\documentclass[11pt]{article}
\usepackage{geometry}
\usepackage{graphicx,subfigure}

\usepackage{microtype}

\newcommand{\siamonly}[1]{}
\newcommand{\tronly}[1]{#1}

\usepackage{fancyhdr}
\fancypagestyle{plain}{%
\fancyhf{} %
\fancyhead[R]{\color{gray}\sffamily\fontsize{9pt}{9pt}\selectfont\thepage} %
\renewcommand{\headrulewidth}{0pt}
\renewcommand{\footrulewidth}{0pt}}
\renewenvironment{abstract}{%
\hfill\begin{minipage}{0.9\textwidth}
\par\small\textbf{Abstract.}}
{\par\noindent\end{minipage}\hfill}

  

\usepackage{xspace}
\usepackage{tabularx}
\usepackage{booktabs}

\usepackage[utf8]{inputenc} %
\usepackage[T1]{fontenc}    %
\usepackage{url}            %
\usepackage{booktabs}       %
\usepackage{nicefrac}       %
\usepackage{microtype}      %
\usepackage{stackengine}

\usepackage{tikz}
\usepackage{tikz-qtree}
\usetikzlibrary{arrows}

\usepackage[compact,small]{titlesec}

\usepackage{amsfonts,amsmath,amssymb}
\usepackage{color}
\usepackage{comment}
\usepackage{paralist}
\usepackage{enumitem}
\usepackage{multirow}
\setitemize{noitemsep,topsep=0pt,parsep=0pt,partopsep=0pt}
\usepackage{adjustbox}

\usepackage{mathtools}

\usepackage{natbib}

\usepackage{algorithm}
\usepackage[noend]{algorithmic}
\usepackage[normalem]{ulem}

\usepackage{caption}
\usepackage{url}
\definecolor{darkcolor}{rgb}{0.03, 0.27, 0.49}
\usepackage[colorlinks = true, pdfstartview = FitV, linkcolor = darkcolor, citecolor = darkcolor, urlcolor = darkcolor,backref=page]{hyperref}

\usepackage{kimon-math}

\numberwithin{theorem}{section}
\numberwithin{algorithm}{section}
\numberwithin{equation}{section}

\usepackage{xspace}

\usepackage{arydshln}

\usepackage[nameinlink]{cleveref}

\newif\iflowresfigs%
\lowresfigsfalse%
\lowresfigstrue%

\usepackage{microtype}

\newcommand{\siamwidth}[1]{%
\begin{minipage}{5.125in}%
{#1}%
\end{minipage}}

\renewcommand{\cite}{\citep}

\definecolor{darkgreen}{rgb}{0.53, 0.66, 0.42}
\renewcommand{\algorithmiccomment}[1]{\textcolor{darkgreen}{\emph{(#1)}}}

\newcommand{\probref}[1]{\hyperref[#1]{Problem~\eqref{#1}}}
\newcommand{\subprobref}[1]{{\hypersetup{hidelinks}\hyperref[#1]{subproblem}}~\eqref{#1}}
\newcommand{\Subprobref}[1]{\hyperref[#1]{Subproblem~\eqref{#1}}}

\newcommand{\stMinCut}{MinCut\xspace}

\newcommand{\secref}[1]{\hyperref[#1]{Chapter~\ref{#1}}}
\newcommand{\Sbar}{\bar{S}}
\newcommand{\Rbar}{\bar{R}}

\newcommand{\rvol}{\text{rvol}}
\renewcommand{\ones}{\textbf{1}}
\DeclareMathOperator{\diag}{diag}

\newcommand*{\metis}{\textsc{metis}\xspace}
\newcommand*{\LFI}[1][]{%
	\textsc{lfi}\ifthenelse{\equal{#1}{}}{}{\mbox{-}\ensuremath{#1}}\xspace%
}
\newcommand*{\FI}{\textsc{fi}\xspace}
\newcommand*{\MQI}{\textsc{mqi}\xspace}
\newcommand*{\mgs}{\texttt{MGS}\xspace}
\newcommand*{\PRs}{\textsc{PR}\xspace}

\crefname{subsection}{section}{sections}
\Crefname{subsection}{Section}{Sections}

\usepackage{blkarray}
\usepackage{enumitem}
\usepackage{framed}
\usepackage{wrapfig}
\definecolor{shadecolor}{gray}{0.9}
\newcounter{AsideNumber}
\crefalias{AsideNumber}{aside}
\Crefname{Aside}{Aside}{Asides}
\crefname{aside}{aside}{asides}
\newcommand*{\aside}[2]{\refstepcounter{AsideNumber}%
	\label{#1}%
	\begin{wrapfigure}{r}{0.4\linewidth}%
		\vspace*{-10pt}\begin{shaded*}%
			\footnotesize%
			\textsc{Aside~\theAsideNumber.~}\textit{#2}%
		\end{shaded*}\vspace*{-10pt}%
	\end{wrapfigure}%
}

\renewcommand{\algorithmicrequire}{\textbf{Input:}}
\renewcommand{\algorithmicensure}{\textbf{Output:}}

\usepackage{titletoc}

\dottedcontents{section}[1.5em]{\small\bfseries}{1.5em}{0.33pc}
\titlecontents{subsection}[4em]{\small\itshape}{\contentslabel{2.5em}}{\hspace*{-2.5em}}{\contentspage}
\titlecontents{part}[0em]{\vspace{2ex}\titlerule\vspace{1ex}\bfseries}{\contentslabel{0em}}{\hspace*{0em}}{}[\titlerule\vspace{1ex}]

\iflowresfigs
\DeclareGraphicsExtensions{%
    .jpg,.pdf,.PDF,.png,.PNG,%
    .mps,.jpeg,.jbig2,.jb2,.JPG,.JPEG,.JBIG2,.JB2}
\graphicspath{{./figs-jpg/}{./}}
\makeatletter
\def\input@path{{}}%
\long\def \IfFileExists#1#2#3{%
    \ifx\input@path\@undefined
      \def\reserved@a{#3}%
    \else
      \def\reserved@a{\@iffileonpath{#1}{#2}{#3}}%
    \fi
  \reserved@a}
\makeatother
\else%
\fi

\newfloat{subroutine}{htbp}{loa}
\floatname{subroutine}{Augmented Graph}
\makeatletter\newcommand\l@subroutine{\@dottedtocline{1}{1.5em}{2.3em}}\makeatother

\renewcommand*{\backref}[1]{}%
\renewcommand*{\backrefalt}[4]{%
  \ifcase #1 %
    No citations.%
  \or
    Cited on page #2.%
  \else
    Cited on pages #2.%
  \fi
}%
\begin{document}

\title{Flow-based Algorithms for Improving Clusters: A Unifying Framework, Software, and Performance}

\author{
        Kimon~Fountoulakis%
        \thanks{School of Computer Science, University of Waterloo, Waterloo, ON, Canada. E-mail: kfountou@uwaterloo.ca.
                KF would like to acknowledge DARPA and NSERC for providing partial support for this work.
        }
        \and
        Meng~Liu%
        \thanks{Department of Computer Science, Purdue University, West Lafayette, IN, USA. e-mail: liu1740@purdue.edu.
        }
        \and
        David~F.~Gleich%
        \thanks{Department of Computer Science, Purdue University, West Lafayette, IN, USA. e-mail: dgleich@purdue.edu.
        	  DFG would like to acknowledge NSF IIS-1546488, CCF-1909528, the NSF Center for Science of Information STC, CCF-0939370, DOE DE-SC0014543, NASA, and the Sloan Foundation for partial support for this work. 
        } 
        \and
        Michael~W.~Mahoney%
        \thanks{ICSI and Department of Statistics, University of California at Berkeley, Berkeley, CA, USA. E-mail: mmahoney@stat.berkeley.edu.
                MWM would like to acknowledge ARO, DARPA, NSF, ONR, Cray, and Intel for providing partial support for this work. 
        }
}

\date{}

\maketitle

\thispagestyle{empty}

\begin{abstract}
\noindent
Clustering points in a vector space or nodes in a graph is a ubiquitous primitive in statistical data analysis, and it is commonly used for exploratory data analysis.  
In practice, it is often of interest to ``refine'' or ``improve'' a given cluster that has been obtained by some other method.\
In this survey, we focus on principled algorithms for this \emph{cluster improvement problem}.\
Many such cluster improvement algorithms are flow-based methods, by which we mean that operationally they require the solution of a sequence of maximum flow problems on a (typically implicitly) modified data graph.\ 
These cluster improvement algorithms are powerful, both in theory and in practice, but they have not been widely adopted for problems such as community detection, local graph clustering, semi-supervised learning, etc.\
Possible reasons for this are: the steep learning curve for these algorithms; the lack of efficient and easy to use software; and the lack of detailed numerical experiments on real-world data that demonstrate their usefulness.\
Our objective here is to address these issues.\
To do so, we guide the reader through the whole process of understanding how to implement and apply these powerful algorithms.\ 
We present a unifying fractional programming optimization framework that permits us to distill, in a simple way, the crucial components of all these algorithms.\
It also makes apparent similarities and differences between related methods.\ 
Viewing these cluster improvement algorithms via a fractional programming framework suggests directions for future algorithm development.\
Finally, we develop efficient implementations of these algorithms in our LocalGraphClustering Python package, and we perform extensive numerical experiments to demonstrate the performance of these methods on social networks and image-based data graphs.
\end{abstract}

	\tableofcontents
	
	\addcontentsline{toc}{part}{Part I. Introduction and Overview of Main Results}
	\section*{\large Part I. Introduction and Overview of Main Results}

\section{Introduction}

Clustering is the process of taking a set of data as input and returning meaningful groups of that data as output. 
The literature on clustering is tremendously and notoriously extensive~\citep{Luxburg-2012-clustering,Ben-David-2018-clustering}; see also comments by Hand in the discussion of~\citet{Friedman-2004-clustering}.
It can seem that nearly every conceivable perspective on the clustering problem---from statistical to algorithmic, from optimization-based to information theoretic, from applications to formulations to implementations---that could be explored, has been explored.  Applications of clustering are far too numerous to discuss meaningfully, and they are often of greatest practical interest for ``soft'' downstream objectives such as those common in \emph{Exploratory Data Analysis}. Yet, despite comprehensive research into the problem, there are still useful and surprising new results on clustering discovered on a regular basis~\citep{Kleinberg:2002:ITC:2968618.2968676,Ackerman-2008-cluster-quality,Awasthi-2015-relax,Abbe-2018-sbm}.

Graph clustering is a special instance of the general clustering problem, where the input is a graph, in this case, a set of nodes and edges, and the output is a meaningful grouping of the graph's nodes. 
\tronly{The ubiquity of sparse relational data from internet-based applications to biology, from complex engineered systems to neuroscience, as well as new problems inspired by these domains~\citep{Newman:2010:NI:1809753,Easley:2010:NCM:1805895,brandes2005-network-analysis,Estrada2010,Traud2011,Grindrod2013,Liberti2014,Bienstock2014,Jia2015,Bertozzi2016,Estrada2016,Rombach2017,Fosdick2018,Fennell2019,Shi2019,Ehrhardt2019}, has precipitated a recent surge of graph clustering research~\citep{Newman-2006-modularity,LLDM2009,Eckles-2017-network-inference}.} 
\siamonly{The ubiquity of sparse relational data from internet-based applications to biology, from complex engineered systems to neuroscience, as well as new problems inspired by these domains~\citep{Newman:2010:NI:1809753,Easley:2010:NCM:1805895,brandes2005-network-analysis} (and within SIAM Review itself during the past decade~\citet{Estrada2010,Traud2011,Grindrod2013,Liberti2014,Bienstock2014,Jia2015,Bertozzi2016,Estrada2016,Rombach2017,Fosdick2018,Fennell2019,Shi2019,Ehrhardt2019}), has precipitated a recent surge of graph clustering research~\citep{Newman-2006-modularity,LLDM2009,Eckles-2017-network-inference}.}
For instance, in graph and network models of complex systems, the \emph{community detection} or \emph{module detection} problem is a specific instance of the graph clustering problem, in which one seeks to identify clusters that exhibit relationships distinctly different from other parts of the network.  
Consequently, there are now a large number of tools and techniques that generate clusters from graph data. 

The tools and techniques we study in this survey arise from a different and complementary perspective.
As such, they are designed to solve a different and complementary problem.
The clustering problem itself is somewhat ill-defined, but the way one often applies it in practice is while performing exploratory data analysis.
That is, one uses a clustering algorithm to ``play with'' and ``explore'' the data, tweaking the clustering to see what insights about the data are revealed.
Motivated by this, and the well-known fact that the output of even the best clustering algorithm is typically imperfectly suited to the downstream task of interest (for example~\citet{Carrasco-2003} mentions ``neither [\ldots] seems to yield really good [\ldots] clusterings of our dataset, so we have resorted to hand-built combinations''), we are interested in tools and techniques that seek to \emph{improve} or \emph{refine} a given cluster---or more generally a representative set of vertices---in a fashion that is computationally efficient, that yields a result with strong optimality guarantees, and that is useful in practice.  

Somewhat more formally, here is the \emph{cluster improvement problem}: given a graph $G=(V,E)$ and a subset of vertices $R$ that serve as a \emph{reference cluster} (or \emph{seed set}), find a nearby set $S$ that results in an \emph{improved cluster}. 
That is, \medskip 
\begin{quote}
when given as input a graph $G=(V,E)$ and a set $R \subset V$, \\
\hspace*{20pt}a \emph{cluster improvement algorithm} returns a set $S \subset V$,\\
\hspace*{20pt}where $S$ is in some sense ``better'' than $R$.
\end{quote} \medskip 
A very important point here is that both $G$ and $R$ are regarded as input to the cluster improvement problem.
This is different from more traditional graph clustering, which typically takes only $G$ as input, and it is a source of potential confusion. 
See \Cref{fig:example-sbm}, which we explain in depth in \Cref{sec:improve-vs-clustering}, for an illustration.

How to choose the set $R$, which is part of the input to a cluster improvement algorithm, is an important practical problem (akin to how to construct the input graph in more traditional graph clustering). It depends on the application of interest, and we will see several examples of it.

In the settings we will investigate in this survey, we will be (mainly) interested in graph conductance (which we will define in \Cref{sec:metrics} formally) as the cluster quality metric.
Thus, the optimization goal will be to produce a set $S$ with smaller (i.e., better) conductance than $R$. 
Generally speaking, a set of small conductance in a graph is a hint towards a bottleneck revealing an underlying cluster. While we focus on conductance, the techniques we review are more general and powerful.
For example, these ideas, algorithms, and approaches can be adapted to other graph clustering objectives such as ratio-cut~\citep{LR04}, normalized-cut~\citep{Hoc13}, and other closely related ``edge counting'' objective functions and scenarios~\citep{VKG18,Veldt-2019-resolution}. We return to the utility of conductance as an objective function to improve clusters, even those output from related objectives and algorithms, via an example in Section~\ref{sec:improve-vs-clustering}. 

We define the precise improvement problems via optimization in subsequent sections. 
For now, we treat them as black-box algorithms to explain how they might be used. 
These introductory examples use one of two algorithms, MQI~\citep{LR04} and LocalFlowImprove~\citep{OZ14}, that we will study in depth. 
Both of these cluster improvement algorithms execute an intricate sequence of max-flow or min-cut computations on graphs derived from $G$ and $R$. 
A technical difference with important practical consequences is the following: 
\medskip \begin{quote} 
$\text{MQI always returns a set $S$ of \emph{exactly optimal} conductance }$ \\ 
\hspace*{20pt}$\text{\emph{contained within the reference cluster $R$}; whereas  }$ \\
$\text{LocalFlowImprove finds an improved cluster $S$ with }$ \\
\hspace*{20pt}$\text{conductance \emph{at least as good} as that found by MQI,}$ \\
\hspace*{20pt}$\text{\emph{by both omitting vertices of $R$ and adding vertices outside $R$.}}$
\end{quote} \medskip 
In addition to these two algorithms, we will also discuss in depth the FlowImprove~\citep{AL08_SODA} method. 

\subsection{Cluster improvement: compared with graph clustering}
\label{sec:improve-vs-clustering}
\newcommand{\asidelouvain}{\aside{aside:louvain}{For this particular example, there are ways of getting a completely accurate answer that involve re-running the Louvain method or tweaking parameters. Our point is simply that we can easily improve existing clustering pipelines with flow-based improvement methods.}}

To start, consider \Cref{fig:example-sbm}, in which we consider a synthetic graph model called a stochastic block model. 
In our instance of the stochastic block model, we plant $5$ clusters of 20 vertices. 
Edges \tronly{\asidelouvain}  
between vertices in the same cluster occur at random
 with probability $0.3$. 
Edges between vertices in different clusters occur at random with probability $0.0157$. 
A popular algorithm for graph clustering is the Louvain method~\citep{Blondel-2008-louvain}. 
\siamonly{\asidelouvain}
On this problem input instance, running the Louvain method often produces a clustering with a small number of errors (\Cref{aside:louvain}). 
By using 
the LocalFlowImprove algorithm on each cluster returned by Louvain, we can directly refine the clusters output by the Louvain method (i.e., we can choose our input set $R$ to be the output of some other method). 
This example involves running the improvement algorithm one time for each cluster returned by the Louvain method. Doing so results in a perfectly accurate clustering for this instance. That said, the Louvain method is designed to \emph{partition} the dataset and insists on a cluster for each node, whereas improving each cluster may result in some vertices unassigned to a cluster or assigned to multiple clusters. Although this does not occur in this instance on the block model, it ought to be expected in general. There are a variety of ways to address this difference in output given the domain specific usage. For instance, to reobtain a partition, one can create clusters of unassigned vertices and pick a single assignment out of the multiple assignments based on problem or application specific criteria.

\begin{figure}[tp]
\centering
\textbf{Graph clustering}\\[-2\baselineskip]
	\subfigure[Input is a graph; this one has 5 planted clusters.]{\includegraphics[width=0.4\linewidth]{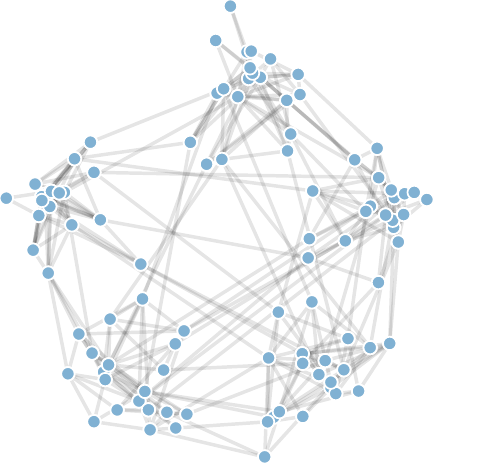}}\hspace{1em}
	\subfigure[Output is a cluster for each node; we highlight mistakes for the 5 groups. ]{\includegraphics[width=0.4\linewidth]{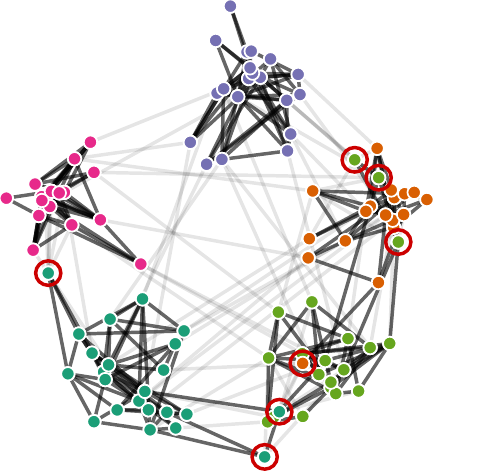}}
~\\[1ex]
	
	\textbf{Cluster improvement}\\[-2\baselineskip]
\subfigure[Input is a graph and a seed set of nodes.]{\includegraphics[width=0.4\linewidth]{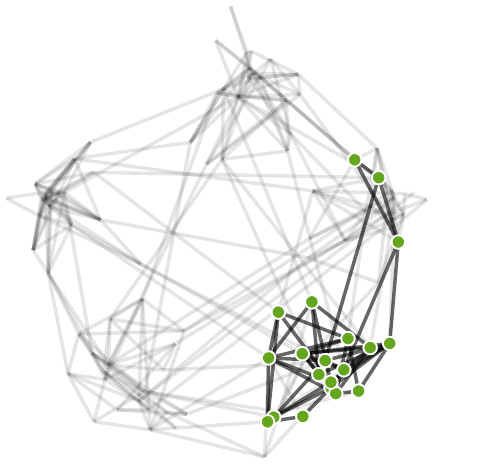}}\hspace{1em}
\subfigure[Output is an improved set of nodes.]{\includegraphics[width=0.4\linewidth]{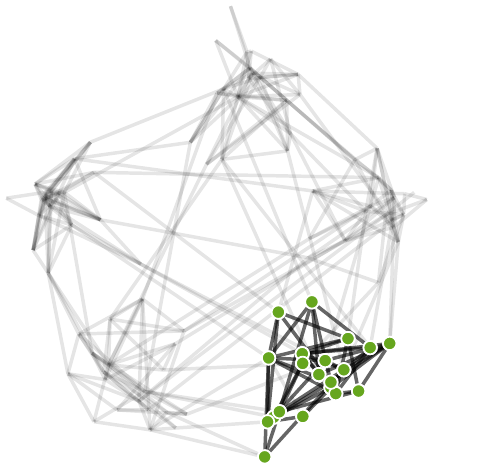}}

	\caption{Graph clustering (known as community detection in some areas) is a problem where the input is a graph and the output is a labeling or partition indicator for each node, indicating the group/cluster to which each node belongs.  This is illustrated in (a) and (b).  Cluster improvement is different problem. In cluster improvement problems, the input is both a graph and a set of nodes, and the output is a set of nodes that is improved in some sense.  As an example, in (c), we show the input as the same graph from (a) along with one of the groups from (b) that has a few mistakes. The result of cluster improvement in (d) has no mistakes. See replication details in the appendix.  }
	\label{fig:example-sbm}
\end{figure}	

For this example, we'd like to highlight the difference in objective functions between the modularity measure optimized by the Louvain algorithm and the conductance measure optimized by LocalFlowImprove. Despite differences in these objectives (modularity compared with conductance), many clustering objective functions are related in their design to balance boundary and size tradeoffs (i.e.~isoperimetry). Consequently, exactly or optimally improving a related objective is likely to result in benefits to nearby measures. %
	Moreover, conductance and modularity are indeed close cousins as established either by how they make cut and volume tradeoffs~\cite{Gleich-2016-mining} or by relationships with Markov stability~\cite{Delvenne2010}. Thus, it is not surprising that LocalFlowImprove is able to assist Louvain, despite the difference in objectives. (Let us also note that flow-based algorithms can be designed around a variety of more general objective functions as well, see, Section~\ref{sec:qcut-general}.) Thus this example mixes pieces that commonly arise in real-world uses: (i) the end goal (find the hidden structures), (ii) an objective function formulation of a related goal (optimize modularity), and (iii) an algorithmic procedure for that task (Louvain method). Given the output from (iii), the improvement algorithms produce an \emph{exactly} optimal solution to a nearby problem that (in this case) captures exactly the true end goal (i).

\subsection{Cluster improvement: compared with seeded graph diffusion}\label{subsec:clustimprov}

Another common scenario in applied work with graphs is what we will call a target identification problem. In this setting, there is a large graph and we are given only one, or a very small number of vertices, from a hidden target set. See \Cref{fig:example-geograph-input} for an illustration. 
Seeded graph diffusions are a common technique for this class of problems. 
In a seeded graph diffusion, the input is a seed node $s$ and the output is a set of nearby graph vertices related to~$s$~\citep{zhu2003semi,Faloutsos:2004:FDC:1014052.1014068,ZBLWS04,tong2006-random-walk-restart,KK2014}. 
Arguably, the most well-known and widely-applied of these seeded graph methods is seeded PageRank~\citep{ACL06, G15}. In essence, seeded PageRank problems identify related vertices as places where a random walk in the graph is likely to visit when it is frequently restarted at $s$.

\begin{figure}[tp]
	\def\stackalignment{r}
	\subfigure[\label{fig:example-geograph-input} Our target and a seed node (orange).]{
		\topinset
			{\fcolorbox{black}{white}{\includegraphics[width=0.2\linewidth]{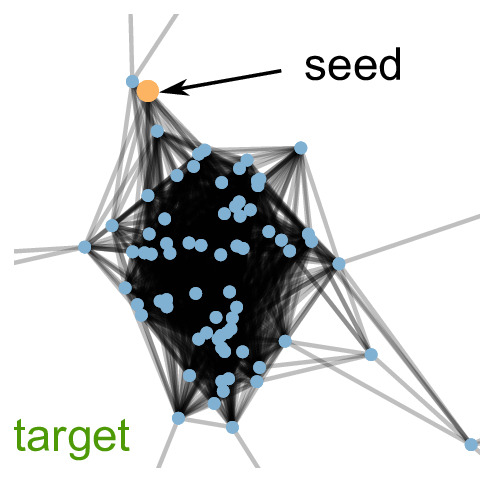}}}
			{\includegraphics[width=0.48\linewidth]{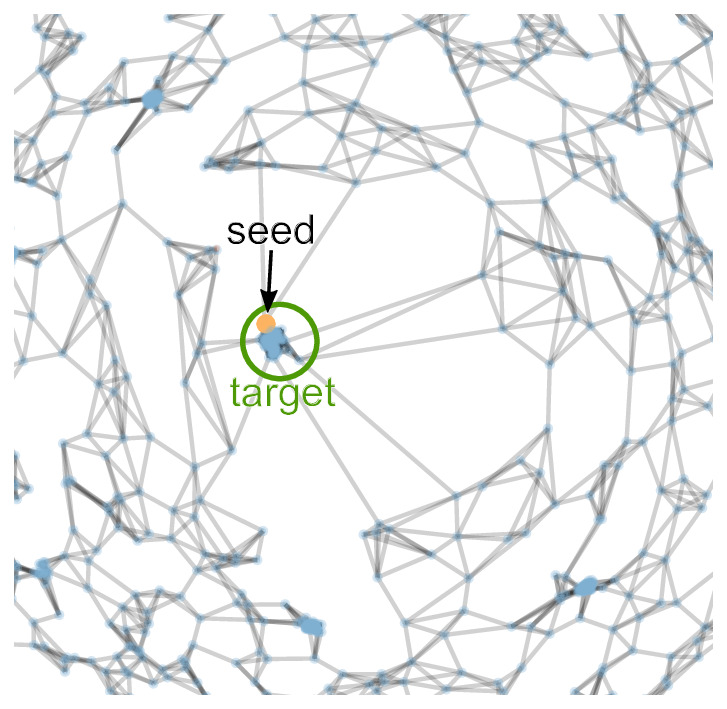}}
			{5pt}{2pt}}\hfill 
	\subfigure[\label{fig:example-geograph-pr}The seeded PageRank result (red).]{\includegraphics[width=0.48\linewidth]{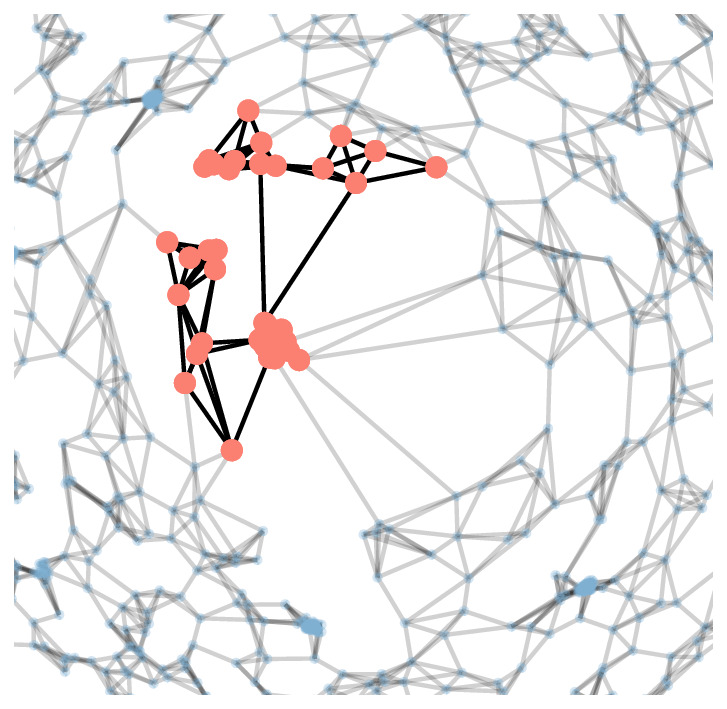}}\\
	\subfigure[\label{fig:example-geograph-mqi}MQI-based improvement (red) of the seeded PageRank result set (inset orange nodes)]{
		\bottominset%
		{\fcolorbox{black}{white}{\includegraphics[width=0.2\linewidth]{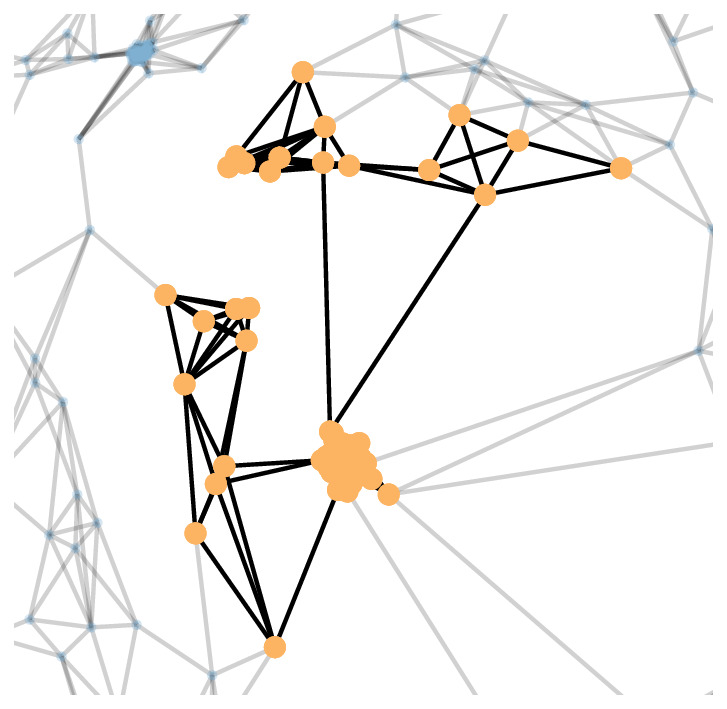}}}%
		{\includegraphics[width=0.48\linewidth]{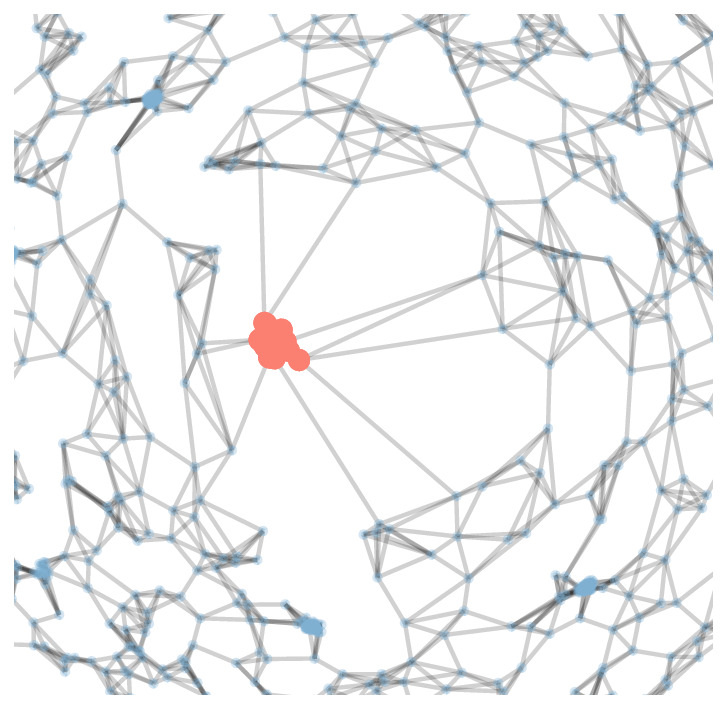}}%
		{5pt}{2pt}}
	\subfigure[\label{fig:example-geograph-lfi}LocalFlowImprove result (red) on a one-step neighborhood of the seed (inset orange nodes)]{
		\bottominset%
			{\fcolorbox{black}{white}{\includegraphics[width=0.2\linewidth]{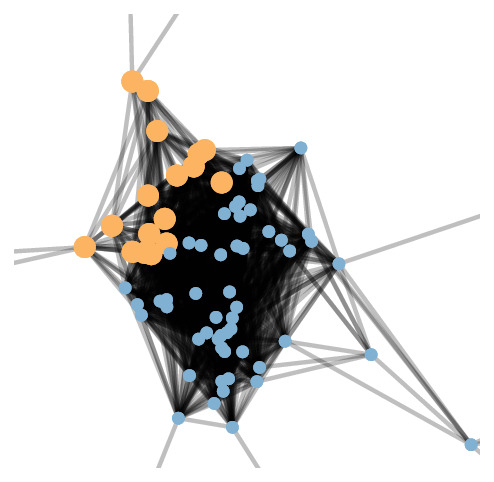}}}%
			{\includegraphics[width=0.48\linewidth]{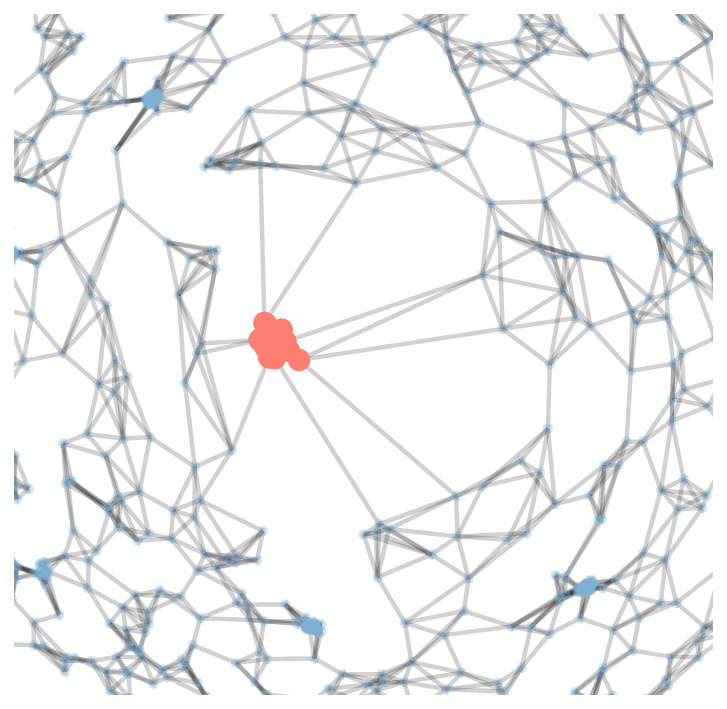}}%
			{5pt}{2pt}}
	\caption{Cluster improvement with MQI~\citep{LR04} and LocalFlowImprove~\citep{OZ14} on a large graph. 
                 We show a piece of a larger graph with a target cluster in the middle of (a) and an expanded view of the target and seed in the inset of (a).
                 If we run a seeded PageRank-based method to search for a cluster nearby the seed, then the result \emph{leaks out} into the rest of the graph and fails to capture the boundary of the cluster, as shown in~(b).
                 If, using the seeded PageRank result as the reference set $R$ (shown in orange in the inset of (c)), we run MQI, then we accurately identify the target in (c) in red.
                 Likewise, if, using the one-step neighborhood of the seed as $R$ (shown in orange in the inset of (d)), we run LocalFlowImprove, then we also accurately identify the target (d) in red.
                 See~\Cref{sxn:app-replication} for details. 
                 }
	\label{fig:example-geograph} 
\end{figure}

Cluster improvement algorithms are different than but closely related to \emph{seeded graph diffusion} problems. 
This relationship is both formal and applied. 
It is related in a formal (and obvious) sense because seeded PageRank and its relatives correspond to an optimization problem that will also provably identify sets of small conductance~\cite{ACL06}. 
It is related in an applied sense for the following (important, but initially less obvious) reason: the improvement methods we describe are excellent choices to refine clusters produced by seeded PageRank and related Laplacian-based spectral graph methods~\citep{Lang2005-spectral-weaknesses,FKSCM2017,veldticml2016}. 
The basic reason for this is that spectral methods often exhibit a ``leak'' nearby a boundary. 
For instance, if a node at the boundary of an idealized target cluster is visited with a non-trivial probability from a random walk, then neighbors will also be visited with non-trivial probability.  In particular, this means that such spectral methods tend to output clusters with larger conductance, more false positives (in terms of the target set), and sometimes fewer true positives as well. 

An illustration of this \emph{leaking out} of a spectral method is given in \Cref{fig:example-geograph}. 
Here, we are using the algorithms to study a graph with a planted target cluster of 72 vertices in the center of a much larger 3000 node graph. 
If we run a seeded PageRank algorithm from a node nearby the boundary of the target, then the result set expands too far beyond the target cluster (\Cref{fig:example-geograph-pr}).
If we then run the MQI cluster improvement method on the output of seeded PageRank, then we accurately identify the target cluster alone (\Cref{fig:example-geograph-mqi}). 
Likewise, if we simply expand the seed node into a slightly larger set by adding all of the seed's neighbors, and we then perform a single run of the LocalFlowImprove method, then we will accurately identify this set. 

\subsection{Cluster improvement: compared with image segmentation}

Our final introductory example is given in \Cref{fig:example-astronaut}, and it illustrates these improvement algorithms in the context of image segmentation. 
Here, an input image is translated into a weighted graph through a standard technique. The goal of that technique is to ensure that similar regions of the image appear as clusters in the resulting graph; this standard process is described formally in \Cref{sec:image-to-graph}.
On this graph representing an image, the target set identification problem from \Cref{subsec:clustimprov} yields an effective image segmentation procedure, albeit with a much larger set of seed nodes. 

\aside{aside:seg-examples}{These image segmentation examples are used to illustrate properties of the algorithms that are difficult to visualize on natural graphs. They are not intended to represent state of the art segmentation procedures.}

We focus on the face of the astronaut Eileen Marie Collins (a retired NASA astronaut and United States Air Force colonel)~\cite{wiki:eileen} as our target set. 
\Cref{fig:astronaut_boundaries_input_mqi} shows a superset of the face. 
When given as the input set to the MQI cluster improvement method (which, recall, always returns a subset of the input), the result closely tracks the face, as is shown in \Cref{fig:astronaut_boundaries_mqi}.
Note that there are still a small number of false positives around the face---see the region left of the neck below the ear---but the number of false positives decreases dramatically with respect to the input.
Similarly, when given a subset of the face, we can use LocalFlowImprove (which, recall, can expand or contract the input seed set) to find most of it. 
We present in \Cref{fig:astronaut_boundaries_input_sl} the input cluster to LocalFlowImprove, which is clearly a subset of the face; and the output cluster for LocalFlowImprove is shown in \Cref{fig:astronaut_boundaries_sl}, which again closely tracks the face with a few false negatives around the mouth.

\begin{figure}[tp]
	\centering
	\subfigure[Input to MQI.\label{fig:astronaut_boundaries_input_mqi}]%
		{\includegraphics[width=0.4\linewidth]{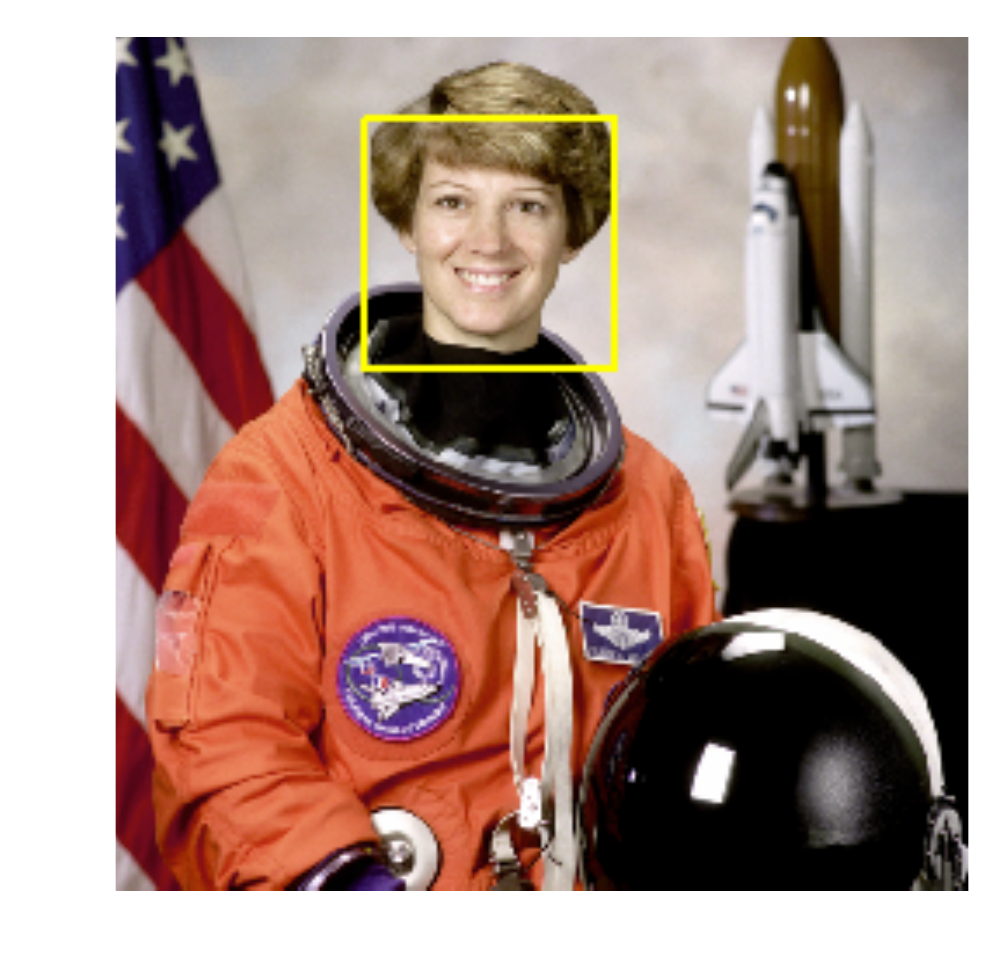}}
	\subfigure[Output of MQI. \label{fig:astronaut_boundaries_mqi}]%
		{\includegraphics[width=0.4\linewidth]{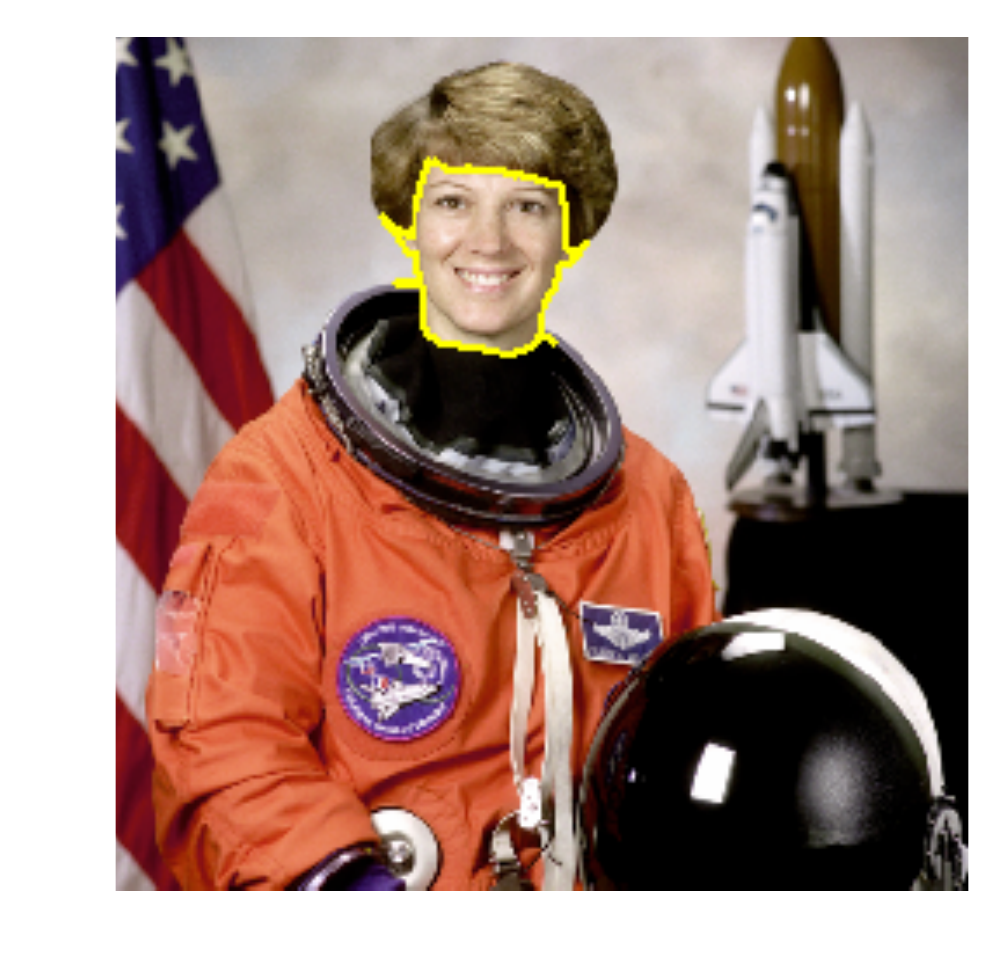}}\\
	\subfigure[Input to LocalFlowImprove. \label{fig:astronaut_boundaries_input_sl}]{\includegraphics[width=0.4\linewidth]{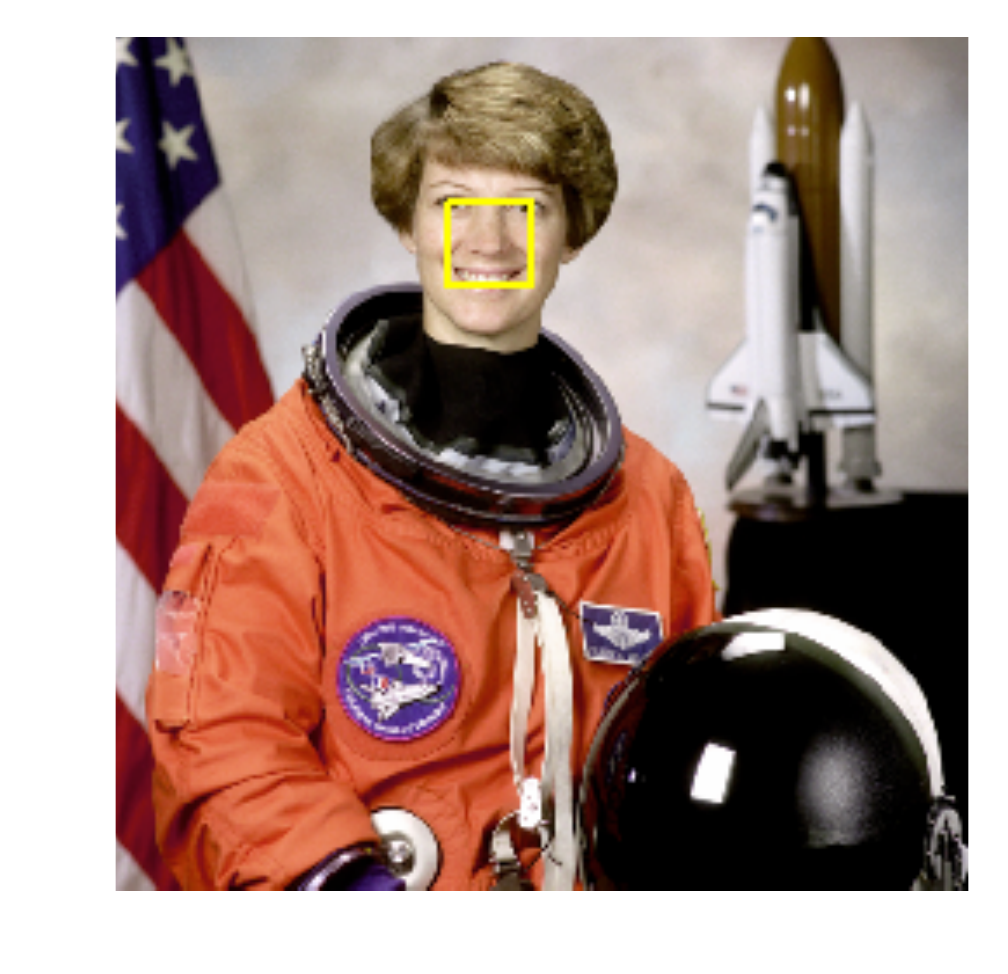}}
	\subfigure[Output of LocalFlowImprove. \label{fig:astronaut_boundaries_sl}]{\includegraphics[width=0.4\linewidth]{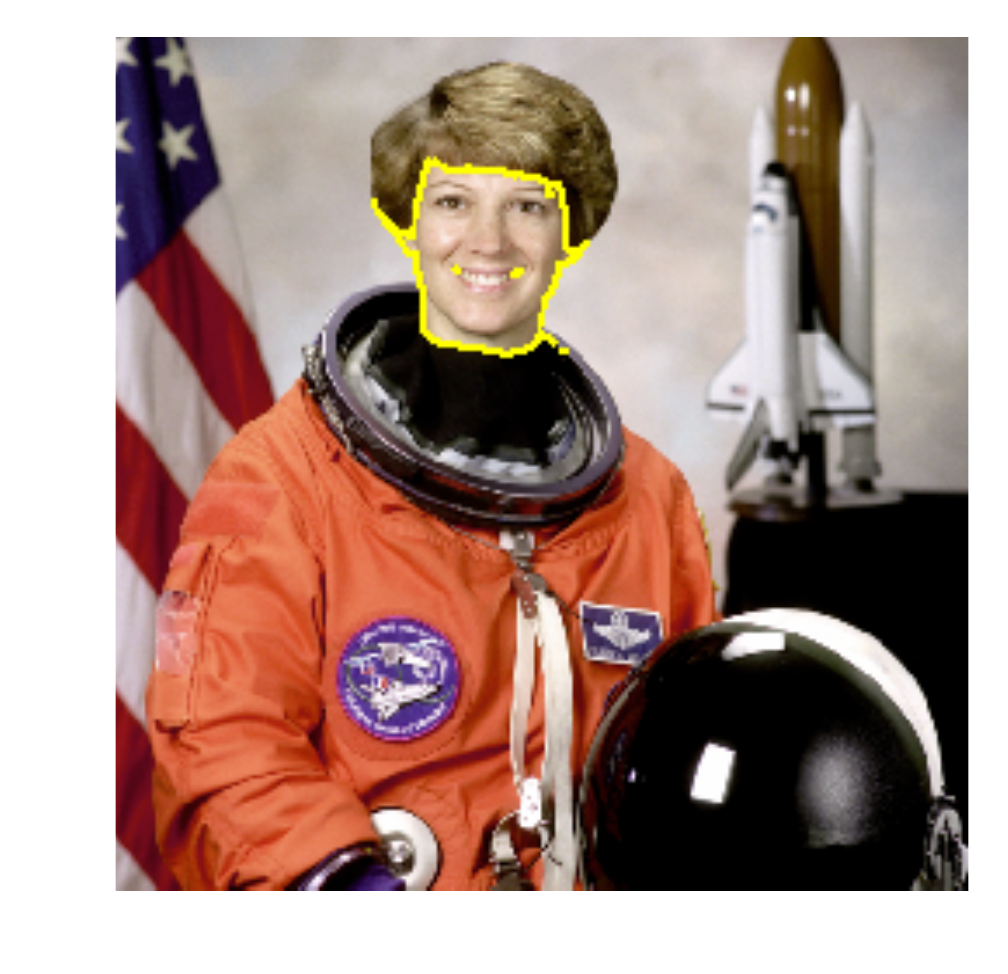}}
	
	\caption{Illustration of cluster improvement with MQI~\citep{LR04} and LocalFlowImprove~\citep{OZ14} on an image. 
	         In Figure~\ref{fig:astronaut_boundaries_input_mqi}, we show the input set of nodes to MQI. 
                 The set of nodes consists of the pixels inside the yellow square. 
                 Note that MQI looks for good clusters within the input square, and the target cluster is the face of Eileen Marie Collins (a retired NASA astronaut and United States Air Force colonel)~\cite{wiki:eileen}.
                 In Figure~\ref{fig:astronaut_boundaries_mqi}, we show the output, which demonstrates that MQI-based cluster improvement decreases the number of false positives. 
	         In Figure~\ref{fig:astronaut_boundaries_input_sl}, we show the input set of nodes to LocalFlowImprove. 
                 The set of nodes consists of the pixels inside the yellow square. 
                 Note that LocalFlowImprove looks for good clusters around the region of the input square and the target cluster is the face of the Eileen Marie Collins.
                 In Figure~\ref{fig:astronaut_boundaries_sl}, we show the output, which demonstrates that LocalFlowImprove-based cluster improvement increases the number of true positives. 
                 See Appendix~\ref{sxn:app-replication} for details.  
                }
	\label{fig:example-astronaut}
\end{figure}

\subsection{Overview and Summary}

One challenge with the flow-based cluster improvement literature is that (so far) it has lacked the simplicity of related spectral methods and seeded graph diffusion methods like PageRank~\citep{G15,zhu2003semi,Faloutsos:2004:FDC:1014052.1014068,ZBLWS04,tong2006-random-walk-restart,KK2014}. 
These spectral methods are often easy to explain in terms of random walks, Markov chains, linear systems, and intuitive notions of \emph{diffusion}. 
Instead, the flow-based literature involves complex and seemingly arbitrary graph constructions that are then used, almost like magic (at least to researchers and downstream scientists not deeply familiar with flow-based algorithms), to show impressive theoretical results. 
Our goal here is to pull back the curtain on these constructions and provide a unified framework based on a class of optimization methods known as fractional programming. 

The connection between flow-based local graph clustering and fractional programming is not new, e.g., \citet{LR04} cite one relevant paper~\citep{gallogrigoriadis}. Both \citet{LR04} and \citet{AL08_SODA} mention binary search for finding optimal ratios akin to root-finding. \citet{Hoc10} was the first to develop a general framework of root-finding algorithms for global flow-based fractional programming problems. 
However, specialization of these results to the FlowImprove problem require special treatment which is not discussed in~\cite{Hoc10}. That said, our purpose in using these connections is that they make the methods simpler to understand. 
Thus, we will make the connection extremely clear, and we will demonstrate that our fractional programming optimization perspective unifies \emph{all} existing flow-based cluster improvement methods. 
\emph{Indeed, it is our hope that this perspective will be used to develop new theoretically-principled and practically useful methodologies.}

\subsection{Reproducible Software: the LocalGraphClustering package}

In addition to the detailed and unified explanation of the flow-based improvement methods, we have implemented these algorithms in a software package with a user-friendly Python interface. 
The software is called LocalGraphClustering~\citep{git:localgraphclustering} (which, in addition to implementing flow improvement methods that we review here we implement spectral diffusion methods for clustering, methods for multi-label classification, network community profiles and network drawing methods).
As an example of using this package, running the seeded PageRank followed by MQI for the results shown in Figure~\ref{fig:example-geograph} is as simple as: 
\begin{quote}\footnotesize 
	\begin{verbatim}
import localgraphclustering as lgc             # load the package 
G = lgc.GraphLocal("geograph-example.edges")   # load the graph 
seed = 305                                     # set the seed and compute 
R,cond = lgc.spectral_clustering(G,[seed],method='l1reg') # seeded PageRank
S,cond = lgc.flow_clustering(G,R,method='mqi') # improve with MQI
	\end{verbatim}
\end{quote}
This software also enables us to explore a number of interesting applications of flow-based cluster improvement algorithms that demonstrate uses beyond simply improving the conductance of sets. The implementation of the methods scales to graphs with billions of edges when used appropriately. In this survey, we explore graphs with up to 117 million edges (\Cref{subsec:findingThousandOfClusters}). 

This package is useful generally. For reproducibility we also provide code that reproduces all the experiments that are presented in this survey. %

\subsection{Outline}

There are three major parts to our survey; and these are designed to be relatively modular to enable one to read parts (e.g., to focus on the theoretical results or the empirical results) separately.

In the first part, we introduce the fundamental concepts and techniques, both informally as in this introduction and formally through our notation (\Cref{sec:notation}) and fractional programming sections (\Cref{sec:fractionalProg}). 
In particular, we introduce graph cluster metrics such as conductance in \Cref{sec:metrics}. 
We also introduce fundamental ideas related to \emph{local graph computations} in \Cref{sec:local-graph-algorithms}, which discusses the distinction between strongly and weakly local graph algorithms. 
These ideas are then used to explain the precise objective functions and settings for flow-based cluster improvement algorithms in \Cref{sec:fractionalProg}. 
This part continues with an overview of how these methods fit into the broader literature of graph-based algorithms (\Cref{sec:background}), and it includes a brief discussion of other scenarios where \emph{max-flow} and \emph{min-cut} algorithms are used as a fundamental computational primitive (\Cref{sec:network-flow}), as well as infinite dimensional analogues to these ideas (\Cref{sec:infinite-flow}). 
We also include a number of ideas that show how the methods generalize beyond using conductance. 

In the second part, we provide the technical core of the survey. 
We begin our description of the details of the methods with a review of concepts from minimum flow and maximum cuts (\Cref{sec:mincutmaxflow}). 
In particular, this section has a careful derivation of these problems as duals in terms of linear programs. 
The next three sections, \Cref{chap:mqi,chap:flowimprove,chap:localflowimprove}, cover the three algorithms that we use in the experiments: MQI, FlowImprove, and LocalFlowImprove. For each algorithm, we provide a thorough discussion on how to define each step of the algorithm.
On a high level, these algorithm require at each iteration the solution of a max-flow problem. However, to actually implement these methods one requires construction of a locally modified version of the given graphs.

In the final part, we provide an extensive empirical evaluation and demonstration of these algorithms (\Cref{sec:experiments}). 
This is done in the context of a number of datasets where it is possible to illustrate clearly and easily the benefits of these techniques. 
Examples in this evaluation include images, as we saw in the introduction, as well as road networks, social networks, and nearest neighbor graphs that represent relationships among galaxies. 
This section also includes experiments on graphs with up to 117 million edges. We also describe strategies to generate local network visualizations from these local graph clustering methods that highlight characteristic differences in how the flow-based methods treat networks. 

In addition, we provide an appendix with full reproducibility details for all of the figures and tables (\Cref{sec:image-to-graph,sxn:app-replication}). 
These include references to specific Python notebooks for replication of the experiments.

\section{Notation, Definitions, and Terminology}
\label{sec:notation}

We begin by reviewing specific mathematical assumptions, notation, and terminology that we will use. 
To start, we use the following standard notations: 

\begin{center} 
	\begin{tabular}{ll} 
$\ZZ$ & denotes the set of integer numbers, \\
$\RR$ & denotes the set of real-valued numbers, \\
$\RR_+$ & denotes the set of real-valued non-negative numbers, \\
$\RR^n$ & denotes the set of real-valued vectors of length $n$, \\
$\RR^{n\times n}$ & denotes the set of real-valued $n\times n$ matrices, \\
$\RR^n_+$ & denotes the set of real-valued non-negative vectors of length $n$, and \\
$\RR^{n\times n}_+$ & denotes the set of real non-negative $n\times n$ matrices.
\end{tabular}\end{center}

\subsection{Graph notation}\label{sec:notation-graph}

Given a graph $G=(V,E)$, we let $V$ denote the set of nodes and $E$ denote the set of edges.
We assume an undirected, weighted graph throughout, although some of the constructions and concepts involved in a flow computation are often best characterized through directed graphs. (See also Aside~\ref{aside:directed}.)
For an unweighted graph, everything we do will be equivalent to assigning an edge weight of $1$ to all edges. 
Also, we also assume that the given graphs have no self-loops.

\aside{aside:directed}{%
Our techniques would extend to any clustering function on directed graphs that defines a hypergraph using techniques from~\citet{Benson-2016-motif-spectral} based on motif enumeration. For adaptations of these techniques to hypergraphs, see~\citet{Veldt-2020-hypergraph,veldt2020localized} for some examples.
}

The cardinality of the set $V$ is denoted by $n$, i.e., there are $n$ nodes, and we assume that the nodes are arbitrarily ordered from $1$ to $n$.
Therefore, we can write $V := \{1,2,\dots,n\}$. %
We use $v_i$ to denote node $i$, and when it is clear, we will use $i$ to denote that node.
We assume that the edges $E$ in the graph are arbitrarily ordered.
The cardinality of the set $E$ is denoted by $m$, i.e., there are $m$ edges.
We will use $e_{ij}$ to denote an edge. Also, if a node $j$ is a neighbor of node $i$, we denote this relationship by $j\sim i$.
A path is a sequence of edges which connect a sequence of distinct vertices. 
A connected component is a subset of nodes such that there exists a path between any pair of nodes in that subset. 

We frequently work with subsets of vertices. Let $S \subseteq V$, for example. Then 
$\Sbar$ denotes the complement of subset $S\subseteq V$, formally, $\Sbar = \{ v \in V \ | \ v \notin S \}$.  
The notation $\partial S$ represents the node-boundary of the set $S$; formally, it denotes the set of nodes that are in $\Sbar$ and are connected with an edge to at least one node in $S$. In set notation, we have $\partial S = \{ v, \text{ where } v \in \Sbar, \text{ and there exists } (u,v) \in E \text{ with } u \in S \}$.

\subsection{Matrices and vectors for graphs}
\label{sec:matvec}

Here, we define matrices that can be used to define models and objective functions on graph data. 
They can also provide a compact way to understand and describe algorithms that operate on graphs.

The \emph{adjacency matrix} $\mA\in\{0,1\}^{n\times n}$ (or $\in\RR^{n\times n}_+$ if the graph is weighted) provides perhaps the most simple representation of a graph using a matrix.
In $\mA$, row $i$ corresponds to node $i$ in the graph, and element $A_{ij}$ is non-zero if and only if nodes $i$ and $j$ are connected with an edge in the given graph. The value of $A_{ij}$ is the edge weight for a weighted graph, or simply $1$ for an unweighted graph. 
Since we are working with undirected graphs, the adjacency matrix is symmetric, i.e., $A_{ij} = A_{ji}$, where $A_{ij}$ is the element at the $i$th row and $j$th column of matrix $A$. 

The \emph{diagonal weighted degree matrix} $\mD\in\ZZ^{n\times n}_+$ (or $\in\RR^{n\times n}_+$ if the graph is weighted) is a matrix that stores the degree information for every node. 
The element $D_{ii}$ is the sum of weights of the edges of node~$i$, i.e., $D_{ii} := \sum_{j\in V : j\sim i} A_{ij}$; and off-diagonal elements, i.e., $D_{ij}$, for $i \ne j$, equal zero. 

The \emph{degree vector} is defined as $\vd = \text{diag}(\mD)$, where $\text{diag}(\cdot)$ takes as input a vector or a matrix and returns, respectively, a diagonal matrix with the vector in the diagonal or a vector with diagonal elements of a matrix.

The \emph{edge-by-node incidence matrix} $\mB\in\{0,-1,1\}^{m\times n}$ (where, recall, $n$ is the number of nodes, and $m$ is the number of edges) is often used to measure differences among nodes. %
Each row of this matrix represents an edge, and each column represents a node. 
For example, row $k$ in $B$ represents the $k$th edge in the graph (arbitrarily ordered) that corresponds (say) to nodes $i$ and $j$ in the graph. 
Row $k$ in $B$ then has exactly two nonzero elements; $-1$ for the source of the edge and $1$ for the target of the edge, at the $i$ and $j$ position, respectively. 
If the graph is undirected, then we can arbitrarily choose which node is the source and which node is the target on an edge, without loss of generality. Note that because we assume no self-loops, the incidence matrix contains the full information about the edges of the graph. 

The \emph{diagonal edge-weight or edge-capacity matrix} $C \in \RR^{m\times m}_+$ is a diagonal matrix where each diagonal element corresponds to the weight of an edge in the graph. This matrix is the identity for an unweighted graph. 
For example, the $k$th diagonal element corresponds to weight of the $k$th edge in the~graph.

The \emph{Laplacian matrix} $\mL\in\ZZ^{n\times n}$ (or $\in\RR^{n\times n}$ if the graph is weighted) is defined as $\mL = \mD - \mA$ or equivalently $\mL = \mB^T \mC \mB$.

Vectors of all-ones and all-zeros, denoted $1_n$ and $0_n$, respectively, are column vectors of length $n$. 
If the dimensions of each vector will be clear from the context, then we omit the subscript. 
The \emph{indicator vector} $\mathbf{1}_i$ is a column vector that is equal to $1$ at the $i$th index and zero elsewhere.  If the indicator is used with a node, then the length of the vector $\mathbf{1}_i$ is $n$. For an edge, its length is $m$. 

If $S$ is a subset of nodes or a subset of indices and $\mA$ is any matrix, e.g., the adjacency matrix, then $\mA_S$ is a submatrix of $\mA$ that corresponds to the rows and columns with indices in $S$. Likewise, $\mathbf{1}_S$ is a column vector with ones in entries for~$S$. These indicator vectors have length $n$. 

\subsection{Vector norms}
\label{sec:norms}

We denote the vector $1$-norm by $\|\vx\|_1 = \sum_{i} |x_i|$ and the $2$-norm by $\|\vx\|_2 = \sqrt{\sum_{i} (x_i)^2}$
We will use these norms to measure differences among nodes that are represented in a vector $x$, i.e., every node corresponds to an element in vector $x$. 
For example, $\|\mB\vx\|_1 = \sum_{e_{ij}\in E} |x_i - x_j|$ is the sum of differences among node representations in $x$. 
In the case of weighted graphs, this can be generalized to $\|\mB\vx\|_{\mC,1} = \sum_{e_{ij}\in E} C_{e_{ij}}|x_i - x_j| = \sum_{e_{ij}\in E} A_{ij}|x_i - x_j|$. %
For the $2$-norm, we have $\|\mB\vx\|_{\mC,2}^2 = \sum_{e_{ij}\in E} C_{e_{ij}}(x_i - x_j)^2=\sum_{e_{ij}\in E} A_{ij}(x_i - x_j)^2$.

\subsection{Graph cuts and volumes using set and matrix notation}

Much of our discussion will fluidly move between set-based descriptions and matrix-based descriptions. 
Here, we give a simple example of how this works in terms of a graph cut and volume of a set. 

\paragraph{Graph cut} 
We say that a pair of complement sets $(S,\Sbar)$, where $S\subseteq V$, is a \emph{global graph partition} of a given graph with node set $V$. 
Given a partition $(S,\Sbar)$, the \emph{cut of the partition} is the sum of weights of edges between $S$ and $\Sbar$, which can be denoted by either
\begin{equation}
\begin{aligned}
\cut(S,\Sbar) & = \sum_{i\in S, j\in \Sbar} A_{ij}, & \quad \text{or} \quad & \cut(S) & = \sum_{i\in S, j\in \Sbar} A_{ij}. 
\end{aligned}
\end{equation}
Instead of using set notation to denote a partition of the graph, i.e., $(S,\Sbar)$, we can use indicator vector notation $\vx = \mathbf{1}_S \in \{0,1\}^n$ to denote a partition.
In this case, the cut of the partition is 
\begin{equation}
\cut(S,\Sbar) = \sum_{i,j} A_{ij} |x_i - x_j| = \|\mB \mathbf{1}_S\|_{C,1} .
\end{equation}
Note that both expressions are symmetric in terms of $S$ and $\Sbar$.

\paragraph{Graph volume}
The \emph{volume of a set of nodes} $S$ is equal to the sum of the degrees of all nodes in $S$, i.e., \
\begin{equation} \vol(S) = \sum_{i\in S} d_i. \end{equation}
We will use the notation $\vol(G)$ to denote the \emph{volume of the graph}, which is equal to $\vol(V)$.
Using this definition and our matrix definitions above, we have that the \emph{volume of a subset of nodes} is $\vol(S) = \mathbf{1}_S^T \vd $.

\subsection{Relative volume}
\label{sec:rvol}

FlowImprove and LocalFlowImprove formulations are simpler to explain by introducing the idea of \emph{relative volume}. 
The \emph{relative volume} of $S$ with respect to $R$ and $\kappa$ is 
\begin{equation} \label{eq:rvol}
	\rvol(S; R, \kappa) = \vol(S \cap R) - \kappa \vol(S \cap \Rbar).
\end{equation} 
The relative volume is a very useful concept that we will use to define the objective functions of the local flow-based problems, MQI, FlowImprove and LocalFlowImprove.
The purpose of the relative volume is to measure the volume of the intersection of $S$ with the input seed set nodes $R$, while penalizing the volume of
the intersection of $S$ with the complement $\Rbar$. This is important when we define the objective functions of MQI, FlowImprove and LocalFlowImprove, since we want to penalize
sets $S$ that have little intersection with $R$ and high intersection with $\Rbar$. This makes sense, since in local flow-based improve methods the goal is often to improve the input set $R$, thus we 
want the output $S$ of a method to be ``related'' to $R$ more than $\Rbar$.

\subsection{Cluster quality metrics}
\label{sec:metrics}

Here, we discuss scores that we use to evaluate the quality of a cluster. 
For all of these measures, \emph{smaller values correspond to better clusters}, i.e., correspond to a cluster of higher quality. 
\paragraph{Conductance} 
The \emph{conductance function} is defined as the ratio between the number of edges that connect the two sides of the partition $(S,\Sbar)$ and the minimum ``volume'' of $S$ and $\Sbar$: 
\[ 
   \phi(S) = \frac{\text{cut}(S)}{\text{min}(\text{vol}(S), \text{vol}(\Sbar))}  .
\]
A set of minimal conductance is a fundamental bottleneck in a graph.
For example, small conductance in a set is often interpreted as an information bottleneck revealing community or module structure, or (relatedly) as a bottleneck to the mixing of random walks on the graph. 
Note that conductance values are always between $0$ and $1$, and they can be interpreted as a probability. 
(Formally, this is the probability that random walk moves between $S$ and $\Sbar$ in a single prescribed step after the walk has fully mixed.) 

\paragraph{Normalized Cuts}
The \emph{normalized cut function} is a related notion that provides a score that is often used in image segmentation problems~\citep{SM2000}, where a graph is constructed from a given image and the objective is to partition the graph in two or more segments. 
In the case of a bi-partition problem, the normalized cuts score reduces to:
\[ \text{ncut}(S) = \frac{\text{cut}(S)}{\text{vol}(S)} + \frac{\text{cut}(\Sbar)}{\text{vol}(\Sbar)}. \]
The normalized cuts and conductance scores are related, in that $\phi(S) \le \text{ncut}(S) \le 2 \phi(S)$. 
There is a related concept, called \text{ncut'}~\cite{Sharon2006,Hoc10} 
that just measures the cut to volume ratio for a single set $\text{ncut'}(S) = \text{cut}(S)/\text{vol}(S)$.
Observe that this is equal to $\phi(S)$ for any set with less than half of the volume. 

\paragraph{Expansion} 
The \emph{expansion function} or \emph{expansion score} is defined as the ratio between the number of edges that connect the two sides of the partition $(S,\Sbar)$ and the minimum ``size'' of $S$ and $\Sbar$: 
\[ \tilde{\phi}(S) = \frac{\text{cut}(S)}{\text{min}(|S|, |\Sbar|)}. \]
\aside{aside:sparsity}{Our definition of expansion used here is sometimes used as the definition for sparsity. The literature is not entirely consistent on these terms.} 
Compared to the conductance score, which uses the volume (related to number of edges) of the sets $S$ and $\Sbar$ in the denominator, the expansion score counts the number of nodes in $S$ or $\Sbar$.
This has the property that the expansion score is less affected by high degree nodes.  
Similarly to conductance, smaller expansion scores correspond to better clusters. 
However, these values are not necessarily between $0$ and $1$.

\paragraph{Sparsity}
The \emph{sparsity measure} of a set is a topic that arises often in theoretical computer science. It is closely related to expansion, but measures the fraction of edges that exist in the cut compared to the total possible number
\[ \psi(S) = \frac{\text{cut}(S)}{|S||\Sbar|}. \]
This value is always between $0$ and $1$. Also, $\tilde{\phi}(S) \le n \psi(S) \le 2 \tilde{\phi}(S)$ because $n \psi(S) = \frac{\text{cut}(S)}{|S|} + \frac{\text{cut}(\Sbar)}{|\Sbar|}.$ Hence, sparsity is a scaled measure akin to normalized cut. 

\paragraph{Ratio cut} 
The \emph{ratio cut function} provides a score that is often used in data clustering problems, where a graph is constructed by measuring similarities among the data, and the objective is to partition the data into multiple clusters \citep{HK1992}.
In the case of the bi-partition problem, the ratio cut score reduces to:
\[ 
   \text{rcut}(S) = \frac{\text{cut}(S)}{|S|}.
\]
Observe that the ratio cut and expansion scores are related, in the sense that the latter is equal to the former if the input set of nodes $S$ has cardinality less than or equal to $n/2$. The ratio cut was popularized due 
to its importance in image segmentation problems \citep{FH2004}. Usually, this ratio is minimized by performing a spectral relaxation \citep{L2007}.

\subsection{Strongly and weakly local graph algorithms}
\label{sec:local-graph-algorithms}

Local graph algorithms and locally-biased graph algorithms are the ``right'' setting to discuss cluster improvement algorithms on large-scale data graphs. 
For the purposes of this survey, there are two key types of (related but quite distinct) local graph algorithms: 
\begin{itemize}
	\item \textbf{Strongly local graph algorithms.} These algorithms take as input a graph $G$ and a reference cluster of vertices $R$; and they have a runtime and resource usage that only depends on the size of the reference cluster $R$ (or the output $S$, but not the size of the entire graph $G$).
	\item \textbf{Weakly local graph algorithms.} These algorithms take as input a graph $G$ and a reference cluster of vertices $R$; and they return an answer whose size will depend on $R$, but whose runtime and resource usage may depend on the size of the entire graph $G$ (as well as the size of $R$). 
\end{itemize}
That is, in both cases, one wants to find a good/better cluster near $R$, and in both cases one outputs a small cluster $S$ that is near $R$, but in one case the running time of the algorithm is independent of the size of the graph $G$, while in the other case the running time depends on the size of $G$. 
For more about local and locally-biased graph algorithms, we recommend~\citet{Gleich-2016-mining,FDM2017} and also~\citet{MOV12,LBM16_TR,lawlor2016mapping} for overviews. 

It is easy to quantify the size of the output $S$ being small; but, in general, the \emph{locality of an algorithm}, i.e., how many nodes/edges are touched at intermediate steps, may depend on how the graph is represented. 
We typically assume something akin to an adjacency list representation that enables:
\begin{itemize}
	\item constant time access to a list of neighbors; and
	\item constant or nearly constant (e.g., $O(\log |V|)$ time access to an arbitrary edge.
\end{itemize}
Moreover, the cost of building this structure is \emph{not counted} in the runtime of the algorithm, e.g., since it may be a one-time cost when the graph is stored. 
Note that, in addition to a reference cluster $R$, these algorithms could take information about vertices in a reference set, such as a vector of values, as well. 

The importance of these characterizations and this discussion is the following: 
\begin{quote} \itshape
for strongly local graph algorithms  \\
\hspace*{1em} the runtime is independent of the size of the graph.
\end{quote}
In particular, this means that the algorithm does not even touch all of the nodes of the graph $G$.
This makes a strongly local graph algorithm an extremely useful tool for studying large data graphs. 
For instance, in \Cref{fig:example-geograph}, none of the algorithms used information from more than about 500 vertices of the the total 3000 vertices of the graph, and this result wouldn't have changed at all if the entire graph was 3 million vertices (or more as in ~\citet{SKFM2016}). 

To contrast with strongly-local graph algorithms, most graph and mesh partitioning tools---and even the improved and refined variations---are global in nature. 
In other words, the methods take as input a graph, and the output of the methods is a global partitioning of the entire graph. 
In particular, this means that the methods have running time which depends on the size of the whole graph.
This makes it \emph{very} challenging to apply these methods to even moderately large graphs. 
\section{Main Theoretical Results: Flow-based Cluster Improvement and Fractional Programming Framework}
\label{sec:fractionalProg}

In this section, we will introduce and discuss the fractional programming problem and its relevance to flow-based cluster improvement.
The motivation is that work on cluster improvement algorithms has thus far proceeded largely on a case-by-case basis; but as we will describe, fractional programming is a class of optimization problems that provides a way to generalize and unify existing cluster improvement algorithms.

\subsection{Cluster improvement objectives and their properties}

\label{sec:improvement-objectives}
For the problem of conductance-based cluster improvement, the three methods we consider exactly optimize the following objective functions:
\begin{equation*}
\begin{array}{@{}rl@{}}
	\text{MQI:} & 	\MINone{S \subset V}{\displaystyle \frac{\cut(S)}{\vol(S)}}{S \subseteq R} \\ \addlinespace
	\text{FlowImprove:} & \MINone{S \subset V}{\displaystyle \frac{\cut(S)}{\rvol(S; R,  \vol(R)/\vol(\Rbar)) }}{\rvol(S; \ldots) > 0} \\ \addlinespace
	\mathop{\text{LocalFlowImprove}(\delta)}\limits_{\smash{\delta \ge 0}}: & \MINone{S \subset V}{\displaystyle \frac{\cut(S)}{\rvol(S; R,  \vol(R)/\vol(\Rbar) + \delta) }}{\rvol(S; \ldots) > 0}  \\
	\end{array}
\end{equation*}
The constraint $\rvol(S; \ldots) > 0$ simply means that we only consider sets where the denominator is positive (we omit repeating all the parameters from the denominator for simplicity). 
Because we are minimizing over discrete sets, there is not a closure problem with the resulting strict inequality ($\rvol(S,\ldots) > 0$), so these are all well-posed.

Recall that $\rvol(S;R, \kappa) = \vol(S \cap R) - \kappa \vol(S \cap \Rbar)$. This definition implies that sets $S$ such that $\rvol(S; R,  \vol(R)/\vol(\Rbar))  \le 0$ cannot be optimal solutions for FlowImprove, and that even fewer sets can be optimal for LocalFlowImprove. 
On the other hand, note that $\text{LocalFlowImprove}(\delta)$ interpolates between the FlowImprove ($\delta = 0$) and MQI ($\delta = \infty$) because when $\delta$ is sufficiently large, then the term $\vol(S \cap \Rbar)$ that arises in rvol must be 0 in order for the set $S$ feasible for the non-negative rvol constraint. In fact, if $\delta > \vol(R)(1-1/\vol(\Rbar))$ then positive denominators alone will require $S \subset R$. 

To understand better the connections between these three objectives, we begin by stating a simple property of these objective functions.
The following theorem states that conductance gets \emph{smaller}, i.e., \emph{better}, as we move from MQI to LocalFlowImprove to FlowImprove.
\begin{theorem}
\label{thm:conductance}
Let $G$ be an undirected, connected graph with non-negative weights.  Let $R \subset V$ have $\vol(R) \le \vol(\Rbar)$, where $\Rbar$ is the complement of $R$. Let $S_{\text{MQI}}, S_{\text{FI}}, S_{\text{LFI}}$ be the optimal solution of the MQI, FlowImprove, and $\text{LocalFlowImprove}(\delta)$ objectives, respectively. 
If the solutions of FlowImprove and LocalFlowImprove satisfy $\vol(S_{\text{FI}}) \le \vol(\Sbar_{\text{FI}})$ and $\vol(S_{\text{LFI}}) \le \vol(\Sbar_{\text{LFI}})$ (that is, the solution set is on the small side of the cut),  
then for any $\delta\ge0$ in LocalFlowImprove, we have that 
\[ \phi(S_{\text{FI}}) \le 	\phi(S_{\text{LFI}}) \le \phi(S_{\text{MQI}}). \]
\end{theorem}
\begin{proof}
The first piece, that $\phi(S_{\text{LFI}}) \le \phi(S_{\text{MQI}})$, is a simple, useful exercise we repeat from~\citet[Theorem 4]{veldticml2016}. 
Note that if $S \subseteq R$ then $\phi(S) = \smash{\tfrac{\cut(S)}{\rvol(S; R, \kappa)}}$ for any $\kappa$. Now, note that for all $\rvol$ terms in the LocalFlowImprove($\delta$) objective with $\delta > 0$, we have $\kappa \ge  \vol(R)/\vol(\Rbar)$. Moreover, solutions are constrained to only consider sets where $\rvol$ is positive. Thus, for the value of $\kappa$ used in LocalFlowImprove, and also any positive $\kappa$, we have 
\[ 
\phi(S_{\text{LFI}})  = \frac{\cut(S_{\text{LFI}})}{\vol(S_{\text{LFI}})} 
\le \frac{\cut(S_{\text{LFI}})}{\vol(S_{\text{LFI}}) - \kappa \vol(S_{\text{LFI}} \cap \Rbar)}  
\le \frac{\cut(S_{\text{LFI}})}{\rvol(S_{\text{LFI}}; R, \kappa)}. 
\]
Next, note that for the chosen setting of $\kappa$, we have that $\rvol(S; R, \kappa) > 0$ for all $S \subseteq R$. Thus, we have 
\[ \phi(S_{\text{LFI}}) \le \mathop{\text{minimum}}_{S \subseteq R} \frac{\cut(S)}{\rvol(S; R, \kappa)} = \mathop{\text{minimum}}_{S\subseteq R} \phi(S) = \phi(S_{\text{MQI}}). \] 
This shows that both LocalFlowImprove and FlowImprove give better conductance sets than MQI. %
	
For the second piece, we use an alternative characterization of LocalFlowImprove as discussed in \citet{OZ14}. LocalFlowImprove($\delta$) is equivalent to solving the following optimization problem for some constant $C$:
\begin{equation*}\label{eq:LFI-optimization}
\MINone{S\subset V}{\frac{\text{cut}(S)}{\text{rvol}(S;R,\text{vol}(R)/\text{vol}(\bar{R}))}}{\frac{\text{vol}(S\cap R)}{\text{vol}(S)}\ge C, \rvol(S; \ldots) > 0 }
\end{equation*}
while FlowImprove solves the same problem without the constraint involving $C$. Then we have:
\[\frac{\text{cut}(S_{FI})}{\text{rvol}(S_{FI};R,\text{vol}(R)/\text{vol}(\bar{R}))} \leq \frac{\text{cut}(S_{LFI})}{\text{rvol}(S_{LFI};R,\text{vol}(R)/\text{vol}(\bar{R}))}\]
\[\frac{\text{cut}(S_{FI})}{\text{cut}(S_{LFI})}\leq \frac{\text{rvol}(S_{FI};R,\text{vol}(R)/\text{vol}(\bar{R}))}{\text{rvol}(S_{LFI};R,\text{vol}(R)/\text{vol}(\bar{R}))}  .  \]
If $\phi(S_{FI})>\phi(S_{LFI})$, we have
\[\frac{\text{cut}(S_{FI})}{\text{cut}(S_{LFI})}>\frac{\text{vol}(S_{FI})}{\text{vol}(S_{LFI})}  .  \]
Thus,
\[\frac{\text{rvol}(S_{FI};R,\text{vol}(R)/\text{vol}(\bar{R}))}{\text{rvol}(S_{LFI};R,\text{vol}(R)/\text{vol}(\bar{R}))} \geq \frac{\text{cut}(S_{FI})}{\text{cut}(S_{LFI})} > \frac{\text{vol}(S_{FI})}{\text{vol}(S_{LFI})}  .  \]
If we now substitute the definition of $\text{rvol}$ and $\text{vol}(S\cap \bar{R})=\text{vol}(S)-\text{vol}(S\cap R)$,
\[\frac{(1+\text{vol}(R)/\text{vol}(\bar{R})) \cdot \text{vol}(S_{FI}\cap R)- \text{vol}(R)/\text{vol}(\bar{R})\cdot \text{vol}(S_{FI})}{(1+\text{vol}(R)/\text{vol}(\bar{R}))\cdot \text{vol}(S_{LFI}\cap R)-\text{vol}(R)/\text{vol}(\bar{R})\cdot \text{vol}(S_{LFI})} > \frac{\text{vol}(S_{FI})}{\text{vol}(S_{LFI})}\]
\[\frac{\text{vol}(S_{FI}\cap R)}{\text{vol}(S_{FI})}>\frac{\text{vol}(S_{LFI}\cap R)}{\text{vol}(S_{LFI})}\geq C  .  \]
This means that $S_{FI}$ also satisfies the additional constraint in the optimization problem of LFI. But $S_{FI}$ has smaller objective value, which is a contradiction to the fact that $S_{LFI}$ is the optimal solution of LFI optimization problem.
\end{proof}

\begin{table}[t] %
	\caption{Characteristics of the MQI, FlowImprove, and LocalFlowImprove methods.}
	\label{tab:algocharacteristics}
	\centering
	\begin{tabularx}{\linewidth}{lXXXX}
		\toprule
		Method & Strongly \mbox{local} & \mbox{Explores} \mbox{beyond} $R$ & \mbox{Easy to} \mbox{implement} & Section\\
		\midrule
		MQI & \checkmark &  & \checkmark & \Cref{chap:mqi}\\
		FlowImprove &  & \checkmark & \checkmark & \Cref{chap:flowimprove}\\
		LocalFlowImprove & \checkmark & \checkmark & & \Cref{chap:localflowimprove} \\
		\bottomrule
	\end{tabularx}
\end{table}

Theorem~\ref{thm:conductance} would suggest that one should always use FlowImprove to minimize the conductance around a reference set $R$, but there are other aspects to implementations which should be taken into account.
The three most important, summarized in Table \ref{tab:algocharacteristics}, are described here.

\begin{itemize}
\item
\textbf{Locality of algorithm.}
For strongly local algorithms, the output is a small cluster around the reference set $R$ \emph{and} the running time depends only on the size of the output but is independent of the size of the graph.
Only the former is true for weakly local algorithms.
As we will show in the coming sections, both MQI and LocalFlowImprove are strongly local.
This enables both of them to be run quickly on very large graphs, assuming $R$ is not too large and $\delta$ is not too~small.
\item
\textbf{Exploration properties of algorithm.}
Some methods ``shrink'' the input, in the sense that the output is a subset of the input, while other methods do not have this restriction, i.e., they can (depending on the input graph and seed set) possibly shrink or expand the input.
This classification is particularly useful when we view the methods as a way to explore the graph around a given set of seed nodes.
For example, MQI only explores the region induced by $R$, and so it is not suitable for various tasks that involve finding \emph{new nodes}.
\item
\textbf{Ease of implementation.}
A final important property of methods regards how easy they are to implement.
MQI and FlowImprove are easy to implement because they rely on standard primitives like simple MaxFlow computations.
This means that one can black-box max-flow computations by calling existing efficient software packages.
For LocalFlowImprove, however, getting a strongly local algorithm requires a more delicate algorithm.
Therefore, we consider it to be a more difficult algorithm to implement.
\end{itemize}

As a simple and quick justification of the locality property of the solution (which is distinct from an algorithmic approach to achieve it), note the following simple-to-establish relationship between $\delta$ and the size of the output set for LocalFlowImprove. This was originally used in~\citet{VKG18} as a small subset of a~proof. 

\begin{lemma} \label{lem:lfi-size}
Let $G$ be an undirected, connected graph with non-negative weights.  Let $S^*$ be an optimal solution of the LocalFlowImprove objective with $\vol(R) \le \vol(\Rbar)$. Then $\vol(S^*) < \left(1 + \frac{\vol(\Rbar)}{\vol(R) + \delta \vol(\Rbar)} \right) \vol(R)$. 
\end{lemma}
\begin{proof}
	For simplicity, let $\sigma = \vol(R)/\vol(\Rbar) + \delta$. Then because the denominator at any solution must be positive, we have $0 < \vol(S^* \cap R) - \sigma \vol(S^* \cap \Rbar)$. Note that $ \vol(S^* \cap \Rbar) = \vol(S^*) - \vol(S^* \cap R)$, so $0 < (1+\sigma)\vol(R \cap S^*) - \sigma \vol(S^*)$. Thus, $\vol(S^*) < (1 + 1/\sigma) \vol(R)$. The result follows by substituting the definition for $\sigma$. 
\end{proof}

As we will show, all of the algorithms for these objectives fit into a standard fractional programming framework, which provides a useful way to reason about the opportunities and trade-offs.
An even more general setting for such problems are \emph{quotient cut problems} that we discuss in \Cref{sec:qcut-general}.
While they are often described in this literature on a case-by-case basis, quotient cut problems are all instances of the more general fractional programming class of problems. 
\subsection{The basic fractional programming problem}
\label{sec:fractionalProg-prob}

A fractional program is a ratio of two objective functions: $N(x)$ for the numerator and $D(x)$ for the denominator. It is often defined with respect to a subset $S$ of $\RR^{n}$
\begin{equation}\label{eq:basic-fractional}
\MINone{x}{N(x) / D(x)}{x \in S}
\end{equation}
where $D(x) > 0$ for all $x \in S$.
Fractional programming is an important branch of nonlinear optimization \citep{Frenk2009}. 
The key idea in fractional programming is to relate \eqref{eq:basic-fractional} to the function
\[ f(\delta) = \text{minimize } N(x)  - \delta D(x) \text{ subject to } x \in S, \]
which captures the minimum value of the objective function for this minimization problem as a function of $\delta$. 
Below, we use ``\text{argmin}'' as the expression for an input or argument that minimizes the problem.
Note that $f(\delta) \le 0$ if there exists $x$ such that $N(x)/D(x) \le \delta$. Moreover, if $N(x)$ and $D(x)$ are linear functions and $S$ is a set described by linear constraints, then $f(\delta)$ can be easily computed by solving a linear program, for instance.

We now specialize this general framework for cluster improvement. Note that we will continue to use $\delta$ as the ratio between the numerator and denominator instead of as the LocalFlowImprove parameter until \Cref{sec:experiments}.

\subsection{Fractional programming for cluster improvement}

\aside{aside:fractional}{Most commonly fractional programming is defined for subsets of $\RR^{n}$ as the domain. In our case, we use set-based domains.} 
When we consider the objective functions from \Cref{sec:improvement-objectives}, note that we can translate them into problems closely related to the fractional programming \probref{eq:basic-fractional}. 
Let $Q \subseteq V$ represent a subset of vertices. For MQI, this is $R$ itself and for the others, it is just $V$.  Now let $g(S \subseteq Q) \to \RR$ represent the denominator terms for the MQI, FlowImprove, or LocalFlowImprove  objectives from~\Cref{sec:improvement-objectives}. Then, in a fractional programming perspective on the problems, we are interested in solving the following problem
\begin{equation}
\label{eq:fracprob}
 \MINone{}{\phi_g(S) := 
   \begin{cases} \frac{\cut(S)}{g(S)} & g(S) > 0 \\ \infty & \text{otherwise} \end{cases}}{S \subseteq Q.}
\end{equation}
Let us assume that there is at least one feasible set $S \subseteq Q$ where $g(S) > 0$. This is satisfied for all the examples above when $S = R$. Also note that if $Q = V$ and if $g(V) > 0$, the entire node set $V$ is immediately a solution. For FlowImprove and LocalFlowImprove, though, $g(V) \le 0$, and so $V$ is never a solution, and, in fact, the value of $\kappa$ in FlowImprove is chosen exactly so that $g(V) = 0$. 

As discussed above, we will use a sequence of related parametric problems to find the optimal solution.
Thus, we introduce the parametric function
\begin{equation*}
z(S,\delta):= \cut(S) - \delta g(S),
\end{equation*}
where the parameter $\delta\in\RR$.
We also define the function
\begin{equation} \label{eq:frac-subprob}
\hat{z}(\delta) := \minimize_S \ z(S,\delta) \text{ where } S\subseteq Q, g(S) > 0.
\end{equation}
Computing the value of $\hat{z}(\delta)$ is a key component that we will discuss in \Cref{sec:frac-algo-parts} and also \Cref{chap:mqi,chap:flowimprove,chap:localflowimprove}.
Given this, we can consider solving the following equation
\begin{equation}
\label{eq:paramprob}
\hat{z}(\delta)=0  ,
\end{equation}
which is a simple root finding problem because $\hat{z}(\delta)$ is monotonically increasing as $\delta \to 0$ and also $\hat{z}(0) \ge 0$ and $\hat{z}(\phi_g(R)) \le 0$ (for our objectives). Note that $\hat{z}(0) = 0$ if $\cut(S) = 0$ for some set $S \subseteq Q$ with $g(S) > 0$, which can happen for a disconnected graph.

We now provide a theorem that establishes the relationship between the root finding \probref{eq:paramprob} and the basic fractional programming \probref{eq:fracprob}.
This theorem establishes that by solving \probref{eq:paramprob} we solve \probref{eq:fracprob} as well.
A similar theorem can be found in~\citet{DIN1967}. %
\begin{theorem}\label{thm:fractionalprogramming}
Let $G$ be an undirected, connected graph with non-negative weights.  A set of nodes
$S^*$ is a solution of \probref{eq:fracprob} iff
\begin{equation*}
\hat{z}\left(\tfrac{\cut(S^*)}{g(S^*)}\right)=0.
\end{equation*}
\end{theorem}
\begin{proof}
For the first part of the proof, let us assume that $S^*$ is a solution of \probref{eq:fracprob}. This implies that $g(S^*)>0$.
We have that
\begin{equation*}
\delta^* := \frac{\cut(S^*)}{g(S^*)} \le \phi_g(S) \quad \text{ for all }  S\subseteq Q, g(S)>0.
\end{equation*}
Hence,
\begin{equation*}
\cut(S^*) - \delta^*g(S^*) = 0,
\end{equation*}
and
\begin{equation*}
\cut(S) - \delta^*g(S) \ge 0 \text{ for all } S \subseteq Q, g(S)>0.
\end{equation*}
Using the above we have that $\{\min z(S,\delta^*) \mid S\subseteq Q, g(S) > 0)\}$ is bounded below by zero, and this bound is achieved by $S^*$.
Therefore, $\hat{z}(\delta^*) = 0$, $z(S^*,\delta^*)=0$.

For the second part of the proof, assume that $\hat{z}(\delta^{*})=0$ such that
\begin{equation}
\delta^{*} = \frac{\mbox{cut}(S^{*})}{g(S^{*})},
\end{equation}
for some optimal $S^*$ of the minimization problem in $\hat{z}$.
Then
\begin{equation}
\mbox{cut}(S^{*}) - \delta^{*}g(S^{*}) = 0 \le \mbox{cut}(S,\Sbar) - \delta^{*}g(S)  \quad \text{ for all } S \subseteq Q, g(S) > 0.
\end{equation}
From the second inequality, we have that $\phi_g(S) \ge \delta^{*}$ $\forall S\subseteq Q, g(S)>0$. This means that the optimal solution of \probref{eq:fracprob} is bounded below by $\delta^{*}$.
From the first equation above, we get that this bound is achieved by $S^{*}$. Therefore, $S^*$ solves \probref{eq:fracprob}.
\end{proof}

\subsection{Dinkelbach's algorithm for fractional programming}
\label{sec:fracalgo}

Based on \Cref{thm:fractionalprogramming}, the root of \probref{eq:paramprob} will be the optimal value of the general cluster improvement \probref{eq:fracprob}.
To find the root of \probref{eq:paramprob}, we will use a modified version of Dinkelbach's algorithm~\citep{DIN1967}.

\paragraph{Dinkelbach's algorithm}
Dinkelbach's algorithm is given in \Cref{algo:fractionalprog}.
Note that we had to modify the original algorithm slightly since we do not assume that $g(S)>0$, $\forall S\subseteq Q$.

\begin{algorithm}
\caption{Dinkelbach's Algorithm}
\label{algo:fractionalprog}
\begin{algorithmic}[1]
\STATE Initialize $k := 1$, $S_1:=R$ and $\delta_1:=\cond_g(S_1)$.
\WHILE{we have not exited via the else clause}
\STATE Compute $\hat{z}(\delta_k)$ by solving
$S_{k+1} \; := \; \mathop{\text{argmin}}_{S}  {z(S,\delta_k)} \; \subjectto \; {S\subseteq Q}$
\IF{$\cond_g(S_{k+1}) < \delta_{k}$ \COMMENT{Recall $\phi_g(S) = \infty$ if $g(S) \le 0$}} 
	\STATE $\delta_{k+1}:=\phi_g(S_{k+1})$
\ELSE
	\STATE $\delta_k$ is optimal, return previous solution $S_k$.
\ENDIF
\STATE $k:=k+1$
\ENDWHILE
\end{algorithmic}
\end{algorithm}

\paragraph{Convergence of Dinkelbach's algorithm}

We now provide a theorem that establishes that the subproblem at Step $3$ of \Cref{algo:fractionalprog} does not output infeasible solutions, such as an $S$ that satisfies $g(S) \le 0$.
Based on this, we can establish that the objective function of \probref{eq:fracprob} is decreased at each iteration of \Cref{algo:fractionalprog}.

\begin{theorem}[Convergence]
\label{thm:monotonicfracprog}
Let $G$ be an undirected, connected graph with non-negative weights. Let $\delta^*$ be the optimal value of \probref{eq:fracprob}.
The subproblem in Step $3$ of  \Cref{algo:fractionalprog} cannot have solutions that satisfy $g(S) \le 0$ for $\delta > \delta^*$.
Such solutions are in the solution set of the subproblem if and only if $\delta \le \delta^*$.
Moreover, the sequence $\delta_k$, which is set to be equal to $\cond_g(S_k)$, decreases monotonically at each iteration. The algorithm returns a solution where $g(S_k) > 0$. 
\end{theorem}
\begin{proof}
For the first part of the theorem, let $\delta \ge 0$, $\hat{S}\in \{\argmin z(S,\delta)\}$; and let us assume for the sake of contradiction that $g(\hat{S})\le 0$.
Then
\begin{equation*}
z(S,\delta) \ge z(\hat{S},\delta) \ge 0 \ \forall S \subseteq Q.
\end{equation*}
Hence,
\begin{equation*}
\phi_g(S) \ge \delta \ \forall S\in \{ S \subset Q \ | \ g(S) > 0\},
\end{equation*}
however, this can only be true if $\delta \le \delta^*$.
Otherwise, for $\delta > \delta^*$ we have a contradiction, and this implies that $g(\hat{S})> 0$.\
Therefore, a solution $\hat{S}\in \{\argmin z(S,\delta)\}$ satisfies $g(\hat{S})> 0$, unless $\delta \le \delta^*$.

For the second part of the theorem, let $k$ be such that $\delta_k > \delta^*$.
Then, we have that $z(S_{k+1},\delta_k) < 0$, since $z(S_{k+1},\delta_k) < z(S_{k},\delta_k) = 0$ (where we get 0 by the definition of $\delta_k$ and $S_k$).
Because $z(S_{k+1},\delta_k)<0$, we have that $\phi_g(S_{k+1}) = \delta_{k+1} < \delta_k = \phi_g(S_k)$.
Note that because $g(S_{k+1})>0$ for any $\delta_k> \delta^*$ then we must have $\delta_{k+1}\ge \delta^*$.

Note that because of the algorithm, $\delta_k$ can never be less than $\delta^*$. Thus, the remaining case is detecting that $\delta_k = \delta^*$. Suppose this is the case and also $g(S_{k+1}) > 0$, then $\delta_{k+1} = \delta_k=\delta^*$, and based on Theorem \ref{thm:fractionalprogramming} the algorithm terminates with an optimal solution because either $S_{k+1}$ or $S_{k}$ are solutions. 
If $\delta_k = \delta^*$ and $g(S_{k+1}) \le 0$, then the algorithm terminates (because $\phi_g(S_{k+1}) = \infty$). Thus, $S_k$ must have been optimal (if not, then $g(S_{k+1})$ must be larger than 0) and so the algorithm outputs an optimal solution.
\end{proof}

\paragraph{Iteration complexity of Dinkelbach's algorithm}

The iteration complexity of a method allows us to deduce a bound on the number of iterations necessary. 
We now provide an iteration complexity result for \Cref{algo:fractionalprog}.
This involves two results. We begin with \Cref{lem:monotonicFracProg}.
This lemma describes several interesting properties of \Cref{algo:fractionalprog} which have an important practical implication. Specifically, it shows that $g(S_{k+1})<g(S_k)$. This result has important \emph{practical} implications, since it shows that \Cref{algo:fractionalprog} is searching for subsets $S$ that have a \emph{smaller} value of the function~$g$.
\Cref{lem:monotonicFracProg} will then allow us to prove an iteration complexity result of \Cref{algo:fractionalprog} in \cref{thm:itercomplexityDinkel}. %
A similar result can be found in~\citet[Lemma 4.3]{gallogrigoriadis}, but we repeat it in \Cref{lem:monotonicFracProg} for completeness.
In \Cref{lem:monotonicFracProg}, we also show that the numerator of the objective function in \probref{eq:fracprob} decreases monotonically.

\begin{lemma}\label{lem:monotonicFracProg}
If \Cref{algo:fractionalprog} proceeds to iteration $k+1$, then it satisfies both $g(S_{k+1}) < g(S_k)$ and $\cut(S_{k+1}) < \cut(S_k)$.
\end{lemma}
\begin{proof}
Consider iterations $k$ and $k-1$ and assume that $\delta_k > \delta^*$.
Then, from Theorem \ref{thm:monotonicfracprog}, in iteration $k-1$, we have that $z(S_k,\delta_{k-1}) < z(S_{k-1},\delta_{k-1}) = 0 $.
In iteration $k$, we have that
\begin{equation*}
z(S_{k+1},\delta_{k}) = \cut(S_{k+1}) - \delta_k g(S_{k+1}) < 0.
\end{equation*}
By adding and subtracting $\delta_{k-1}g(S_{k+1})$ to the latter, we~get
\begin{equation*}
z(S_{k+1},\delta_{k}) = \cut(S_{k+1}) - \delta_{k-1}g(S_{k+1}) + \delta_{k-1}g(S_{k+1}) - \delta_k g(S_{k+1}) < 0.
\end{equation*}
Note that the first two terms on the right side of the equality are the minimization problem for that gave the solution $S_k$. Hence, we can lower-bound $\cut(S_{k+1}) - \delta_{k-1}g(S_{k+1})$ via $S_k$ to get
\begin{equation*}
\cut(S_{k}) - \delta_{k-1}g(S_{k}) + \delta_{k-1}g(S_{k+1}) - \delta_k g(S_{k+1}) \le  z(S_{k+1},\delta_{k}) < 0.
\end{equation*}
Because $z(S_{k},\delta_{k}) = 0$, we get that $\cut(S_k) = \delta_k g(S_k)$.
Thus, using this in the latter inequality, we get
\begin{equation*}
\delta_k g(S_k) - \delta_{k-1}g(S_{k}) + \delta_{k-1}g(S_{k+1}) - \delta_k g(S_{k+1}) \le  z(S_{k+1},\delta_{k}) < 0,
\end{equation*}
which is equivalent to
\begin{equation*}
(g(S_k) - g(S_{k+1}))(\delta_k- \delta_{k-1}) < 0.
\end{equation*}
However, because the algorithm monotonically decreases $\delta_k$, we have that $\delta_{k-1}  - \delta_k < 0$, and therefore we must have that
\begin{equation*}
g(S_k) > g(S_{k+1}).
\end{equation*}
This means that the denominator of the objective function in \probref{eq:fracprob} decreases monotonically.
Additionally, from Theorem \ref{thm:monotonicfracprog} we have that the objective function decreases monotonically.
These two imply that the numerator of the objective function, i.e., $\cut(S,\Sbar)$, decreases monotonically.
\end{proof}

Given this result, we can establish the following theorem, which provides an iteration complexity for \Cref{algo:fractionalprog}. This basic result can be improved, as we describe in \Cref{subsec:fastdinkel}, next.

\begin{theorem}[Iteration complexity for Dinkelbach's algorithm]
\label{thm:itercomplexityDinkel}
Consider using Dinkelbach's algorithm \Cref{algo:fractionalprog} for solving MQI, FlowImprove, or LocalFlowImprove on an undirected, connected graph with non-negative integer weights when starting with the set $R$. Then the algorithm needs at most $\cut(R) \le \vol(R)$ iterations to converge to a solution. 
\end{theorem}
\begin{proof}
For all of the above programs, $R$ is a feasible set and thus we can initialize our algorithms with $R$. From \Cref{lem:monotonicFracProg}, we have that $\cut(S)$ decreases monotonically at each iteration.
Since we assume that the graph is integer-weighted, then $\cut(S)$ is integer valued and so $\cut(R)$ gives an upper bound on the number of iterations. Note that $\cut(R) \le \vol(R)$ for any set and so the algorithms need at most $\cut(R) \le \vol(R)$ iterations to converge to a solution~$S^*$. 
\end{proof}

\begin{remark} 
\label{remark:itercomplexityWeightedGraphs}
A weakness of the previous result is that it does not give a complexity result for graphs with non-integer weights. For weighted graphs with non-integer weights, if the weights come from an ordered field where the minimum relative spacing between elements is $\mu$, such as would exist for rational-valued weights or floating point weights, then the above argument gives $\cut(R)/\mu$ iterations. This is essentially tight as the following construction gives two sets whose cut and volume differ only by $\mu$.
\[
\text{\includegraphics[height=0.85in]{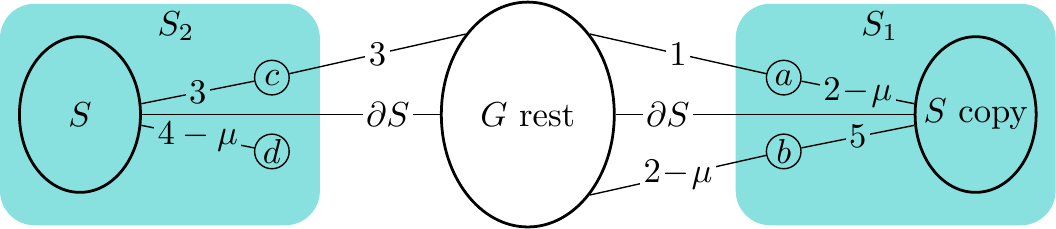}}
\]
Here, $S$ and $S$ copy are duplicates of the same subgraph, so their cut $\partial S$ is identical. Assume $S$ is small enough that we do not need to take into consideration the min term in conductance. Then  
note that $\phi(S_1) = \frac{\cut(S_1)}{\vol(S_1)} = \frac{\cut(S_2) - \mu}{\vol(S_2) - \mu}$. Furthermore, there is no obvious way to detect this scenario as we have a set of well-spaced distinct edge weights (1, $2-\mu$, 3, $4-\mu$, $5-\mu$). (Assuming all other edges in the graph have weight~1.) 
\end{remark}
\emph{For this reason, we do not consider the iteration complexity of algorithms for graphs with non-integer weights and we would recommend the algorithm in the next section to get an approximate answer.}

\subsection{A faster version of Dinkelbach's algorithm via root finding}
\label{subsec:fastdinkel}

\Cref{algo:fractionalprog} requires at most $\vol(Q)$ iterations to converge to an \textit{exact} solution for non-negative integer-weighted graphs.
If we are not interested in exact solutions, then we can improve the iteration complexity of \Cref{algo:fractionalprog} by performing a binary search on $\delta$. This is possible because it is easy to get bounds on the optimal range of $\delta$. We have zero as a simple lower-bound and, for the MQI, FlowImprove, and LocalFlowImprove objectives, $\phi(R) \le 1$ is an easy-to-compute upper-bound on the optimal $\delta$. 
\Cref{algo:fasterfracprog} presents a modified version of Dinkelbach's algorithm that accomplishes this.
In particular, the subproblem in Step $4$ in \Cref{algo:fasterfracprog} is the same as the subproblem in Step $3$ of the original \Cref{algo:fractionalprog}.

\begin{algorithm}
\caption{Fast Dinkelbach's Algorithm for \probref{eq:fracprob}}\label{algo:fasterfracprog}
\begin{algorithmic}[1]
\STATE Initialize $k := 1$, $\delta_{\min} := 0$, $\delta_{\max} \ge p:=\{\max \phi_g(S) \ | \ S\subseteq Q\}$ and $\eps \in (0,1]$
\WHILE{$\delta_{\max} - \delta_{\min} > \eps \delta_{\min}$}
\STATE $\delta_k:= (\delta_{\max} + \delta_{\min})/2$
\STATE Compute $\hat{z}(\delta_k)$ by solving
$S_{k+1} \; := \; \text{argmin}_{S} \; {z(S,\delta_k)} \;\; \subjectto \; {S\subseteq Q}$
\IF{$g(S_{k+1}) > 0$ \COMMENT{Then $\delta_k \ge \delta^*$.}}
\STATE $\delta_{\max} := \phi_g(S_{k+1})$ and set $S_{\max} := S_{k+1}$. \COMMENT{Note $\phi_g(S_{k+1}) \le \delta_k$.}
\ELSE
\STATE $\delta_{\min} := \delta_k$
\ENDIF
\STATE $k:=k+1$
\ENDWHILE
\STATE Return $\ARGMINzero{}{S\subseteq Q}{z(S,\delta_{\max})}$ or $S_{\max}$ based on minimum $\phi_g$
\end{algorithmic}
\end{algorithm}
At Steps $5$ to $8$ of \Cref{algo:fasterfracprog}, we make the decision to update $\delta_{\max}$ and $\delta_{\min}$ based on the optimal value of the subproblem. 
We further store the best solution so far in $S_{\max}$. In Step 10, we test if another solve with $\delta_{\max}$ produces a solution with a better objective than $S_{\max}$. This test would allow us to certify that $S_{\max}$ was optimal if the subsequent objective was not lower.  

In order to have a convergent algorithm, we have to guarantee that this decision results in a well-defined binary search.
In the following lemma, we discuss this~issue.

\begin{theorem}[Convergence of \Cref{algo:fasterfracprog} and iteration complexity]\label{thm:complexityfastdinkel}
Let $G$ be an undirected, connected graph with non-negative weights.
The binary search procedure in \Cref{algo:fasterfracprog} is well-defined, in the sense that the binary search interval includes the optimal solution, and condition in Step $5$ tells us on which side of the optimal solution the current solution is. Moreover, the sequence $\delta_k$ of \Cref{algo:fasterfracprog} converges to an approximate solution $|\delta^* - \delta_k|/\delta^* \le \eps$ in $\mathcal{O}(\log (\delta_{\max}/\eps))$ iterations, where $\delta^* = \cond_{g}(S^*)$ and $S^*$ is an optimal solution to problem \eqref{eq:fracprob}.
\end{theorem}
\begin{proof}
Let $p:=\{\max \phi_g(S) \ | \ S\subseteq Q\}$ and $\delta_{\max} \ge p$.
Let $S^*$ be an optimal solution of \probref{eq:fracprob}.
From Theorem~\ref{thm:fractionalprogramming}, we get that for $(S^*,\delta^*)$ we have that $z(S^*,\delta^*)=0$, which gives $\phi_g(S^*)= \delta^*$.
Therefore, $\delta^*\in [0,\delta_{\max}]$.
We will use this interval as our search space for the binary search. Moreover, if $g(S_{k+1})>0$, then  we get from Theorem \ref{thm:monotonicfracprog} that $\delta_k > \delta^*$. Therefore, we can use $\delta_k$ to update $\delta_{\max}$ in Step~$6$. In fact, because we have a specific set, we know that $\phi_g(S_{k+1}) \le \delta_k$ and so we can use a slightly tighter update. 
However, if $g(S_{k+1})\le0$, then we get from Theorem \ref{thm:monotonicfracprog} that $\delta_k \le \delta^*$, and we can use $\delta_k$ to define $\delta_{\min}$ in Step $8$.
If the initial $\delta_{\max}$ is greater than $p$, then it is easy to see that \Cref{algo:fasterfracprog} converges to an optimal solution of \probref{eq:fracprob}
in at most $\log (\delta_{\max}/\eps)$ iterations, where $\eps>0$ is an accuracy parameter.
\end{proof}

Note that \Cref{thm:complexityfastdinkel} is an improvement over \Cref{thm:itercomplexityDinkel}. The former requires $\mathcal{O}(\log (\delta_{\max}/\eps))$ iterations in the worst-case, while the latter states that Dinkelbach's algorithm requires $\vol(Q)$ (number of edges) iterations.
Similar results about binary search have been discussed in \citet{LR04,AL08_SODA,Hoc10}.  Among other details, what is missing from these references is an exact quantification of the value of $\eps$ necessary for an exact solution which we provide in subsequent sections. 

\subsection{The algorithmic components of cluster improvement}
\label{sec:frac-algo-parts}

We have now shown how to solve cluster improvement problems in the form of \probref{eq:fracprob} via either Dinkelbach's algorithm or the bisection-based root finding variation. The last component of the algorithmic framework is a solver for the \subprobref{eq:frac-subprob} in the appropriate Step (3 or 4) of each algorithm. Solving these subproblems is where the MinCut and MaxFlow-based algorithm arises as they allow us to test $\hat{z}(\delta) < 0$. In \Cref{chap:mqi,chap:flowimprove,chap:localflowimprove} we work through how appropriate \stMinCut and MaxFlow problems can be derived constructively. 

At this point, we summarize the major results and show an overview of the running times of the methods we will establish in these sections.
In particular, in Table~\ref{tab:summary_subproblems}, we provide pointers of algorithms and convergence theorems for each method. Also, in Table~\ref{tab:summary_subproblems} we provide a short summary of running times for each method where we make it clear that the subproblem solve time is a dominant term.

\begin{table}[t] %
    \caption{Specifics of MQI, FlowImprove and LocalFlowImprove, as special cases of Dinkelbach's \Cref{algo:fractionalprog} and its binary search version \Cref{algo:fasterfracprog}. In the table, $R$ is the input seed set of nodes. 
    	The column \textit{Subproblem} refers to the specialized subsolver that is used to solve the subproblem at Step $3$  of \Cref{algo:fractionalprog} or Step $4$ of \Cref{algo:fasterfracprog}. The \textit{Augmented Graph} entry refers to an augmented graph construction that is used to understand the subproblem that is solved at each iteration of Dinkelbach's algorithm. Note that we omit all log-factors and constants from the running times of the algorithms, more detailed running times can be found in the referenced theorems. We use $\tilde{\mathcal{O}}$ as $\mathcal{O}$-notation without logarithmic factors.}
  \centering

\label{tab:summary_subproblems}
\siamwidth{
	\setlength{\tabcolsep}{0.25\tabcolsep}
	\begin{tabularx}{\linewidth}{@{}p{0.8in}@{\hspace{10pt}}XXX@{}}
\toprule 
		\textsf{\itshape Method} & \textsf{\itshape Dinkelbach} & \textsf{\itshape Binary Search} & \textsf{\itshape Subproblem} \\ 
		& \footnotesize and Runtime & \footnotesize and Runtime & \footnotesize \mbox{Construction, Runtime,} and Solvers \\
		\midrule 
		MQI 
		& \Cref{algo:mqi} \vspace{1ex}\newline \footnotesize $\mathcal{O}(\cut(R) \cdot \texttt{subproblem})$ \newline \Cref{thm:itercomplexityMQI} \newline \scriptsize \cite{LR04}
		& \Cref{algo:fastermqi} \vspace{1ex}\newline \footnotesize $\mathcal{\tilde{O}}(\texttt{subproblem})$ \newline  \Cref{thm:complexityfastmqi} \newline  \scriptsize \cite{LR04}
		& \probref{eq:mqi-orig-problem} \vspace{1ex}\newline \footnotesize Augmented Graph \ref{proc:aug_graph_mqi} \newline  \mbox{MaxFlow with $\vol(R)$} edges (\S\ref{sec:mqi_subproblem})  \\
\midrule 
		FlowImprove 
		& \Cref{algo:flowImprove} \vspace{1ex}\newline \footnotesize $\mathcal{O}(\cut(R) \cdot \texttt{subproblem})$ \newline \Cref{thm:itercomplexityFlowImp} \newline \scriptsize \rlap{\mbox{\cite{AL08_SODA}}}
		& \Cref{algo:fasterflowimprove} \vspace{1ex}\newline \footnotesize  $\mathcal{\tilde{O}}(\texttt{subproblem} )$  \newline \Cref{thm:complexityfastflowimp} \newline \scriptsize \rlap{\mbox{\cite{AL08_SODA}}}
		& \probref{eq:mincutflowimprove_2} \vspace{1ex}\newline \footnotesize Augmented Graph~\ref{proc:aug_graph_flowimprove} \newline \mbox{MaxFlow with $\vol(G)$} edges (\S\ref{sec:flowimprove_subproblem})  \\
\midrule
		LocalFlow-Improve$(\delta)$  \vspace{1ex} \newline\footnotesize $\sigma = \delta + \tfrac{\vol(R)}{\vol(\Rbar)}$
		& \raggedright SimpleLocal \vspace{1ex}\newline \footnotesize 
		$\mathcal{O}(\cut(R) \cdot \texttt{subproblem})$ \newline 
		\Cref{thm:localfloworiginallocaltime}, 
		\scriptsize \cite{veldticml2016}
		& \Cref{algo:localflowimprove} \vspace{1ex}\newline \footnotesize $\mathcal{\tilde{O}}(\texttt{subproblem} )$ \newline \Cref{thm:localfloworiginallocaltime} \newline \scriptsize \cite{OZ14}
		& \probref{eq:mincutlocalflowimprove_2} \vspace{1ex}\newline \footnotesize Augmented Graph~\ref{proc:aug_graph_localflowimprove} \newline 
		\mbox{$\mathcal{\tilde{O}}((1\!+\!1/\sigma)^2\vol(R)^2)$ with} Alg~\ref{algo:simplelocal}  (\S\S\ref{subsubsec:localgraph}-\ref{subsubsec:simplelocal})
		\\
\bottomrule 
	\end{tabularx}}
\end{table}

\subsection{Beyond conductance and degree weighted nodes}
\label{sec:qcut-general}
Our discussion and analysis of fractional programming for cluster improvement objectives has, so far, focused on the MQI, FlowImprove and LocalFlowImprove problems as unified through Problem~\ref{eq:fracprob}. However, there is a broader class of objectives that generalizes beyond these specific types of cuts and volume ratios. We will highlight a few definitions that are reasonably straightforward to understand, although we will return to the MQI, FlowImprove, and LocalFlowImprove definitions above in the subsequent discussions. 

As an instance of a more generalized setting, we can define a generalized volume of a set $S$, which we call $\nu$, with respect to an arbitrary vector of positive weights $\vw$, 
\[ \nu(S; \vw) = \sum_{i \in S} \vw_i = \mathbf{1}_S^T \vw. \] 
Note that setting $\vw$ to be the degree vector $\vd$ gives the standard definition of volume, i.e., $\nu(S; \vd) = \vol(S)$. Then we can seek solutions of 
\[ \MINone{S \subset V}{\displaystyle \frac{\cut(S)}{\nu(S \cap R; \vw) - \kappa \nu(S \cap \Rbar; \vw) }}{\text{denominator} > 0} \] as a generalized notion of MQI, FlowImprove, and LocalFlowImprove (where $\kappa \ge \nu(R; \vw)/\nu(\Rbar; \vw)$). 

A particularly useful instance is where $\vw$ is simply the vector of all ones $1_n$. In which case $\nu(S, 1_n)$ is simply the cardinality of the set $S$. In this case $\cut(S)/\nu(S, 1_n)$ is the expansion or ratio-cut value of a set (\Cref{sec:metrics}). This approach was used in the original MQI paper~\cite{LR04}, as that paper discussed ratio-cuts instead of conductance values. This more general notion of volume also appeared in the FlowImprove paper~\cite{AL08_SODA} in order to unify the analysis of ratio-cuts and conductance objectives. While these two choices have been explored, of course, the theory allows us to choose virtually any vector and this gives a large amount of flexibility. The MaxFlow and mincut constructions for the subproblems in the subsequent sections would need to be adjusted to account for this type of arbitrary choice. This is reasonably straightforward given our derivations. For example, we could set $\vw = \sqrt{\vd}$ to generate a hybrid objective between expansion and conductance. 

As another example of how the framework can be even more general, we mention the ideas from \citet{VKG18} that \emph{penalize} excluding nodes from $R$ in the solution set $S$. These penalties can be set sufficiently large such that we can solve variations of FlowImprove and LocalFlowImprove where all the nodes in $R$ \emph{must} be in the result, for instance 
\[ \MINone{S \subset V}{\displaystyle \frac{\cut(S)}{\nu(S \cap R; \vw) - \kappa \nu(S \cap \Rbar; \vw) }}{R \subset S, \text{denominator} > 0}. \]
They can also be set smaller, however, such that we wish to have \emph{most} of $R$ within the solution $S$. This scenario is helpful when the element of $R$ may have a \emph{confidence} associated with them.

All of the analysis in subsequent sections -- including the locality of computations -- applies to these more general settings; however, the generalized details often obscure the simplicity and connections among the methods. So we do not conduct the most general description possible.  We simply wish to emphasize that it is possible and useful to do so.

\section{Cluster Improvement, Flow-based, and Other Related Methods}
\label{sec:background}
\label{sec:related-work}

As we have already briefly discussed, graph clustering is a well-established problem with an extensive literature. 
Cluster improvement algorithms have received comparatively little attention. 
In this section, we will discuss how the cluster improvement problem and algorithms for solving this problem are similar to and different than other related techniques in the literature. 
Our goal is to draw a helpful distinction and explain the relationship between cluster improvement problems/algorithms and a number of other (sometimes substantially but sometimes superficially) related topics. 

For instance, we will discuss how the cluster improvement perspective yields the best results on graph and mesh partitioning benchmark problems (\Cref{sec:mesh-clustering}). We will then highlight key differences between the types of graphs arising in scientific and distributed computing and the types of graphs based on sparse relational data and complex systems (\Cref{sec:data-in-complex-systems}), which strongly motivates the use of \emph{local algorithms} for these data.
These local graph clustering algorithms, in turn, have strong relationships with the community detection problem in networks as well as with inferring metadata, which we will explore more concretely in the empirical sections. 

Taking a step back, we explain our cluster improvement algorithms in terms of finding sets of small conductance, and so we also briefly survey the state of conductance optimization techniques more generally (\Cref{sec:conductance-optimization}). 
Likewise, our algorithms are all based on using a network flow optimizer as a subroutine to accomplish something else. 
Since this scenario is surprisingly common, e.g., because there are fast algorithms for network flow computations, we highlight a few notable applications of network-flow based computing (\Cref{sec:network-flow}) as well as the current state of the art for computing network flows (\Cref{sec:flow-algorithms}). 

Finally, we conclude this section by relating our cluster improvement perspective to network flows in continuous domains (\Cref{sec:continuous}), total variation metrics, and a wide range of work in using graph cuts and flows in image segmentation (\Cref{sec:image-graph-cuts}).

\subsection{Graph and mesh partitioning in scientific computing}
\label{sec:mesh-clustering}

Graph and mesh partitioning are important tools in  parallel and distributed computing, where the goal is to 
partition a computation into \emph{many, large} pieces that can be treated with minimal dependencies among the pieces. 
This can then be used to maximize parallelism and minimize communication in large scientific computing algorithms~\cite{Pothen-1990-partitioning,simon1991partitioning,karypis1998-metis,hendrickson1995improved,hendrickson1995multi,Karypis-1999-parmetis,leland1995chaco,Walshaw-MPDD-07,walshaw2000mesh,pellegrini1996scotch,Knight-2014-powers}. 
The traditional inputs to graph partitioning for scientific computing are graphs representing computational dependencies involved in solving a spatially discretized partial differential equation. 
In these problems, there is often a strong underlying geometry, where nodes are localized in space and edges are between nearby nodes. 
Furthermore, one of the key goals (indeed, almost a constraint in this application) is that the partitions be very well balanced so that no piece is much larger than the others.

In the context of this literature, our goal is not to produce an overall partitioning of the graph. Rather, given a piece of a partition, our tools and algorithms would enable a user to \emph{improve} that partition in light of an objective function such as graph conductance or another related objective. Indeed, work on improving and refining the quality of an initial graph bisections can be found in the Fiduccia-Mattheyses implementation of the Kernighan-Lin method~\cite{Fiduccia-1982-linear}.  Given a quality score for a two-way partition of a graph and a desired balance size, this algorithm searches among a class of local moves that could improve the quality of the partition. This improvement technique is incorporated, for instance, into the SCOTCH~\cite{pellegrini1996scotch}, Chaco~\cite{leland1995chaco}, and METIS~\cite{karypis1998-metis} partitioners.

This strategy for partition-and-improvement is also a highly successful paradigm for generating the best quality bisections and partitions on benchmark data. 
For example, on the Walshaw collection of partitioning test cases~\cite{Soper-2004-benchmark}, around half of the current best known results are the result of improving an existing partitioning using an improvement algorithm~\cite{Henzinger-2018-ilp-improve}. This has occurred a few times in the past as well~\cite{Sanders-2010-engineering,Hein-2011-relaxations,LR04}. There are important differences between the applications we consider (which are more motivated by machine learning and data science) and those in mesh partitioning for scientific computing. 
Most notably, having good balance among all the partitions is extremely important for efficient parallel and distributed computing, but it is much less so for social and information networks, as we discuss in the next section.

\subsection{The nature of clusters in sparse relational data and complex systems}
\label{sec:data-in-complex-systems}

Beyond the runtime difference between local and global graph analysis tools, there is another important reason to consider local graph analysis for sparse relational data such as social and information networks, machine learning, and complex systems. 
There is strong evidence that large-scale graphs arising in these fields~\cite{LLDM2009,LLDM08_communities_CONF,LLM10_communities_CONF,Gargi-2011-youtube,Jeub15} have interesting small-scale structure, as opposed to interesting and non-trivial large-scale global structure. 
Even aside from running time considerations, this means that global graph methods tend to have trouble identifying these small and good clusters and thus may not be well-applicable to many large graphs that arise in large-scale data applications. 
As a simple example of the impact the differences of data may have on a method, note that for graphs such as discretizations of a partial differential equation, simply enlarging a spatially coherent set of vertices results in a set of better conductance (until it is more than half the graph).
On the other hand, the sets of small conductance in machine learning and social network based graphs tend to be small, in which case enlarging them simply makes them worse in terms of conductance. 
This has been quantified by the Network Community Profile (NCP) plot~\cite{LLDM2009,Jeub15}.

\subsection{Local graph clustering, community detection, and metadata inference}
\label{sec:local-graph-clustering}
Local graph clustering is, by far, the most highly developed setting for local graph algorithms. 
A local graph clustering method seeks a cluster nearby the reference set $R$, which can be as small as a single node. 
Cluster improvement algorithms are, from this perspective, instances of local graph clustering where the input is a good cluster $R$ and the output is an even better cluster $S$. 
Local graph clustering itself emerged simultaneously out of the study of partitioning graphs for improvement in theoretical runtime of Laplacian solvers~\cite{ST13} and the limitations of global algorithms applied to machine learning and data analysis based graphs~\cite{Lang2005-spectral-weaknesses,andersen2006-communities,ACL06}. 
Subsequently, there have been a large number of developments in both theory, practice, and applications. 
These include:
\begin{itemize}
	\item improved theoretical bounds~\cite{ZLM13,AGPT2016}, 
	\item novel recovery scenarios \cite{KK2014},
	\item optimization-based approaches and formulations \cite{GM14_ICML,Gleich-2015-robustifying,FDM2017,FKSCM2017},
	\item heat kernel-based approaches \cite{chung2007-pagerank-heat,C09,CS14,KG14,Avron-2015-pagerank}, 
	\item Krylov and Lanczos-based approaches~\cite{Li-2015-lemon,2017-ecml-pkdd},
	\item local higher-order clustering based on triangles \cite{YBLG2017,Tsourakakis-2017-motif},
	\item large-scale parallel approaches \cite{SKFM2016}.
\end{itemize}

One reason for the diversity of methods in this area is that local graph clustering is a common technique to study the community structure of a complex system or social network~\cite{LLDM2009,LLDM08_communities_CONF,LLM10_communities_CONF}. 
The communities, or modules, of a network represent a coarse-grained view of the underlying system~\cite{Newman-2006-modularity,Palla2005-overlapping}. 
In particular, local clustering, local improvement, and local refinement algorithms are often used to generate overlapping groups of communities from any community partition~\cite{Lancichinetti-2009-overlapping,Xie-2013-overlapping,Whang-2016-community}. 
This is often called a \emph{local optimization and expansion} methodology. %

Another application of local graph clustering is metadata inference. The metadata inference problem is closely related to semi-supervised learning, where the input is a graph and a set of labels with many missing entries. The goal is to \emph{interpolate} the labels around the remainder of the graph. Hence, any local clustering method can also be used for semi-supervised learning problems~\cite{joachims2003transductive,ZBLWS04,liu2009robust,belkin2006manifold,zhu2003semi} (and thus, metadata inference). That said, the metadata application raises a variety of statistical consistency questions~\cite{WFM2019}, methodological questions due to a no-free-lunch theorem~\cite{Peel-2016-metadata}, as well as data suitability questions~\cite{Peel-2017-relational}. We omit these discussions in the interest of brevity and note that some caution with this approach is advisable.

Among the local graph clustering methods, the Andersen-Chung-Lang algorithm for seeded PageRank computation~\cite{ACL06} is often the de facto choice.  This method has both useful theoretical and empirical properties, namely, recovery guarantees in terms of small conductance clusters~\cite{ACL06,ZLM13} and extremely fast computation~\cite{ACL06}. It also has close relationships to many other perspectives on graph problems (e.g.~\cite{Gleich-2015-robustifying,FKSCM2017,FDM2017}, including robust and $1$-norm regularized versions of these problems. 

Cluster improvement algorithms are a natural fit for both community detection and metadata inference setting. Given any partition of the network, set of communities, set of overlapping communities, or other set of vertex sets, we can study the results of improving each set individually. This is exactly the setting of \Cref{fig:example-sbm}, where we were able to find a better partition of the network given an initial partition. (Although, these techniques may not result in a partition.) 
 Second, for metadata inference,  we simply seek to use a given label as a reference set that we \emph{improve}. 
We explore these applications from an empirical perspective in \Cref{sec:experiments}, where we compare them to a relative of the Andersen-Chung-Lang method for these tasks.

\subsection{Conductance optimization}
\label{sec:conductance-optimization}

Taking a step back, the cluster improvement algorithms we discuss improve the \emph{conductance} or \emph{ratio-cut} scores. Finding the overall minimum conductance set in a graph is a well-known NP-hard problem~\cite{SM90,LR99}. That said, there exist approximation algorithms based on linear programming~\cite{LR88,LR99}, semi-definite programming~\cite{ARV09}, and so-called cut-matching games~\cite{KRV09,OSVV12}. A full comparison and discussion of these ideas is beyond the scope of this survey. We note that these techniques are not often implemented due to complexities in the theory needed to get the sharpest possible bounds. However, these techniques do inspire new scalable approaches, for instance~\cite{Lang-2009-partitioning}.

\subsection{Network flow-based computing}
\label{sec:network-flow}

More broadly beyond conductance optimization, our work relates to the idea of using \emph{network flow} as a fundamental computing primitive itself. 
By this, we mean that many other algorithms can be cast as an instance of network flow or a sequence of network flow problems. 
When this is possible, it enables us to use highly optimized solvers for this specific purpose that often outperform more general methods. 
Bipartite matching is a well known, textbook example of this scenario~\cite[Section 7.5]{Kleinberg-book}. 
Other examples include finding the densest subgraph of a network, which is the subset of vertices with highest average degree. 
Formally, if we define
\[ 
   \text{density}(S) = \frac{\vol(S)-\cut(S)}{|S|}, 
\]
then the set $S$ that maximizes this quantity is polynomial time computable via a sequence of network flow problems~\cite{Goldberg-1984-densest-subgraph}. 
Another instance is one of the many definitions of \emph{communities} on the web that can be solved exactly as a max-flow problem~\cite{Flake-2000-communities}. 
More relevant to our setting is the work of \citet{Hoc13}, who showed that the sets that minimize 
\[ \minimize_S \; \frac{\cut(S)}{\vol(S)} \quad \text{ and } \quad  \minimize_S \;  \frac{\cut(S)}{|S|} \]
can be found in polynomial time through a sequence of max-flow and min-cut computations. 
Although feasible to compute, in general these sets are unlikely to be interesting on many machine learning and data analysis based graphs, as they will tend to be very large sets that cut off a small piece of the rest of the graph. 
(Formally, suppose there exists a node of degree 1 in an unweighted graph, then the complement set of that node will be the solution.) Among other reasons, this is the reason we use the objective functions that are symmetric in $S$ and $\Sbar$. 

Four other interesting cases show the diversity of this technique. First, the semi-supervised learning algorithm of \citet{Blum-2001-mincuts} uses the mincut algorithm to identify other vertices likely to share the same label as those that are given. Second is the use of flows to estimate a gradient in an algorithm for ranking a set of data due to \citet{Osting-2013-statistics}. Third, there are useful connections between \emph{matching} algorithms (which can be solved as flow problems) and semi-supervised learning problems~\cite{Jacobs2018}. Finally, there is a recent set of research on \emph{total variation} or \emph{TV} norms in graphs and the connections to network flow~\cite{Jung2019}. These were originally conceptualized for semi-supervised learning. They can also be used to build local clustering mechanisms that optimize a combination of 2-norm and 1-norm  objectives with max-flow techniques~\cite{Jung2021}.

\subsection{Recent progress on network flow algorithms}
\label{sec:flow-algorithms}
Having flow as a subroutine is useful because there is a large body of work in both theory and practice at making flow computations fast. For an excellent survey of the overall problem, the challenges, and recent progress, we recommend~\citet{Goldberg-2014-flows}. This overview touches on the exciting line of work in theory that showed a connection between Laplacian linear system solving and approximate maximum flow computations~\cite{Christiano-2011-max-flow,Lee-2013-max-flow} as well as recent progress on the exact problem~\cite{Orlin-2013-max-flow}. %
We refer readers to~\citet{LS13,PS19} as well. Also, we refer the reader to software packages that compute maximum flows fast~\cite{LEMON}.

\subsection{Continuous and infinite dimensional network flow and cuts}
\label{sec:infinite-flow}
\label{sec:continuous}
Our approach in this survey begins with a finite graph based on data and is entirely finite dimensional. 
Alternative approaches seek to understand problems in the continuous or infinite dimensional setting. 
For instance, \citet{Strang-1983-continuous-max-flow} posed a continuous maximum-flow problem in a domain, where the goal is to identify a function that satisfies continuous generalizations of the flow-conditions. 
As a quick example of these generalizations, recall that the cut of a set $S$ can be computed as $\normof[\mC,1]{\mB \vx}$. The total variation of an indicator function for a set generalizes the cut quantity to a continuous domain. 
This connection, and it's relationship to sharp boundaries, motivates total variation image denoising~\cite{Rudin-1992-tvd} as well as ideas of continuous minimum cuts~\cite{Chan-2006-denoising}.
Continued development of the theory~\cite{Strang-2010-max-flow-plane} has led to interesting new connections between the infinite dimensional and finite dimensional cases~\cite{Yuan-2010-continuous-max-flow}. 
There are strong connections in motivation between our cluster improvement framework and finding optimal continuous functions in these settings -- e.g., we can think of sharpening a blurry, noisy image as improving a cluster (see \Cref{fig:mqi-for-images}) -- but the details of the algorithms and data are markedly different. 
In particular, we largely think of the cluster improve routine as a strongly local operation.
Understanding how these ideas generalize to continuous or infinite dimensional scenarios is an important problem raised by our approach.

\begin{figure}
\newlength{\binaryimagesize}
\setlength{\binaryimagesize}{0.97in}
\tronly{\setlength{\binaryimagesize}{1.23in}}
\centering \footnotesize 
\noindent\begin{tabular}{*{5}{@{}p{\binaryimagesize}}@{}l@{}}
\includegraphics[width=\linewidth]{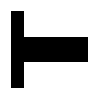} & 	
\includegraphics[width=\linewidth]{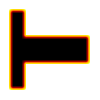} & 
\includegraphics[width=\linewidth]{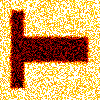} & 
\includegraphics[width=\linewidth]{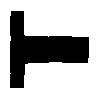} &
\includegraphics[width=\linewidth]{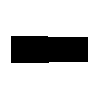} &
\includegraphics{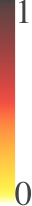} \\
 (a) Original & 
 (b) Boundary blur & 
 (c) Blur and noise & 
 (d) MQI-like \newline ($\delta\!=\!0.11$) solution & 
(e) MQI-like \newline ($\delta\!=\!0.04$) solution 
\end{tabular}
\caption{An example of using MQI-like procedures to reconstruct a binary image (a) from a blurry (b), noisy sample (c). Here, the result set is a binary image, which is a set in the grid graph. The value $\delta$ is from the fractional programming 
subproblem~\eqref{eq:frac-subprob}
with a custom denominator term as described in the reproduction details.  Using $\delta = 0.11$ produces 209 error pixels around the boundary (d). Reducing $\delta$ to 0.04 (e) produces a convex shape due to the conductance-like bias in this setting where convex shapes are optimal for isoperimetric-like objectives on grids.}
\label{fig:mqi-for-images}	
\end{figure}

\subsection{Graph cuts and max flow-based image segmentation}
\label{sec:image-graph-cuts}

One final application of maximum flows is graph cut-based image processing~\cite{Boykov-2006-graph-cut,Marlet-2017-graph-cuts}. 
The general setting in which these arise is an energy minimization framework~\cite{Greig-1989-exact,Kolmogorov-2004-graph-cuts} with binary variables. 
The goal is to identify a binary latent feature in an image as an exact or approximate solution of an optimization problem. 
An extremely large and useful class of these energy functions can be solved via a single or a sequence of max-flow computations. 
The special properties of the max-flow problems on image-like data motivated the development of specialized max-flow solvers that, empirically, have running time that scales linearly in the size of the data~\cite{Boykov-2004-maxflow}. 

This methodology has a number of applications in image segmentation in 2d and 3d images~\cite{Boykov-2006-graph-cuts} such as MRIs. For instance, one task in medical imaging is separating water from fat in an MRI, for which a graph cut based approach is highly successful~\cite{Hernando-2010-robust}. More recently, deep learning-based methods have often provided a substantial boost in performance for image processing tasks. Even these, however, benefit from a cluster improvement perspective. Multiple papers have found that post-processing or refining the output of a convolutional neural net using a graph cut approach to yield improved results in segmenting tumors~\cite{Ullah-2018-cnn-graph-cut,Ma-2018-tumor-graph-cut}. These recent applications are an extremely close fit for our cluster improvement framework, where the goal is to find a small object in a big network starting from a good reference region. We often illustrate the benefits and differences between our methodologies with a closely related problem of refining a local image segmentation output, e.g, \Cref{fig:example-astronaut}.

	\addcontentsline{toc}{part}{Part II. Technical Details Underlying the Main Theoretical Results}
	\section*{\large Part II. Technical Details Underlying the Main Theoretical Results}

\section{Minimum Cut and Maximum Flow Problems}
\label{sec:mincutmaxflow}

As a simple introduction to our presentation of the technical details of MQI, FlowImprove, and LocalFlowImprove, we will start with the minimum cut and maximum flow problems. We will review the basics of these problems from an optimization and duality perspective. This is because our technical discussions in subsequent sections will constitute related, but more intricate, transformations, and will use maximum flow problems as a subroutine. To simplify the text, we use the names MinCut and MaxFlow to refer to the $s$-$t$ minimum cut and $s$-$t$ maximum flow problems,  which are the fully descriptive terms for these~problems.

\subsection{\stMinCut}\label{subsec:mincut}

Given a graph $G=(V,E)$, let $s$ and $t$ be two special nodes where $s$ is commonly called the \emph{source node} and $t$ is the \emph{sink node}.
The undirected \stMinCut problem is:
\begin{equation}
\label{eq:mincut}
 \MINone{S}{\cut(S,\Sbar)}{s \in S, t \in \Sbar,\ S\subseteq V  .}
\end{equation}
The objective function of the \stMinCut problem measures the sum of the weights of edges between the sets $S$ and $\Sbar$.
The constraints encode the idea that we want to separate the source from sink and so we want the source node $s$ to be in $S$ and the sink node $t$ to be in $\Sbar$. Putting the objective function
and the constraints together, we see that the purpose of the \stMinCut problem is to find a partition $(S,\Sbar)$ that minimizes the number of edges needed to separate node $s$ from node $t$.
As an example, see \Cref{fig:mincutfig}, where we demonstrate the optimal partition for the MinCut \probref{eq:mincut} on a toy graph.

\begin{figure}[t]
	\centering
		\includegraphics[width=0.5\linewidth]{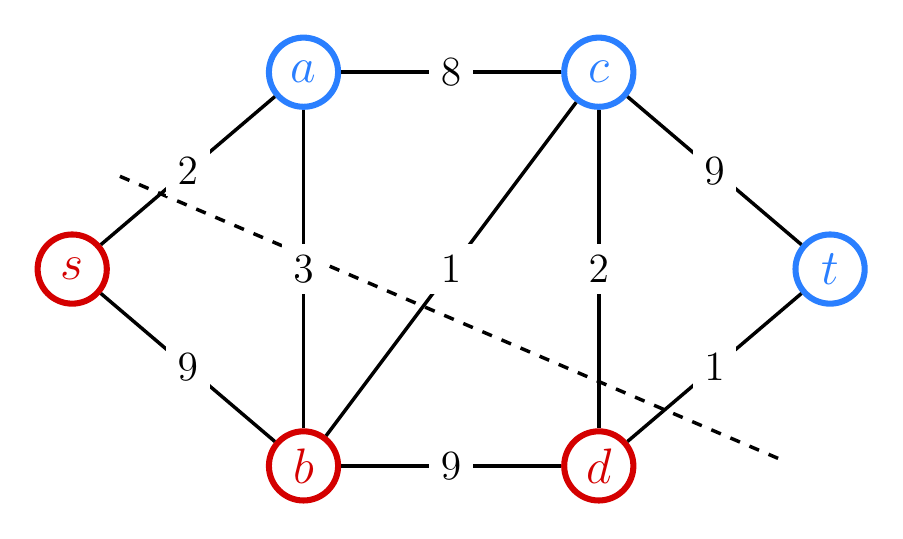}
	\caption{Demonstration of the optimal MinCut solution of \probref{eq:mincut}.
                 The numbers show the weight of each edge.
                 The red nodes ($s$, $b$, $d$) and the blue nodes ($a$, $c$, $t$) denote the optimal partition $(S, \Sbar)$, respectively, for \probref{eq:mincut}.
                 The black dashed line denotes the edges that are being cut, i.e., the edges that cross the partition $S$ and $\Sbar$.
                 The optimal objective value of \probref{eq:mincut} for this example is equal to $9$.
                 }
	\label{fig:mincutfig}
\end{figure}

We can express the \stMinCut problem in other equivalent ways, some of which are more convenient for analysis and  implementations.
For example, we use indicator vector notation and the incidence matrix from \Cref{sec:notation} to represent \probref{eq:mincut}~as
\begin{equation}
\label{eq:mincut_B}
\MINone{\vx}{\|\mB\vx\|_{\mC,1}}{x_s = 1, x_t = 0, \vx \in \{0,1\}^{n}.}
\end{equation}
\tronly{\enlargethispage{-\baselineskip}}
Expressing the \stMinCut problem with this notation will be especially useful later when we develop a unified framework for many cluster improvement algorithms.
In practice, when implementing a solver for this problem, we need not take the binary constraints into account.\ This is because we can relax them \emph{without changing the objective value} to obtain the following equivalent form of the \stMinCut problem:
\begin{equation}
\label{eq:mincut_b_relax}
\MINone{\vx}{\|\mB\vx\|_{\mC,1}}{x_s = 1, x_t = 0, \vx \in \RR^{n}.}
\end{equation}
It can be shown that there exists a solution to \eqref{eq:mincut_b_relax} that has the same objective function value as the optimal solution of \eqref{eq:mincut_B}.
Given any solution to the relaxed problem, the integral solution can be obtained by an exact rounding procedure. In that sense, the relaxed problem \eqref{eq:mincut_b_relax} and the integral problem \eqref{eq:mincut_B} are equivalent~\citep{PS82}.
In the next subsection, we will obtain a solution to \eqref{eq:mincut_B} through the MaxFlow problem.

\subsection{Network Flow and MaxFlow}\label{sec:flow}
\label{sec:maxflow}
We provide a basic definition of a network flow, which is crucial for defining MaxFlow. For more details about network flows we recommend reading the notes of \citet{trevisan2011}.

Network flows are commonly defined on directed graphs. Given an undirected graph, we will simply allow flow to go in both directions of an edge.
This means that instead of doubling the number of edges, which is a common technique in the literature, we fix an arbitrary direction of the edges, encoded in the $\mB$ matrix, and simply let flow go in either direction by allowing the flow variables to be negative.
Also, in the context of flows, edge weights are usually called edge capacities. We will use these terms interchangeably, but we tend to use capacities when discussing flow and weights when discussing~cuts.

A network flow is a mapping that assigns values to edges, i.e., a mapping $f : E\to \mathbb{R}$ from the set of edges $E$ to $\mathbb{R}$, which also satisfies capacity and flow conservation constraints. We view $\vf$ as a vector that encodes this mapping for a fixed ordering of the edges consistent with the incidence matrix. The capacity constraints are easy to state. Let $\vc = \diag(\mC)$ be the capacity for each edge, we need \[ -\vc \le \vf \le \vc \]
so that the flow along an edge is bounded by its respective capacity.
The flow preservation constraints ensure that flow is only created at the source and removed at the sink and that all other nodes neither create nor destroy flow. This can be evaluated using the incidence matrix that, given a flow $\vf$ mapping, computes the changes via $\mB^T \vf$. Consequently, flow conservation is written
\begin{align*}
\mB^T \vf = \vq-\vp
\end{align*}
where $p_s\in\RR, p_i=0 \ \forall i\in V\backslash \{s\} $ and $q_t\in\RR, q_i=0 \ \forall i\in V\backslash \{t\}$.
The maximum flow problem is to compute a feasible network flow with the maximum amount of flow that emerges from the source and gets to the sink.
The corresponding MaxFlow optimization problem can be expressed~as:
\begin{equation}
\label{eq:maxflow}
\MAXfour{\vf,\vp,\vq}{\vp^T \mathbf{1}_s}{\mB^T \vf = \vq-\vp }{p_s\in\RR, p_i=0 \ \forall i\in V\backslash \{s\}}{q_t\in\RR, q_i=0 \ \forall i\in V\backslash \{t\}}{-\vc \le \vf \le \vc.}
\end{equation}
See \Cref{fig:maxflowfig} for a visual demonstration of the flow variables and the optimal solution of \probref{eq:maxflow} for the same graph used in \Cref{fig:mincutfig}.
\begin{figure}[ht]
	\centering
	\subfigure[Demonstration of a (non-optimal) flow\label{fig:aflow}]%
	{\includegraphics[width=0.45\linewidth]{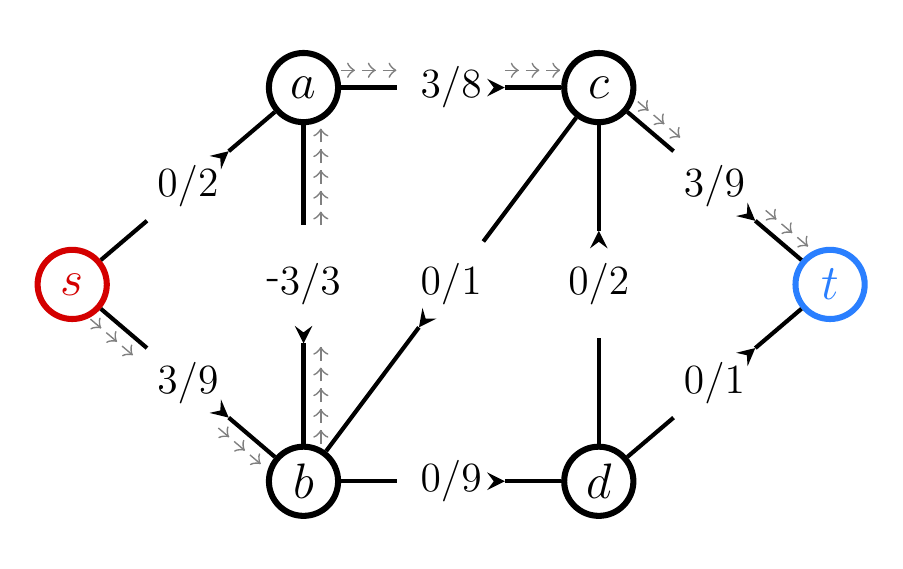}}
	\hfill
	\subfigure[Demonstration of the MaxFlow solution\label{fig:maxflowsol}]%
	{\includegraphics[width=0.45\linewidth]{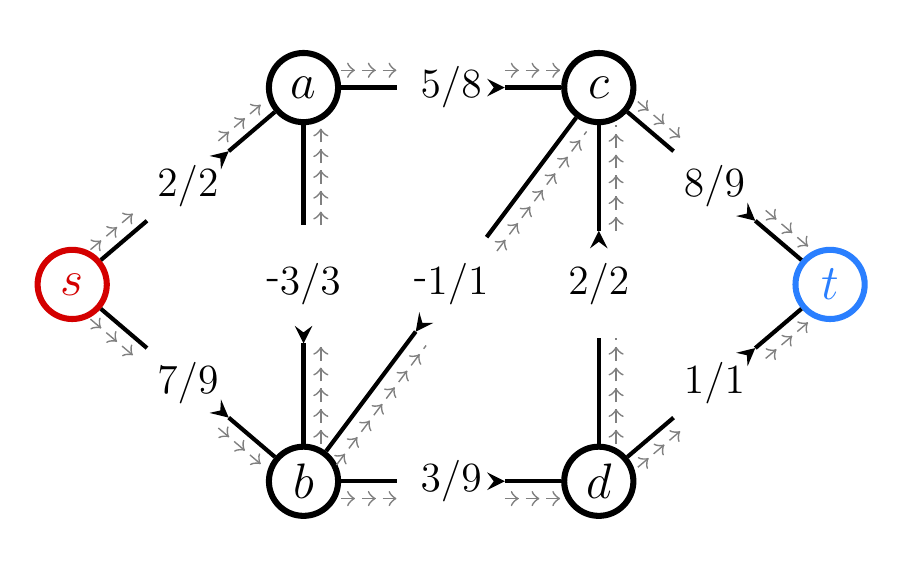}}
  \caption{In this figure, all edges are undirected edges but each edge has an arrow in the middle indicating the positive direction of flow. A negative flow value on an edge means that the flow is flowing against the positive direction. The numerators in each expression show the flow that passes through an edge, and the denominators in each edge show the capacity of each edge. In Subfigure \ref{fig:aflow}, we demonstrate a flow that starts from the source node $s$ and ends to the sink node $t$ and has value equal to $3$.
		The path of that flow is highlighted by gray dashed arrows and includes nodes $s$, $b$, $a$, $c$ and $t$.
		Note that the flow in Subfigure \ref{fig:aflow} is not optimal since we can send more flow from the source to the sink while satisfying the constraints of Problem \eqref{eq:maxflow}.
		The optimal solution of the MaxFlow \probref{eq:maxflow} for this toy graph is shown in Subfigure \ref{fig:maxflowsol}.
		The optimal flow that can be sent from the source to the sink is equal to $9$. 
	}
	\label{fig:maxflowfig}
\end{figure}

We will obtain the MaxFlow \probref{eq:maxflow} by computing the Lagrange dual of the relaxed MinCut \probref{eq:mincut_b_relax}.\
For basics about Lagrangian duality, we refer the reader to Chapter $5$ in~\citet{boyd2004convex}.\
The process of obtaining the dual of a problem is important for us, because it will allow us to understand how to implement flow-based clustering methods in subsequent sections.\
First, we will convert \probref{eq:mincut_b_relax} into an \emph{equivalent} linear program
\begin{equation}
\label{eq:mincut_b_relax_lp_uv}
\MINthree{\vx, \vu, \vv}{\vc^T\vu + \vc^T\vv}{\mB\vx = \vu - \vv}{x_s = 1, x_t = 0, \vx \in \RR^{n}}{\vu,\vv\ge 0.}
\end{equation}
This can be done by starting with \probref{eq:mincut_b_relax} and following standard steps in conversion of a linear program into standard form. Here, this involves introducing non-negative variables $\vu$ and $\vv$ such that $\mB \vx = \vu - \vv$ and then writing the objective as above. (Note that due to the minimization, at optimality, we will never have both $\vu$ and $\vv$ non-zero in the same index.)
Consequently, the Lagrangian function of \probref{eq:mincut_b_relax_lp_uv} is given by
\begin{equation}\label{eq:mincut-lagrangian}
\begin{aligned}
L(\vu,\vv,\vx,\vf,\vs,\vg,\vp,\vq) & = \vc^T\vu + \vc^T\vv - \vf^T(\mB\vx - \vu + \vv) - \vs^T\vu - \vg^T\vv \\
    					       & \qquad - \vp^T(\vx-\mathbf{1}_s) + \vq^T\vx \\
					       & = (\vf - \vs + \vc)^T\vu + (-\vf - \vg + \vc)^T\vv + (-\mB^T\vf - \vp \\
					       & \qquad + \vq)^T\vx + \vp^T \mathbf{1}_s,
\end{aligned}
\end{equation}
where $\vs,\vg \ge 0$, $\vf\in\RR^m$, $\vp_s\in\RR$ and $\vp_i=0$ $\forall i\in V\backslash \{s\}$, $\vq_t\in\RR$ and $\vq_i=0$ $\forall i\in V\backslash \{t\}$.
The latter constraints are important for Lagrangian duality because they guarantee that the dual function (that we will derive below) will provide a lower bound for the optimal solution of the primal \probref{eq:mincut_b_relax}.
See Chapter $5$ in~\citet{boyd2004convex}.
The dual function is
\begin{equation}
h(\vf,\vs,\vg,\vp,\vq) := \min_{\vu,\vv,\vx} \ L(\vu,\vv,\vx,\vf,\vs,\vg,\vp,\vq).
\end{equation}
Note that the Lagrangian function $L$ is a linear function with respect to $\vu,\vv,\vx$.\ Therefore, we can obtain an analytic form for the dual function by requiring the partial derivatives of $L$ with respect to $\vu$, $\vv$ and $\vx$ to be zero.\
The following three equations arise from the latter process:
\[ \mB^T \vf + \vp - \vq = 0_n \qquad \vf - \vs + \vc = 0_n \qquad -\vf - \vg + \vc  = 0_n. \]
By substituting these conditions into~\eqref{eq:mincut-lagrangian}, we have
\begin{equation}
h(\vf,\vs,\vg,\vp,\vq) = \vp^T \mathbf{1}_s,
\end{equation}
with domain that is defined by the following constraints
\[ \begin{array}{ccc}
\mB^T \vf + \vp - \vq = 0_n & \vf - \vs + \vc = 0_n & -\vf - \vg + \vc  = 0_n \\
p_s\in\RR, p_i=0 \ \forall i\in V\backslash \{s\}
& q_t\in\RR, q_i=0 \ \forall i\in V\backslash \{t\}
& \vs,\vg\ge 0.
\end{array}\]
Thus, we obtain that the dual problem of \probref{eq:mincut_b_relax_lp_uv} is
\begin{equation}\label{eq:maxflow_ugly}
\MAXsix{\vf,\vs,\vg,\vp,\vq}{\vp^T \mathbf{1}_s = h(f,s,g,p,q)}{\mB^T \vf = \vq -\vp }{\vf - \vs + \vc = 0}{-\vf - \vg + \vc = 0}{p_s\in\RR, p_i=0 \ \forall i\in V\backslash \{s\}}{q_t\in\RR, q_i=0 \ \forall i\in V\backslash \{t\}}{\vs,\vg\ge 0.}
\end{equation}
By eliminating the variables $\vs$ and $\vg$ we obtain the MaxFlow problem \eqref{eq:maxflow}. (These correspond to \emph{slack} variables associated with $-\vc \le \vf \le \vc$.)

Both the primal~\eqref{eq:mincut_b_relax_lp_uv} and dual~\eqref{eq:maxflow_ugly} are feasible (with a trivial cut and a zero flow, respectively) and also have finite solutions (0 is a lower bound on the cut and $\vol(G) = 1^T \vc $ is an upper bound on the flow). So, strong duality will hold between the two solutions at optimality, and the optimal value of the MaxFlow \probref{eq:maxflow} is equal
to the optimal value of the relaxed MinCut \probref{eq:mincut_b_relax} (which is equal to the optimal value of~\eqref{eq:mincut_B}). This fact is often one component of the so-called the MaxFlow-MinCut Theorem. Another important piece is discussed next.

\subsection{From MaxFlow to MinCut}
\label{sec:maxtomin}

Assume that we have solved the MaxFlow problem to optimality and that we have obtained the optimal flow $\vf$.\
Then the MaxFlow-MinCut Theorem is a statement about the equivalence between the objective function value of the optimal solution to the MinCut \probref{eq:mincut} and the objective function value of the optimal solution of the MaxFlow \probref{eq:maxflow}.\
In many cases, obtaining this quantity suffices; but, in some cases, we want to work with the actual solutions themselves.\

To obtain the optimal MinCut solution from an optimal MaxFlow solution, we define the notion of a \textit{residual graph}.\
A residual graph $G_{\vf}$ of a given $G$ has the same set of nodes as $G$, but for each edge $e_{ij}\in E$, it has a forward edge $\tilde{e}_{ij}$, i.e., from node $i$ to node $j$, with capacity $\mbox{max}(c_{ij}-f_{ij},0)$ and a backward edge $\hat{e}_{ji}$, i.e., from node $j$ to node $i$, with capacity $\mbox{max}(f_{ij},0)$, where $f$ is the optimal solution of the MaxFlow \probref{eq:maxflow}.\
A demonstration of a residual graph for a given flow is shown in \Cref{fig:residualgraph}.
\begin{figure}[tb]
	\centering
	\subfigure[The residual graph of the non-optimal flow \Cref{fig:aflow}\label{fig:residual_graph_blocking}]%
		{\includegraphics[width=0.45\linewidth]{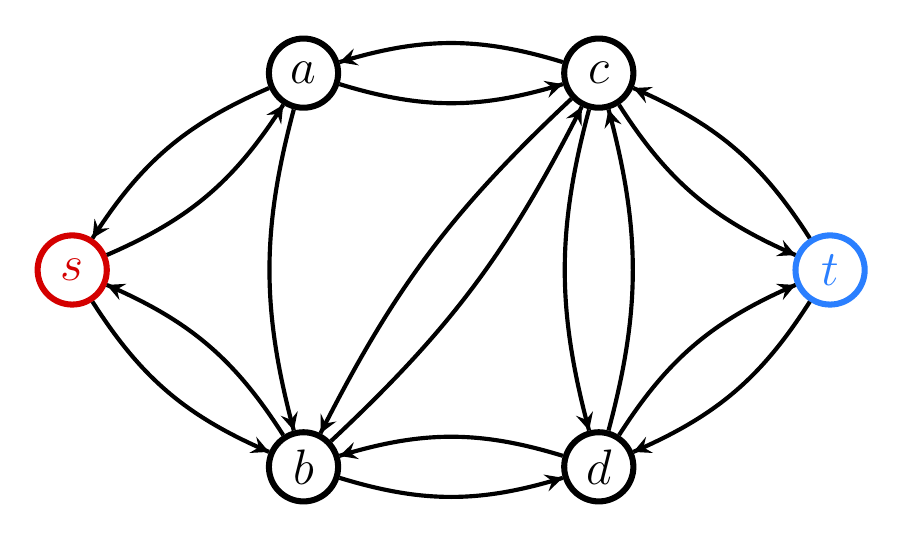}}
				\hfill
	\subfigure[The residual graph of the MaxFlow solution \Cref{fig:maxflowsol}\label{fig:residual_graph_maxflow}]%
		{\includegraphics[width=0.45\linewidth]{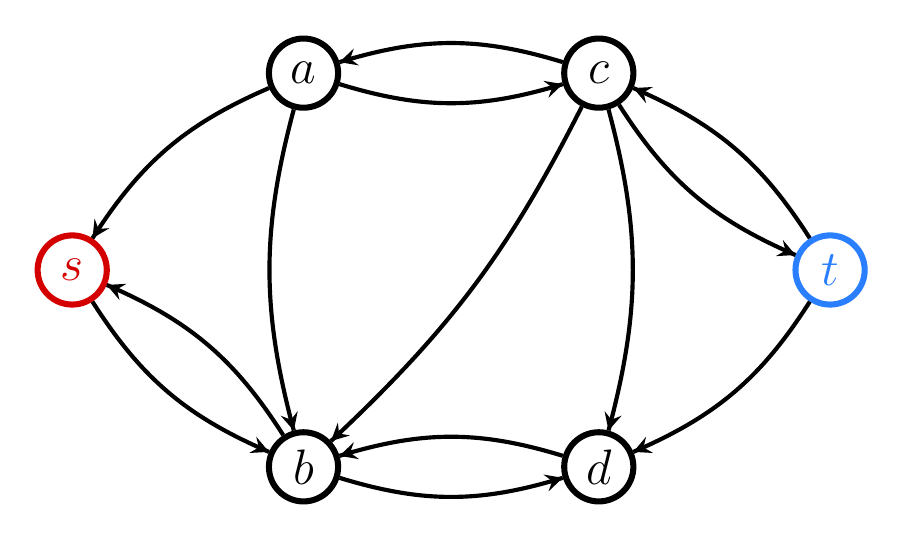}}
	\caption{The two subfigures show the directed residual graph for the flows from \Cref{fig:maxflowfig}. The edge capacities are removed for simplicity. Edges are only shown if they have positive capacity. %
                 Note that the flow in \Cref{fig:aflow} is not optimal since we can send more flow from the source to the sink while satisfying the constraints of \probref{eq:maxflow}, this is equivalent to having a $s$ to $t$ path in the residual graph in Subfigure \ref{fig:residual_graph_blocking}.
				In Subfigure \ref{fig:residual_graph_maxflow}, we show the corresponding residual graph for the optimal flow, and note that in the residual graph of the MaxFlow solution there is no path from the source node to the sink node. 
                }
	\label{fig:residualgraph}
\end{figure}

Note that there cannot exist a path from $s$ to $t$ in the residual graph at a max-flow solution. (Otherwise, we would be able to increase the flow!)
Consequently, we can look at the set $S$ of vertices reachable starting from the source node $s$ (this can be algorithmically identified using a breadth-first or depth-first search starting from $s$). It is now a standard textbook argument that the cut of the set $S$, which does not contain $t$, is equal to the maximum flow.

\subsection{MaxFlow solvers for weighted and unweighted graphs}
 MaxFlow problems can be solved substantially faster than general linear programs. See our discussion in \Cref{sec:flow-algorithms} for more information on state-of-the-art solvers. 

It is often assumed that the graphs are unweighted or have integer positive weights. All of the MaxFlow problems we need to solve will be weighted with rational weights that depend on the current estimate of the ratio in the fractional programming problem. Many of the same algorithms can be applied for weighted problems as well. We explicitly mention both Dinic's algorithm~\cite{Dinitz-1970-max-flow} and the Push-Relabel algorithm~\cite{Goldberg-Rao}, both of which can be implemented for the types of weighted graphs we need. In our implementations, we use Dinic's algorithm. In these cases, however, the runtime becomes slightly tricky to state and is fairly pessimistic. Consequently, when we have a runtime that depends on MaxFlow, we simply state the number of edges involved in the computation as a proxy for the~runtime.

\section{The MQI Problem and Algorithm}
\label{chap:mqi}

In this section, we will describe the MaxFlow Quotient-Cut Improvement (MQI) algorithm, due to~\citet{LR04}.
This cluster improvement method takes as input a graph $G=(V,E)$ and a reference set $R\subset V$, with $\vol(R) \le \vol(G)/2$, and it returns as output an ``improved'' cluster, in the sense that the output is a subset of $R$ of minimum conductance.

The basic MQI problem is 
\begin{equation}
\label{eq:mqiprob}
 \boxed{\MINone{S}{\dfrac{\cut(S)}{\vol(S)}}{S \subseteq R.}}
\end{equation}
Due to the assumption that $\vol(R) \le \vol(G)/2$, this problem is equivalent to 
\begin{equation}
 \MINone{S}{\cond(S)}{S \subseteq R.}
\end{equation}
\aside{aside:big-set}{A curious implication of the MQI objective is that it is NP-hard to find a set $S$ with $\vol(S) \le \vol(G)/2$ that even \emph{contains} the set of minimum conductance.} 
In the equivalence with conductance, this constraint that $\vol(R) \le \vol(G)/2$ is crucial because it makes this problem polynomially solvable. Without this constraint, the problem with conductance is intractable, however we can still minimize the cut to volume ratio even when $\vol(R) > \vol(G)/2$. 

Recall that this MQI problem is related to the fractional programming \probref{eq:fracprob} by setting $g(S):=\vol(S)$ and $Q=R$.
\citet{LR04} describe an algorithm to solve the MQI problem, which is equivalent to what is presented as \Cref{algo:mqi}. (They describe solving \cref{eq:mqi-orig-problem} via the flow procedure we will highlight shortly.) 
It is easy to see that this algorithm is simply \Cref{algo:fractionalprog} for fractional programming specialized to this scenario. Consequently, we can apply our standard~theory.

\begin{algorithm}
\caption{MQI~\citep{LR04}}\label{algo:mqi}
\begin{algorithmic}[1]
\STATE Initialize $k := 1$, $S_1:=R$ and $\delta_1:=\cond(S_1)$.
\WHILE{we have not exited via else clause}
\STATE Solve 
$S_{k+1}:= \mathop{\text{argmin}}\limits_{S \subseteq R} {\cut(S) - \delta_k \vol(S)}$
\IF{$\cond(S_{k+1}) < \delta_{k}$}
	\STATE $\delta_{k+1}:=\cond(S_{k+1})$
\ELSE
	\STATE $\delta_k$ is optimal, return previous solution $S_k$.
\ENDIF
\STATE $k:=k+1$
\ENDWHILE
\end{algorithmic}
\end{algorithm}

The following theorem implies that MQI monotonically decreases the objective function in \probref{eq:mqiprob} at each iteration. 
It was first shown by~\citet{LR04}, but it is a corollary of Theorem \ref{thm:monotonicfracprog}.
Note that $\delta_k$ is equal to the objective function of \probref{eq:mqiprob} evaluated at $S_k$.

\begin{theorem}[Convergence of MQI] 
Let $G$ be an undirected, connected graph with non-negative weights. 
Let $R$ be a subset of vertices with $\vol(R) \le \vol(\Rbar)$. The sequence $\delta_k$ monotonically decreases at each iteration of MQI. 
\end{theorem}

\subsection{Solving the MQI subproblem using MaxFlow algorithms}
\label{sec:mqi_subproblem}

In this subsection, we will discuss how to solve efficiently the subproblem at Step $3$ of MQI \Cref{algo:mqi}, namely 
 \begin{equation} \label{eq:mqi-orig-problem} 
 \ARGMINone{}{S}{\cut(S) - \delta \vol(S)}{S\subseteq R.}	
 \end{equation}
The summary of this subsection is that the subproblem corresponds to a MinCut-like problem and by introducing a number of modifications, we can turn it into an instance of a \stMinCut problem. This enables us to use MaxFlow solvers to compute a binary solution efficiently.  The final solver will run a MaxFlow problem on the subgraph of $G$ induced by $R$ along with a few additional edges. 

By translating \probref{eq:mqi-orig-problem} into indicator notation, we have 
\begin{equation}
\label{eq:mincutmqi_2}
\MINone{\vx}{\|\mB\vx\|_{\mC,1} - \delta\vx^T \vd}{x_i = 0 \ \forall i \in \Rbar, \vx \in \{0,1\}^{n}.}
\end{equation}
This is not a \stMinCut problem as stated, but there exists an equivalent problem that is a \stMinCut problem. To generate this problem, we'll go through two steps. First, we'll shift the objective to be non-negative. This is necessary because a \stMinCut problem always has a non-negative objective. Second, we'll introduce a source and sink to handle the terms that are not of the form $\|\mB\vx\|_{\mC,1}$ and the equality constraints.  Again, this step is necessary because these problems must have a source and sink. 

For step 1, note that the maximum negative term is $\delta \ones^T \vd$. (It's actually smaller due to the equality constraints, but this overestimate will suffice.) Thus, we shift the objective by this value and regroup terms
\begin{equation}
\label{eq:mincutmqi_3}
\MINone{\vx}{\|\mB\vx\|_{\mC,1} + \delta(\ones - \vx)^T \vd}{x_i = 0 \ \forall i \in \Rbar, \vx \in \{0,1\}^{n}.}
\end{equation}

Note that $\ones-\vx$ is simply an indicator for $\Sbar$, the complement solution set. Consequently, we want to introduce a penalty for each node placed in $\Sbar$. To do so, we introduce a source node $s$ that will connect to each node of the graph with weight proportional to the degree of each node. (A penalty for $\Sbar$ corresponds to an edge from the source $s$.) 
Since nothing in $\Rbar$ can be in the solution, we can introduce a sink node $t$ and connect it with infinite weight edges to each node in $\Rbar$. Thus, these edges will never be cut at optimality, as there is a finite-valued objective possible. Also note that the infinite weight can be replaced by a sufficiently large graph-dependent weight to achieve the same effect. 

This \stMinCut construction is given in \Cref{fig:mqi_maxflow_graph_c}, although this omits the edges from $s$ to nodes in $\Rbar$. This construction, however, is not amenable to a strongly local solution method, as it naively involves the entire graph. 
Note that in practice, we can form the graph construction in \Cref{fig:mqi_maxflow_graph_d} with the collapsed vertices without ever examining the whole graph.  

To generate a strongly local method, note that we can collapse all the vertices in $\Rbar$ and $t$ into a single super-sink $t$. This simply involves rewiring all edges $(u,v)$ where $u \notin \Rbar$ and $v \in \Rbar$ into a new edge $(u,t)$ where we handle multiedges by summing their weights. This results in a number of $s$ to $t$ edges, one for each node in $\Rbar$, which we can further delete as they exert a constant penalty of $\delta \vol(\Rbar)$ the final objective. An illustration is given in \Cref{fig:mqi_maxflow_graph_d}.
Importantly, in \Cref{fig:mqi_maxflow_graph_d}, there are only a small number of nodes in $\Rbar$ that are collapsed into the sink node $t$, but $\Rbar$ could have had thousands or millions or billions of nodes.
In that case, the final graph would still have only a very small number of nodes, in which case strongly local algorithms would be \emph{much} faster.

\begin{subroutine}
	\caption{for the subproblem at Step $3$ of MQI \Cref{algo:mqi}}\label{proc:aug_graph_mqi}
	\begin{algorithmic}[1]
		\STATE Extract the subgraph with nodes in $R$ and the edges of these nodes, which we denote by $E(R)$.
		\STATE Add to the set of nodes $R$ a source node $s$ and a sink node $t$. 
		\STATE Add to the set of edges $E(R)$ an edge from the source node $s$ to every node in the seed set of nodes $R$ with weight the degree of that node times $\delta$.
		\STATE For any edge in $G$ from $R$ to $\Rbar$, rewire it to node $t$ and combine multiple edges by summing their weights.
	\end{algorithmic}
\end{subroutine}

To recap, see the Augmented Graph~\ref{proc:aug_graph_mqi} procedure. We now give an explicit instance of the \stMinCut problem to illustrate how it maps to our desired binary objective. Let $\mB(R)$ and $\mC(R)$ be the incidence and weight matrix for the subgraph induced by the set $R$. Then consider the incidence matrix and the diagonal edge-weight matrix of the modified graph, which are 
\[ 
\tilde{\mB} :=
\begin{blockarray}{ccc}
s & R & t \\
\begin{block}{[ccc]}
\ones & -\mI          & 0  \\
0 & \mB(R)            		  & 0 \\
0 & \mI               & -\ones \\
\end{block}
\end{blockarray} 
\qquad 
\tilde{\mC} :=
\begin{blockarray}{ccc}
  &   &  \\
\begin{block}{[ccc]}
\delta\mD_{R} & 0                  & 0  \\
0 & \mC(R)   & 0 \\
0 & 0                   & \mZ \\
\end{block}
\end{blockarray}  ,
\] 
where $\mD_R$ is the submatrix of $\mD$ corresponds to nodes in $R$ (ordered conformally), and 
$\mZ$ is a diagonal matrix that stores the weights of the rewired edges from $R$ to the sink $t$, i.e., 
\[ Z_{ii} = \sum_{ e } c_e, \text{ where $e$ is an edge from $i \in R$ to any node in $\Rbar$ }. \]
(These weights can be zero if there are no edges leaving from a node $i \in R$.) 
The first column of matrix $\tilde{\mB}$ corresponds to the source node, the last column corresponds to the sink node, and all other columns in-between
correspond to nodes in $R$. The first block $\delta\mD_{R}$ in $\tilde{\mC}$ corresponds to edges from the source to nodes in $R$, the second block $\mC_{R}$ in $\tilde{\mC}$ corresponds to edges from $R$ to $R$, and the third block $\mZ$ in $\tilde{\mC}$
corresponds to edges from nodes in $R$ to the sink node $t$.
Let 
\[ \tilde{\vx} := \begin{bmatrix} \vx_s \\
			   		  \vx_R \\
					  \vx_t
\end{bmatrix},
\text{ so that } \tilde{\vx}_1 = \vx_s \text{ and } \tilde{\vx}_{|R|+2} = \vx_t
\]
then the MinCut problem with respect to the modified graph is 
\begin{equation}
\label{eq:mincutmqi}
\MINone{\tilde{\vx}}{\|\tilde{\mB}\tilde{\vx}\|_{\tilde{\mC},1} = \|\mB(R) \vx_R\|_{\mC(R),1} + \delta \ones^T D_{R} (\ones_R - \vx_R) + \ones^T \mZ \vx_R }{\tilde{\vx}_1 = 1, \tilde{\vx}_{|R|+2}=0, \tilde{\vx}_i \in \{0,1\}.}
\end{equation}
It is straightforward to verify that \probref{eq:mincutmqi} is equivalent to a shifted version of \probref{eq:mincutmqi_3} where the objectives differ by $\delta \vol(\Rbar)$. Finally, to get a solution of the original problem, we have to further decrease the objective by the constant $\delta \vol(R)$. 

To solve this \stMinCut problem, we then simply use an undirected MaxFlow solver. The input has $\mathcal{O}(\vol(R))$ edges and $|R|+2$ nodes. %

\begin{figure}
	\centering
\subfigure[Graph and seed set $R$]{\label{fig:mqi_maxflow_graph_a}\includegraphics[scale=0.59]{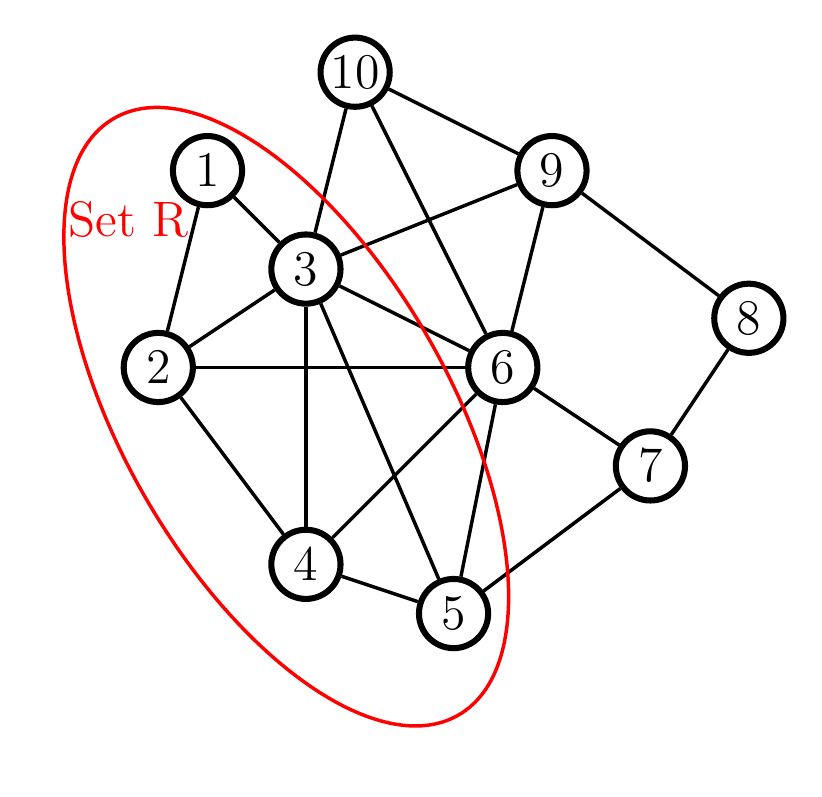}}%
\subfigure[Non-local \stMinCut problem]{\label{fig:mqi_maxflow_graph_c}\includegraphics[scale=0.59]{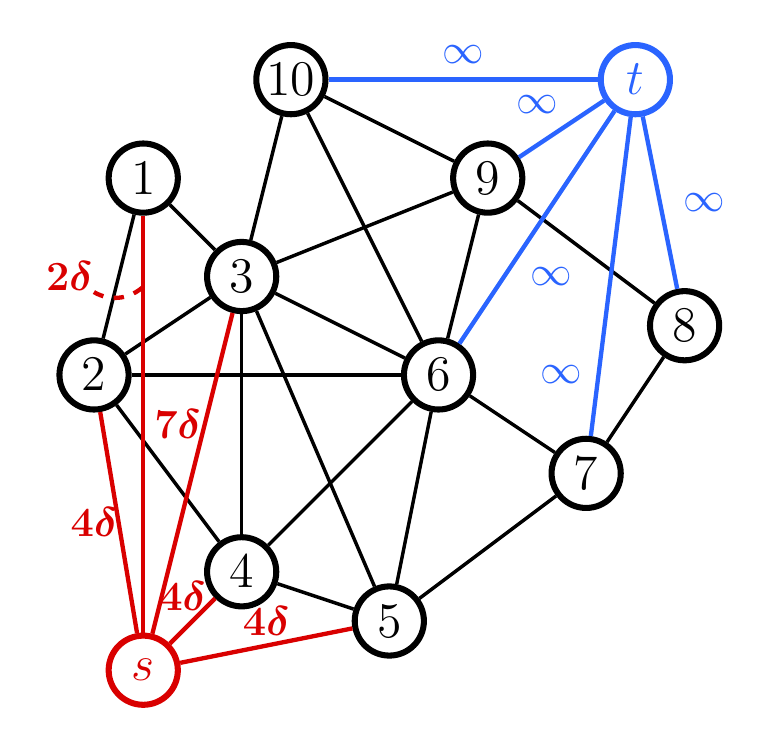}}%
\subfigure[Local \stMinCut problem]{\label{fig:mqi_maxflow_graph_d}\tronly{\hspace*{1em}}\includegraphics[scale=0.6]{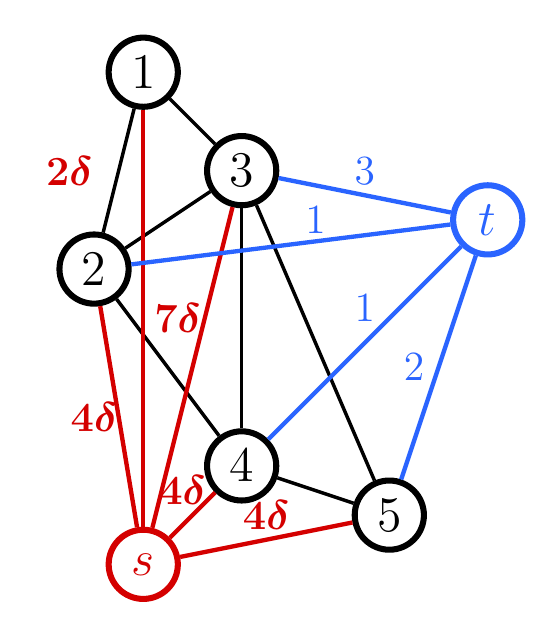}\tronly{\hspace*{1em}}}
\caption{%
Illustration of the augmented graph for solving the MQI subproblem. 
Subfigure \ref{fig:mqi_maxflow_graph_a} illustrates a small graph and a seed set  $R$ denoted by the red ellipse. This set includes nodes with ID $1$ to $5$. Subfigure \ref{fig:mqi_maxflow_graph_c} demonstrates the addition of a source node $s$ and sink node $t$ that involves the entire graph but solves the subproblem. 
Subfigure \ref{fig:mqi_maxflow_graph_d} illustrates the collapse of all nodes in $\Rbar$ into a single sink node $t$. Edges from $R$ to $\Rbar$ are maintained with the same weights but they are rewired to the sink node $t$. 
The final \stMinCut problem in Subfigure \ref{fig:mqi_maxflow_graph_d} can be solved via MaxFlow problem from the source to the sink. }
\label{fig:mqi_maxflow_graph}
\end{figure}

\subsection{Iteration complexity}

We now specialize  our general analysis in \Cref{sec:fractionalProg}, and we present an iteration complexity result for \Cref{algo:mqi}. %
First, we present \Cref{lem:monotonicMqi}, which will be used in the iteration complexity result in Theorem \ref{thm:itercomplexityMQI}.

An interesting property of MQI that is shown in \Cref{lem:monotonicMqi} is that the volume of $S_k$ monotonically decreases at each iteration, i.e., $\vol(S_{k+1})<\vol(S_k)$.
This result has important \emph{practical} implications since it shows that MQI is searching for subsets $S$ that have \emph{smaller} volume than the set $S_1$. 
Moreover, \Cref{lem:monotonicMqi} shows that the numerator of \probref{eq:mqiprob} decreases monotonically.

\begin{lemma}\label{lem:monotonicMqi}
Let $G$ be an undirected, connected graph with non-negative weights.  
If the MQI algorithm proceeds to iteration $k+1$ it satisfies both $\vol(S_{k+1}) < \vol(S_k)$ and $\cut(S_{k+1}) < \cut(S_k)$.
\end{lemma}
\begin{proof}
The proof for this claim is given in the proof of  \Cref{lem:monotonicFracProg}.
\end{proof}

\begin{theorem}[Iteration complexity of MQI]\label{thm:itercomplexityMQI}
Let $G$ be a connected, undirected graph with non-negative integer weights.  \Cref{algo:mqi} has at most $\cut(R)$ iterations before converging to a solution. %
\end{theorem}
\begin{proof}
This is just an explicit specialization of \Cref{thm:itercomplexityDinkel}.
\end{proof}

\begin{remark}[Time per iteration]
At each iteration a weighted MaxFlow problem is being solved. %
Therefore, the worst-case time of MQI will be its iteration complexity times the cost of computing a MaxFlow on a graph of size $\vol(R)$. Here $\vol(R)$ is an upper bound on the number of edges incident to vertices in $R$ because the weights are integers.
\end{remark}

\subsection{A faster version of the MQI algorithm}

The original MQI algorithm requires at most $\vol(R)$ iterations to converge to the optimal solution for graphs with integer weights. After at most that many iterations the algorithm returns the \textit{exact} output.
However, in the case that we are not interested in exact solutions we can improve the iteration complexity of MQI to at most $\mathcal{O}\left(\log \frac{1}{\eps}\right)$ where $\eps>0$ is an accuracy parameter. 
To achieve this we will use binary search for the variable $\delta$. It is true that for $(S^*,\delta^*)$ we have $\cut(S^*)/\vol(S^*) = \delta^*$.
Therefore, $\delta^*\in [0,1]$. We will use this interval as our search space for the binary search.
The modified algorithm is shown in \Cref{algo:fastermqi}. This algorithm is an instance of \Cref{algo:fasterfracprog}.
Note that the subproblem in Step $4$ in \Cref{algo:fastermqi} is the same as the subproblem in Step $3$ of the original \Cref{algo:mqi}.
The only part that changes is that we introduced binary search for $\delta$.

\begin{algorithm}
\caption{Fast MQI}\label{algo:fastermqi}
\begin{algorithmic}[1]
\STATE Initialize $k := 1$, $\delta_{\min} := 0$, $\delta_{\max} := \phi(R)$ and $\eps \in (0,1]$
\WHILE{$\delta_{\max} - \delta_{\min} > \eps \delta_{\min}$}
\STATE $\delta_k:= (\delta_{\max} + \delta_{\min})/2$ 
\STATE Solve 
$S_{k+1}:= \mathop{\text{argmin}}\limits_{S \subseteq R} {\cut(S) - \delta_k \vol(S)}$ via MaxFlow on Augmented Graph~\ref{proc:aug_graph_mqi}. 
\IF{$\vol(S_{k+1}) > 0$ \COMMENT{Then $\delta_k$ is above $\delta^*$}} 
\STATE $\delta_{\max} := \phi(S_{k+1})$, and set $S_{\max} = S_{k+1}$ \COMMENT{Note $\phi(S_{k+1}) \le \delta_k$}
\ELSE
\STATE $\delta_{\min} := \delta_k$ 
\ENDIF 
\STATE $k:=k+1$
\ENDWHILE
\STATE Return $\mathop{\text{argmin}}_{S \subseteq R} {\cut(S) - \delta_{\max} \vol(S)}$ or $S_{\max}$ based on minimum conductance.
\end{algorithmic}
\end{algorithm}

Putting the iteration complexity of Fast MQI together with its per iteration computational complexity we get the following theorem. 

\begin{theorem}[Iteration complexity of the Fast MQI \Cref{algo:fastermqi}]\label{thm:complexityfastmqi}
Let $G$ be an undirected, connected, graph with non-negative weights. Let $R$ be a subset of vertices with $\vol(R) \le \vol(\Rbar)$. The sequence $\delta_k$ of \Cref{algo:fastermqi} converges to an approximate solution $|\delta^* - \delta_k|/\delta^* \le \eps$ in $\mathcal{O}(\log 1/\eps)$ iterations, where $\delta^* = \cond(S^*)$ and $S^*$ is an optimal solution to problem \eqref{eq:mqiprob}. Moreover, if $G$ has non-negative integer weights then the algorithm will return the exact minimizer when $\eps < \tfrac{1}{\vol(R)^2}$.
\end{theorem}
\begin{proof}
The iteration complexity of MQI is an immediate consequence of Theorem~\ref{thm:complexityfastdinkel}. %
The exact solution piece is a consequence of the smallest difference between values of conductance among subsets of $R$ for integer weighted graphs. Let $S_1$ and $S_2$ be arbitrary subsets of vertices in $R$ with $\phi(S_1) > \phi(S_2)$. Then
\[\phi(S_1)-\phi(S_2)=\frac{\cut(S_1)\vol(S_2)-\cut(S_2)\vol(S_1)}{\vol(S_1)\vol(S_2)}\geq (\vol(R))^{-2}.\]
The last piece occurs because if $\cut(S_1)\vol(S_2)-\cut(S_2)\vol(S_1)$ is an integer, the smallest possible difference is 1. At termination Fast MQI satisfies $\delta_{\max} - \delta_{\min} \le \eps \delta_{\min}$. By the above difference bound, the next objective function value that is larger than $\delta^*$ is at least $\delta^* + \frac{1}{\vol(R)^2}$. Therefore, setting $\eps<\frac{1}{\vol(R)^2}$, we get that $\delta_{\max} < \delta^* + \frac{1}{\vol(R)^2}$.
\end{proof}

\begin{remark}[Time per iteration]
Each iteration involves a weighted MaxFlow problem on a graph with volume equal to $O(\vol(R))$. 
\end{remark}

\section{The FlowImprove Problem and Algorithm}
\label{chap:flowimprove}

In this section, we will describe the FlowImprove method, due to~\citet{AL08_SODA}. 
This cluster improvement method was designed to address the issue that the MQI algorithm will always return an output set that is strictly a subset of the reference set $R$.
The FlowImprove method also takes as input a graph $G=(V,E)$ and a reference set $R\subset V$, with $\vol(R) \le \vol(G)/2$, and it also returns as output an ``improved'' cluster.
Here, the output is ``improved'' in the sense that it is a set with conductance at least as good as $R$ that is also highly correlated with $R$. 

To state the FlowImprove method, consider the following variant of conductance:
\begin{equation} 
\label{eqn:conductance_FlowImprove}
\cond_R(S) = \begin{cases} \dfrac{\cut(S,\Sbar)}{\rvol(S; R, \theta)} & \text{when the denominator is positive } \\ 
\infty & \text{otherwise} \end{cases}
\end{equation} 
where $\theta = \vol(R)/\vol(\Rbar)$, and where the value is $\infty$ if the denominator is negative. 
This particular value of $\theta$ arises as the smallest value such that the $\rvol(S; R, \theta) = \vol(S \cap R) - \theta \vol(S \cap \Rbar)$ denominator is exactly zero when $S = V$ and hence will rule out trivial solutions. (This idea is equivalent to picking $\theta$ so that the total weight of edges connected to source is equal to the total weight of edges connected to the sink in the coming flow problem~\cite{AL08_SODA}.) Note that this setup is also equivalent to the statement in \Cref{sec:improvement-objectives} where the denominator constraint is adjusted to be a positive infinity value. 

For any set $S$ with $\vol(S) \le \vol(\Sbar)$, since $\rvol(S; R, \theta) = \vol(S \cap R) - \theta \vol(S \cap \Rbar)$, it holds that $\cond_R(S) \ge \cond(S)$.
Thus, this modified conductance score $\cond_R(\cdot)$ provides an upper-bound on the true conductance score $\cond(\cdot)$ for sets that are not too big; but this objective provides a bias toward $R$, in the sense that the denominator penalizes sets $S$ that are outside of the reference set $R$. 

Consequently, the FlowImprove problem is:
\begin{equation}
\label{eq:flowImpProb}
 \boxed{\MINone{S}{\cond_R(S)}{S\subset V.}}
\end{equation}
This FlowImprove problem is related to the fractional programming \probref{eq:fracprob} by setting $g(S):=\vol(S \cap R) - \theta \vol(S \cap \Rbar)$ and $Q=V$.
\citet{AL08_SODA} describe an algorithm to solve the FlowImprove problem, which is equivalent to what we present  as \Cref{algo:flowImprove}.
It is easy to see that this algorithm is a special case of \Cref{algo:fractionalprog} for general fractional programming.

\begin{algorithm}
\caption{FlowImprove~\citep{AL08_SODA}}\label{algo:flowImprove}
\begin{algorithmic}[1]
\STATE Initialize $k = 1$, $S_1:=R$ and $\delta_1=\cond_R(S_1)$.
\WHILE{we have not exited via else clause}
\STATE Solve  
$\ARGMINzero{S_{k+1}:=}{S}{\cut(S) - \delta_k \left(\vol(S\cap R) - \theta \vol(S\cap \Rbar)\right)}$
\IF{$\cond_R(S_{k+1}) < \delta_{k}$}
	\STATE $\delta_{k+1}:=\cond_R(S_{k+1})$
\ELSE
	\STATE $\delta_k$ is optimal, return previous solution $S_k$.
\ENDIF
\STATE $k:=k+1$
\ENDWHILE
\end{algorithmic}
\end{algorithm}

The following theorem implies that FlowImprove monotonically decreases the objective function in \probref{eq:flowImpProb} at each iteration.
It was first shown by~\citet{AL08_SODA}, but it is a corollary of \Cref{thm:monotonicfracprog}.
Note that $\delta_k$ is equal to the objective function of \probref{eq:flowImpProb} evaluated at $S_k$.

\begin{theorem}[Convergence of FlowImprove]
Let $G$ be an undirected, connected graph with non-negative weights.
 Let $R$ be a subset of vertices with $\vol(R) \le \vol(\Rbar)$.  The sequence $\delta_k$ monotonically decreases at each iteration of FlowImprove (\Cref{algo:flowImprove}). 
\end{theorem}

\subsection{The FlowImprove subproblem}
\label{sec:flowimprove_subproblem}

In this subsection, we will discuss how to solve efficiently the subproblem at Step $3$ of FlowImprove. 
We will follow similar steps as we did for MQI in \Cref{sec:mqi_subproblem}. That is, we convert the \stMinCut-like problem into a true \stMinCut problem on an augmented graph, and then we use MaxFlow to find the set minimizing the objective. 
As a summary and overview, see the  Augmented Graph~\ref{proc:aug_graph_flowimprove} procedure and an example of this new modified graph in \Cref{fig:flowimp_maxflow_graph}. 
(Observe that here we do \emph{not} have a fourth step where we combine multiple edges, as we did in Augmented Graph~\ref{proc:aug_graph_mqi} and \Cref{fig:mqi_maxflow_graph_d}---thus, the FlowImprove \Cref{algo:flowImprove} will \emph{not} be strongly local.)

\begin{subroutine}[t]
\caption{for the subproblem at Step $3$ of FlowImprove \Cref{algo:flowImprove}}\label{proc:aug_graph_flowimprove}
\begin{algorithmic}[1]
\STATE Add to the set of nodes $V$ a source node $s$ and a sink node $t$. 
\STATE Add to the set of edges $E$ an edge from the source node $s$ to every node in the seed set of nodes $R$ with weight the degree of that node times $\delta$.
\STATE Add to the set of edges $E$ an edge from the sink node $t$ to every node in the set of nodes $\Rbar$ with weight the degree of that node times $\delta \theta$.
\end{algorithmic}
\end{subroutine}

\begin{figure}
	\centering
\subfigure[Graph and seed set $R$ ]{\label{fig:flowimp_maxflow_graph_b}\includegraphics[width=0.5\linewidth]{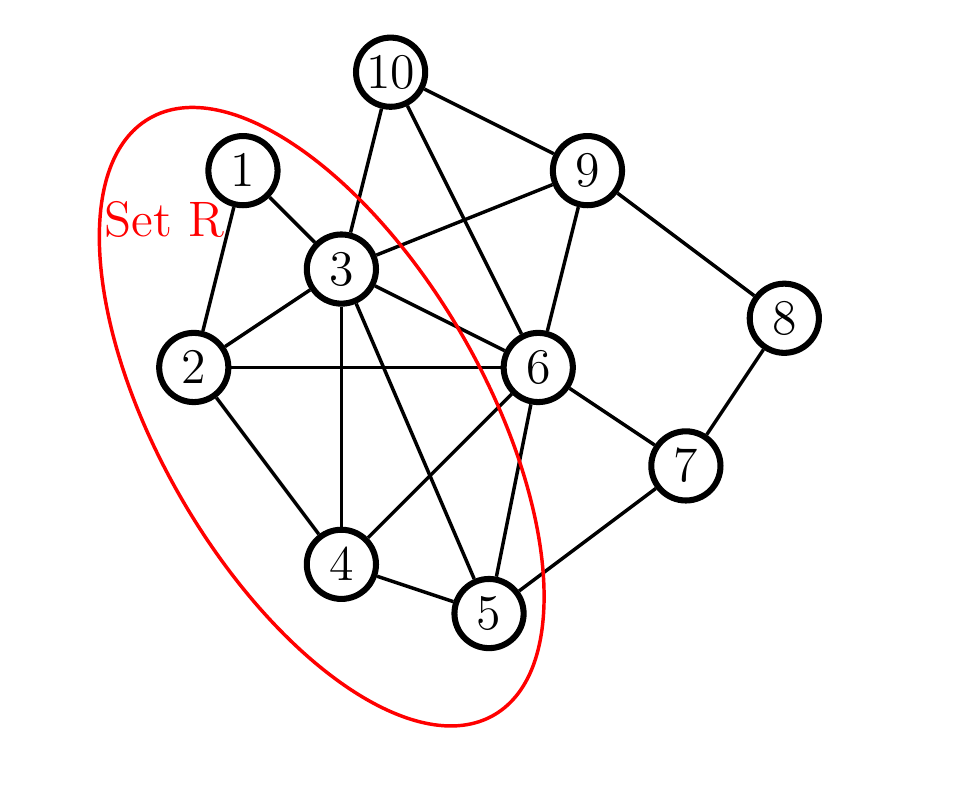}}%
\subfigure[MinCut graph for FlowImprove subproblem]{\label{fig:flowimp_maxflow_graph_d}\includegraphics[width=0.5\linewidth]{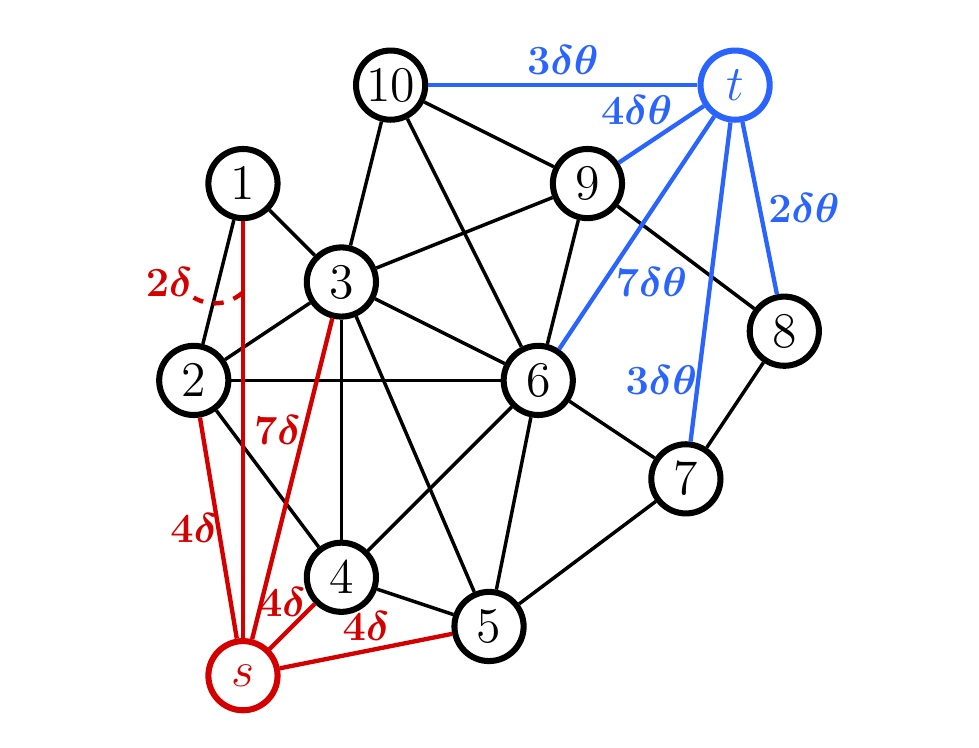}}%
\caption{%
Illustration of the augmented graph for solving the FlowImprove subproblem. 
Subfigure \ref{fig:flowimp_maxflow_graph_b}  illustrates the same graph and seed set from \Cref{fig:mqi_maxflow_graph}.  
Subfigure \ref{fig:flowimp_maxflow_graph_d} demonstrates the addition of a source node $s$ and sink node $t$, along with corresponding edges from $s$ to nodes in $R$ and node $t$ to every node in $\Rbar$.
The \stMinCut problem in Subfigure \ref{fig:flowimp_maxflow_graph_d} can be solved to identify a set via a MaxFlow problem from the source to the sink. 
}
\label{fig:flowimp_maxflow_graph}
\end{figure}

Turning back towards the derivation of this formulation, the MinCut sub-problem at Step $3$ of FlowImprove problem is equivalent to
\begin{equation}
\label{eq:mincutflowimprove_2}
\MINone{\vx}{\|\mB\vx\|_{\mC,1} - \delta\vx^T \hat{\vd}_R + \delta \theta \vx^T \hat{\vd}_{\Rbar}}{\vx \in \{0,1\}^{n},}
\end{equation}
where $\hat{\vd}_R$ is a $n$-dimensional vector that is equal to $\vd$ for components with index in $R$ and zero elsewhere. Similarly for $\hat{\vd}_{\Rbar}$. Consequently, $\vd = \hat{\vd}_R + \hat{\vd}_{\Rbar}$ where the two pieces have disjoint support.

As with the previous case, we shift this and then add sources and sinks. First, the largest possible negative value is at least $ \delta \ones^T\hat{\vd}_{R}$. Adding this yields 
\begin{equation}
\label{eq:mincutflowimprove_3}
\MINone{\vx}{\|\mB\vx\|_{\mC,1} + \delta(\ones - \vx)^T \hat{\vd}_R + \delta \theta \vx^T \hat{\vd}_{\Rbar}}{\vx \in \{0,1\}^{n}.}
\end{equation}
Again, we have penalty terms associated with $S$ (given by non-zero entries of $\vx$) and $\Sbar$ (given by non-zero entries of $\ones - \vx$). For these, we introduce a source and sink. The source connects to penalties associated with $\Sbar$ and the sink connects to penalties associated with $S$. Note that these penalties partition into two groups, associated with $R$ and $\Rbar$. Consequently, we add a source node $s$ and connect it to all nodes in $R$ with weight $\delta \hat{\vd}_R$, and we also add a sink node $t$ and connect it to all nodes in $\Rbar$ with weight $\delta \theta \hat{\vd}_{\Rbar}$. 

The resulting \stMinCut problem is associated with the 
incidence matrix and the diagonal edge-weight matrix of a modified problem as follows
\[ \tilde{\mB} := 
\begin{blockarray}{ccc}
s & V & t  \\
\begin{block}{[ccc]}
\ones & -\mI_R         & 0  \\
0 & \mB  &     		 0 \\
0 & \mI_{\Rbar}         & -\ones \\
\end{block}
\end{blockarray} 
\qquad 
\tilde{\mC} :=
\begin{blockarray}{ccc}
  &   &  \\
\begin{block}{[ccc]}
\delta\mD_{R} & 0                  & 0  \\
0 & \mC   & 0 \\
0 & 0                   & \delta \theta \mD_{\Rbar} \\
\end{block}
\end{blockarray}   .
\]
Here, $\mD_{R}$ and $\mD_{\Rbar}$ are diagonal submatrices of $\mD$ corresponding to nodes in $R$ and $\Rbar$, respectively. Also, $I_R$ and $I_{\Rbar}$ are matrices where each row contains an indicator vector for a node in $R$ and $\Rbar$, respectively.  These matrices give $\mI_R \vx = \vx_R$ and $\mI_{\Rbar} \vx = \vx_{\Rbar}$. They are ordered in the same way as was $D_R$ and $D_{\Rbar}$. 
Let 
\[ \tilde{\vx} := \begin{bmatrix} \vx_s \\
			   		  \vx \\
					  \vx_t
\end{bmatrix}, \text{ so that } \tilde{\vx}_1 = \vx_s \text{ and } \tilde{\vx}_{|R|+2} = \vx_t
\]
then the \stMinCut problem with respect to the modified graph is 
\begin{equation}
\label{eq:mincutflowimprove}
\MINone{\tilde{\vx}}{\|\tilde{\mB}\tilde{\vx}\|_{\tilde{\mC},1} = \|\mB \vx\|_{\mC,1} + \delta (\ones - \vx_R)^T \vd_R + \delta \theta \vx_{\Rbar} \vd_{\Rbar} }{\tilde{\vx}_1 = 1, \tilde{\vx}_{n+2}=0, \tilde{\vx}_i \in \{0,1\} \ \forall i=2\dots n+1.}
\end{equation}
Again, note that this objective corresponds to a constant shift with respect to \probref{eq:mincutflowimprove_2}.
This problem can be solved via MaxFlow to give a set solution. 

\subsection{Iteration complexity}

In \Cref{lem:monotonicFlowImp} we show that when using FlowImprove the denominator of \probref{eq:flowImpProb}, i.e., $\vol(S\cap R) - \theta \vol(S\cap \Rbar)$, decreases monotonically at each iteration.
Moreover, the numerator of \probref{eq:flowImpProb} decreases monotonically as well.

\begin{lemma}\label{lem:monotonicFlowImp}
If the FlowImprove algorithm proceeds to iteration $k+1$ it satisfies $\vol(S_{k+1}\cap R) - \vol(S_k\cap R) < \theta \left(\vol(S_{k+1}\cap \Rbar) - \vol(S_k\cap \Rbar)\right)$ and $\cut(S_{k+1}) < \cut(S_k)$.
\end{lemma}
\begin{proof}
This result is a specialization of \Cref{lem:monotonicFracProg} and the proof is the same.
\end{proof}

\begin{theorem}[Iteration complexity of the FlowImprove \Cref{algo:flowImprove}]\label{thm:itercomplexityFlowImp}
Let $G$ be a connected, undirected graph with non-negative integer weights.  Then \Cref{algo:flowImprove} needs at most $\cut(R)$ iterations to converge to a solution. %
\end{theorem}
\begin{proof}
This is just an explicit specialization of \Cref{thm:itercomplexityDinkel}.
\end{proof}

\begin{remark}[Time per iteration] \label{rem:flowimprove-time}
At each iteration a weighted MaxFlow problem is being solved, see \Cref{sec:flowimprove_subproblem}. The MaxFlow problem size is proportional to the whole graph.%
\end{remark}

\subsection{A faster version of the FlowImprove algorithm}

The original FlowImprove algorithm requires at most $\cut(R) \le \vol(R)$ iterations to converge to the optimal solution. After at most that many iterations the algorithm returns the \textit{exact} output.
However, in the case that we are not interested in exact solutions we can improve the iteration complexity of FlowImprove to at most $\mathcal{O}\left(\log \frac{1}{\eps}\right)$ where $\eps>0$ is an accuracy parameter. 
To achieve this we will use binary search for the variable $\delta$. It is true that $\cond_R(R) \in [0,1]$.
Therefore, $\delta^*\in [0,1]$. We will use this interval as our search space for the binary search.
The modified algorithm is shown in \Cref{algo:fasterflowimprove}.
Note that the subproblem in Step $4$ in \Cref{algo:fasterflowimprove} is the same as the subproblem in Step $3$ of the original \Cref{algo:flowImprove}.
The only part that changes is that we introduced binary search for $\delta$.

\begin{algorithm}
\caption{Fast FlowImprove}\label{algo:fasterflowimprove}
\begin{algorithmic}[1]
\STATE Initialize $k := 1$, $\delta_{\min} := 0$, $\delta_{\max} := 1$ and $\eps \in (0,1]$
\WHILE{$\delta_{\max} - \delta_{\min} > \eps \delta_{\min}$}
\STATE $\delta_k:= (\delta_{\max} + \delta_{\min})/2$ 
\STATE Solve $S_{k+1}:= \mathop{\text{argmin}}_S {\cut(S) - \delta_k \left(\vol(S\cap R) - \theta \vol(S\cap \Rbar)\right)}$ via MaxFlow

\IF{$\vol(S_{k+1}\cap R) > \theta \vol(S_{k+1}\cap \Rbar)$ \COMMENT{Then $\delta_k$ is above $\delta^*$}}
\STATE $\delta_{\max} := \phi_R(S_{k+1})$ and set $S_{\max} := S_{k+1}$ \COMMENT{Note $\phi_R(S_{k+1}) \le \delta_k$}
\ELSE
\STATE $\delta_{\min} := \delta_k$ 
\ENDIF 
\STATE $k:=k+1$
\ENDWHILE
\STATE Return $\mathop{\text{argmin}}_S {\cut(S) \!-\!\delta_{\max} (\vol(S\!\cap\!R) \!-\! \theta \vol(S \!\cap\! \Rbar)) }$ or $S_{\max}$ based on min $\phi_R$.
\end{algorithmic}
\end{algorithm}

Putting the iteration complexity of Fast FlowImprove together with its per iteration computational complexity we get the following theorem.
\begin{theorem}[Iteration complexity of the Fast FlowImprove \Cref{algo:fasterflowimprove}]\label{thm:complexityfastflowimp}
Let $G$ be an undirected, connected graph with non-negative weights. Let $R$ be a subset of $V$ with $\vol(R) \le \vol(\Rbar)$. The sequence $\delta_k$ of \Cref{algo:fasterflowimprove} converges to an approximate solution $|\delta^* - \delta_k|\delta^* \le \eps$ in $\mathcal{O}(\log 1/\eps)$ iterations, where $\delta^* = \cond_R(S^*)$ and $S^*$ is an optimal solution to problem \eqref{eq:flowImpProb}. Moreover, if $G$ has non-negative integer weights, then the algorithm will return the exact minimizer when $\eps < \frac{1}{\vol(R)^2 \vol(\Rbar)}$. 
\end{theorem}
\begin{proof}
Iteration complexity of FlowImprove is an immediate consequence of \Cref{thm:complexityfastdinkel}. %
The exact solution piece is a consequence of the smallest difference between values of relative conductance for integer weighted graphs. Let $S_1$ and $S_2$ be arbitrary sets of vertices in the graph with $\phi_{R}(S_1) > \phi_{R}(S_2)$. Let
$k_1=\text{cut}(S_1)\vol(S_2\cap R)-\text{cut}(S_2)\vol(S_1\cap R)$ and $k_2=\text{cut}(S_1)\vol(S_2\cap \Rbar)-\text{cut}(S_2)\vol(S_1\cap \Rbar)$. Both are integers. Then
\[
\frac{\text{cut}(S_1)}{\rvol(S_1;R,\theta)}
-\frac{\text{cut}(S_2)}{\rvol(S_2;R,\theta)}
=\frac{k_1\vol(\Rbar)-k_2\vol(R)}{\vol(\Rbar)\rvol(S_1;R,\theta)\rvol(S_2;R,\theta)}\geq \frac{1}{\vol(R)^2 \vol(\Rbar)}. \]
The last piece occurs because $k_1$ and $k_2$ are integers, and thus the smallest positive value of $k_1\vol(\Rbar)-k_2\vol(R)$ is 1. The rest of the argument on the exact solution is the same as the proof of \Cref{thm:complexityfastmqi}. 
\end{proof}

The subproblem is the same and so the cost per iteration is the same as discussed in \Cref{rem:flowimprove-time}.

\subsection{Non-locality in FlowImprove} 

The runtime bounds for FlowImprove assume that we may need to solve a MaxFlow problem with size proportional to the entire graph. We now show that this is essentially tight and that the solution of a FlowImprove problem, in general, is not strongly local.   Indeed, the following example shows that FlowImprove will return one fourth of the graph even when started with a set $R$ that is a singleton. 

\noindent\begin{minipage}{0.55\linewidth}%
\begin{lemma}\label{thm:cycle-graph}
Consider a cycle graph with some extra edges connecting neighbors or neighbors in some parts (illustrated on the right) with $4N+8$ nodes in 4 major regions. Each set $A$ and $B$ has $N$ nodes of degree $4$ corresponding to a contiguous piece of the cycle graph with neighbors and neighbors of neighbors connected. Each set $C$ and $D$ has $N$ degree 2 nodes. This introduces two extra nodes, of degree 3, between each pair of adjacent degree $2$ and degree $4$ regions. Consider using any node of degree $4$ as the seed node to FlowImprove algorithm. Then, at optimality, FlowImprove will return a set with $N+4$ nodes that is a continuous degree 4 region plus the four adjacent degree 3 nodes.%
\end{lemma}%
\end{minipage}\begin{minipage}{0.45\linewidth}
	\includegraphics[width=0.99\linewidth]{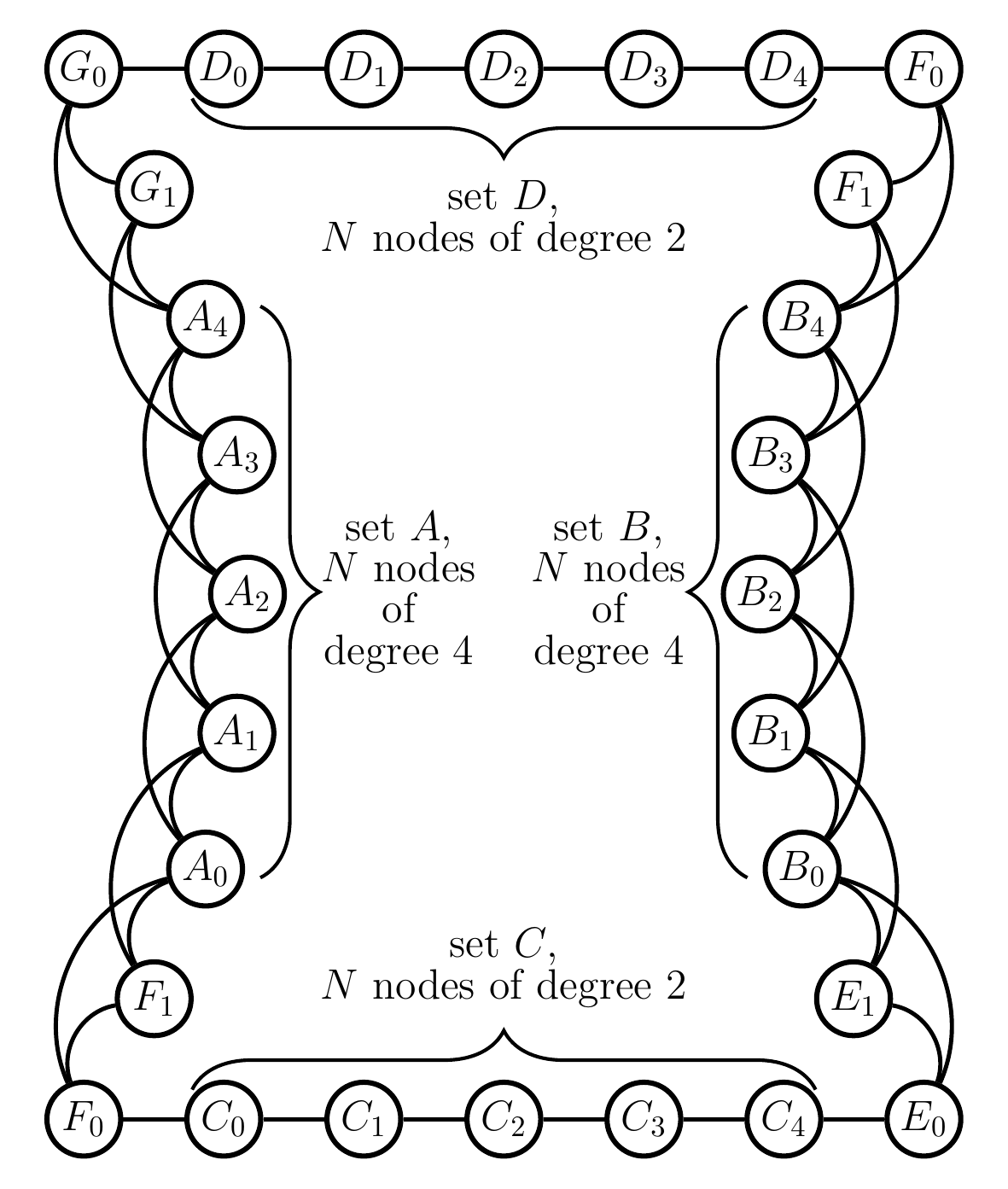}
\end{minipage}%
\begin{proof}
	Without loss of generality, suppose we seed on a node from set $A$. According to \Cref{lem:monotonicFlowImp}, when Dinkelbach's algorithm for FlowImprove proceeds from iteration to iteration, it must return a set with a strictly smaller cut value or the seed set $R$ was optimal. This means FlowImprove will only return one of the following sets. (Due to symmetry, there may be equivalent sets that we don't list.) 
	\begin{enumerate}
		\item The seed node with cut 4. 
		\item A continuous subset of the $A$ region, $G_0$, $G_1$, and a continuous subset of the set $D$, with cut 3. 
		\item All of the $A$ region, two adjacent degree 3 nodes (without loss of generality, $G_0$ and $G_1$) on one end and one adjacency degree 3 node on the other edge ($F_1$), with cut 3.
		\item All of the $A$ region and all adjacency degree 3 nodes ($G_0, G_1, F_0, F_1$), with cut 2. 
		\item All of the $A$ region and all adjacency degree 3 nodes ($G_0, G_1, F_0, F_1$ and additional nodes from sets $C$ and $D$), with cut 2. 
	\end{enumerate}
	The goal is to show that case (4) is optimal, i.e., has the smallest objective value. Obviously, case (5) cannot be optimal since it has the same cut value as case (4) but smaller relative volume. Similarly, case (3) has the same cut value as case (2) but smaller relative volume. So case (3) won't be optimal either. So we only need to compare $\phi_R(S_1)$, $\phi_R(S_2)$ and $\phi_R(S_4)$. Observe that in this setting, $\theta=\frac{\text{vol}(R)}{\text{vol}(\Rbar)}=\frac{4}{(2N-1)4+2N\cdot2+8\cdot3}=\frac{1}{3N+5}$, so we can compute that
	\[\phi_R(S_4)=\frac{2}{4-\theta (4(N-1)+3\cdot4)}=\frac{3N+5}{4N+6}<1=\phi_R(S_1)  .  \]
	On the other hand, suppose in case (2), there are $1\leq k<N$ nodes from $A$ and $m\geq 0$ nodes from $D$, then we can write
	\[\phi_R(S_2)=\frac{3}{4-\theta (4(k-1)+3\cdot2+2m)}\geq \frac{3}{4-6\theta}=\frac{9N+15}{12N+14}>\phi_R(S_4)  .  \]
	So case (4) is optimal.
\end{proof}

\subsection{Relationship with PageRank}
\label{sec:fi-pagerank}
The FlowImprove \subprobref{eq:mincutflowimprove} is closely related to the PageRank problem if the 1-norm objective is translated into a $2$-norm objective and we relax to real-valued vectors (and make a small perturbation to the resulting systems). This was originally observed, in slightly different ways, in our previous work~\cite{GM14_ICML,Gleich-2015-robustifying}. For the same matrix $\tilde{\mB}$, consider the problem 
\begin{equation}
\MINone{\tilde{\vx}}{\|\tilde{\mB}\tilde{\vx}\|_{\tilde{\mC},2}^2}{x_s = \tilde{\vx}_1 = 0, x_t = \tilde{\vx}_{n+2}=1, \tilde{\vx}_i \in \{0,1\} \ \forall i=2\dots n+1.}
\end{equation}
Note that this problem, with the binary constraints, is exactly equivalent to the original problem. However, if we relax the binary constraints to real-valued vectors and substitute in $x_1=1$ and $x_{n+2} = 0$, then this is a strongly-convex quadratic objective, which can be solved as the following linear system: 
\begin{equation}
(\mB^T \mC \mB + \delta \diag(\tilde{d}_R) + \theta \delta \diag (\tilde{d}_{\Rbar})) \vx = \delta/2 \tilde{d}_R  .
\end{equation}
Here, $\mB^T \mC \mB = \mL = \mD - \mA$ is the Laplacian of the original graph. Also, if we had $\theta = 1$ (or simply assume this is true), then $ \delta \diag(\tilde{d}_R) + \theta \delta \diag (\tilde{d}_{\Rbar}) = \delta \mD$. This yields the linear system  
\[ (\mL + \delta \mD) \vx = \delta/2  \quad \Leftrightarrow \quad (\mI - \tfrac{1}{1+\delta} \mA \mD^{-1}) \mD \vx = \delta/(2+2\delta) \tilde{d}_R. \]
The second system is equivalent to a rescaled PageRank problem for an undirected graph $(\mI - \alpha \mA \mD^{-1}) \vy = \gamma \vv$ where $\vy = \mD \vx$. This form, or a scaled version, is widely used in practice~\cite{G15}. 
\section{The LocalFlowImprove (and SimpleLocal) Problem and Algorithm}
\label{chap:localflowimprove}

In this section, we will describe the LocalFlowImprove method, due to~\citet{OZ14}, and the related SimpleLocal, due to~\citet{veldticml2016}.
This cluster improvement method was designed to address the issue that FlowImprove is weakly (and not strongly) local, i.e., that the FlowImprove method has a running time that depends on the size of the entire input graph and not just on the size of the reference set $R$.  %
The setup is the same: LocalFlowImprove method takes as input a graph $G=(V,E)$ and a reference set $R\subset V$, with $\vol(R) \le \vol(G)/2$, and it returns as output an ``improved''~cluster.

To understand the LocalFlowImprove method, consider the following variant of conductance:
\begin{equation}
\label{eqn:conductance_LocalFlowImprove}
\cond_{R,\sigma}(S) = \begin{cases} \dfrac{\cut(S,\Sbar)}{\vol(S \cap R) - \sigma \vol(S \cap \Rbar)} & \text{when the denominator is positive } \\ 
\infty & \text{otherwise} \end{cases}
\end{equation}
where $\sigma \in [\vol(R)/\vol(\Rbar),\infty)$. This is identical to FlowImprove~\eqref{eqn:conductance_FlowImprove}, but we change 
$\theta$ into $\sigma$ 
and allow it to vary. Given this, the basic LocalFlowImprove problem is:
\begin{equation}\label{eq:localflowImpProb}
\boxed{\MINone{S}{\cond_{R,\sigma}(S)}{S\subset V.}}
\end{equation}

\newcommand{\asidedelta}{\aside{aside:sigma-vs-delta}{For the theory in this section, we parameterize the LocalFlowImprove objective with $\sigma$ instead of $\theta + \delta$ as in \Cref{sec:improvement-objectives} and \Cref{sec:experiments}. This choice reduces the number of constants in the statement of theorems. The previous choice of $\delta$ is designed to highlight the FlowImprove to MQI spectrum.}}
\tronly{\asidedelta}
On the surface, it is straightforward to adapt between FlowImprove and LocalFlowImprove. Simply ``repeating'' the entire previous section with $\sigma$ instead of $\theta$ will result in correct algorithms. For example, the original algorithm proposed for LocalFlowImprove by \citet{OZ14} is presented in an equivalent fashion \Cref{algo:localflowimprove}, which is simply an instance of the bisection-based fractional programming \Cref{algo:fasterfracprog}. 

\siamonly{\asidedelta}
The key difference between FlowImprove and LocalFlowImprove is that by setting $\sigma$ larger than $\vol(R)/\vol(\Rbar)$ we will be able to show that the running time is independent of the size of the input graph. Recall that we have already shown the \emph{output set} has a graph-size independent bound in \Cref{lem:lfi-size}.

This strongly-local aspect of LocalFlowImprove manifests in the subproblem solve step. Put another way, we need to crack open the black-box flow techniques in order to make them run in a way that scales with the size of the output rather than the size of the input.  As a simple example of how we'll need to look inside the black box, note that when $\sigma = \infty$, then LocalFlowImprove corresponds to MQI, as discussed in \Cref{sec:improvement-objectives}, which has an extremely simple strongly local algorithm. We want algorithms that will be able to take advantage of this property without needing to be told this will happen. Consequently, in this section, we are going to discuss the subproblem solver extensively.

In particular, we will cover how to adapt a sequence of standard MaxFlow solves to be strongly local, as in the SimpleLocal method of \citet{veldticml2016} (\Cref{subsubsec:localgraph}), as well as improvements that arise from using blocking flows and adapting Dinic's algorithm (\Cref{sec:blockingflow,subsubsec:simplelocal}). We will also cover differences with solvers with different types of theoretical tradeoffs that were discussed in the original \citet{OZ14} paper (\Cref{subsec:approxDinic}).

Note that the SimpleLocal algorithm of~\citet{veldticml2016} did not use binary search on $\delta$ as in \Cref{algo:localflowimprove} (and nor do our implementations), instead it used the original Dinkelbach's algorithm. As we have pointed out a few times,  binary search is not as useful as it may seem for these problems, as a few iterations of Dinkelbach's method is often sufficient on real-world data. The point here is that the tradeoff between bisection and the greedy Dinkelbach's method is independent of the \emph{subproblem} solves that is the heart of what differentiates LocalFlowImprove from~FlowImprove. Finally, note that \Cref{algo:localflowimprove} is also a special instance of \Cref{algo:fasterfracprog}.

\begin{algorithm}
	\caption{LocalFlowImprove~\citep{OZ14}}\label{algo:localflowimprove}
	\begin{algorithmic}[1]
		\STATE Initialize $k := 1$, $\delta_{\min} := 0$, $\delta_{\max} := 1$,  $\sigma \in \left[\frac{\vol(R)}{\vol(\Rbar)},\infty\right)$, and $\eps \in (0, 1]$
		\WHILE{$\delta_{\max} - \delta_{\min} > \eps \delta_{\min}$}
		\STATE $\delta_k:= (\delta_{\max} + \delta_{\min})/2$ 
		\STATE Solve $S_{k+1}:= \mathop{\text{argmin}}_S {\cut(S) - \delta_k \left(\vol(S\cap R) - \sigma \vol(S\cap \Rbar)\right)}$ via MaxFlow
		\IF{$\vol(S_{k+1}\cap R) > \sigma \vol(S_{k+1}\cap \Rbar)$ \COMMENT{Then $\delta_k \ge \delta^*$}}
		\STATE $\delta_{\max} := \phi_{R,\sigma}(S_{k+1})$ and set $S_{\max} := S_{k+1}$ \COMMENT{Note $\phi_{R,\sigma}(S_{k+1}) \le \delta_k$.} 
		\ELSE
		\STATE $\delta_{\min} := \delta_k$ 
		\ENDIF 
		\STATE $k:=k+1$
		\ENDWHILE
		\STATE Return $\mathop{\text{argmin}}_S \cut(S)\!-\!\delta_{\max} (\vol(S \!\cap\! R)\!-\! \sigma \! \vol( S \!\cap\! \Rbar)) $ or $S_{\max}$ based on min $\phi_{R,\sigma}$
	\end{algorithmic}
\end{algorithm}

The iteration complexity of \Cref{algo:localflowimprove} is now just a standard application of the fractional programming theory. 

\begin{theorem}[Iteration complexity of LocalFlowImprove]\label{cor:localflowimprove}
	Let $G$ be an undirected, connected graph with non-negative weights. Let $R$ be a subset of nodes with $\vol(R) \le \vol(\Rbar)$. In \Cref{algo:localflowimprove}, the sequence $\delta_k$ converges to an approximate optimal value $|\delta^* - \delta_k|/\delta^* \le \eps$ in $\mathcal{O}(\log 1/\eps)$ iterations, where $\delta^* = \cond_{R,\sigma}(S^*)$ and $S^*$ is an optimal solution to problem \eqref{eq:localflowImpProb}. Moreover, if $G$ has non-negative integer weights and $\sigma = (\eta + \vol(R))/\vol(\Rbar)$ for an integer value of $\eta$, then the algorithm will return the exact minimizer when $\eps < \frac{1}{\vol(R)^2 \vol(\Rbar)}$.
	\end{theorem}
\begin{proof}
	The first part is an immediate consequence of Theorem \ref{thm:complexityfastdinkel} with $\delta_{\max}=1$. 
The exact solution piece is a consequence of the smallest difference between values of relative conductance for integer weights. Let $S_1$ and $S_2$ be arbitrary sets of vertices in the graph with $\phi_{R,\sigma}(S_1) > \phi_{R,\sigma}(S_2)$. Let
$k_1=\text{cut}(S_1)\vol(S_2\cap R)-\text{cut}(S_2)\vol(S_1\cap R)$ and $k_2=\text{cut}(S_1)\vol(S_2\cap \Rbar)-\text{cut}(S_2)\vol(S_1\cap \Rbar)$, which are both integers. Then
\[
\frac{\text{cut}(S_1)}{\rvol(S_1;R,\sigma)}
-\frac{\text{cut}(S_2)}{\rvol(S_2;R,\sigma)}
=\frac{k_1\vol(\Rbar)-k_2(\eta+\vol(R))}{\vol(\Rbar)\rvol(S_1;R,\sigma)\rvol(S_2;R,\sigma)}\geq \frac{1}{\vol(R)^2 \vol(\Rbar)}. \]
The inequality follows because the integrality of $k_1$ and $k_2$ ensures tha the smallest positive value of $k_1\vol(\Rbar)-k_2(\eta+\vol(R))$ is 1. The rest of the argument on the exact solution is the same as the proof of \Cref{thm:complexityfastmqi}.
\end{proof}

\begin{subroutine}
	\caption{for the subproblem at Step $4$ of LocalFlowImprove \Cref{algo:localflowimprove}. This is identical to the FlowImprove procedure with $\sigma$ instead of $\theta$; for LocalFlowImprove we develop algorithms to work with this problem implicitly.}\label{proc:aug_graph_localflowimprove}
	\begin{algorithmic}[1]
		\STATE Add to the set of nodes $V$ a source node $s$ and a sink node $t$. 
		\STATE Add to the set of edges $E$ an edge from the source node $s$ to every node in the seed set of nodes $R$ with weight the degree of that node times $\delta$.
		\STATE Add to the set of edges $E$ an edge from the sink node $t$ to every node in the set of nodes $\Rbar$ with weight the degree of that node times $\delta \sigma$, where $\sigma \in [\vol(R)/\vol(\Rbar),\infty)$.
	\end{algorithmic}
\end{subroutine}

Moreover,  the subproblem construction and augmented graph are identical to FlowImprove, except with $\sigma$ instead of $\theta$.
For the construction of the modified graph to use at the subproblem step, see Augmented Graph~\ref{proc:aug_graph_localflowimprove}. The \stMinCut problem for a specific value of $\delta = \delta_k$ from the algorithm is also equivalent with $\sigma$ instead of $\theta$, 
\begin{equation}
\label{eq:mincutlocalflowimprove_2}
\MINone{\vx}{\|\mB\vx\|_{\mC,1} + (\ones - \delta)\vx^T \hat{\vd}_R + \delta \sigma \vx^T \hat{\vd}_{\Rbar}}{\vx \in \{0,1\}^{n},}
\end{equation}
using the same notation from \eqref{eq:mincutflowimprove_2}. (Here, we have not implemented the subsequent step of associating terms with sources and sinks, as that follows an identical reasoning to FlowImprove problem.) 

However, in practice, we \emph{never explicitly} build this augmented graph, as that would \emph{immediately} preclude a strongly local algorithm, where the runtime depends on $\vol(R)$ instead of $n$ or $m$ (the number of vertices or edges). 
Instead, the algorithms seek to iteratively identify a \emph{local graph}, whose size is bounded by a function of $\vol(R)$ and $\sigma$ that has all of $R$ and just enough of the rest of $G$ to be able to guarantee a solution to \eqref{eq:mincutlocalflowimprove_2}. 

As some quick intuition for \emph{why} the LocalFlowImprove subproblem might have this property, we recall \Cref{lem:lfi-size}, which showed that there is a bound on the output size that is independent of the graph size. We further note the following \emph{sparsity-promoting} intuition in the LocalFlowImprove subproblem. 
\begin{lemma}[Originally from \citet{veldticml2016}, Theorem 1]
The subproblem solve in LocalFlowImprove~\eqref{eq:mincutlocalflowimprove_2} corresponds to a degree-weighted $1$-norm regularized variation on the subproblem solve in FlowImprove~\eqref{eq:mincutflowimprove_3}. More specifically, the $1$-norm regularized problem is
\begin{equation}\label{eq:mincutlocalflowimprove_l1}
\MINone{\vx}{\|\mB\vx\|_{\mC,1} + \hat{\delta}(\ones - \vx)^T \hat{\vd}_R + \hat{\delta} \theta \vx^T \hat{\vd}_{\Rbar} + \kappa \| \mD \vx \|_1 }{\vx \in \{0,1\}^{n}} 
\end{equation}
with $\hat{\delta} = \delta + \kappa$ and $\kappa = \tfrac{\delta \sigma - \delta}{1+\theta}$, where $\theta = \vol(R)/\vol(\Rbar)$. 
\end{lemma}
\begin{proof}
The proof follows from expanding~\eqref{eq:mincutlocalflowimprove_l1} using $\hat{\delta}$ and $ \kappa \| \mD \vx \|_1 = \kappa \vx^T\hat{\vd}_R + \kappa \vx^T \hat{\vd}_{\Rbar}$ for indicator vectors. And then ignoring constant terms.
\end{proof}

\aside{aside:l1-strongly-local}{We also note that this idea of adding a $1$-norm penalty is a common design pattern to create strongly local algorithms.}

Given the rich literature on solving $1$-norm regularized problems in time much smaller than the ambient problem space or with provably fewer samples~\cite{tibshirani1996-lasso,Efron-2004-lars,Candes-2006-incomplete,Donoho2008}, these results are perhaps somewhat less surprising.

In the remainder of this section, we will explain two solution techniques for the subproblem solve that will guarantee the following runtime for finding the set that minimizes the LocalFlowImprove objective. 
\begin{theorem}[Running time of LocalFlowImprove, based on \citet{veldticml2016}]
	\label{thm:localfloworiginallocaltime} %
Let $G$ be a connected, undirected graph with non-negative integer weights.  A LocalFlowImprove problem can be solved via Dinkelbach's \Cref{algo:fractionalprog} or 
	\Cref{algo:localflowimprove}. The algorithms terminate in worst-case time  
	\begin{equation*}
	\mathcal{O}({\cut(R)} \cdot \texttt{subproblem}) \text{ for Dinkelbach and } \mathcal{O}(\log\tfrac{1}{\eps} \cdot \texttt{subproblem}) \text{ for bisection}. 
	\end{equation*}
	Let $\gamma = 1 + \frac{1}{\sigma}$. For solving the subproblem, we have the following possible runtimes, %
	\begin{equation*}
	\begin{aligned}
	\text{\Cref{algo:maxflowsimplelocal}} & \quad \gamma\vol(R) \text{ calls to MaxFlow on } \gamma\vol(R) \text{ edges}&& \text{(MaxFlow based)}  \\
	\text{\Cref{algo:simplelocal}} & \quad \mathcal{O}\left(\gamma^2\vol(R)^2 \log [\gamma \vol(R)]\right) & & \text{(BlockingFlow based)}
	\end{aligned}
	\end{equation*}
\end{theorem}

\begin{proof}
	This result can be obtained by combining the iteration complexity of Dinkelbach's \Cref{thm:itercomplexityDinkel} or LocalFlowImprove from \Cref{cor:localflowimprove} with either the running time of the MaxFlow-based SimpleLocal subsolver \Cref{algo:maxflowsimplelocal} or the running time of the blocking flow algorithm, \Cref{thm:simplelocaltime}. 
\end{proof}

See \Cref{subsec:approxDinic} for details on faster algorithms from~\citet{OZ14}.

\subsection{Strongly Local Constructions of the Augmented Graph}\label{subsubsec:localgraph}
Before we present algorithms for the LocalFlowImprove subproblem, we discuss a crucial result from~\citet{OZ14} that reduces the Augmented Graph~\ref{proc:aug_graph_localflowimprove} for the MaxFlow problem to a reduced modified graph that includes only nodes relevant to the optimal solution. The crux of this section is an appreciation of the following statement: 
\begin{quote}\itshape
	An unsaturated edge in a flow is an edge where the flow value is strictly less than the capacity. 
	If, in a solution of MaxFlow on the augmented graph, there is an
	unsaturated edge from a node in $\Rbar$ to $t$, then that node is
	not in the solution MinCut set.
\end{quote}
This result is a fairly simple structural statement about how we might \emph{verify} a solution to such a MaxFlow problem. We will illustrate it first with a simple example where the optimal solution set is contained within $R$, akin to MQI but without that explicit constraint, and then we move to the more general case, which will involve introducing the idea of a \emph{bottleneck set} $B$.  Throughout these discussions, we will use $\hat{G}$ to denote the full $s,t$ augmented graph construction for a LocalFlowImprove subproblem with $R, \delta, \sigma$ fixed. 

Consider what happens in solving a MaxFlow on $\hat{G}$ where $\sigma > \vol(R)$. In this scenario, LocalFlowImprove will always return $S \subseteq R$ and this will be true on subproblem solve as well. (See discussion in \Cref{sec:improvement-objectives}). We will show how we can locally certify a solution on $\hat{G}$---without even creating the entire augmented graph. We first note the structure of the $\hat{G}$ partitioned into the following sets: $R$, $\partial R$ and everything else, i.e.,~$\Rbar - \partial R$. This results in a view of the subproblem as follows: 
\[ \begin{array}{c@{}c@{}c}
\text{full subproblem } \hat{G} & \text{or} & \text{edge subset.}\\
\text{\includegraphics[width=0.45\linewidth]{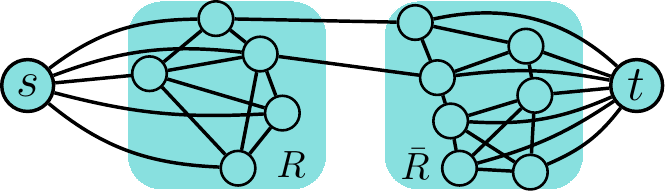}} && \text{\includegraphics[width=0.45\linewidth]{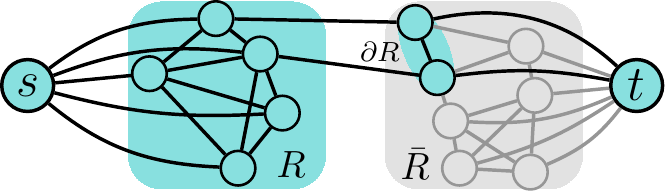}} 
\end{array} \] 
Suppose that we \emph{delete} all the gray edges and solve the resulting MaxFlow problem (or we just solve the problem for the teal-colored subset). This will result in a MaxFlow problem on a subset of the edges of the augmented graph---and one that has size bounded by $\vol(R)$. 
In any solution of the resulting MaxFlow problem, we have that all of the edges from $\partial R$ to $t$ will be unsaturated, meaning that the flow along those edges will be strictly smaller than the capacity. This is straightforward to see because the total flow out of the source is $\delta \vol(R)$ and \emph{each edge} from $\partial R$ to $t$ has weight $d_i \delta \sigma > d_i \delta \vol(R)$. Consequently the nodes in $\partial R$ will always be on the sink side of the MinCut solution. 

This ability of unsaturated edges to provide a local guarantee that we have found a solution arises from two aspects. 
First, we have a strict edge-subset of the true augmented graph, so any flow value we compute will be a lower-bound on the max flow  objective function on the entire graph. Second, we have not removed any edges from the source. Consequently we can locally certify this solution because none of the edges leading to $t$ are saturated, so the bottleneck must have been outside of the boundary of $R$. Put another way, since the edges from $\partial R$ to $t$ are unsaturated, there is no way the omitted gray nodes and edges could have helped get more flow from the source to the sink.

Now, suppose that $\sigma$ was smaller such that at least one node in the boundary of $R$ has a saturated edge to $t$. Then we lose the proof of optimality because it's possible those missing gray nodes and edges could have been used to increase the flow. 
Suppose, however, we add those \emph{bottleneck} nodes in $\partial R$ to a set $B$ and solve for the MaxFlow where $B$ is: 
\[ \begin{array}{c@{}c@{}c}
\text{$B$ is one node} & \text{or} & \text{$B$ is two nodes.}\\
\text{\includegraphics[width=0.45\linewidth]{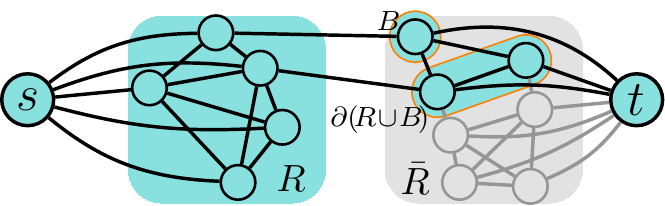}} && \text{\includegraphics[width=0.45\linewidth]{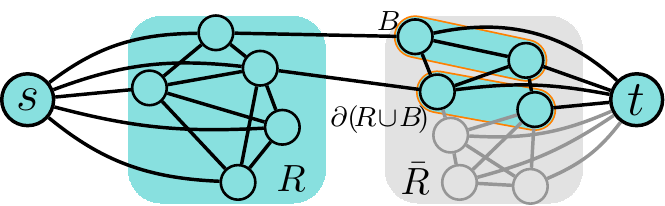}} 
\end{array} \] 
That is, 
the subgraph of $\hat{G}$ with all edges among $s, t$, $R$, $B$, $\partial(R \cup B)$. As long as the bottleneck is not in the boundary $\partial(R \cup B)$, then we have an optimal solution. 
The best way to  think about this is to look at the missing edges in the picture. If all the edges to $t$ from the boundary $\partial (R \cup B)$ are unsaturated, then we must have a solution, as the other edges could not have increased the flow. 

Of course, there may still be saturated edges in the boundary, but this suggests a simple algorithm. To state it, let $G_{R,\sigma,\delta}(B)$ be the $s$,$t$ MaxFlow problem with
\begin{itemize}
	\item all vertices $\{s,t\}\cup R \cup B \cup \partial (R\cup B)$,
	\item all edges from the source $s$ to nodes in set $R$,
	\item all edges with nodes in the set $B \cup \partial (R\cup B)$ to the sink node $t$,
	\item all edges from nodes in $R\cup B$ to nodes in $V$.
\end{itemize}
We iteratively grow $B$ by nodes whose edges to $t$ are saturated in a MaxFlow solve on $G_{R,\sigma,\delta}(B)$, starting with $B$ empty. This procedure is described in \Cref{algo:maxflowsimplelocal}. It uses the idea of solving MaxFlow problems consistently with previous iterations, to which we will return shortly.  
What this means is that among multiple optimal solutions, we choose the one that would saturate edges to $B$ in the same way as previous solutions. There is a simple way to enforce this by using the residual graph, and this really just means that once a node goes into $B$, it stays in $B$.  
\begin{algorithm}[t]
	\caption{MaxFlowSimpleLocal~\citep{veldticml2016}}\label{algo:maxflowsimplelocal}
	\begin{algorithmic}[1]
		\STATE Set $B :=\emptyset$
		\WHILE{the following procedure has not yet returned}
		\STATE Solve the MaxFlow problem on $G_{R,\sigma,\delta}(B)$ consistent with previous iterations.
		\STATE Let $J$ denote the vertices in $\partial (R\cup B)$ whose saturated edges to the sink $t$.
		\STATE \textbf{if} $|J| = 0$ \textbf{then} return the MinCut set $S$ as the solution
		\STATE \textbf{else}  $B \leftarrow B \cup J$ and repeat
		\ENDWHILE
	\end{algorithmic}
\end{algorithm}

\emph{Locally finding the set $B$ is, in a nutshell, the idea behind strongly local algorithms for LocalFlowImprove.} 
These strongly local algorithms construct the set $B$ for each subproblem solve by doing exactly what we describe here, along with a few small ideas to make them go fast. The algorithms to accomplish this will always produce a set $B$ whose size is bounded in terms of $\sigma$ and $\vol(R)$, as guaranteed by the following~result.

\begin{lemma}[Lemma 4.3~\citep{OZ14}]
We have $\vol(B) \le \frac{1}{\sigma} \vol(R)$ for every iteration for the iteratively growing procedure in \Cref{algo:maxflowsimplelocal} . 
\end{lemma}
\begin{proof} The proof follows because each time a node $v$ is added to $B$, we know there was a flow that saturated the edge with weight $\sigma \delta d_v$. Since the total flow from $s$ is $\delta \vol(R)$, this implies that if $\vol(B) \ge \vol(R)/\sigma$, then we have expanded enough edges to $t$ to guarantee that the flow can be fully realized with no bottlenecks. 
\end{proof}

\subsection{Blocking Flow}
\label{sec:blockingflow}

In each iteration of \Cref{algo:maxflowsimplelocal}, we need to identify the set $J$. We motivated this set with a \emph{maximum flow} on the graph $G_{R,\sigma,\delta}(B)$. 
It turns out that we do not actually need to solve a MaxFlow problem. Instead, the concept of a \emph{blocking flow} suffices. The difference is subtle but important. 
A blocking flow is a flow such that every path from the source to the sink contains at least one saturated edge. For a demonstration of a blocking flow, see Figure \ref{fig:blocking}. See also the helpful descriptions in \citet[Chapter 4]{Williamson2019}. By this definition, a maximum flow is always a blocking flow because the source and sink are disconnected in the residual graph. 
\begin{figure}[t]
	\centering
	\includegraphics[width=0.5\linewidth]{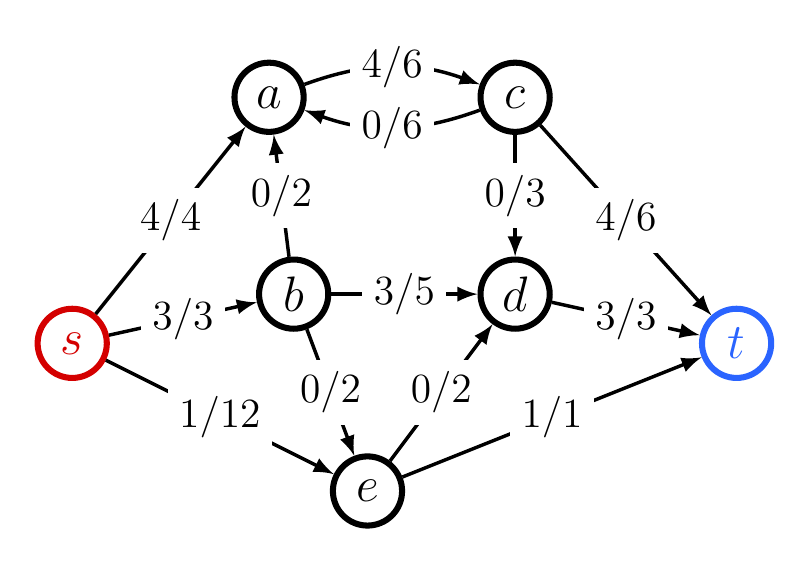}
	\caption{We demonstrate a flow that starts from the source node $s$ and ends to the sink node $t$. This flow includes the paths ($s$,$b$,$d$,$t$), ($s$,$a$,$c$,$t$) and ($s$,$e$,$t$). Note that this flow is a blocking flow since every path from $s$ to $t$ includes at least one saturated edge. It is not a maximum flow because there is a path from $t$ to $s$ with reversed edges.}
	\label{fig:blocking}
\end{figure}

The relevance of blocking flows is that after finding a blocking flow and looking at the residual graph, then the distance from the source and sink increases by one. This is essentially what we do with the set $B$ and $J$. If $B$ is not yet optimal and we find nodes $J$ in \Cref{algo:maxflowsimplelocal} then the \emph{length} of the next path from the source to the sink found to become optimal must increase by one. Blocking flow algorithms use this property, along with other small properties of the residual graph, to accelerate flow computations. We defer additional details to \citet{Williamson2019} in the interest of space. 

The best algorithm for computing blocking flows has been suggested in~\citet{DD83}. There, the authors proposed a link-cut tree data structure that is used to develop a strongly polynomial time algorithm for weighted graphs that computes blocking flows in $\mathcal{O}(m\log n)$ time, where $m$ is the number of edges in the given graph. Blocking flows are a major tool and subroutine inside other solvers for MaxFlow problems. For example, Dinic's algorithm~\cite{Dinitz-1970-max-flow}, simply runs successive blocking flow computations on the residual graph to compute a maximum flow (see \citet[Algorithm 4.1]{Williamson2019} as well). This iteratively finds the maximum flow \emph{up to a distance $d$}. Here, they serve the purpose of giving us a lower bound on the maximum flow that could saturate some edges of the graph.

\subsection{The SimpleLocal subsolver}\label{subsubsec:simplelocal}

For the SimpleLocal subsolver, we will use the concept of local bottleneck graph $G_R(B)$ that was introduced in \Cref{subsubsec:localgraph}. (We omit $\sigma, \delta$ for simplicity.) The only other idea involved is that we can \emph{iteratively} update the entire flow itself using the residual graph. So, rather than solving MaxFlow at each step, we compute a blocking flow to find new elements $J$ and update the residual graph. This ensures that the flow between iterations is consistent in the fashion we mentioned in \Cref{algo:maxflowsimplelocal}. The algorithm is presented in \Cref{algo:simplelocal}. SimpleLocal is exactly Dinic's algorithm but specialized for our LocalFlowImprove~problem. 
\begin{algorithm}
	\caption{SimpleLocal~\citep{veldticml2016}}\label{algo:simplelocal}
	\begin{algorithmic}[1]
		\STATE Initialize the flow variables $f$ to zero and $B :=\emptyset$
		\WHILE{True}
		\STATE Compute a blocking flow $\hat{f}$ for the residual graph of $G(B)$ with the flow $f$, if flow is zero, then stop. 
		\STATE $f \leftarrow f + \hat{f}$
		\STATE Let $J$ denote the vertices in $\partial (R\cup B)$ whose edges to the sink node $t$ get saturated using the new flow variables $f$.
		\STATE $B \leftarrow B\cup J$
		\ENDWHILE
		\STATE The current flow variables $f$ are optimal for the MaxFlow in $G(B)$, return a MinCut set $S$ on the source side $s$.
	\end{algorithmic}
\end{algorithm}

\begin{theorem}[Iteration complexity and running time for SimpleLocal]
	\label{thm:simplelocaltime}
Let $G$ be an undirected, connected graph with non-negative weights. 
	SimpleLocal requires $(1+\frac{1}{\sigma}) \vol(R)$ iterations to converge to the optimal solution of the MaxFlow subproblem and $\mathcal{O}(\vol(R)^2 (1+\frac{1}{\sigma})^2 \log [(1+\frac{1}{\sigma})\vol(R))]$ running time.
\end{theorem}
\begin{proof}
	Dinic's algorithm converges in at most $(1+\frac{1}{\sigma}) \vol(R)$ iterations (Proposition A.1 in \citet{OZ14} and Lemma 4.3 of~\citet{OZ14}).
	Each iteration requires a blocking flow operation which costs $\mathcal{O}((1+\frac{1}{\sigma}) \vol(R) \log [(1+\frac{1}{\sigma})\vol(R)])$ time (Lemma 4.2 of~\citet{OZ14}). Hence, SimpleLocal requires $\mathcal{O}(\vol(R)^2 (1+\frac{1}{\sigma})^2 \log [(1+\frac{1}{\sigma})\vol(R)])$ time.
\end{proof}

In~\cite{veldticml2016}, SimpleLocal is described using MaxFlow to compute the blocking flows in Step 3 of \Cref{algo:simplelocal}. We also used Dinkelbach's algorithm instead of binary search. Otherwise, however, the two algorithms are identical. In practice, both of those modifications result in faster computations, although they are slower in theory.

\subsection{More sophisticated subproblem solvers} \label{subsec:approxDinic}
There are more advanced solvers for the LocalFlowImprove algorithm possible in theory. For instance, \citet{OZ14} also presents a solver based on the Goldberg-Rao push relabel method~\cite{Goldberg-Rao} that will yield a strongly local algorithm. Finally, note that the goal in using these algorithms is often to minimize the \emph{conductance} of a set $S$ instead of the relative conductance $\phi_{R,\sigma}(S)$, in which case relative conductance is just a computationally useful proxy. In the analysis of \citet{OZ14}, they show that running algorithm \Cref{algo:simplelocal} for a bounded number of iterations will either return a set $S$ that minimizes the relative conductance exactly, or find an easy-to-identify bottleneck set $S'$ that has conductance $\phi(S') \le 2 \delta$. Using this second property, they are able to relate the runtime of the algorithm to the conductance of the set returned for a slightly different type of guarantee than exactly solving the LocalFlowImprove subproblem \cite[Theorem 1a]{OZ14}.

	\addcontentsline{toc}{part}{Part III. Empirical Performance and Conclusion}
	\section*{\large Part III. Empirical Performance and Conclusion}
\section{Empirical Evaluation}
\label{sec:experiments}

In this section, we provide a detailed empirical evaluation of the cluster improvement algorithms we have been discussing.
The focus of this evaluation is on illustrating how the methods behave and how they might be incorporated into a wide range of use cases.
The specific results we show include the~following.

\begin{enumerate}[leftmargin=*]
\item
\textbf{Reducing conductance.}
(\Cref{sec:small-conductance}.)
Flow-based cluster improvement algorithms are effective at finding sets of smaller conductance near the reference set---as the theory promises.
This is illustrated with examples from a road network, see \Cref{fig:usroads-results} and \Cref{tab:usroads}, where the algorithm finds geographic features to make the conductance small, as well as on a data-defined graph from astronomy, see \Cref{fig:astro} and \Cref{fig:cluster-improvement}.
We also illustrate empirically \Cref{thm:conductance}, which states that FlowImprove and LocalFlowImprove always return smaller conductance sets than MQI.
In our experiments, these improvement algorithms commonly return sets of nodes in which the conductance is cut in half, occasionally reducing it by up to one order of magnitude or more.

\item
\textbf{Growing and shrinking.}
(\Cref{sec:target-set}.)
Flow-based improvement algorithms are useful for the target set recovery task (basically, the task of finding a desired set of vertices in a graph, when given a nearby reference set of nodes), even when the conductance of the input is not especially small.
In particular, we show how these methods can grow and shrink input sets to identify these hidden target sets when seeded nearby, by improving precision (the fraction of correct results) or recall (the fraction of all possible results).
In this case, we use a weighted graph constructed from images, where the goal is to identify an object inside the image, see \Cref{fig:grow-shrink-image}.
We also use a social network, where the goal is to identify students with a specific class year or major within the Johns Hopkins school community, see \Cref{fig:JH}.

\item
\textbf{Semi-supervised learning.}
(\Cref{sec:semi-supervised}.)
Going beyond simple unsupervised clustering methods, semi-supervised learning is the task of predicting the labels of nodes in a graph, when the nearby nodes share the same label and when given a set of true labels.
Flow-based improvement algorithms accurately refine large collections of labeled data in semi-supervised learning experiments.
Our experiments show that flow algorithms are effective for this task, moreso when one is given large collections of true labels, and somewhat less so when one is given only a small number of true labels, see \Cref{fig:ssl-all}.

\item
\textbf{Scalable implementations.}
(\Cref{subsec:findingThousandOfClusters}.)
Our software implementations of these algorithms can be used to find thousands of clusters in a given graph in parallel.
These computations scale to large graphs, see Table~\ref{tab:runtime}. The implementations we have use Dinkelbach's method and Dinic's algorithm for exact solutions of the MaxFlow problems.

\item
\textbf{Locally-biased flow-based coordinates.}
(\Cref{subsec:flowlocalstructure}.)
We can use our flow improvement algorithms to define locally-biased coordinates or embeddings, in a manner analogous to how global spectral methods are often used to define global coordinates or embeddings for data, see \Cref{fig:usroads-local-embed} and \Cref{fig:astro-embed}.  This involves a novel flow-based coordinate system that will highlight subtle hidden structure in data that is distinctly different from what is found by spectral methods, as illustrated on road networks and in the spectra of galaxies.
\end{enumerate}

\aside{aside:largeset}{\textbf{Large set results.} The set found by \FI and \LFI may not have $\vol(S) \le \vol(\Sbar)$. For instance, the \FI result in Figure~\ref{fig:usroads-results_d} has $\vol(S) > \vol(\Sbar)$. In our computer codes, we always give $\vol(S) \le \vol(\Sbar)$ and flip $S$ and $\Sbar$ to force this property. Figure~\ref{fig:usroads-results_d} reverses this flip to show the relationship with $R$.
}
To simplify and shorten the captions, throughout the remainder of this section, we will use the abbreviations \MQI, \FI (FlowImprove), and \LFI (LocalFlowImprove).
Because \LFI depends on a parameter $\delta$, we will simply write \LFI[\delta], e.g.,~\LFI[1.0].
The formal interpretation of this parameter is LocalFlowImprove($R, \sigma=\vol(R)/\vol(\Rbar) + \delta$), where $\delta$ is a non-negative real number.
Recall that \LFI[0.0] is equivalent to \FI and \LFI[\infty] is equivalent to MQI. 

\subsection{Flow-based cluster improvement algorithms reduce conductance}
\label{sec:small-conductance}

The first result we wish to illustrate is that the algorithms \MQI, \FI, and \LFI reduce the conductance of the input reference set, as dictated by our theory.
For this purpose, we are going to study the US highway network as a graph (see~\Cref{fig:usroads}).
Edges in this network represent nationally funded highways, and nodes represent intersections.
There are ferry routes included, and there exist other major roads that are not in this data.
This network has substantial local and small-scale structure that makes it a useful example. It has a natural large-scale geometry that makes it easy to understand visually. And there are large (in terms of number of nodes) good (in terms of conductance) partitions of this network.

\begin{figure}[t]
	\centering
	\includegraphics[width=0.9\linewidth]{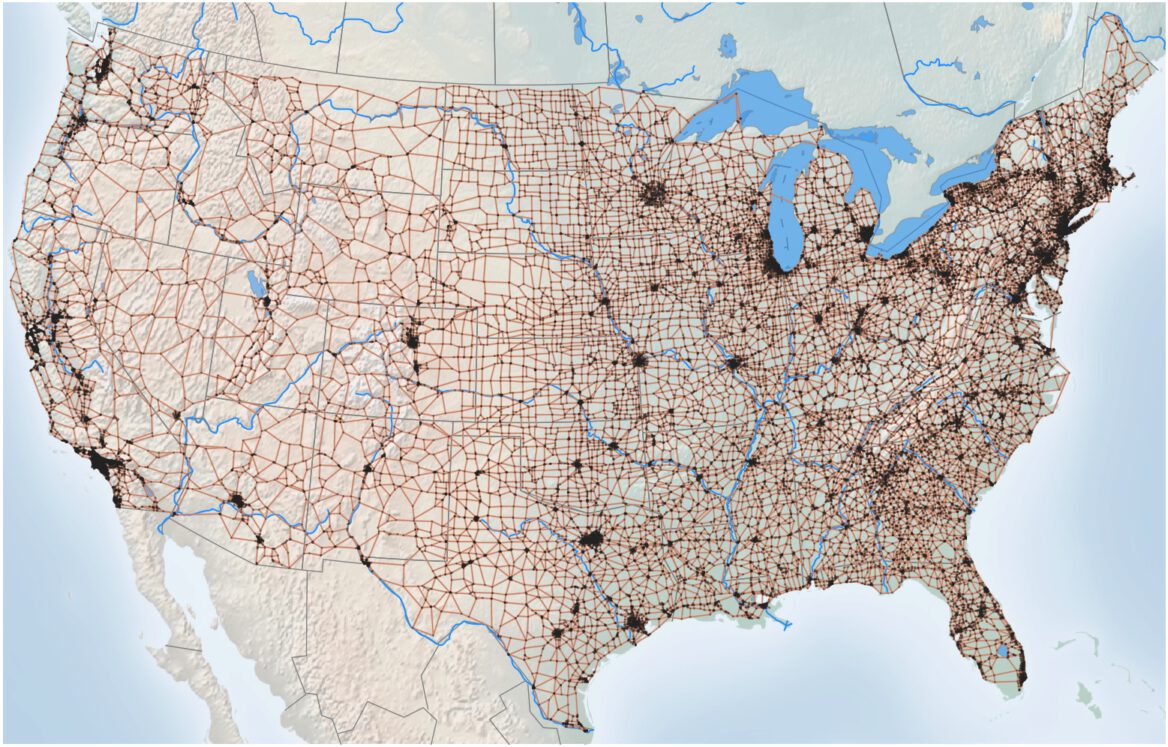}

	\caption{The US National Highway Network as a simplified graph has $51,144$ nodes and $86,397$ undirected edges. Edges represent roads and are shown as orange lines, nodes are places where roads meet and are shown as black dots. This display highlights major topographical features as well as major rivers. Mountain ranges, rivers, and lakes create interesting fine-scale features for our flow algorithms to find. There are also dense local regions around cities akin to small well-connected  pieces of social networks.  }
	\label{fig:usroads}
\end{figure}

We create a variety of reference sets for our flow improvement methods to refine. In \Cref{fig:usroads-results}, nodes in black show a set and purple edges with white interior show the~cut.
\begin{itemize}
\item We start with two partitions, one horizontal and one vertical (\Cref{fig:usroads-results_a,fig:usroads-results_c}) of the network. These are simple-to-create sets based on using latitude and longitude, and they 
 roughly bisect the country into two pieces.
They are also inherently good conductance sets.  This is due to the structure of roads on the surface of the Earth: they are tied to the two-dimensional geometric structure, and thus they have good isoperimetric or surface-to-volume properties

\item Next, we consider a large region in the western US centered on Colorado (\Cref{fig:usroads-results_e}).
Again, this set is shown in black, and the purple edges (with white interior) highlight the cut. The rest of the graph shown in orange.

\item We further consider using the vertices visited in 200 random walks of length 60 around the capital of Virginia (\Cref{fig:usroads-results_g}).
This example will show our ability to refine a set which, due to the noise in the random walks, is of lower~quality.

\item Finally, we consider the result of the \metis program for bisection, which represents our ability to refine a set that is already high quality.
This is not shown because it looks visually indistinguishable from \Cref{fig:usroads-results_d}, although the cut and volume are slightly different, as discussed below and in \Cref{tab:usroads}.
\end{itemize}

The conductance improvement results from a number of our algorithms are shown in \Cref{tab:usroads} and \Cref{fig:usroads-results}.
This table shows additional results that are not present in the figures.
We make several observations.
First, as given by \MQI, the optimal subset of the horizontal split of \Cref{fig:usroads-results_a} identifies a region in the lower US, specifically, the southern California region around Los Angeles, San Diego, and Santa Barbara (\Cref{fig:usroads-results_b}).
The southern California area is separated by mountains and deserts that are spanned by just 12 national highways to connect to the rest of the country.
Second, the result of \FI on the vertical split of \Cref{fig:usroads-results_c} of the US traces the Mississippi, Ohio, and Wabash rivers up to Lake Michigan (\Cref{fig:usroads-results}), splitting just 42 highways and ferry routes.
Note that although we start with the reference on the \emph{east coast}, the set returned by the algorithm is entirely disjoint.
This is because optimizing the FlowImprove objective expanded the set to be larger than half the volume, which caused the returned set to flip to the other coast.
Third, the region around Colorado in \Cref{fig:usroads-results_e} is refined by \LFI[1.0] to include Dallas (which was split in the initial set) and follows the Missouri river up into Montana.
Finally, a set of random walks around the Virginia capital visit much of the interior region of the state, albeit in a noisy fashion.
Using \LFI[1.0] (\Cref{fig:usroads-results_h}) refines the edges of this region to reduce conductance.
Reducing $\delta$ to 0.1 and using \LFI[0.1] (\Cref{fig:usroads-results_i}) results in a bigger set that includes the nearby city (and dense region) of Norfolk.
Note that, for the high quality \metis partition, all of our algorithms return exactly the same result.
(Again, these are not shown because the results are indistinguishable.)
We also note that this set is the overall smallest conductance result in the entire table because the volume is slightly larger than vertical split experiments.

\begin{table}
	\footnotesize
	\begin{tabularx}{\linewidth}{p{50pt}X@{}X@{}X@{}XlX@{}X@{}X@{}X@{}X}
		\toprule
		Input &  &  &  &  & Result &  &  &  &  & \\
		\cmidrule(r){1-5}
		\cmidrule(r){6-11}
		& cut & vol & size & cond. &
		Alg.	& cut & vol & size & cond. & ratio \\
		\midrule
		Horiz. & 233 & 85335 & 25054 & 0.0027 & \MQI & 12 & 9852 & 2763 & 0.0012 & 225\% \\
		&     &       &       &        & \FI & 29 & 35189 & 10471 & 0.0008 & 330\% \\
		\midrule
		Vert. &  131 & 72780 & 21552 & 0.0018 & \MQI & 29 & 35195 & 10473 & 0.0008 & 220\% \\
		&      &       &       &        & \FI & 42 & 84582 & 25030 & 0.0005 & 365\% \\
		\midrule
		Colorado
		& 195 & 23377 & 6982 & 0.0083 & \MQI & 9 & 1799 & 506 & 0.0050 & 167\%  	\\
		region   &     &       &      &        & \LFI[1.0] & 97 & 23617 & 7037 & 0.0041 & 203\% \\
		&     &       &      &        & \LFI[0.1] & 101 & 26613 & 7941 & 0.0038 & 220\% \\
		&     &       &      &        & \FI & 42 & 84204 & 24916 & 0.0005 & 1672\% \\
		\midrule
		Virginia
		& 112 & 1344 & 393 & 0.0833 & \LFI[1.5] & 23 & 1067 & 312 & 0.022 & 386\% \\
		\rlap{random walks} &     &      &     &       & \LFI[1.0] & 24 & 1212 & 357 & 0.0198 & 420\% \\
		&     &      &     &       & \LFI[0.1] & 26 & 1938 & 572 & 0.0134 & 621\% \\
		\midrule
		\metis & 56 & 85926 & 25422 & 0.0007 & \MQI & 42 & 84594 & 25034 & 0.0005 & 131\% \\
		&    &         &     &        & \LFI[0.1] & 42 & 84594 & 25034 & 0.0005 & 131\% \\
		&    &         &     &        & \FI & 42 & 84594 & 25034 & 0.0005 & 131\% \\
		\bottomrule
	\end{tabularx}
	\caption{The results of applying our algorithms to input sets of various quality for the graph of \Cref{fig:usroads-results}. A few of the sets and cuts are illustrated in \Cref{fig:usroads-results}. All of the methods reduce the conductance score considerably, with improvement ratios from 131\% to 621\%. The smallest improvements happen when the input is high-quality, such as the output from \textsc{metis}.
                }
	\label{tab:usroads}
\end{table}

\begin{figure}[p]
	\centering
	\siamwidth{
	\subfigtopskip=0pt
	\subfigcapskip=-5pt
	\subfigcaptopadj=0pt
	\subfigbottomskip=0pt
	\subfigcapmargin=0pt
	\abovecaptionskip=5pt
	\subfigure[A simple horizontal split $\phi=0.002$\label{fig:usroads-results_a} $\longrightarrow$]%
	{\includegraphics[width=0.5\linewidth]{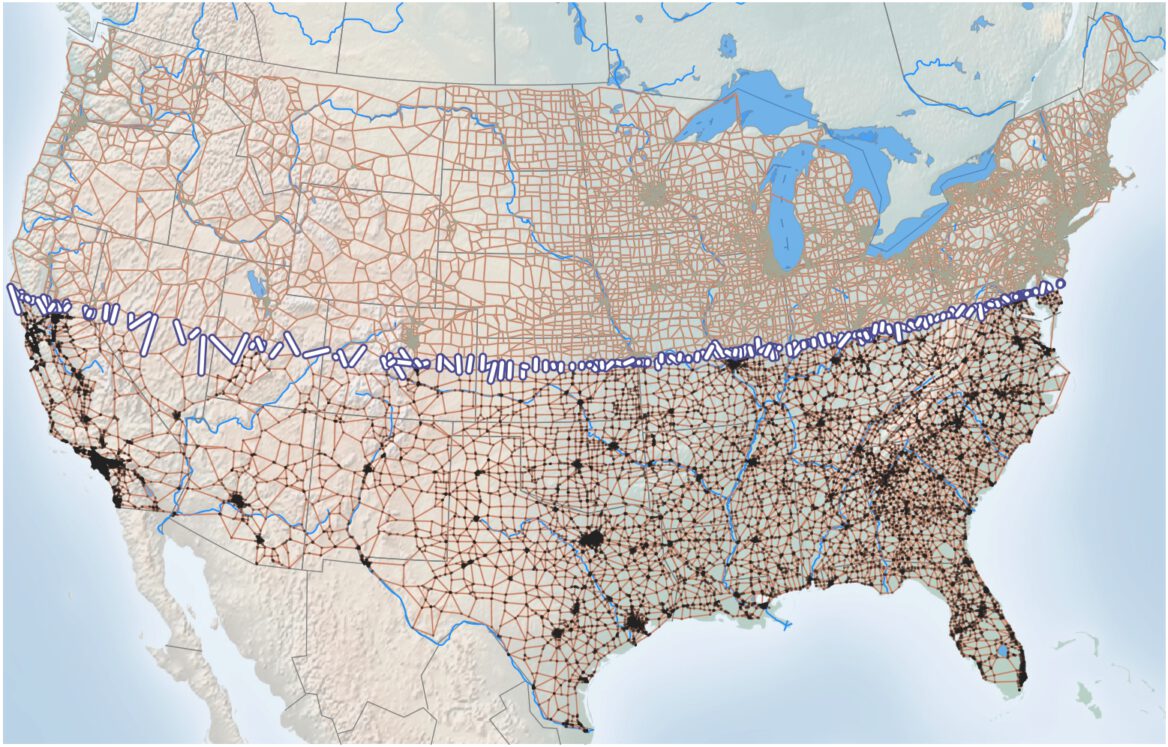}}%
	\subfigure[\MQI finds Southern California $\phi=0.0012$ \label{fig:usroads-results_b}]{\topinset{\includegraphics[width=0.37\linewidth]{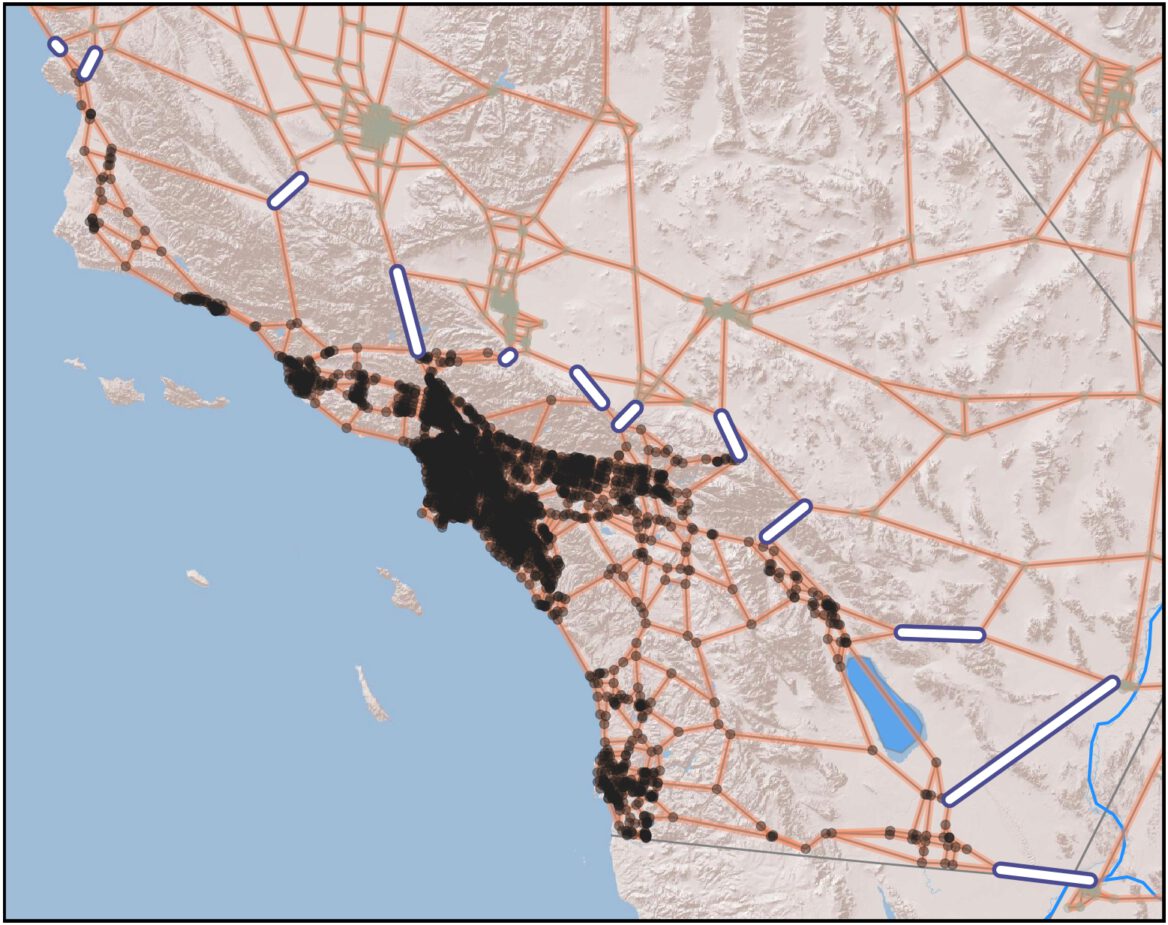}}{\includegraphics[width=0.5\linewidth]{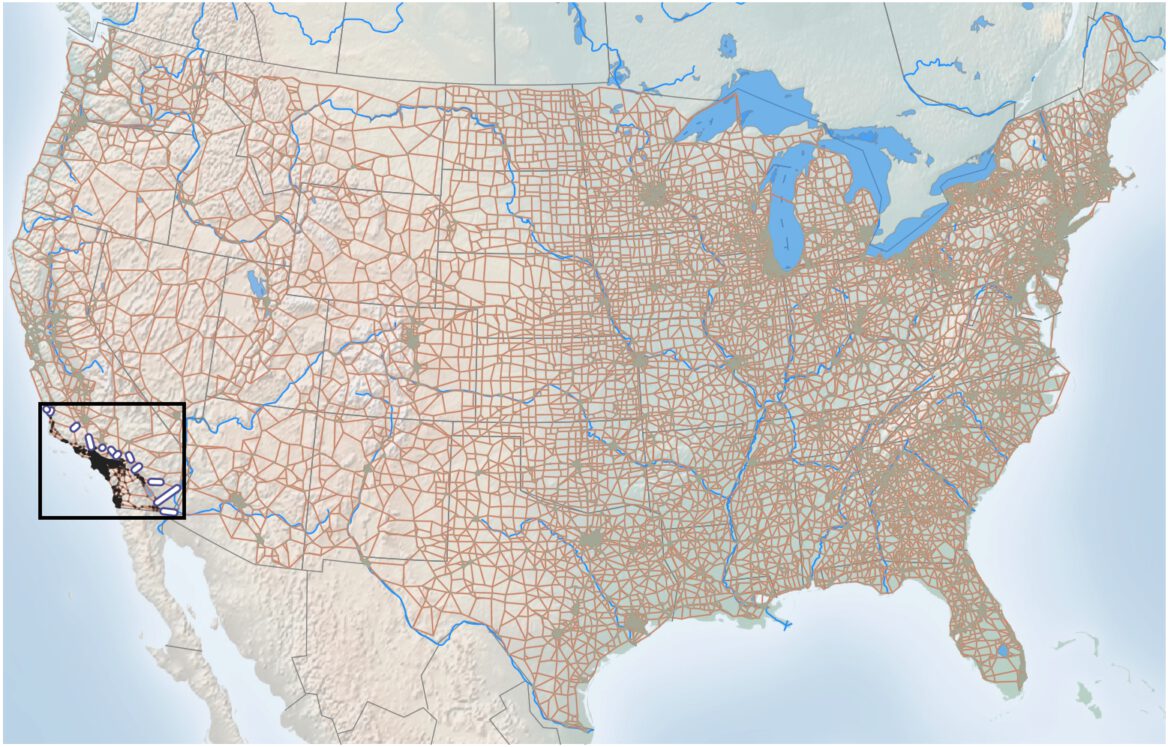}}{5pt}{15pt}}

	\subfigure[A simple vertical split $\phi=0.0018$ \label{fig:usroads-results_c} $\longrightarrow$]{\includegraphics[width=0.5\linewidth]{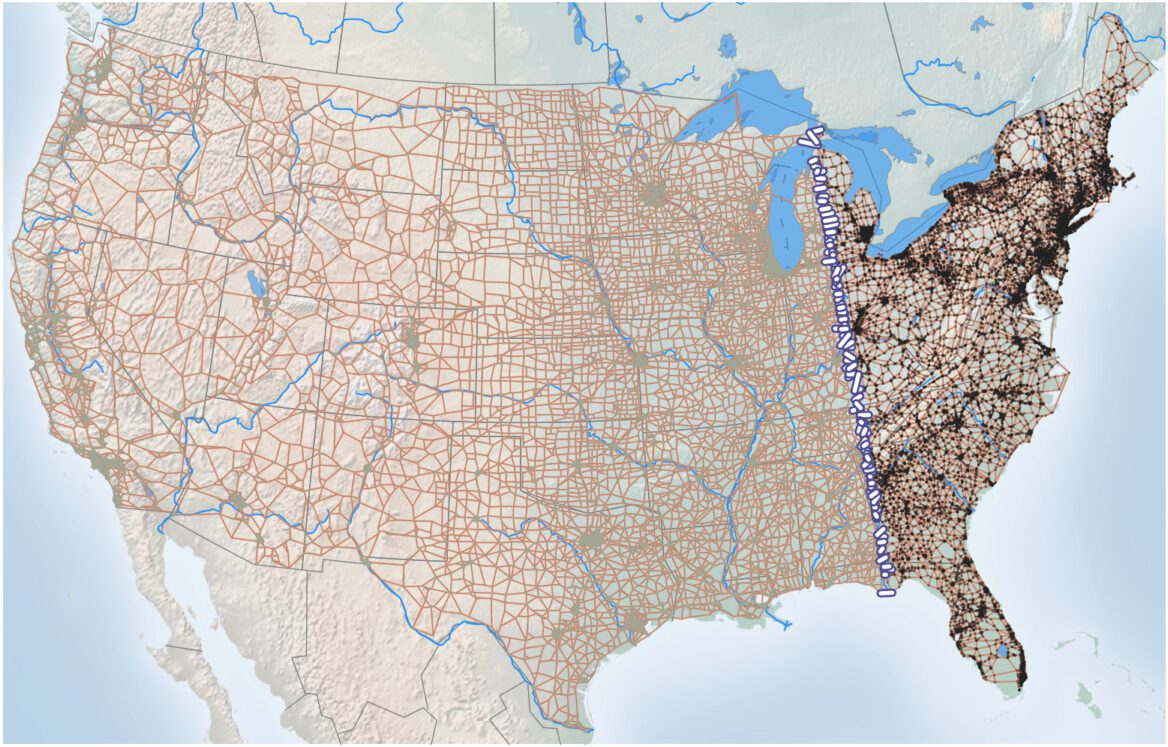}}%
	\subfigure[\FI (*) finds rivers and lakes $\phi=0.0005$ \label{fig:usroads-results_d}]{\topinset{\includegraphics[height=3.925cm]{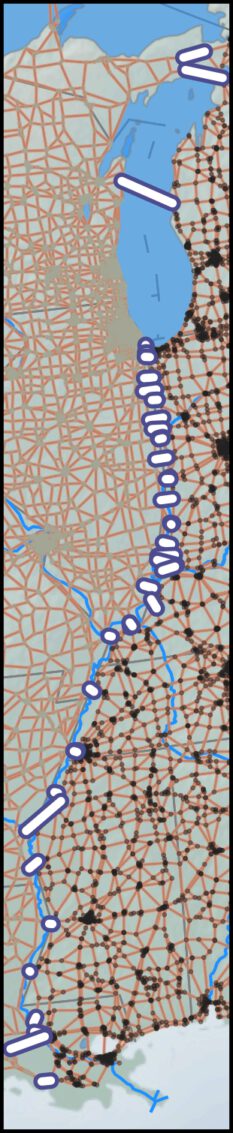}}{\includegraphics[width=0.5\linewidth]{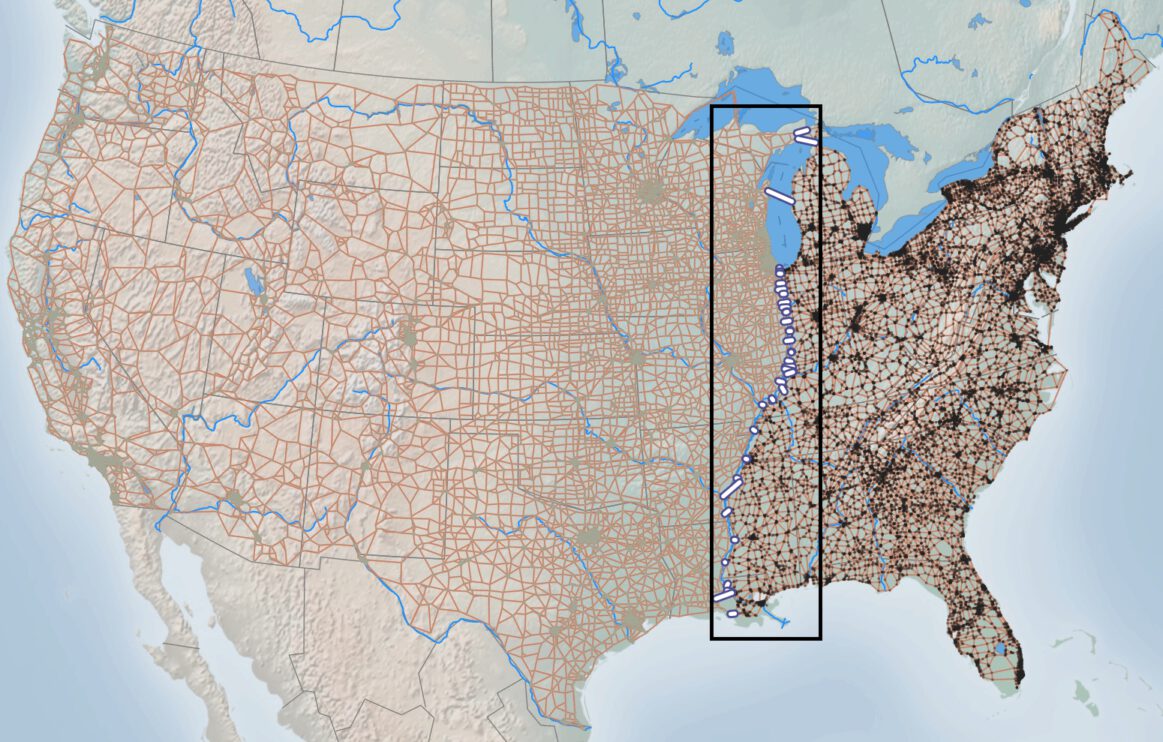}}{3pt}{-1pt}}

	\subfigure[A region around Colorado $\phi=0.008$ \label{fig:usroads-results_e} $\longrightarrow$]%
	{\includegraphics[width=0.5\linewidth]{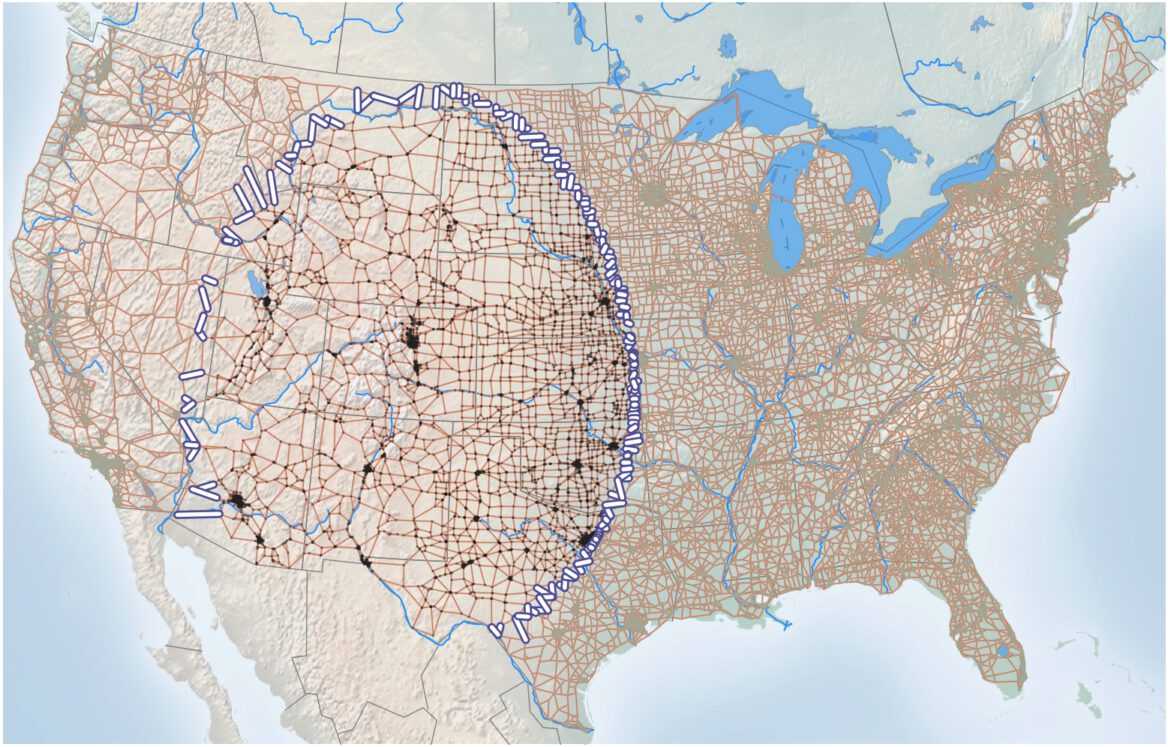}}%
	\subfigure[\mbox{\LFI[1.0]}  tracks a river and adds Dallas $\phi=0.004$\label{fig:usroads-results_f}]%
	{\includegraphics[width=0.5\linewidth]{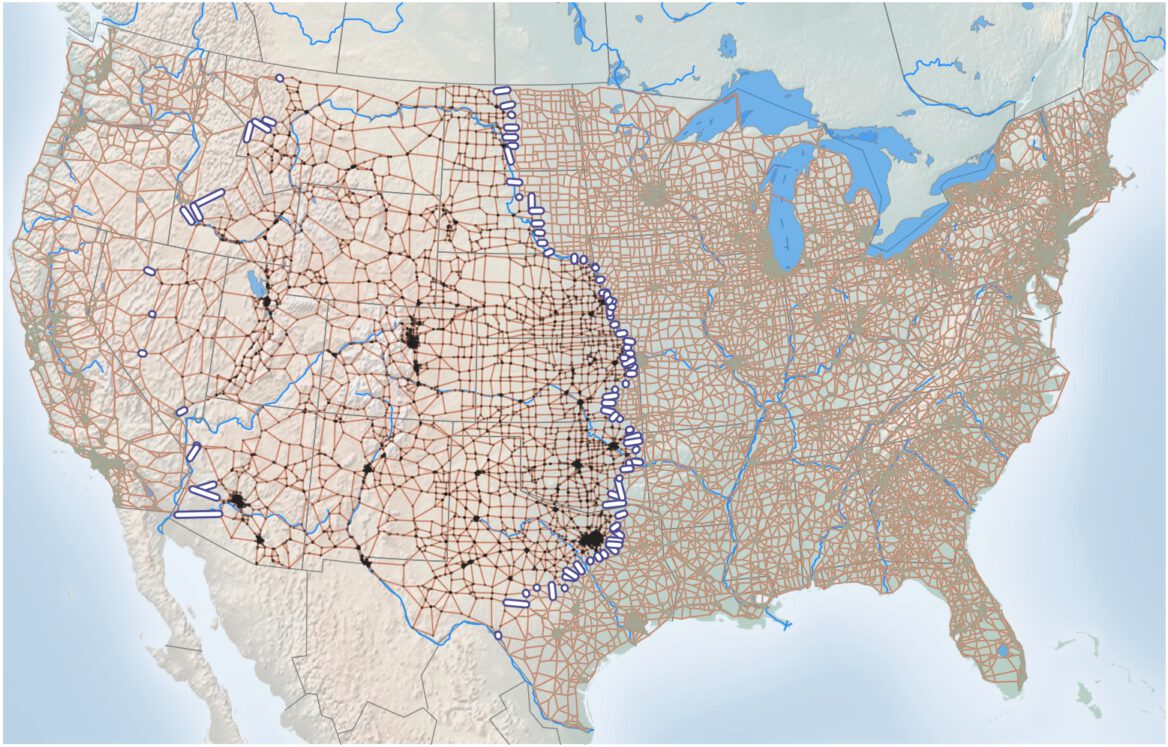}}

	\subfigure[Random walks $\phi=0.083$ \siamonly{$\longrightarrow$}%
	\label{fig:usroads-results_g}]{\includegraphics[width=0.33\linewidth]{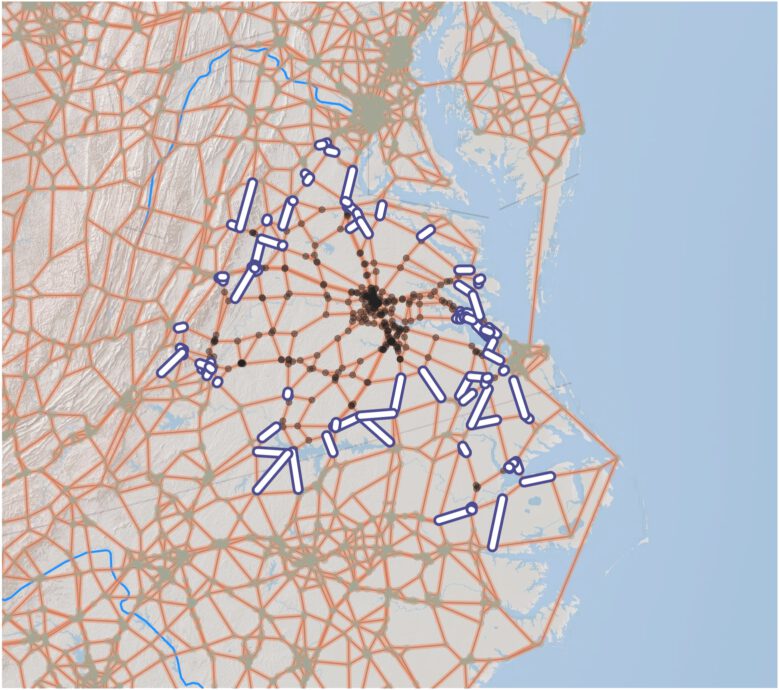}}%
	\subfigure[\mbox{\LFI[1.0]} gives $\phi=0.020$\label{fig:usroads-results_h}]{\includegraphics[width=0.33\linewidth]{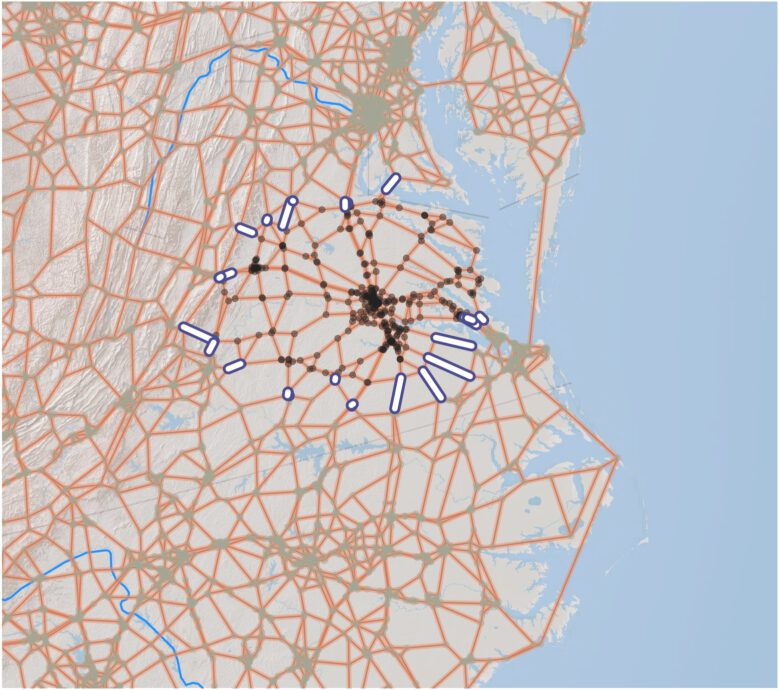}}%
	\subfigure[\mbox{\LFI[0.1]} gives $\phi=0.013$\label{fig:usroads-results_i}]{\includegraphics[width=0.33\linewidth]{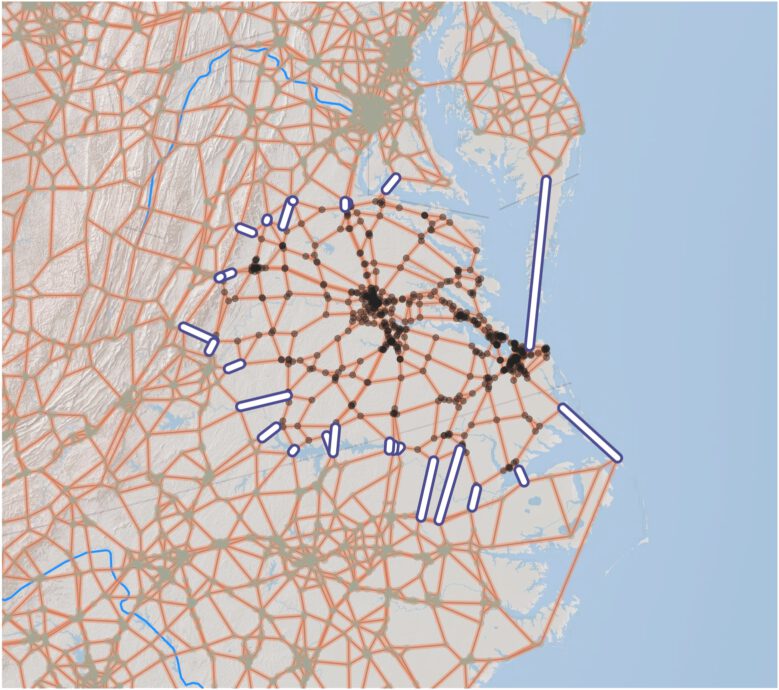}}
}
	\vspace{1.5pt}
\hrule
	\siamonly{\vspace{-8.4pt}}
	\caption{Our flow-based cluster improvement algorithms reduce the conductance of simple input sets by finding natural features including mountains, rivers, and cities.  The purple edges highlight the boundary of the  set shown in black nodes, and
		      $\phi$ is the conductance of the depicted set. Panel (a) shows an input that cuts horizontally the map and (b) is the corresponding output of \MQI.
		      Panel (c) shows an input that cuts the map vertically and
			      (d) shows the output of \FI. %
		      Panel (e) shows an input which corresponds to a large region in the western US centered on Colorado and  (f) shows the output of \LFI.
		      Finally,
		      panel (g) shows an input around the capital of Virginia, which has been created using random walks, and (h) and (i) are the corresponding output of \mbox{\LFI[1.0]} and \mbox{\LFI[0.1]}, respectively.
                }%
	\vspace{-10pt}
	\begin{minipage}{1\columnwidth}%
	\begingroup\begin{NoHyper}%
	\footnotesize\let\thefootnote\relax\footnotetext{(*) See the large set results aside (Aside~\ref{aside:largeset}).}%
	\end{NoHyper}\endgroup%
	\end{minipage}
	\label{fig:usroads-results}
\end{figure}

Overall, these results show the ability of our flow improvement algorithms to improve conductance by up to a factor of 16 in the best case scenario and by a ratio of 1.31 on the high quality \metis partition.
The most useful summary from these figures are as follows:
\begin{itemize}
	\item Reducing the value of $\delta$ in \LFI corresponds to finding smaller conductance sets compared to \MQI. We also observe that reducing $\delta$ in \LFI results in larger clusters in terms of number of nodes and volume.
        \item As predicted by \Cref{thm:conductance}, the results for \LFI and \FI are always better in terms of conductance than \MQI in terms of conductance.
\end{itemize}

While visually useful to understand our algorithms, obtaining such results on a road network is less useful and less interesting than obtaining similar results on graphs representing data with fewer or different strutural constraints. 
Thus, we now illustrate these same points in another, larger dataset with a study of around 2500 improvement calls.
This second dataset is a $k=32$-nearest neighbor graph constructed on the Main Galaxy Sample (MGS) in SDSS Data Release 7.
We briefly review the details of this standard type of graph construction and provide further details in~\Cref{sxn:app-replication}.
This data begins with the emission spectra of 517,182 galaxies in 3841 bands.
We create a node for each galaxy and connect vertices if either is within the $16$ closest vertices to the other based on a Euclidean distance-like measure (see~\Cref{sxn:app-replication}).
The graph is then weighted proportional to this distance.
The result is a weighted undirected graph with 517,182 nodes and 15,856,315 edges (and 517,182 self-loops) representing nearest neighbor relationships among galaxy spectra.
\Cref{fig:astro} provides a visualization of a global Laplacian eigenvector embedding of this graph.
For more details on this dataset, we refer readers to~\citet{LBM16_TR,lawlor2016mapping}.

\begin{figure}
	\centering
	\setlength{\fboxsep}{0pt}
	\subfigure[The full graph]{%
	\colorbox{shadecolor}{\includegraphics[width=0.5\linewidth]{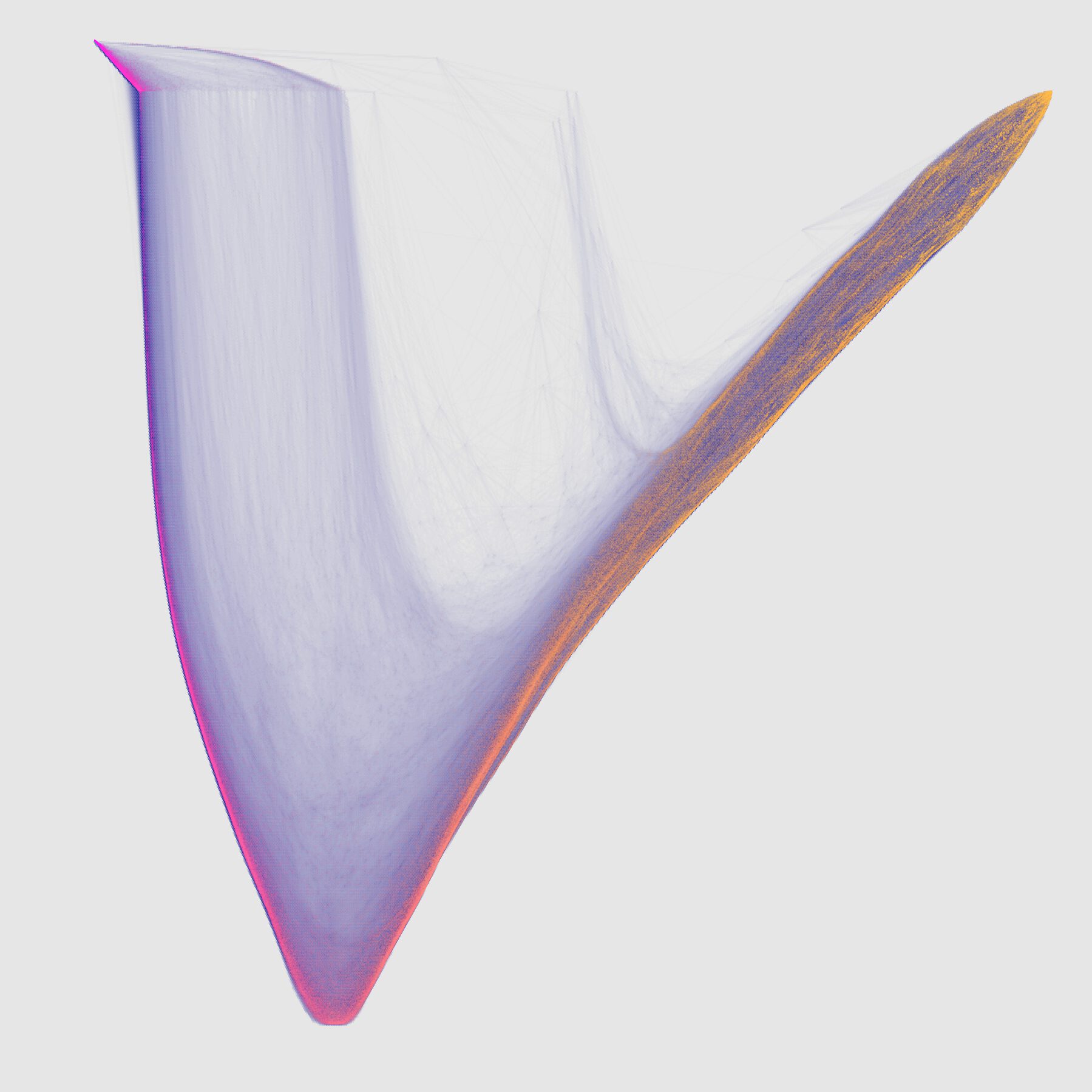}}}%
	\subfigure[Zoom into top-right]{\colorbox{shadecolor}{\includegraphics[width=0.5\linewidth,trim={12in 12in 1in 1in},clip]{figures/flow_embedding/hexbin_plots/hexbin_plot_global_low_res_spring}}}

	\caption{The Main Galaxy Sample (MGS) dataset has 517,182 nodes and 32,229,812 edges. This display shows an eigenvector embedding of the graph along with edges shown in light blue. The edges dominate the visualization in parts and nodes are only shown where there is sufficient density. The node color is determined by the horizontal coordinate (pink to orange). The right part of the visualization (dark orange to light orange coordinates in b) hints at structure hidden within the upper band, which we will study in \Cref{subsec:flowlocalstructure}.
                }
	\label{fig:astro}
\end{figure}

In this case, we compute reference sets using seeded PageRank using a random node, followed by a sweepcut procedure by~\citet{ACL06} to locally optimize the conductance of the result.
Consequently, the reference sets we start with are already fairly high quality.
Then we run \MQI, \LFI[1], \LFI[0.1], and \LFI[0.01] on the results.
We repeat this experiment 2526 times.
The output to input conductance ratio is shown in \Cref{fig:cluster-improvement} with reference to the original reference conductance from seeded PageRank (\Cref{fig:cluster-improvement_a}) and also with reference to the MQI conductance (\Cref{fig:cluster-improvement_b}).
Like the previous experiments with the road network, reducing $\delta$ in this less easily visualizable data set results in improved conductance.
Also like the previous experiments, \LFI always reduces the conductance more than \MQI.
\begin{figure}
	\subfigure[Conductance improvement relative to seeded PageRank \label{fig:cluster-improvement_a}]{\includegraphics[width=0.5\linewidth]{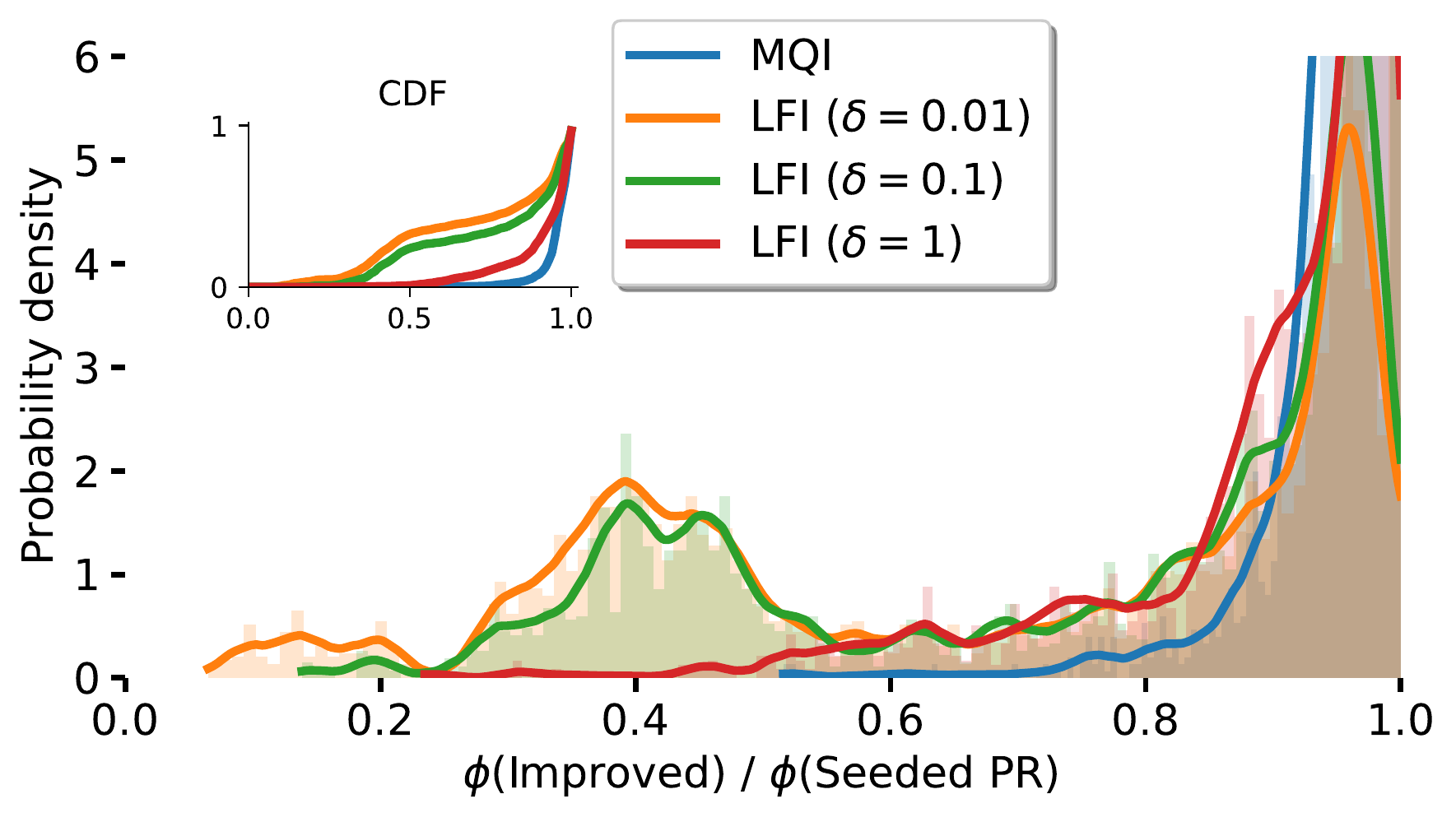}}%
	\subfigure[Conductance improvement relative to MQI \label{fig:cluster-improvement_b}]{\includegraphics[width=0.5\linewidth]{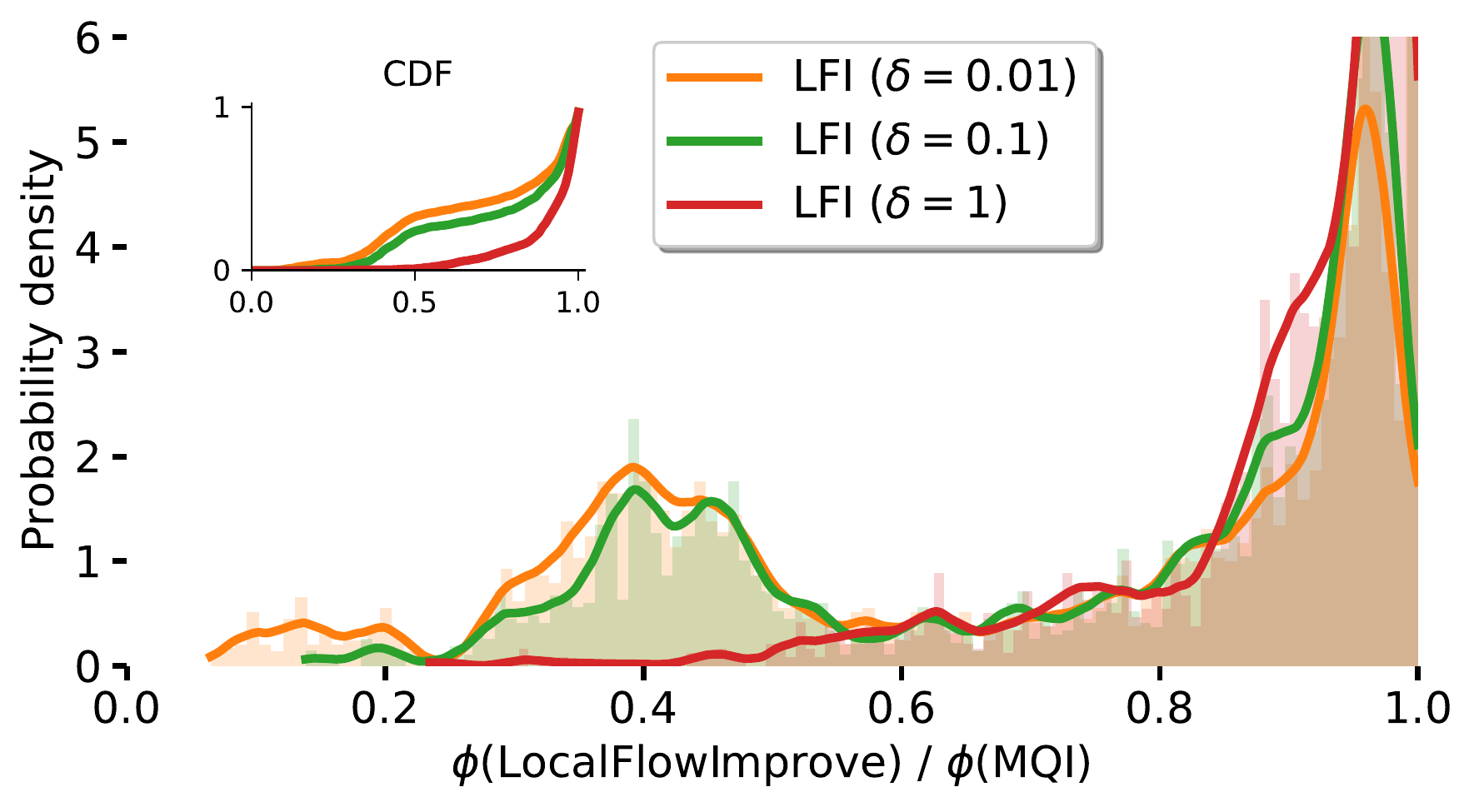}}

	\subfigure[Conductance improvement relative to 2-hop BFS \label{fig:cluster-improvement_c}]{\includegraphics[width=0.5\linewidth]{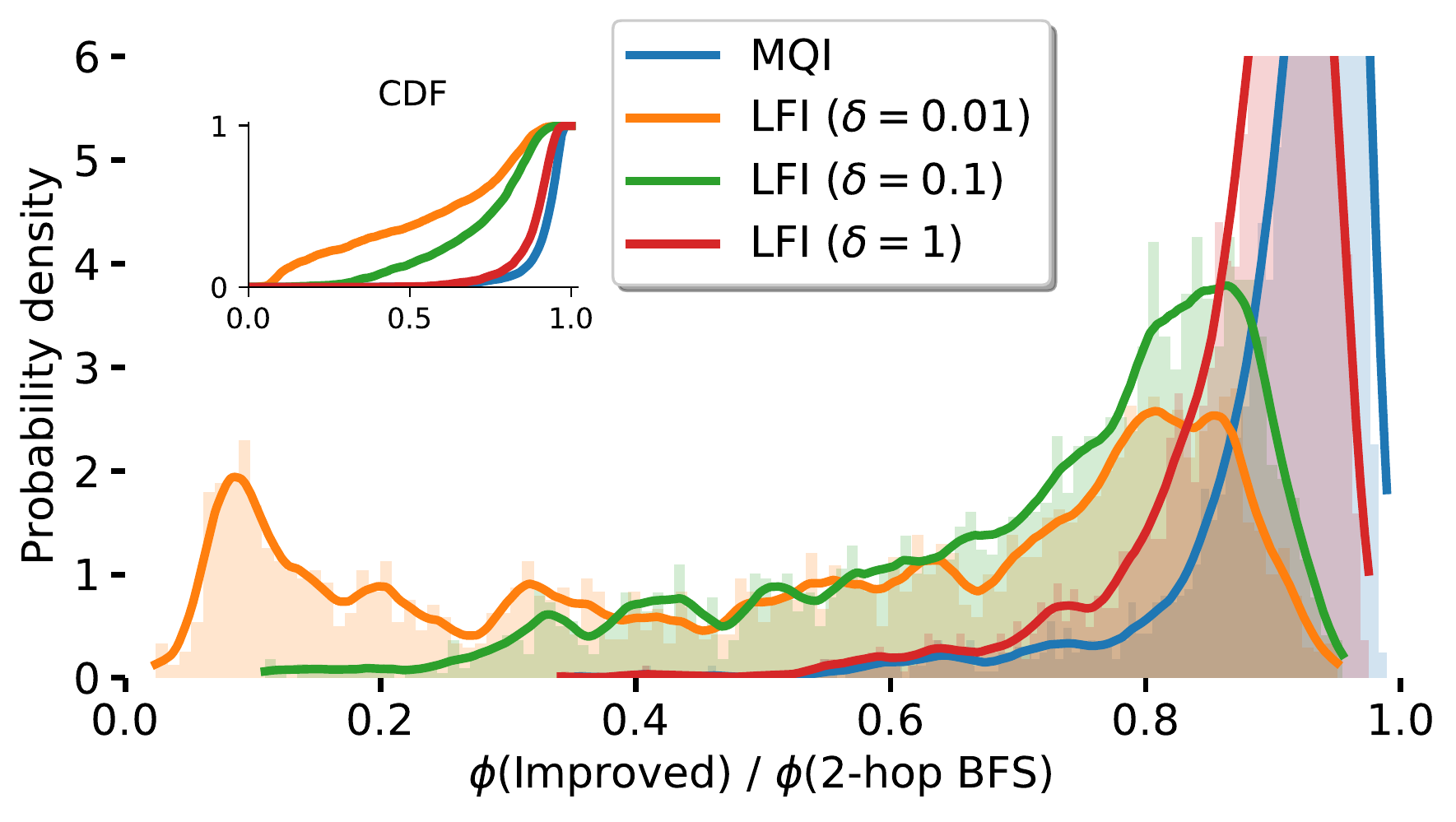}}%
	\subfigure[Conductance improvement relative to MQI \label{fig:cluster-improvement_d}]{\includegraphics[width=0.5\linewidth]{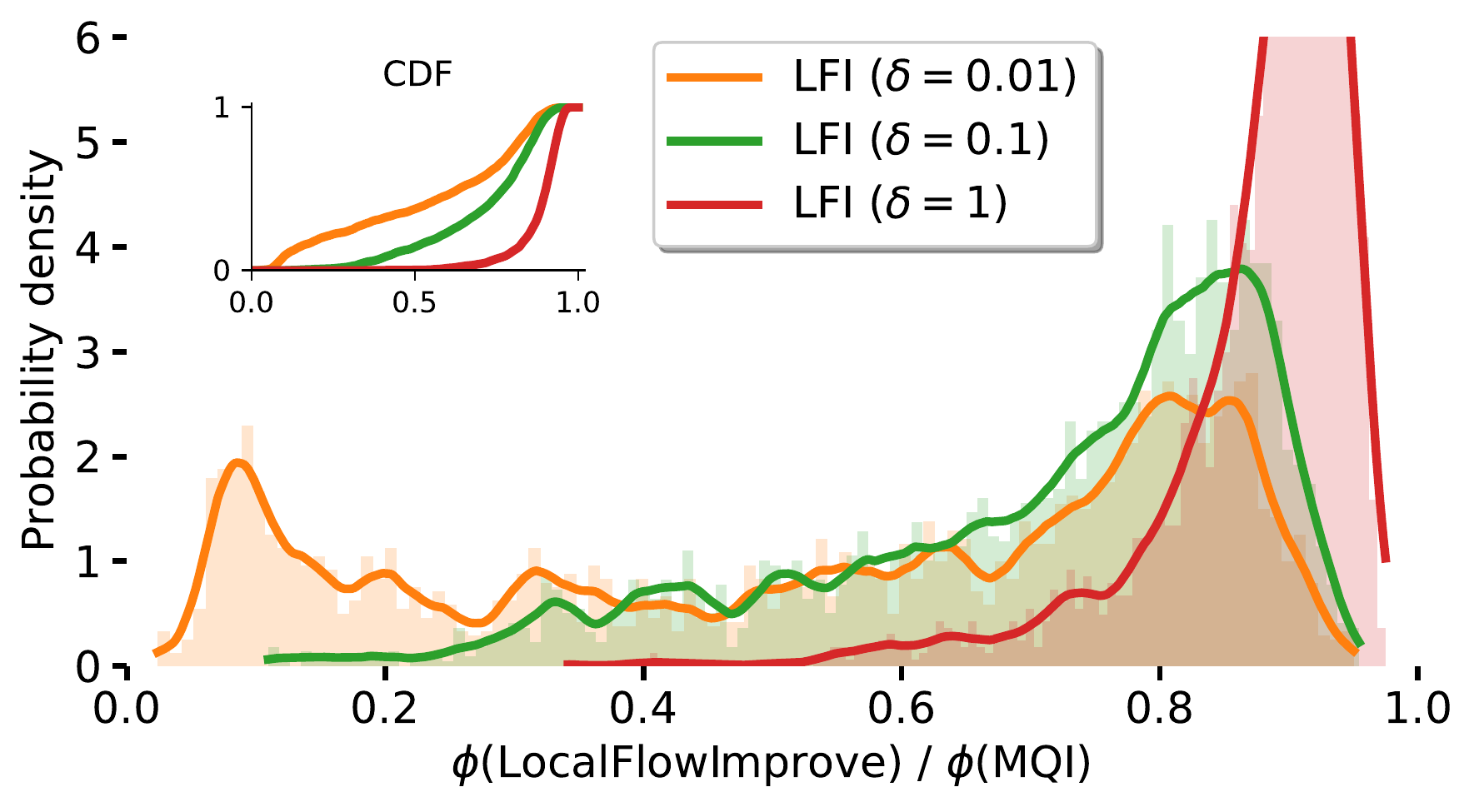}}
	\caption{A summary of 3102 (top row) and 2585 (bottom row) experiments in the \mgs dataset that show (i) that reducing $\delta$ in \LFI produces sets of smaller conductance, when the input set is a of another conductance minimizing procedure (top row) or 2-hop BFS set (bottom row), and also (ii) that \LFI and \FI always find smaller conductance sets than \MQI. The inset figures shows the cumulative density function (CDF) of the probability density.}
	\label{fig:cluster-improvement}
\end{figure}

The point of these initial experiments is to demonstrate that these algorithms achieve their primary goal of finding small conductance sets in a variety of scenarios.
They can do so both in a graph with an obviously geometric structure as well as in a graph without an obvious geometric structure that was constructed from noisy observational data. 
In addition, they can do so starting from higher or lower quality~inputs.
In the next section, we are going to evaluate our algorithms on specific tasks where finding small conductance sets is not the end goal.

\subsection{Finding nearby targets by growing and shrinking}
\label{sec:target}
\label{sec:target-set}

Another use for cluster improvement methods is to recover a hidden target set of vertices from a nearby reference set, e.g., a conjectured subregion of a graph, or a coherent section of an image.
The goal here is accuracy in returning the vertices of this set, and we can measure this in terms of precision and recall.
Let $T$ be a target set we seek to find and let $S$ be the set returned by the algorithm.
Then the precision score is $|T \cap S|/|S|$, which is the fraction of results that were correct, and the recall score is $|T \cap S|/|T|$, which is the fraction of all results that were obtained.
The ideal scenario is that both precision and recall are near $1$.

We begin by looking at the simple scenario when the initial reference $R$ is entirely contained within $T$, and also a scenario when $R$ is a strict superset of $T$.
This setting allows us to see how the flow-based algorithms grow or shrink sets to find these targets $T$, and it gives us a useful comparison against simple greedy improvement algorithms as well as against spectral graph-based approaches.
For simplicity of illustration, we examine these algorithms on weighted graphs constructed from images.
The construction of a graph based on an image is explained in \Cref{sec:image-to-graph}.

The results of the experiment are shown in \Cref{fig:grow-shrink-image}.
We consider three distinct targets within a large image, as shown in \Cref{fig:grow-shrink-image_a} and \Cref{fig:grow-shrink-image_b}: the left dog, middle dog, and right person.
In our first case, the reference is entirely contained within the target.
In this case, we can use either \FI or \LFI to attempt to enlarge to the target.
(Note that we cannot use \MQI, as the target set is larger than the seed set.)
For comparison, we use a seeded PageRank algorithm as well.
Our choice of this algorithm largely corresponds to replacing $\normof[1]{\mB \vx}$ in the flow-based objective with the minorant function $\normof[2]{\mB\vx}^2$ as discussed in \Cref{sec:fi-pagerank}.
We use two seeded PageRank scenarios that correspond to both \FI and \LFI, see Figure~\ref{fig:grow-shrink-image_c} to Figure~\ref{fig:grow-shrink-image_f}.
These show that spectral methods that grow tend either find a region that is too big or fail to grow large enough to capture the entire region.
This is quantified by a substantial drop in precision compared with the flow method.
Second, we consider the case when the target is contained within the reference set.
This corresponds to the \MQI setting as well as a variation of spectral clustering that called Local Fiedler~\cite{Chung07_localcutsLAA} (because it uses the eigenvector with minimal eigenvalue in a submatrix of the Laplacian).
The results are in \Cref{fig:grow-shrink-image_g} and \Cref{fig:grow-shrink-image_h}, and they show a small precision advantage for the flow-based methods (see the text below each image).
Finally, for reference, in \Cref{fig:grow-shrink-image_i} and \Cref{fig:grow-shrink-image_j}, we also include the results of a purely greedy strategy that grows or shrinks the reference set $R$ to improve the conductance.
This is able to find reasonably good results for only one of the test cases and shows that these sets are not overly \emph{simple} to identify, e.g., since they cannot be detected by algorithms that trivially grow or expand the seed set.
Let us also note that the results here measure 

\begin{figure}
	\subfigtopskip=0pt
	\subfigcapskip=-5pt
	\subfigcaptopadj=0pt
	\subfigbottomskip=0pt
	\subfigcapmargin=0pt
	\abovecaptionskip=5pt

	~\hfill%
	\subfigure[The full image \label{fig:grow-shrink-image_a}]{\includegraphics[width=2in]{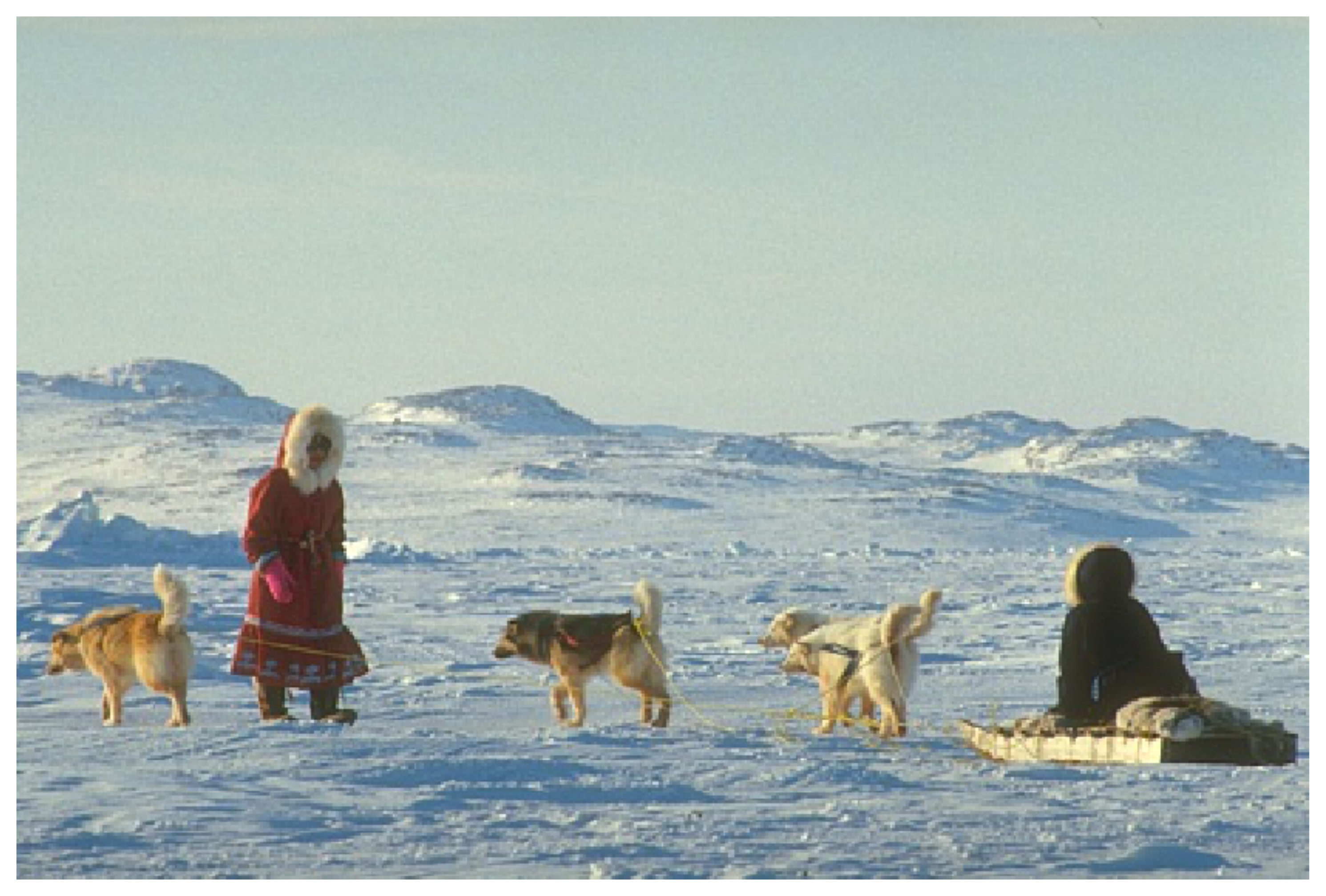}}%
	\hfill~%
	\subfigure[The targets \label{fig:grow-shrink-image_b}]{\includegraphics[width=2in]{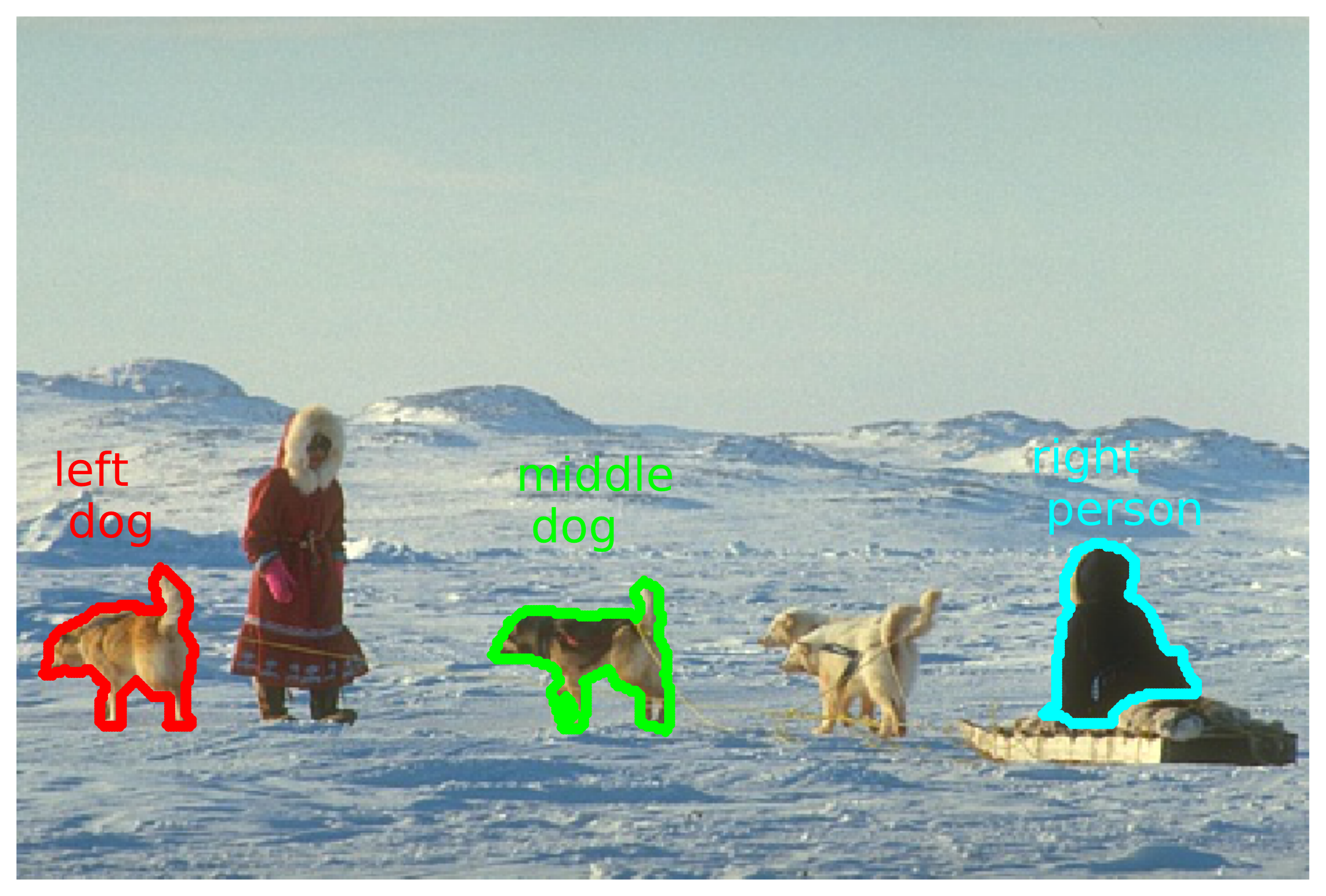}}%
	\hfill~%

	\subfigure[\FI \label{fig:grow-shrink-image_c}]{\includegraphics[width=0.5\linewidth]{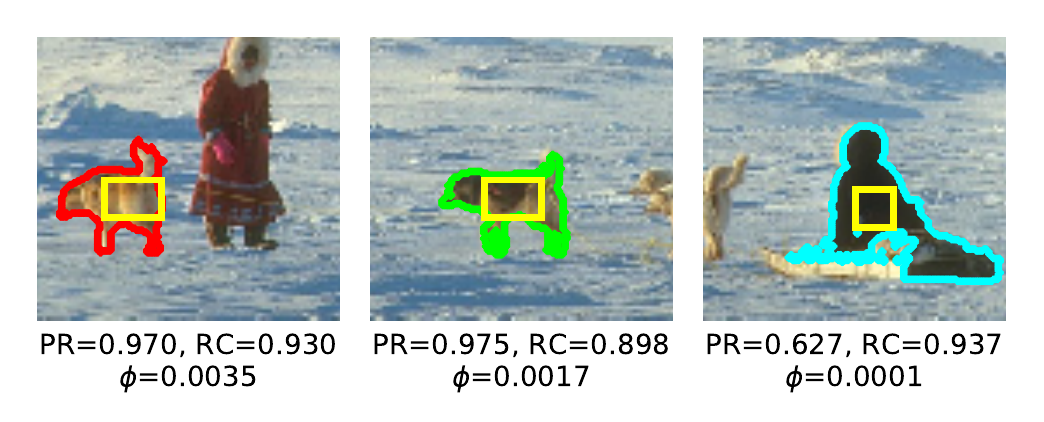}}%
	\subfigure[Seeded PageRank, $\rho=10^{-12}$ \label{fig:grow-shrink-image_d}]{\includegraphics[width=0.5\linewidth]{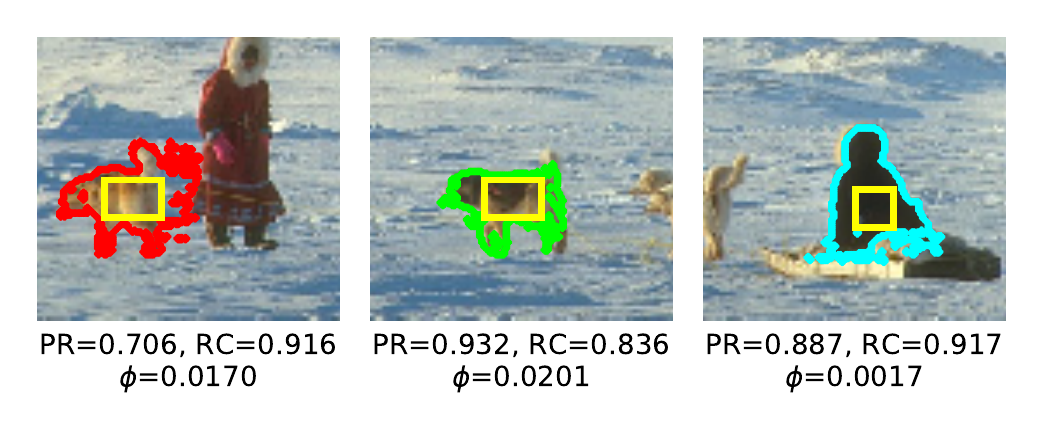}}%

	\subfigure[\mbox{\LFI[0.3]}\label{fig:grow-shrink-image_e}]{\includegraphics[width=0.5\linewidth]{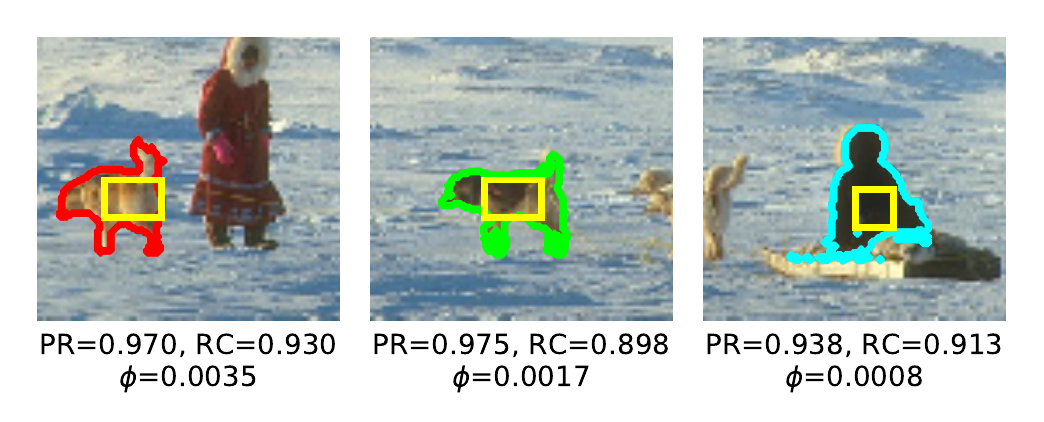}}%
	\subfigure[Seeded PageRank, $\rho=10^{-6}$ \label{fig:grow-shrink-image_f}]{\includegraphics[width=0.5\linewidth]{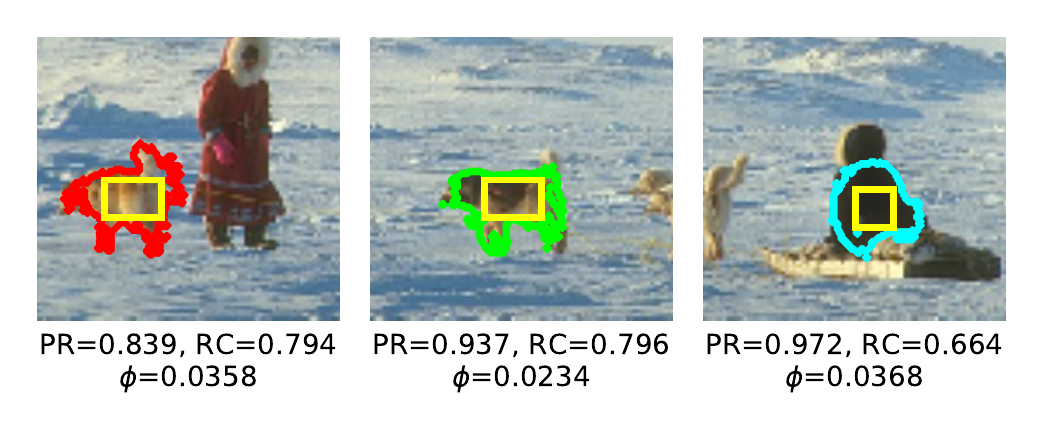}}%

	\subfigure[\MQI \label{fig:grow-shrink-image_g}]{\includegraphics[width=0.5\linewidth]{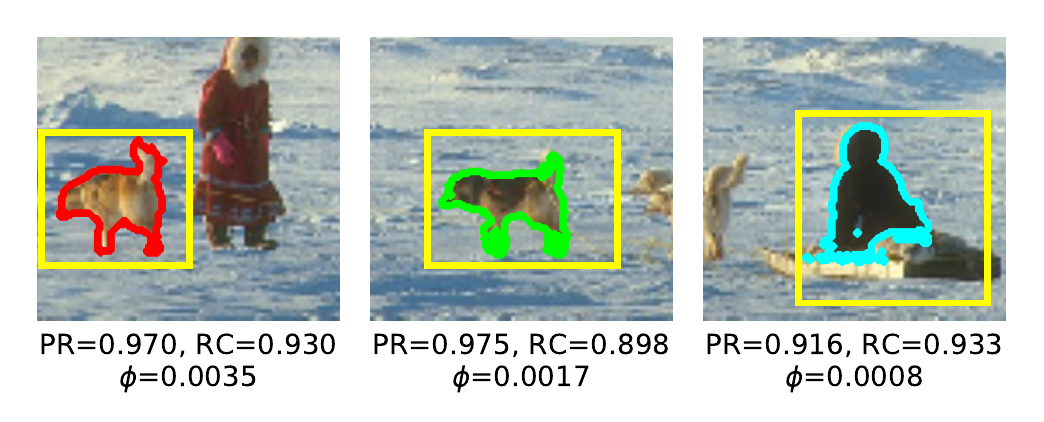}}%
	\subfigure[LocalFiedler \label{fig:grow-shrink-image_h}]{\includegraphics[width=0.5\linewidth]{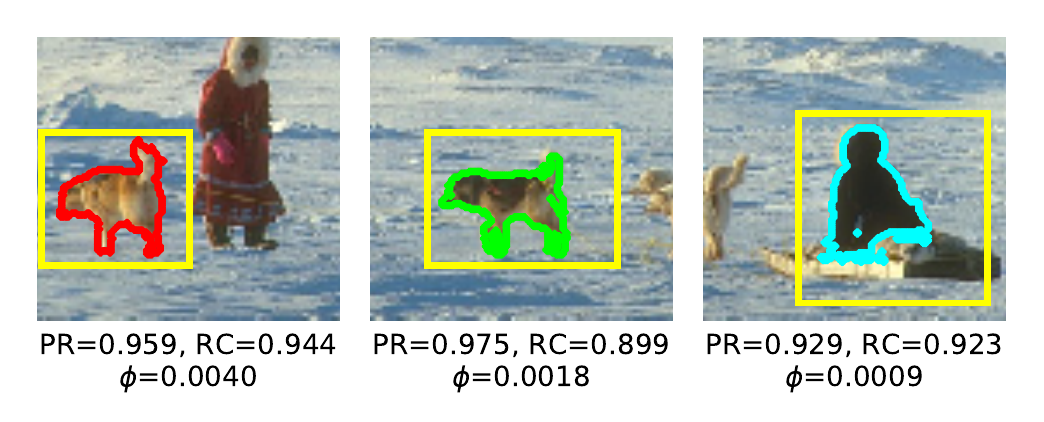}}%

	\subfigure[Greedy Grow \label{fig:grow-shrink-image_i}]{\includegraphics[width=0.5\linewidth]{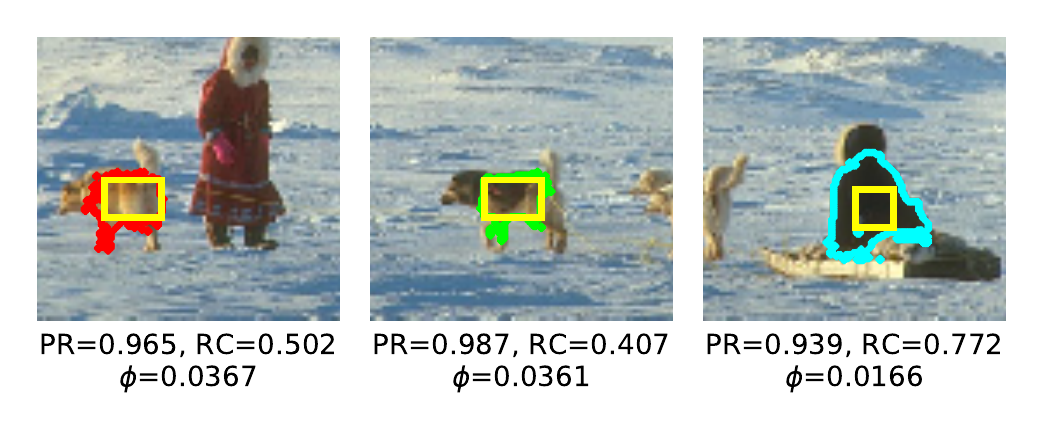}}%
	\subfigure[Greedy Shrink \label{fig:grow-shrink-image_j}]{\includegraphics[width=0.5\linewidth]{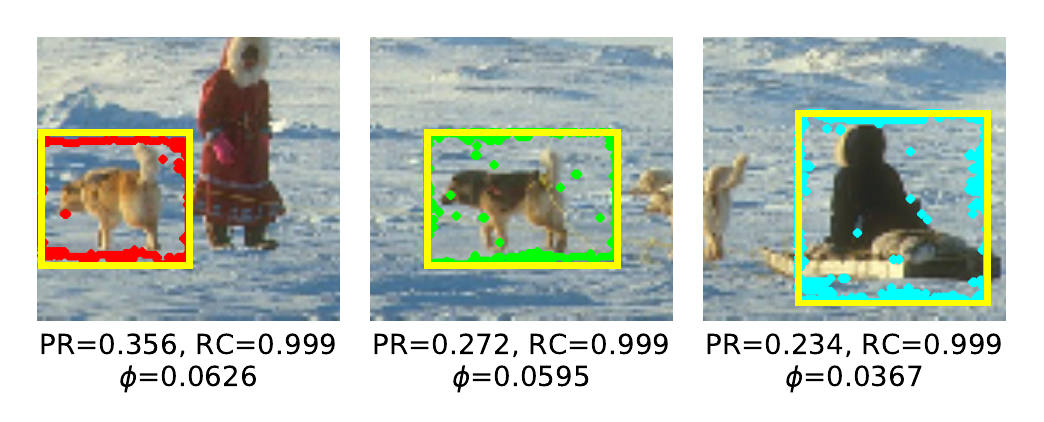}}%

	\vspace{5pt}
\hrule
	\caption{Illustration of finding targets within an image (a) corresponding to the three low-conductance regions shown in (b). The reference sets given to \MQI, \FI, and \LFI are given by the yellow regions, which either need to be grown or shrunk to find the target. For growing, we compare against seeded PageRank, which is a spectral analogue of \FI and \LFI; for shrinking, we compare against a local Fiedler vector, a spectral analogue of MQI, as well as simple greedy approaches for both. The flow-based methods capture the borders nicely and give high recall for growing and high precision for shrinking. Among other things, in this case, \FI grows too large on the right person (c) whereas \LFI (e) captures this target better. RC stands for recall and PR stands for precision.}
	\label{fig:grow-shrink-image}
\end{figure}

Next, we repeat these target set experiments using the Johns Hopkins network, a less visualizable network, for which we see similar results.
The data are a subset of the Facebook100 data-set from~\citet{Traud2011,TMP12}.
The graph is unweighted. It represents ``friendship'' ties and it has $5157$ nodes and $186572$ edges.
This dataset comes along with $6$ features: major, second major, high school, gender, dorm, and year.
We construct two targets by using the features: students with a class year of $2009$ and student with major id $217$.
The visualization shows that major id 217 looks like it will be a fairly good cluster as the graph visualization has moved the bulk away.
However, the conductance of this set is 0.26.
Indeed, neither of these sets is particularly small conductance, which makes the target identification problem much harder than in the images.
Both sets are illustrated in \Cref{fig:JH_a}.

Here, we use a simple breadth first search (BFS) method to generate the input to \MQI and \LFI to mirror the previous experiment with images.
Given a single and arbitrary node in the target cluster, we generate a seed set $R$ by including its neighborhood within 2 hops.
Like the previous examples then, we use \MQI to refine precision and \LFI to boost recall. We repeat this generating and refining procedure for 25 times for distributional statistics. The inputs as well as results from \MQI and \LFI are shown in \Cref{fig:JH_b} to \Cref{fig:JH_h}. The colors show the regions that are excluded (black) or included (red or blue) by each input set or algorithm over 25 trials.

We first summarize in \Cref{fig:JH_b} the increase in precision for \MQI and the increase in both precision and recall for \LFI. The remaining figures illustrate the sets on top of the graph layout showing where the error occurs or regions missed over these 25 trials.
In particular, in \Cref{fig:JH_c} and \Cref{fig:JH_f} we illustrate the BFS input for the target clusters. These inputs include a lot of false positives. In \Cref{fig:JH_d} and \Cref{fig:JH_g} we illustrate the corresponding outputs of \MQI. Note that
\MQI removes the most false positives by contracting the input set. But the outputs of \MQI can have low recall as \MQI can only shrink the input set. In \Cref{fig:JH_e} and \Cref{fig:JH_h} we show that \LFI is able to both contract and expand the input set and it obtains a good approximation to the target cluster.

\begin{figure}
	\centering

	\subfigure[Target sets for Johns Hopkins \label{fig:JH_a}]{\begin{minipage}[t]{0.5\linewidth}\noindent\footnotesize\centering%
			\includegraphics[width=\linewidth]{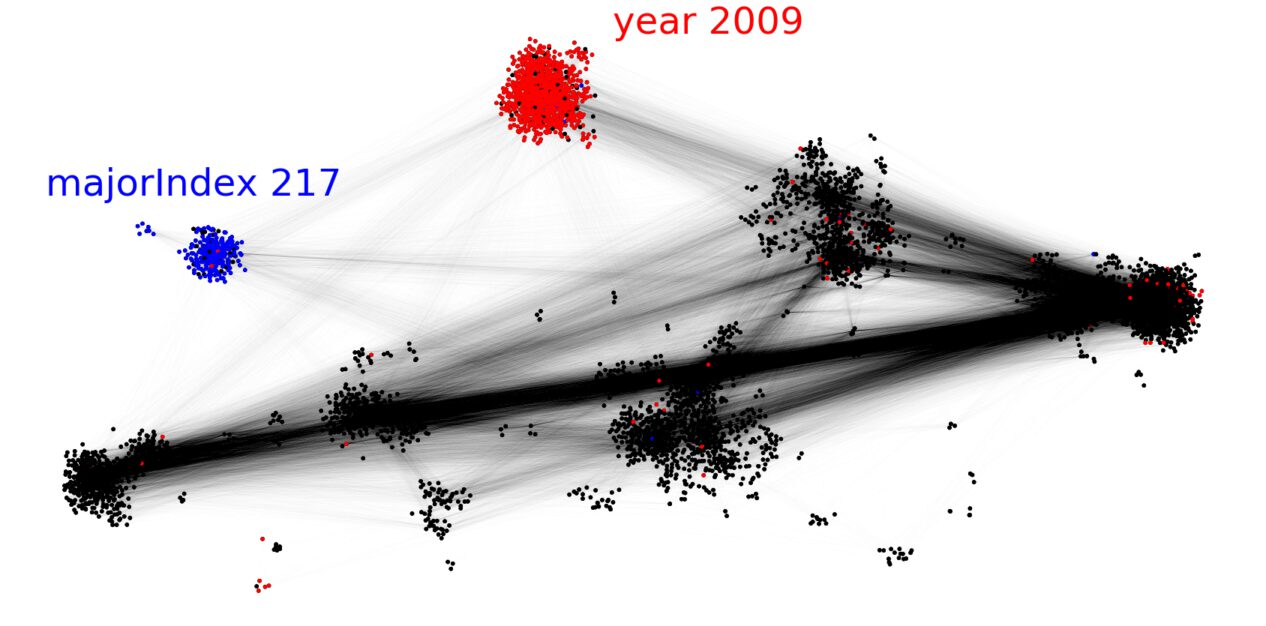}

				\begin{tabularx}{0.8\linewidth}{@{}lXXX@{}}
					\toprule
					Feature & Vol. & Size & Cond.\\
					\midrule
					Major-$217$ & $10696$ &$200$ & $0.26$ \\
					Class-$2009$ & $32454$  &$886$ & $0.19$ \\
					\bottomrule
				\end{tabularx}
		\end{minipage}}%
	\subfigure[Median statistics on input sets as well as \MQI and \LFI results \label{fig:JH_b}]{%
	\begin{minipage}[c]{0.5\linewidth}
			\footnotesize\noindent%
		\begin{tabularx}{\linewidth}{@{}l@{\;}l@{\;}XXXX@{}}
			\toprule
			Target & Set & Size & \rlap{Cond.} & Prec. & Rec. \\
			\midrule
				Major-217 & Input & 1282 & 0.58 & 0.15 & 0.94 \\
				\hfill \MQI & Result & 203 & 0.19 & 0.90 & 0.90 \\
				\hfill \LFI & Result & 218 & 0.18 & 0.88 & 0.95 \\
				\midrule
				Class-2009 & Input & 1129 & 0.52 & 0.35 & 0.51 \\
				\hfill \MQI & Result & 472 & 0.29 & 0.96 & 0.50 \\
				\hfill \LFI & Result & 802 & 0.18 & 0.94 & 0.83 \\
			\bottomrule
		\end{tabularx}%
		\end{minipage}%
		}
	\subfigure[Class-2009 Input \label{fig:JH_c}]{\includegraphics[width=0.33\linewidth]{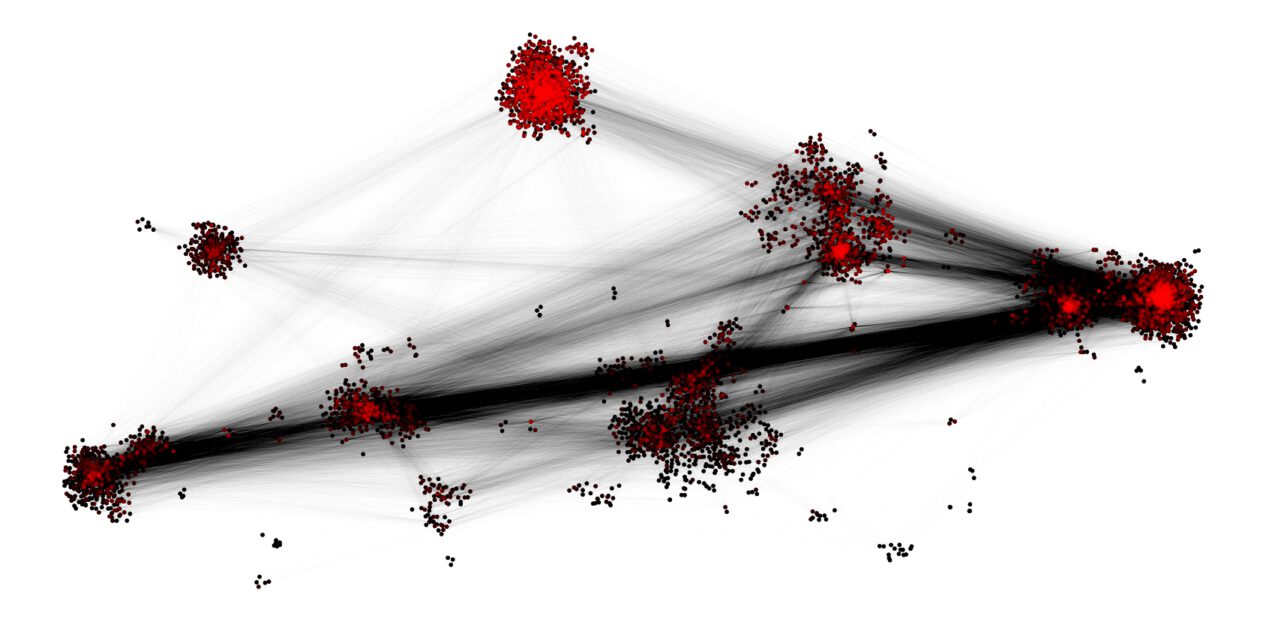}}%
	\subfigure[Class-2009 \mbox{\MQI}\label{fig:JH_d}]{\includegraphics[width=0.33\linewidth]{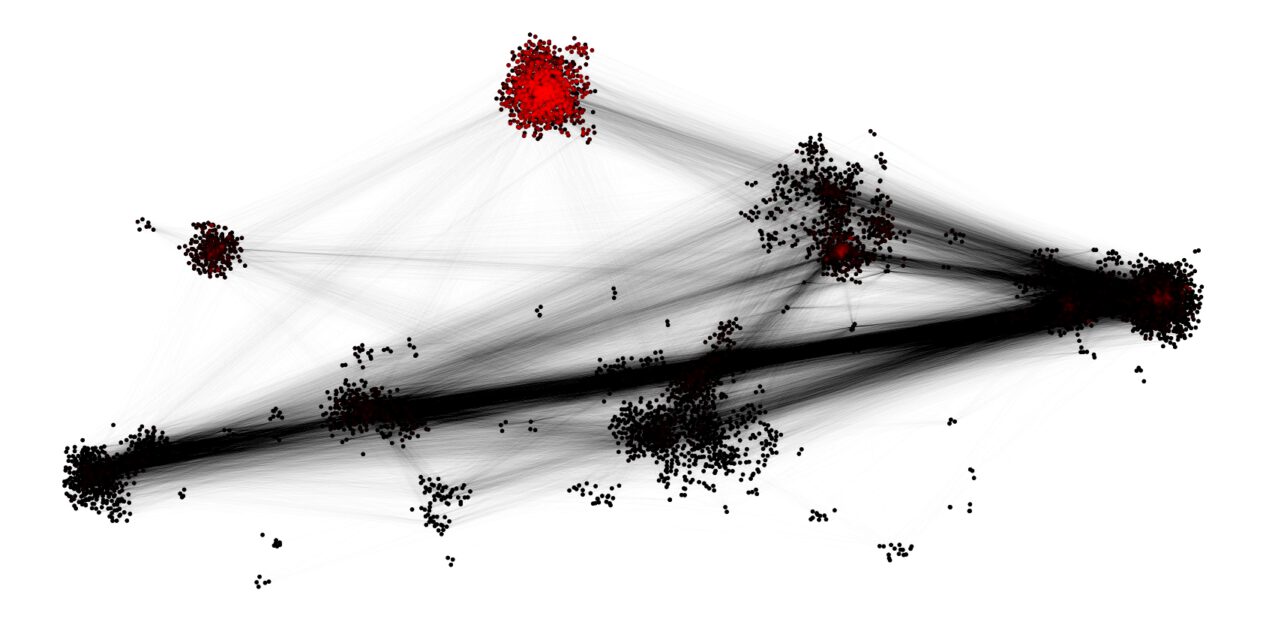}}%
	\subfigure[Class-2009 \mbox{\LFI}\label{fig:JH_e}]{\includegraphics[width=0.33\linewidth]{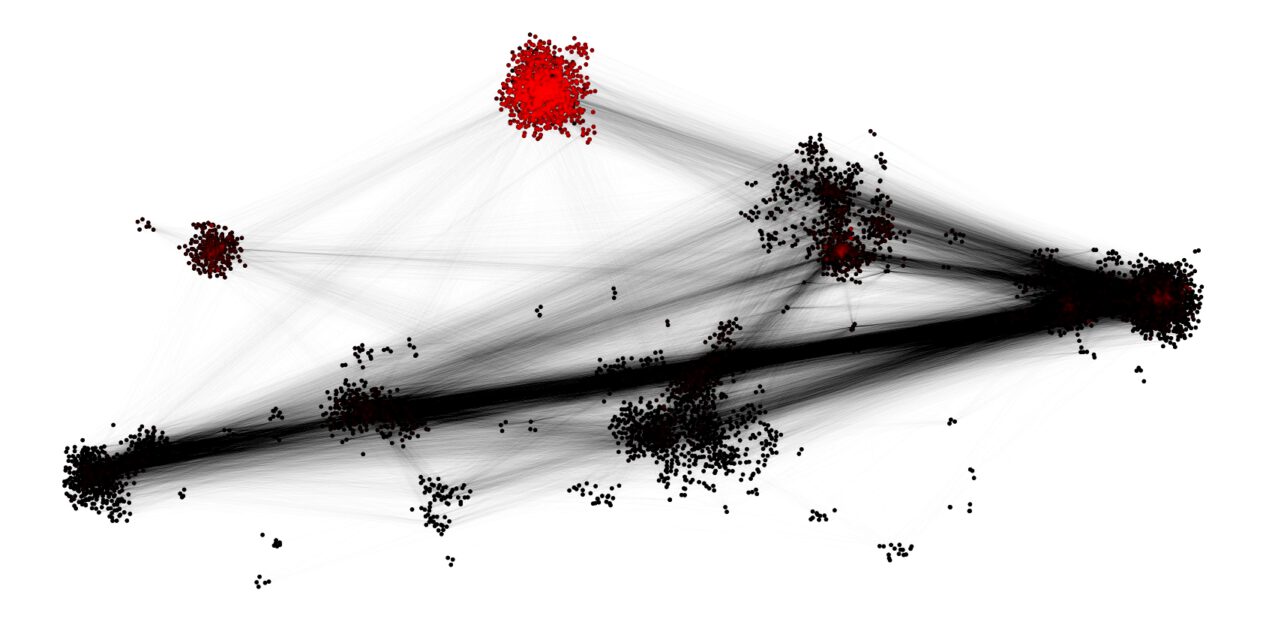}}%

	\subfigure[Major-217 Input\label{fig:JH_f}]{\includegraphics[width=0.3\linewidth]{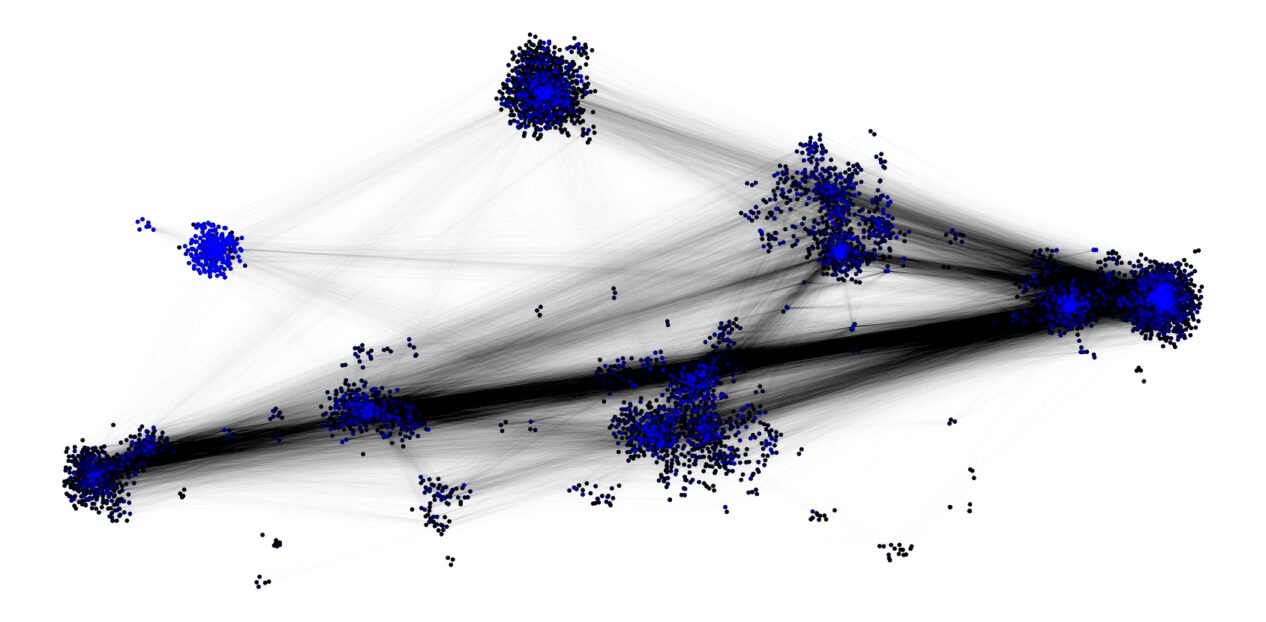}}%
	\subfigure[Major-217 \mbox{\MQI}\label{fig:JH_g}]{\includegraphics[width=0.3\linewidth]{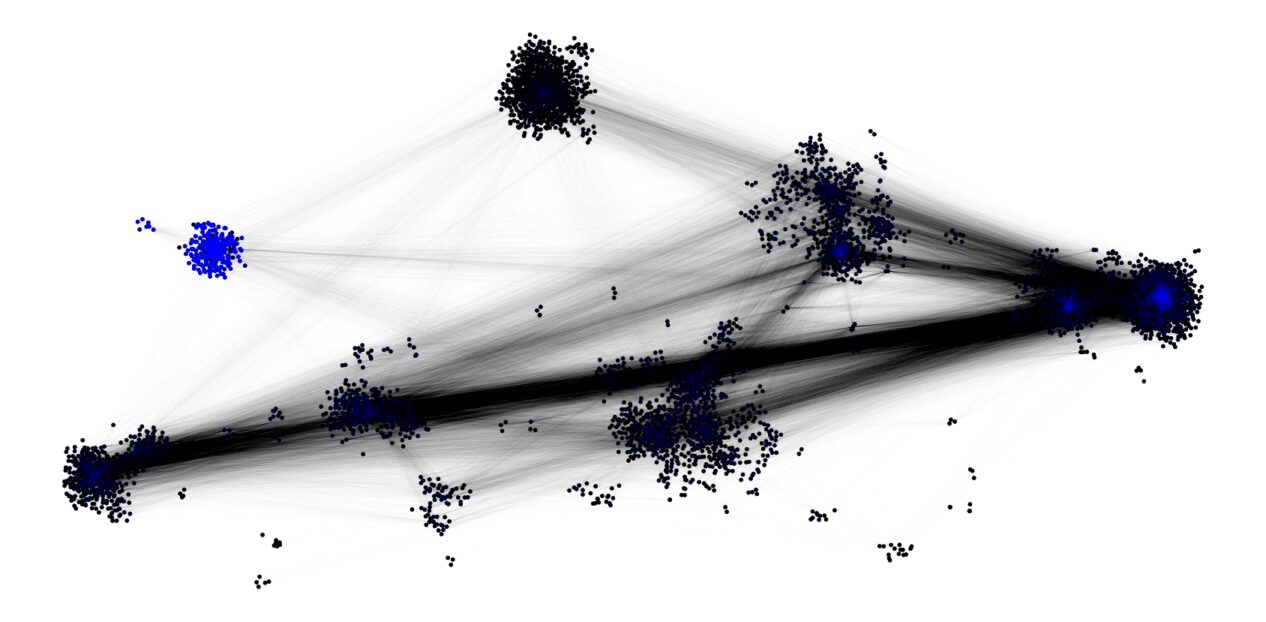}}%
	\subfigure[Major-217 \mbox{\LFI}\label{fig:JH_h}]{\includegraphics[width=0.3\linewidth]{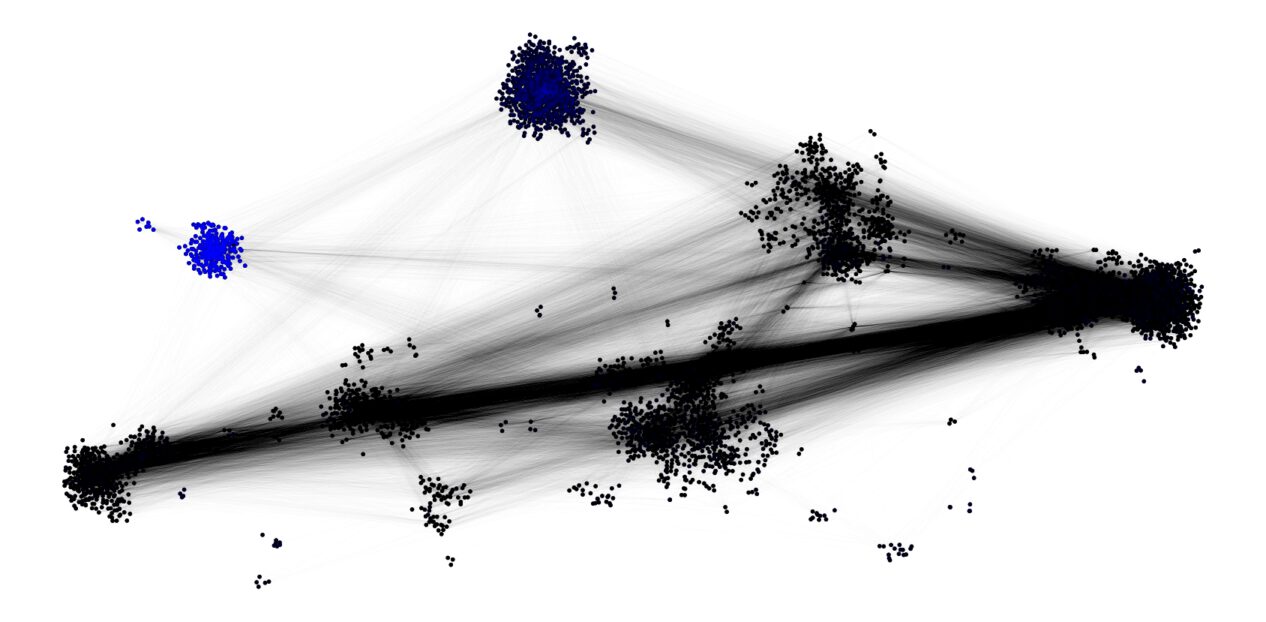}}%

	\caption{The \MQI and \LFI[0.1] algorithms can also find target sets in the Johns Hopkins Facebook social network, even though they are fairly large conductance, which makes them more challenging. The algorithms use as the reference set a simple 2-hop BFS sets with low precisions starting from a random target node. The layout for this graph has been obtained using the force-directed algorithm, which is available from the graph-tool project~\cite{peixoto_graph-tool_2014}. The colors show the regions that are excluded (black) or included (red or blue) by each input set or algorithm over 25 trials.}
	\label{fig:JH}
\end{figure}

This particular set of examples is designed to illustrate algorithm behavior in a simplified and intuitive setting. In both cases we see results that reflect the interaction between the algorithm and the data transformation. For the images, we create a graph designed to correlate with segments, then run an algorithm designed to improve and find additional structure. For the social network, we simply take the data as given, without any attempt to change it to make algorithms perform better. However, with more specific tasks in mind, we would suggest altering the graph data in light of the targeted use as in~\cite{Gleich-2015-robustifying,Benson-2016-motif-spectral,Peel-2017-relational}.

\subsection{Using flow-based algorithms for semi-supervised learning}
\label{sec:semi-supervised}
\newcommand{\asideedges}{\aside{aside:nonedge}{If edges from the graph do not represent a high likelihood of a shared attributed, then there are scenarios where the graph data itself can be transformed such that this becomes the case~\cite{Peel-2017-relational}.}}

Semi-supervised learning on graphs is the problem of inferring the value of a hidden label on all vertices, when given a few vertices with known labels. Algorithms for this task need to assume that the graph edges represent a high likelihood of sharing a label.
For instance, one of the datasets we will study is the MNIST data. 
Each node in the MNIST 
\siamonly{\asideedges}
\tronly{\asideedges}
graph represents an image of a handwritten digit, and edges connect two images based
 on a nearest neighbor relationships. 
The idea is that images that show similar digits should share many edges. Hence, knowing a few node labels would allow one to infer the hidden labels.  Note that this is a related, but distinct, problem to the target set identification problem (\Cref{sec:target-set}). The major difference is that we need to handle multiple values of a label and produce a value of the label for all vertices.

An early paper on this topic suggested that MinCut and flow-based approaches should be useful~\cite{Blum-2001-mincuts}.
In our experiments, we compare flow-based algorithms \LFI and \FI with seeded PageRank, and we find that the flow-based algorithms are more sensitive to an increase in the size of the set of known labels.
In the coming experiments, this manifests as an increase in the recall while keeping the precision fixed.
For these experiments, \MQI is not a useful strategy, as the purpose is to \emph{grow and generalize} from a fixed and known set of labels to the rest of the graph.

There are three datasets we use to evaluate the algorithm for semi-supervised learning: a synthetic stochastic block model, the MNIST digits data, and a citation~network.
\begin{itemize}
	\item \textbf{SBM}. SBM is a synthetic stochastic block model network. It consists of $6000$ nodes in three classes, where each class has $2000$ nodes. The probability of having a link between nodes in the same class is $0.005$ while the probability of having a link between nodes in different classes is $0.0005$. The one we use in the experiment has $36102$ links. By our construction, the edges preferentially indicate class similarity.

	\item \textbf{MNIST}. MNIST is a $k$-NN (nearest neighbor) network~\cite{Lecun-1998-mnist}.
              The raw data consists of $60000$ images. Each image represents a handwritten sample of one arabic digit. Thus, there are 10 total classes. In the graph, each image is represented by a single node and then connected to its 10 nearest neighbors based on Euclidean distance between the images when represented as vectors of greyscale pixels. We assume that edges indicate class similarity.

	\item \textbf{PubMed}. PubMed is a citation network~\cite{namata2012query}. It consists of $19717$ scientific publications about diabetes with 44338 citations links. Each article is labeled with one of three types. By our assumption, articles about one type of diabetes cite others about the same type more often.
\end{itemize}

The experiment goes as follows: for each class, we randomly select a small subset of nodes, and we fix the labels of these nodes as known. We then run a spectral method or flow method where this set of nodes is the reference. We vary the number of labeled nodes included from 0.5\% to 15\% of the class size. For each fixed number of labeled nodes, we repeat this $30$ times to get a distribution of precision, recall, and $F1$ scores (where $F1$ is the harmonic mean of precision and recall), and we represent an aggregate view of this.  For the flow methods, the output is a binary vector with $1$ suggesting the node belongs to the class of reference nodes. Thus, it's possible that some nodes are classified into multiple classes, while some other nodes remain unclassified. We consider the first case as false positives and the second case as false negatives when computing precision and recall. For the spectral method, we use the real-valued solution vector to uniquely assign a node to a class.

The results are in \Cref{fig:ssl-all} and show that the flow-based methods have uniformly high precision. As the set of known labels increases, the recall increases, yielding a higher overall $F1$ score. Furthermore, the regularization in \LFI[0.1] causes the set sizes to be smaller than \FI, which manifests as a decrease in recall compared with \FI. In terms of why  the flow-based algorithms have low-recall with small groups of known labels sizes, note \Cref{lem:monotonicFracProg}, which requires that the cut and denominator to reduce at each step. This makes it challenging for the algorithms and objectives to produce high recall when started with small sets unless the sets are exceptionally well-separated.

\begin{figure}
	\includegraphics[width=\linewidth]{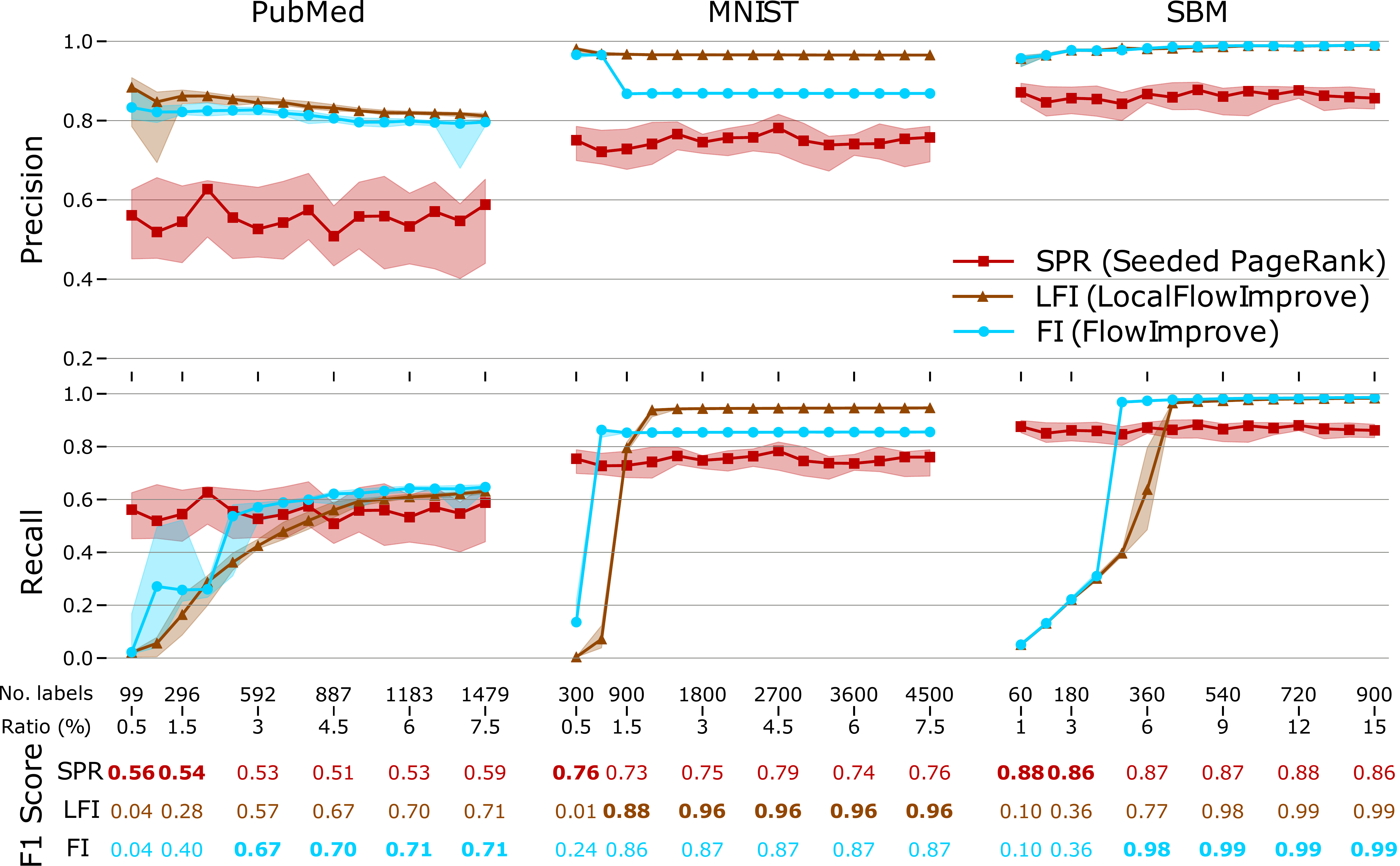}
	\caption{The horizontal axis shows the number of true labels included in the seeds and the plots are aligned with the tables so you can read off the F1 score as well as the associated precision and recall for each choice. These results on semi-supervised learning show that the flow-based methods \LFI[0.1] and \FI are more sensitive to the number of known true labels included in the reference seed sets compared with seeded PageRank.
                }
	\label{fig:ssl-all}
\end{figure}

\subsection{Improving thousands of clusters on large scale data}
\label{subsec:findingThousandOfClusters}

In practice, it is often the case that one might want to explore thousands of clusters in a single large graph.
For example, this is a common task in many computational topology pipelines~\cite{Lum-2013-topology}.
Another example that requires thousands of clusters is computing the network community profile~\cite{LLDM2009,LLDM08_communities_CONF,LLM10_communities_CONF}, which shows a conductance-vs-size result for a large number of sets, as a characteristic feature of a network.
In this section, we will explore the runtime of the flow-based solvers on two small biological networks and two large social networks, where nodes are individuals and edges represent declared friendship relationships.
\begin{itemize}
        \item \textbf{sfld}, has $232$ nodes and $15570$ edges in this graph~\citep{brown2006gold}. %
        \item \textbf{ppi}, has $1096$ nodes and $13221$ edges~\citep{pagel2004mips}. %
	\item \textbf{orkut}, has $3,072,441$ nodes and $117,185,083$ edges~\cite{Mislove-2007-measurement}. This data set can be accessed via~\citet{snapnets}.
	\item \textbf{livej}, has $4,847,571$ nodes and $68,993,773$ edges~\cite{Mislove-2007-measurement}. This data set can be accessed via~\citet{snapnets}.
\end{itemize}

The goal will be to enable studies such as those discussed above using instead the flow-based algorithms as a subroutine on these graphs.
To generate seed sets to refine, we use seeded PageRank.
Each input set is the result of a seeded PageRank algorithm on a random node with a variety of settings to generate sets between a few nodes and up to around 10,000 nodes. For each resulting set, we then run the \MQI, \LFI[0.9], \LFI[0.6], and \LFI[0.3] improvement methods. We use our code~\citep{git:localgraphclustering}, which has a Python interface and methods that are implemented using C++, for all of these experiments and runtimes. Our environment has a dual CPU Intel E5-2670 (8 cores) CPU with 128 GB RAM.
We parallelize over individual runs of the seeded PageRank and flow methods using the Python Multiprocessing module using a common shared graph structure.  Note that each individual run is independent.

In this way, we are able to explore tens of thousands of clusters in around 30-40 minutes, as we demonstrate in \Cref{tab:runtime}.
There, we present running times for producing the seeded PageRank sets and then refining it with the flow-based methods.
Note that the fastest method is \MQI.
It is even faster than the seeded PageRank method that generates the input sets. This is because \MQI only explores the input subcluster, while LocalFlowImprove reaches outside of the input seed set of nodes. Also, note the dependence of the runtime for \LFI on the parameter $\delta$. The larger the parameter $\delta$ for \LFI, the smaller the part of the graph that it explores outside of the input set of nodes.
This property is also captured in \Cref{tab:runtime} by the running time of \LFI.

\begin{table}[t] %
	\caption{Running times in seconds for generating and improving clusters on small scale biological networks and large-scale social networks. The input cluster to the flow-based improvement methods is the output of seeded PageRank. It takes around 20 minutes to generate the input clusters for large scale social networks. Running the flow-based improvement algorithms takes around the same amount of time, except for \LFI[0.3] on LiveJournal, which takes roughly 30 minutes. The time measurements reflect the pleasingly parallel computation of results for all clusters on a 16-core machine.
                }
                \centering
		\begin{tabularx}{\linewidth}{@{}lXXXXXXXX@{}}
			\toprule
			& & & & Time & (s) & \\
			\cmidrule(l){5-9}
			graph & nodes  & edges & \rlap{\mbox{clusters}} found & seeded \PRs & \MQI  & \LFI[0.3] & \LFI[0.6] & \LFI[0.9] \\
			\midrule
			sfld & 232 & 16k & $342$ & $18$& $0.5$& $1.7$ & $1.6$ & $1.5$\\
			ppi & 1096 & 13k & $1199$ & $46$& $1$& $2.6$ & $2.5$ & $2.3$\\ \addlinespace
			orkut & 3M & 117M & $13799$ & $1130$& $171$& $838$ & $701$ & $628$\\
			livej & 4.8M & 69M & $31622$& $1057$ & $105$ & $1940$ &  $1326$  & $1094$ \\
			\bottomrule
		\end{tabularx}%
		\label{tab:runtime}%
\end{table}%

\subsection{Using flow-based methods for local coordinates}
\label{subsec:flowlocalstructure}

\newcommand{\asideembed}{\aside{embed-weakness}{A bigger issue with spectral embeddings is that they often produce useless results for many large networks; see~\citet{Lang2005-spectral-weaknesses}. Here, we use networks where these techniques yield interesting results.}}

A common use case for global (but also local~\citep{LBM16_TR,lawlor2016mapping}) spectral methods on graphs and networks is using eigenvector information in order to define coordinates 
\tronly{\asideembed}
for the vertices of a graph. 
This is often called a spectral embedding or eigenvector embedding~\cite{Hall-1970}; it may use two or more eigenvectors directly or with simple transformations of them in order to define coordinates for each node~\cite{ng2001spectral}. The final choice is typically made for aesthetic reasons. 
An example of a spectral embedding for the US highway network is shown in \Cref{fig:usroads-embed}.
One of the problems with such global embeddings is that they often squash interesting and relevant regions of the network into filamentary structures.
For instance, notice that both the extreme pieces of this embedding compress massive and interesting population centers of the US on the east and west coast. Alternative eigenvectors show different but similar structure. 
A related problem is that they smooth out interesting features, making them difficult to use.

\siamonly{\asideembed}

\begin{figure}
	\centering
	\includegraphics[width=0.6\linewidth]{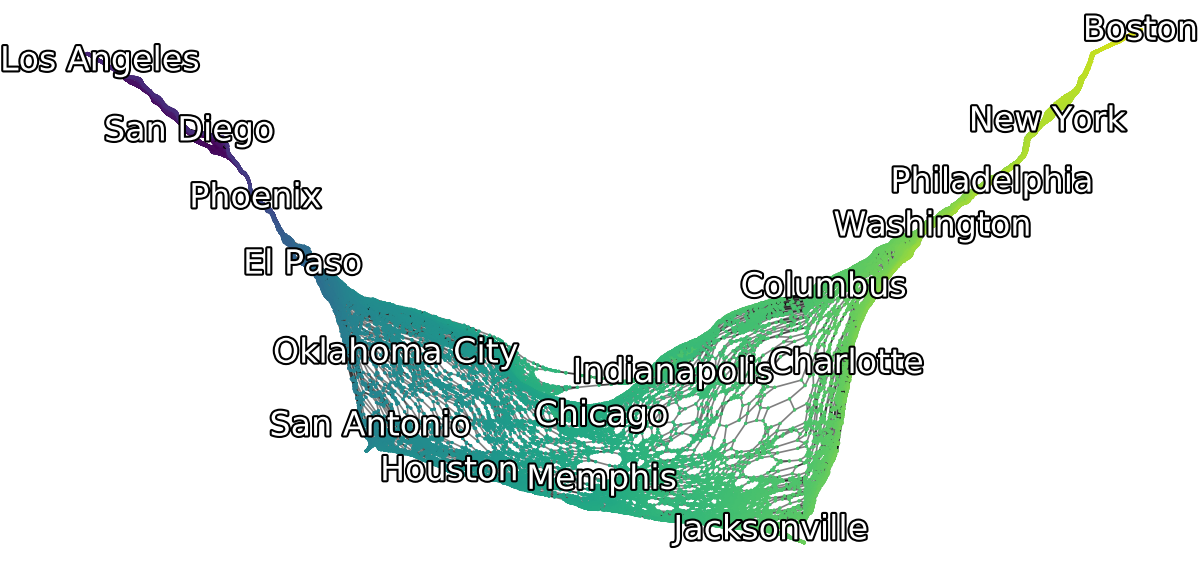}
	\caption{A spectral embedding of the US Highway Network corresponding to the first and second non-trivial eigenvectors of the Laplacian matrix. The embedded locations of major cities are  labeled as well. Node color is determined by the true longitude of a node, which shows that the first eigenvector of the Laplacian correlates with an east-west split of the network. This global embedding, however, compresses major regions of the northeastern US (Washington, New York, Boston) as well as the Western US (Los Angeles, San Diego, Phoenix). }
	\label{fig:usroads-embed}
\end{figure}

Semi-supervised eigenvectors are one way to address this aspect of global spectral embeddings~\cite{Hansen-2014-semi-supervised-eigenvectors,LBM16_TR,lawlor2016mapping}.
These seek orthogonal vectors that minimize a constrained Rayleigh quotient. One challenge in using related ideas to study \emph{flow-based} computation is that the solution of flow-problem is fundamentally discrete and binary. That is, a spectral solution produces a real-valued vector whose entries, e.g., for seeded PageRank, can be interpreted as probabilities. We can thus meaningfully discuss and interpret sub-optimal, orthogonal solutions. Flow computations only gives 0 or 1 values, where orthogonality implies disjoint sets.

In this section, we investigate how flow-based methods can be used to compute real-valued coordinates that can show different types of structure within data compared with spectral methods.
In the interest of space, we are going to be entirely procedural with our description, justification is provided in \Cref{sxn:app-replication}.

Given a reference set $R$, we randomly choose $N$ subsets (we use 500-2500 subsets) of $R$ with exactly $k$ entries; for each subset we add all nodes within distance $d$ and call the resulting sets called $R_1, \ldots, R_N$. These serve as inputs to the flow algorithms. For each subset, we compute the result of a flow-based improvement algorithm, which gives us sets $S_i$. For each $S_i$, we form an indicator vector over the vertices, $\vx_i$,  where the entry is $1$ if the vector is in the set and $0$ otherwise. We assemble these vectors as columns of a matrix $\mX$, and we use the coordinates of the dominant two left singular vectors as flow-based coordinates. This procedure is given as an algorithm in \Cref{algo:flow-coords}. Note also that this procedure can be performed with spectral algorithms as well. (See the appendix for additional details.)

\begin{algorithm}[t]
	\caption{The local flow-based algorithm to generate flow-based coordinates. }
	\label{algo:flow-coords}
	\begin{algorithmic}[1]
		\REQUIRE A graph $G$, a set $R$ and parameters
		\begin{itemize}
			\item $N$ : the number of sets to sample
			\item $k$: the size of each subset
			\item $d$: the expansion distance
			\item $c$: the dimension of the final embedding
			\item \texttt{improve}: a cluster improvement algorithm
		\end{itemize}
		\ENSURE An embedding of the graph into $c$ coordinates for each node

\STATE Let $n$ be the number of vertices.
\STATE Allocate $\mX$, an $n$-by-$N$ matrix of zeros.
\FOR{$i$ in $1$ to $n$}
\STATE Let $T$ be a sample of $k$ entries from $R$ at random without replacement
\STATE Let $R_i$ be the set of $T$ and also all vertices within distance $d$ from $T$
\STATE Let $S_i$ be the set that results from \texttt{improve}$(G,R_i)$
\STATE Set $X[v,i] = 1$ for all $v \in S_i$
\ENDFOR
\STATE Compute the rank-$c$ truncated SVD of $\mX$ and let $\mU$ be the left singular vectors.
\STATE \textbf{Return} $\mU$, each row gives the $c$ coordinates for a node
	\end{algorithmic}
\end{algorithm}

The results of using \Cref{algo:flow-coords} (see parameters in \Cref{sxn:app-replication-tables}) to generate local coordinates for a set of vertices on the west coast of the United States Highway Map is shown in \Cref{fig:usroads-local-embed}.
The set of vertices shown in \Cref{fig:usroads-embed-seed} is in a region where the spectral embedding compresses substantial information.
This region is shown on a map in  \Cref{fig:usroads-embed-map}, and it includes major population centers on the west coast.
In \Cref{fig:usroads-embed-local-spectral}, we show the result of a local spectral embedding that uses seeded PageRank in \Cref{algo:flow-coords}, along with a few small changes that are discussed in our reproducibility section.
(Here, we note that these changes do not change the character of our findings, they simply make the spectral embedding look better.)
In the spectral embedding, the region shows two key areas: 1.~Seattle, Portland, and San Francisco and 2.~Los Angeles, San Diego, and Phoenix.
In \Cref{fig:usroads-embed-local-flow}, we show the result of the local flow-based embedding that uses \LFI[0.1] as the algorithm.
This embedding clearly and distinctly highlights major population centers, and it does so in a way that is clearly qualitatively different from spectral methods.

\begin{figure}
	\centering
	\subfigure[\label{fig:usroads-embed-seed} The subset of nodes (left) from the spectral embedding of the US Highway Network used to compute local embeddings]{\includegraphics[width=0.4\linewidth]{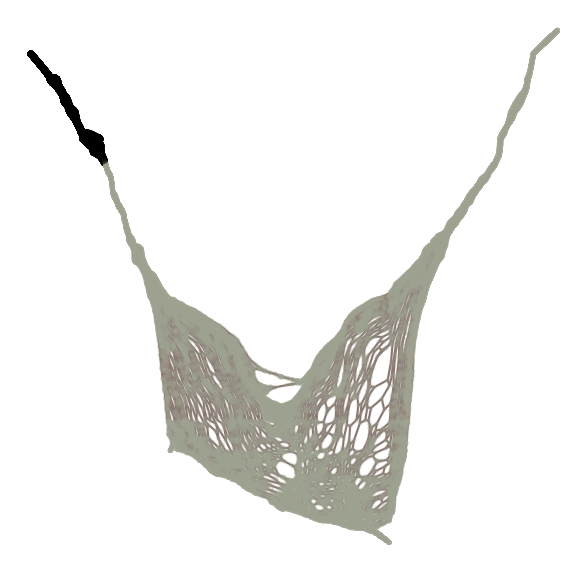}}
	\hspace{2em}
	\subfigure[\label{fig:usroads-embed-map} The same subset shown on a map with major cities labeled]{\includegraphics[width=0.4\linewidth]{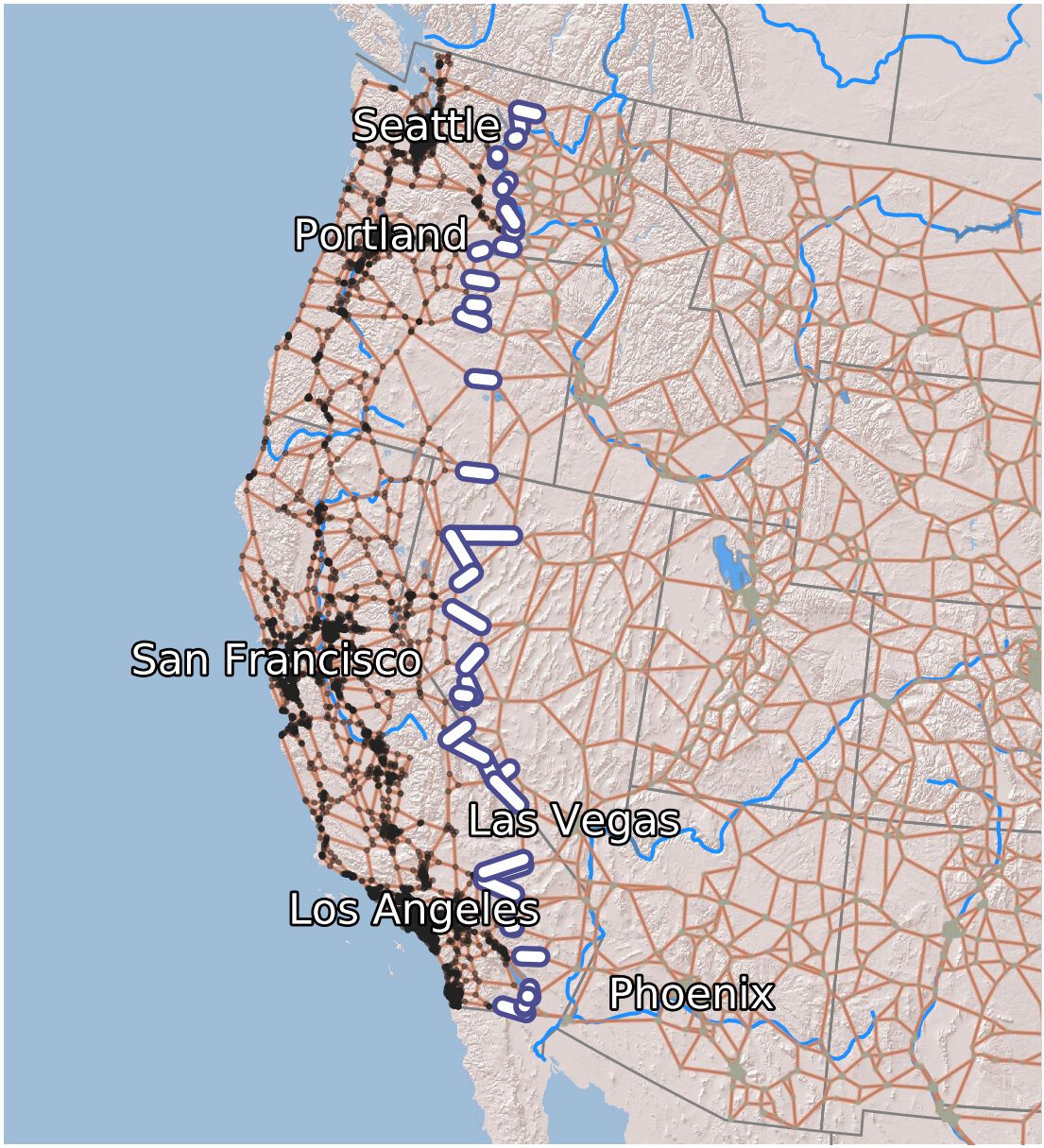}}

	\subfigure[\label{fig:usroads-embed-local-spectral} The local spectral embedding of the US west coast cities]{\includegraphics[width=0.4\linewidth]{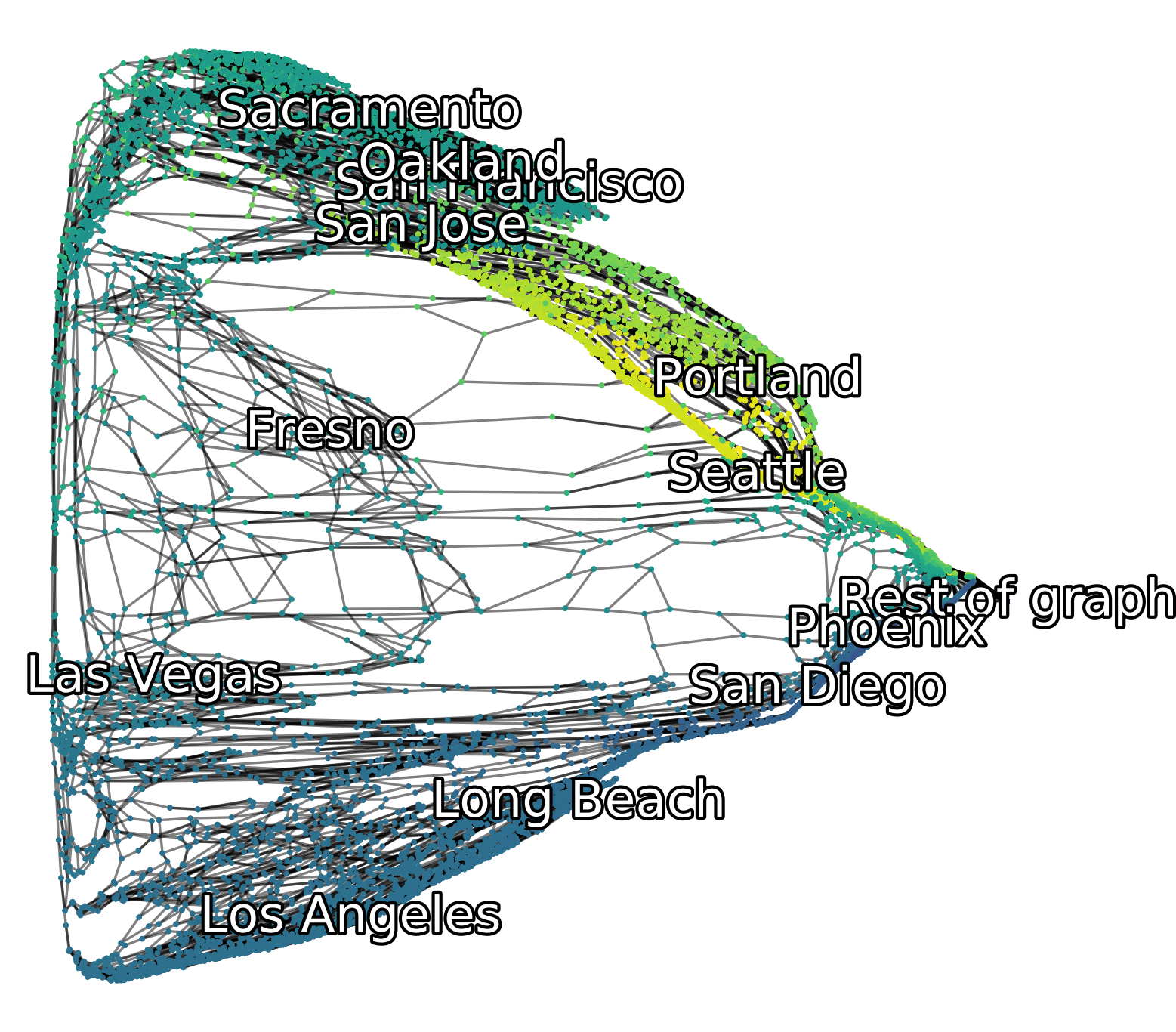}}
	\hspace{2em}
	\subfigure[The local flow embedding of the US west coast cities \label{fig:usroads-embed-local-flow}]{\includegraphics[width=0.4\linewidth]{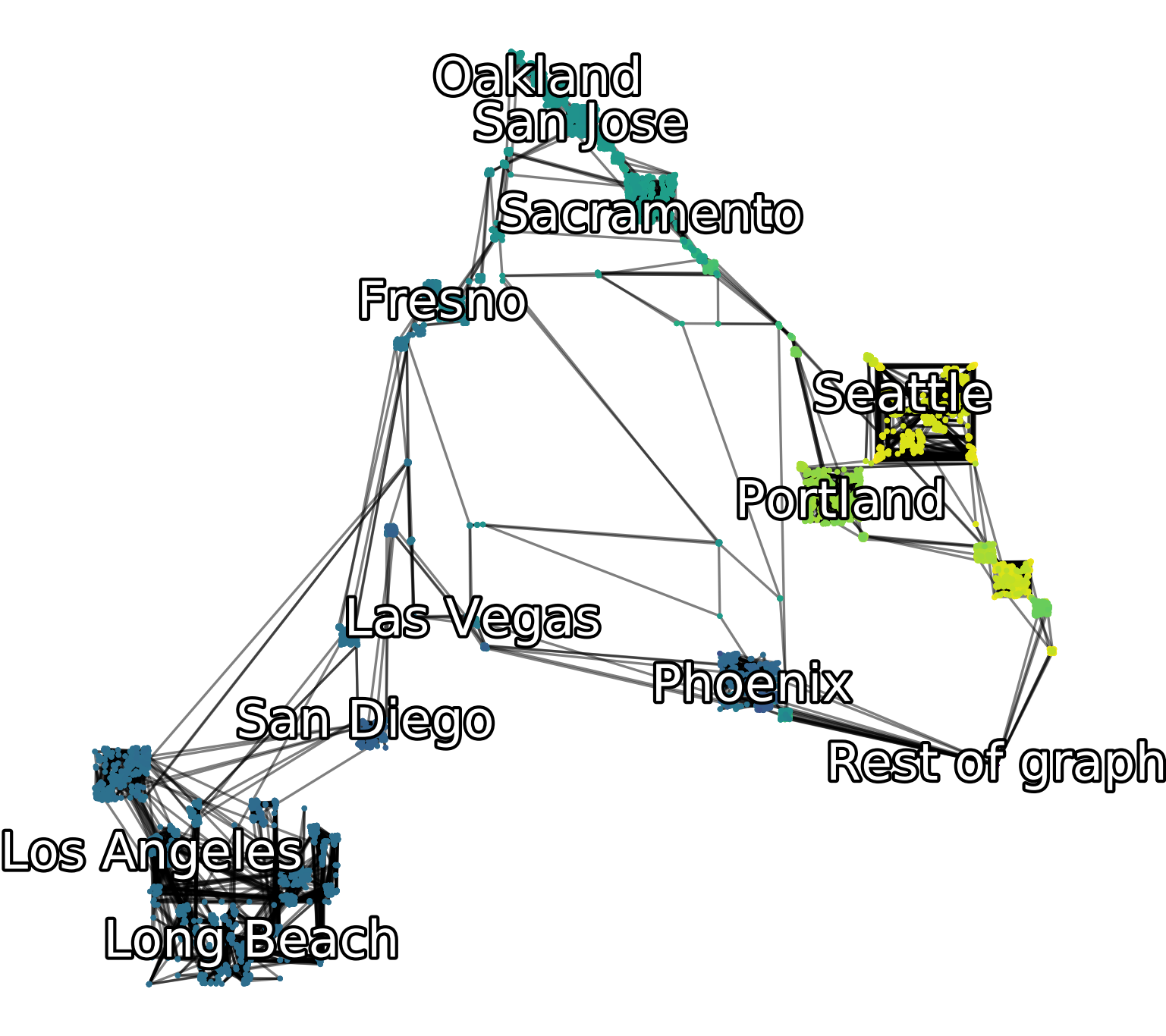}}
	\caption{We select a subset of 7143 nodes that are compressed in the spectral embedding of the US Highway network (shown in red and blue in figures (a) and (b)) that represent the majority of major cities on the west coast. Note that interior cities such as Phoenix and Las Vegas are not included in the set. In (c) and (d) we show the results of running the pipeline from \Cref{algo:flow-coords} to generate local spectral and local flow based embeddings into two dimensions. The color of a node is determined by its north-south latitude. Note that both include Phoenix and Las Vegas. 
	The local flow embedding clearly and distinctly delineates clusters corresponding to major population centers whereas the local spectral embedding shows a smooth view with only two major regions: 1.~Northern California to Seattle and 2.~Southern California to Phoenix and Las Vegas. }
	\label{fig:usroads-local-embed}
\end{figure}

We repeat this analysis on the Main Galaxy Sample (MGS) dataset to highlight the local structure in a particularly dense region of the spectral embedding that was used for \Cref{fig:astro}. The seed region we use is shown in \Cref{fig:astro-embed-seed} and has 201,252 vertices, which represents almost half the total graph. We use \Cref{algo:flow-coords} (see parameters in the \Cref{sxn:app-replication}) again to get local spectral~(\Cref{fig:astro-embed-spectral}) and local flow embeddings~(\Cref{fig:astro-embed-flow}). Again, we find the the local flow embedding shows considerable substructure that is useful for future analysis.

\begin{figure}
		\setlength{\fboxsep}{0pt}
		\subfigure[The seed region \label{fig:astro-embed-seed}]{\colorbox{shadecolor}{\includegraphics[width=0.32\linewidth,trim={12in 12in 0.5in 0.5in},clip]{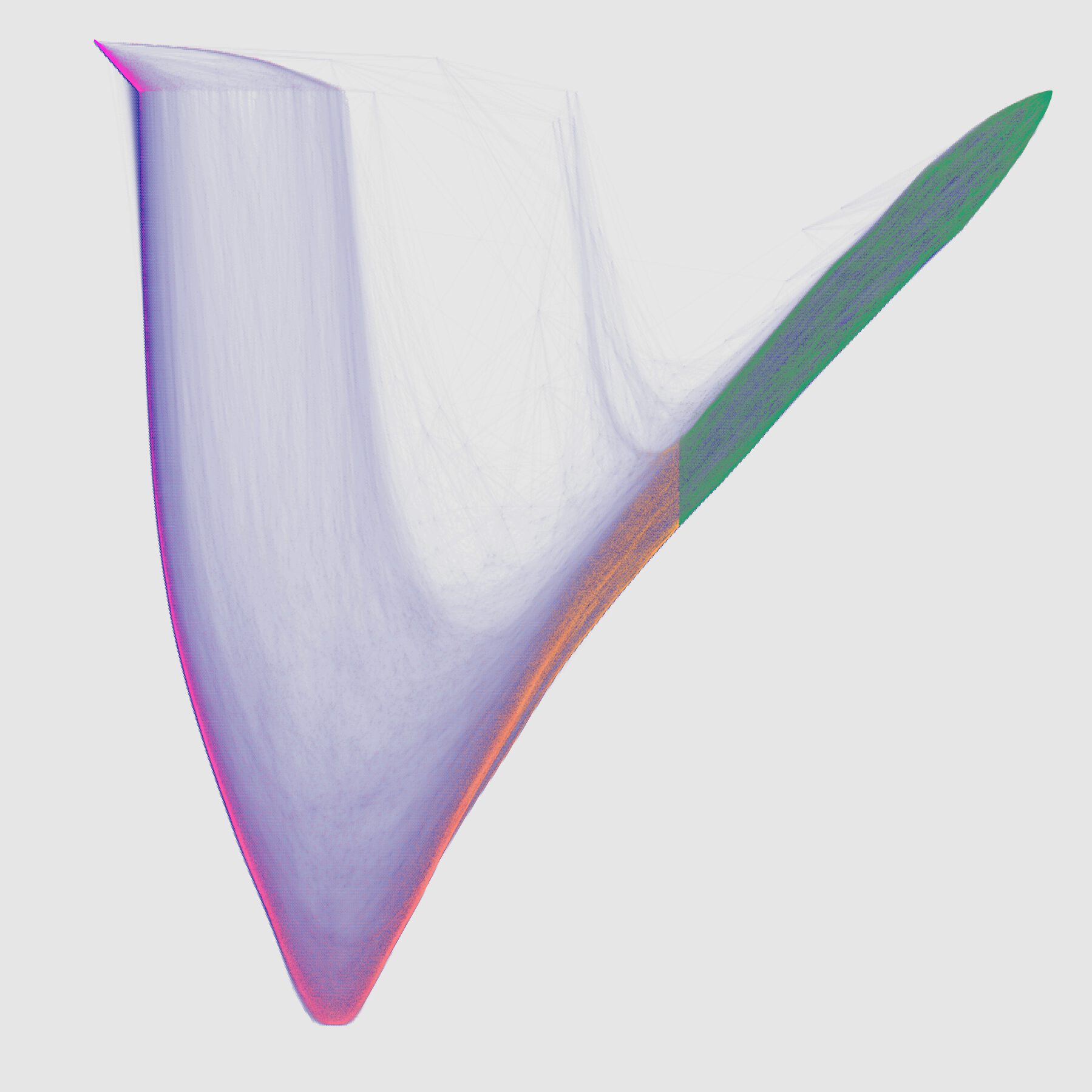}}}
		\subfigure[Local spectral  embedding \label{fig:astro-embed-spectral}]{\colorbox{shadecolor}{\includegraphics[width=0.32\linewidth]{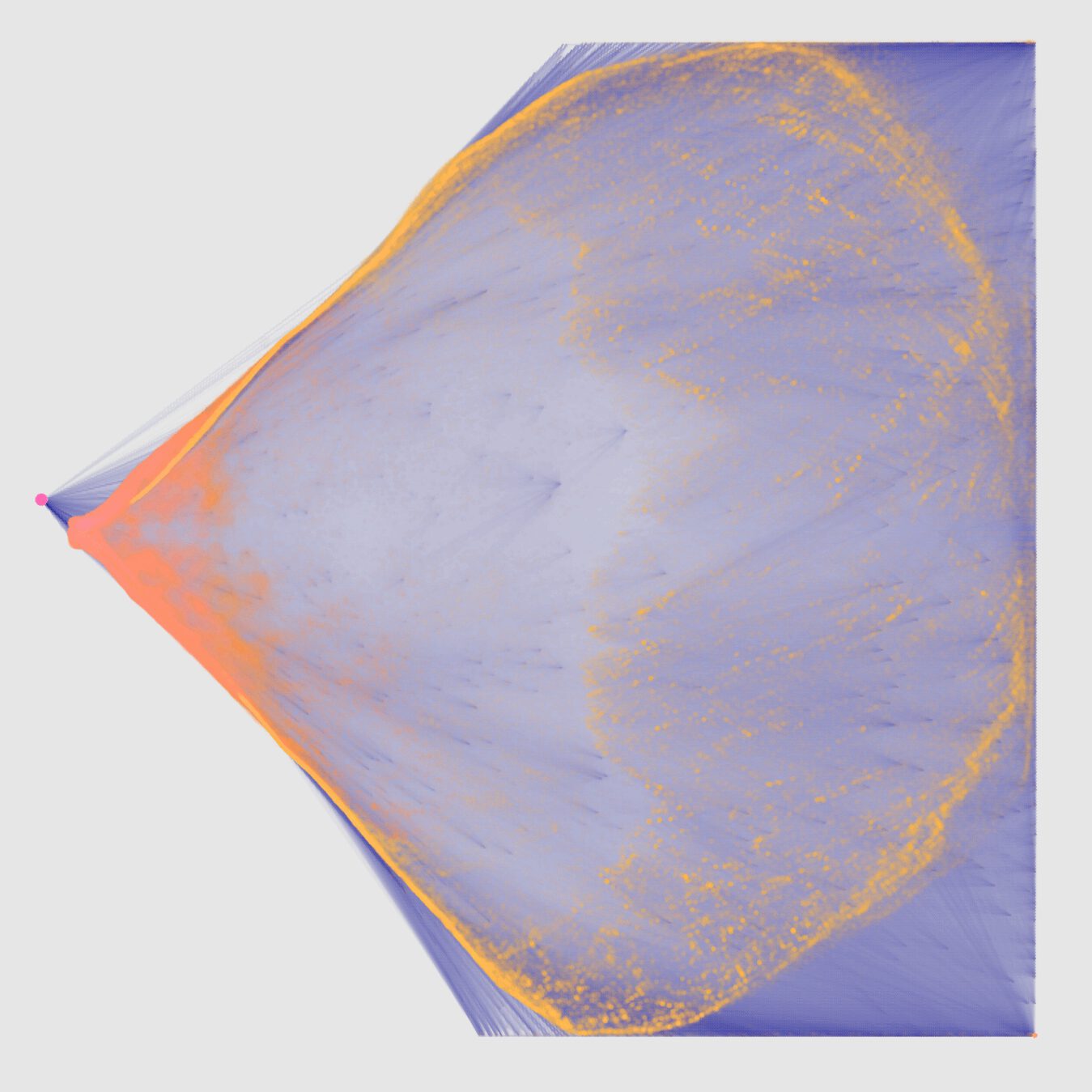}}}
		\subfigure[Local flow  embedding \label{fig:astro-embed-flow}]{\colorbox{shadecolor}{\includegraphics[width=0.32\linewidth]{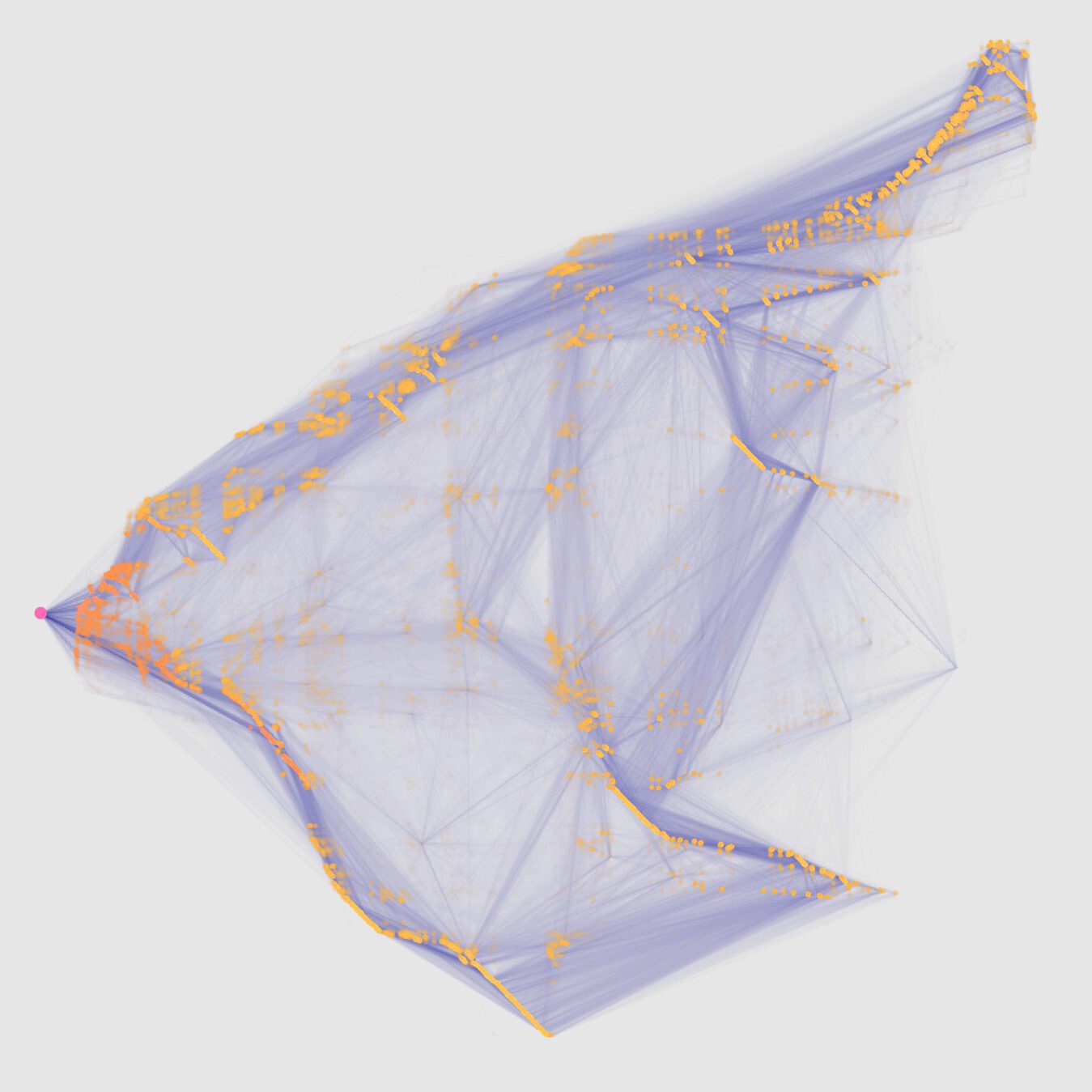}}}

\caption{Local spectral and local flow embeddings of the large, 201,252 node, seed region -- shown in green in (a) -- that is compressed in the global spectral embedding from \Cref{fig:astro}. In (b), the local spectral shows the nodes colored with the same color as in \Cref{fig:astro}. Nodes that were not touched by the local embedding are shown with the big node on the right hand side. In (c), the local flow embedding with the same color scheme and same big node on the right hand side. Note that the spectral embedding does not show any clear sub-structure besides a top-bottom split. In contrast, the flow embedding shows a number of pockets of structure indicative of small conductance subsets. }
\label{fig:astro-embed}
\end{figure}

As a simple validation that this substructure is real, we use the 2d embedding coordinates as input to a $k$-means clustering procedure on both the local spectral and local flow coordinates. For each cluster that results from this procedure, we compute its conductance. Histograms of conductance values are shown in \Cref{fig:astro-cond-hist} for $k=50$ and $k=100$. Both of these histograms show consistently smaller conductance values for the flow-based embedding.

\begin{figure}
	\subfigure[$k=50$]{\includegraphics[width=0.5\linewidth]{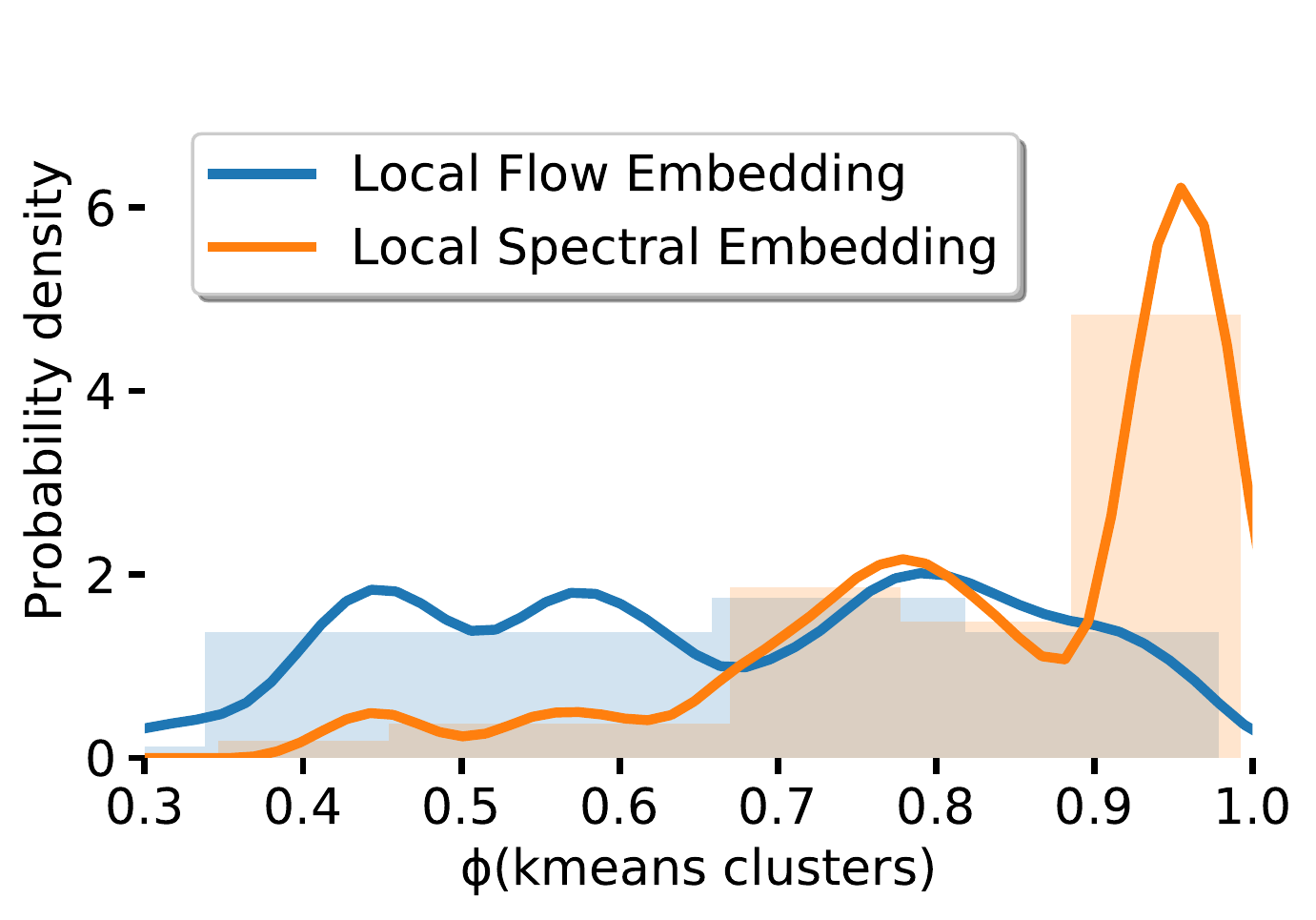}}%
	\subfigure[$k=100$]{\includegraphics[width=0.5\linewidth]{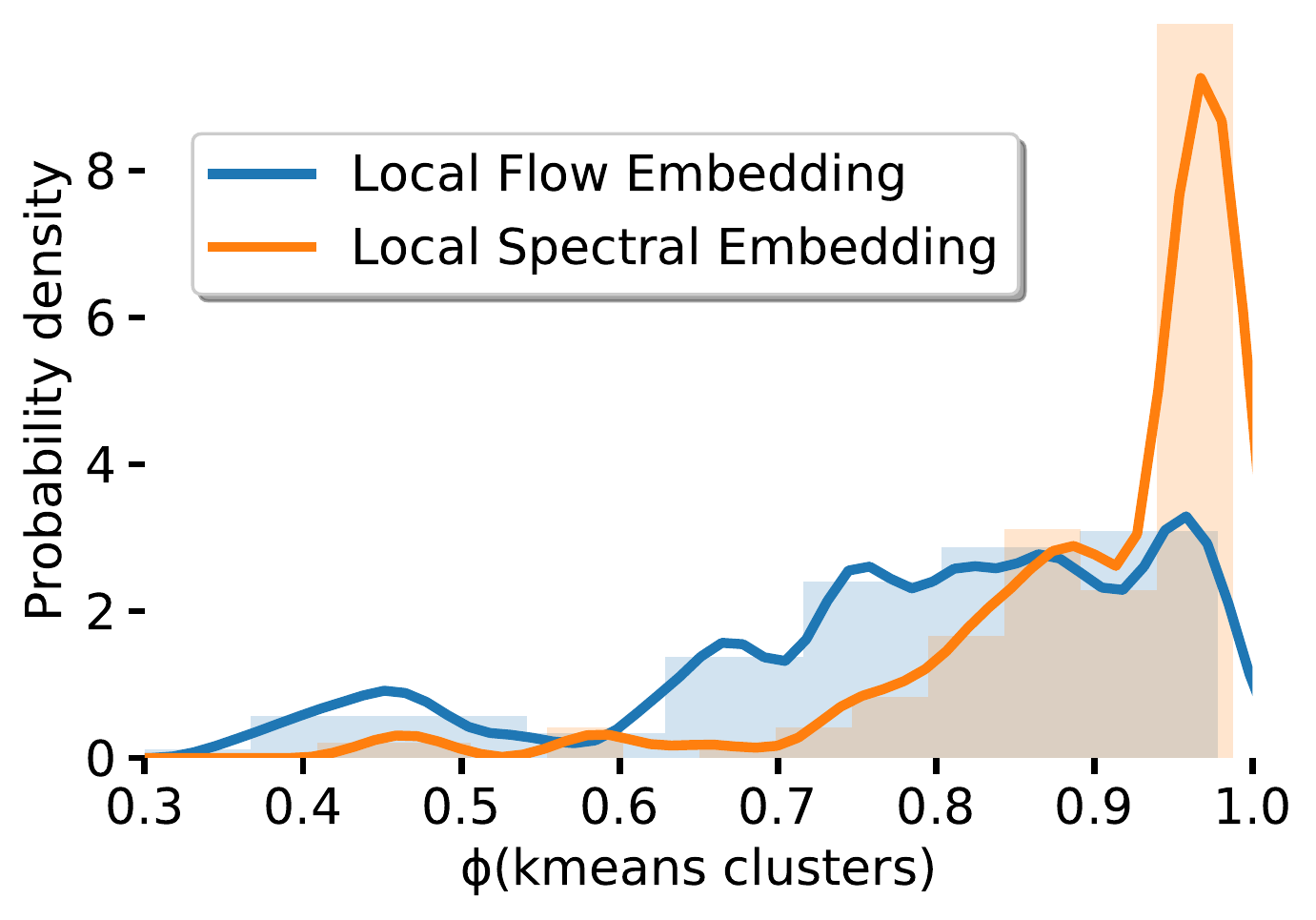}}

\caption{A histogram of cluster conductance scores that come from using $k$-means on the two-dimensional local spectral and local flow embeddings from \Cref{fig:astro-embed}. These show that the flow embedding produces clusters with smaller conductance, and they support the intuition from \Cref{fig:astro-embed} that the additional structure suggested by the flow embedding reflects meaningful sub-structure within the data.  }
\label{fig:astro-cond-hist}
\end{figure}

\section{Discussion and Conclusion}
\label{sec:conclusion}

Our goal with this survey is to highlight the advantages and wide utility of flow-based algorithms for improving clusters.
The literature on these methods is much smaller compared with other graph computation methodologies, despite attractive theoretical benefits.
For example, global spectral methods based on random walks or eigenvectors are ubiquitous in computer science, machine learning, and statistics. 
Here, we illustrated similar possibilities for flow based methods.
We have also shown that these local flow based improvement algorithms can scale to very large graphs, often returning outputs without even touching most of the graph, and that many popular machine learning and data analysis uses of flow based methods can be applied to them. 
This is the motivation behind our software package where these algorithms have been implemented~\cite{git:localgraphclustering}. 
An alternative implementation is available in Julia~\cite{Veldt-2019-localflow}. 
These results and methods open the door for novel analyses of very large graphs.

As an example of these types of novel analyses, note that the fractional ratio $\delta_k$ inside MQI, FlowImprove, and LocalFlowImprove (\Cref{algo:mqi,algo:flowImprove,algo:localflowimprove}) can be interpreted as a ratio between the numerator and denominator. This enables one to search for a value of $\delta$ that would correspond to a given solution. See the ideas in \citet{Veldt-2019-resolution} for how to use search methods to choose $\delta$ for a specific application of clustering. 

In our explanation of the theory behind the methods, we often encountered decisions where we could have made more general, albeit more complex, statements. Our guiding principle was to make it easy to appreciate the opportunities for these methods. As an example where there are more general results, note that much of the theory of this survey holds where $\vol(R) \ge \vol(\Rbar)$ for the seed. For instance, the MQI, FI, and LFI procedures are all well defined algorithms in this scenarios although our theory statements list the explicit requirement that $\vol(R) \le \vol(\Rbar)$. What happens in these scenarios is that some of the details of runtime and other aspects change. In terms of another generalization, the methods could have been stated in terms of a general volume function as noted in \Cref{sec:qcut-general}. Again, however, this setting becomes more complex to state for non-integer volume functions and a number of other subtle issues. In these cases, we sought for the explanations that would make the underlying issues clear and focused on conductance in order to do that.

There are a number of interesting directions that are worth further exploring. First, in the theory from this survey, the binary search or bisection-based search methods have superior worst-case time. However, in practice, these methods are rarely used. For instance our own implementations always use the Dinkelbach greedy framework. This is because this strategy commonly terminates in just a few iterations. This was noted in both the MQI and FlowImprove papers~\cite{LR04,AL08_SODA}, yet there is still no theoretically satisfying explanation. To provide some data, for the LocalFlowImprove experiments in \Cref{fig:cluster-improvement}, we never needed to evaluate more than 10 values of the fractional ratio to find the optimal solution in Dinkelbach's algorithm. As evidence that this effect is real, note that \Cref{lem:monotonicFracProg} actually shows that $\cut(R)$ is a bound on the number of iterations for Dinkelbach's algorithm. But for weighted graphs, this becomes $\cut(R)/\mu$, where $\mu$ is the minimum spacing between elements (think of the floating point machine value $\eps$). A specific case where this type of insight would be helpful is in terms of weighted graphs with non-negative weights. Dinkelbach's algorithm does not appear to be much slower on such problems, yet the worst-case theory bound is extremely bad and depends on the minimum spacing between elements. 

Another direction is a set of algorithms that span the divide between the fractional program and the \stMinCut problem. For instance, it not necessary to completely solve the flow problem to optimality (until the last step) as all that is needed is a result that there is a better solution available. This offers a wide space to deploy recent advances in Laplacian-based solvers to handle the problem -- especially because the electrical flow-based solution largely correspond to PageRank problems.  It seems optimistic, but reasonable, to expect good solutions in time that is more like a random walk or spectral algorithm. 

Finally, another direction for future research is to study these algorithms in hypergraphs~\cite{veldt2020localized} and other types of higher-order structures. This was a part of early work on graph-cuts in images that showed that problems where hyperedges had at most three nodes could be solved exactly~\cite{Kolmogorov-2004-graph-cuts}. More recently, hypergraphs have been used to identify refined structure in networks based on motifs~\cite{Li-2017-hypergraph}. 

In closing, our hope is that this survey and the associated LocalGraphClustering package~\citep{git:localgraphclustering} helps to make these powerful and useful methods---both basic flow-based analysis of smaller graphs, but especially local flow-based analysis of very large data graphs---more common in the future.

	\addcontentsline{toc}{part}{Part IV. Replicability Appendices and References}
	\section*{\large Part IV. Replicability Appendices and References}
	\siamonly{~\newline}
	
	\appendix

\section{Replicability Details for Figures and Tables}
\label{sxn:app-replication}
\label{sxn:app-replication-tables}

In the interest of reproducibility and replicability, we provide additional details on the methods underlying the figures. To replicate these experiments, see our publicly-available code~\citep{git:localgraphclustering_experiments}. All of the seeded PageRank examples in this survey use an $\ell_1$-regularized PageRank method~\citep{FKSCM2017}. We use $\rho$ to denote the regularization parameter in $\ell_1$-regularized PageRank, $\alpha$ to denote the teleportation parameter in $\ell_1$-regularized PageRank and $\delta$ to denote the parameter of LocalFlowImprove.

The implementations we have use Dinkelbach's method, \Cref{algo:fractionalprog}, for the fractional programming problem and Dinic's algorithm for exact solutions of the weighted MaxFlow problems at each step. Put another way, for MQI, we use \Cref{algo:mqi} and Dinic's algorithm to solve the MaxFlow problems. For FlowImprove and LocalFlowImprove, we use the Dinkelbach variation on \Cref{algo:localflowimprove} with Dinic's algorithm used to compute blocking flows in \Cref{algo:simplelocal}. We use the same implementation for LocalFlowImprove and FlowImprove and simply set $\delta = 0$ for FlowImprove. Using Dinkelbach's method and Dinic's MaxFlow has a worse runtime in theory, but better performance in practice.  The implementations always return the smallest connected set that achieves the minimum of the objective functions. They also always return a set with less than half the total volume of the graph.

\paragraph{\Cref{fig:example-sbm}}
We use the implementation of the Louvain algorithm in~\citep{git:community}.
We use our own code to generate the SBM. %
The code for this experiment is in the notebook \texttt{sbm\_demo.ipynb} in the subdirectory~\path{ssl} available in~\citep{git:localgraphclustering_experiments}.

\paragraph{\Cref{fig:example-geograph}}
This is a geometric-like stochastic block model ``hybrid''. A short description of the data-generation procedure follows but the code is the precise description. Create $g$ groups of $n$ points. Each group is assigned the same random 2d spatial coordinate from a standard mean 0, variance 1 normal distribution. But each node within a group is also perturbed by a mean 0 random normally distributed amount with variance $\sigma$.  Add $p$ additional points with normally distributed 
$\rho$ (a ``center'' group). These determine the coordinates of all the nodes. Now, add edges to $k$ nearest neighbors
and also within radius $\epsilon$. For this experiment, we set $g=25$, $n=100$, $\sigma=0.05$, $p=2000$, $\rho=5$, $k=5$ and $\epsilon=0.06$.
The code is in the Jupyter notebook \texttt{Geograph-Intro.ipynb}~\citep{git:localgraphclustering_experiments}.

\paragraph{\Cref{fig:example-astronaut}}
The image can be downloaded from~\citet{scikit-image}.
The image is turned into a graph using the procedure described in \Cref{sec:image-to-graph}. In particular, we set $r=80$, $\sigma_p^2=\mathcal{O}(10^2)$ and $\sigma_c^2=l/10$, where $l$ is the maximum between the row and column length of the image.
The code for this experiment is in the Jupyter notebook \texttt{astronaut.ipynb} in the subdirectory~\path{usroads} available in~\citep{git:localgraphclustering_experiments}.

\paragraph{\Cref{fig:mqi-for-images}}
The original image is 100 by 100 pixels with a 13 pixel wide by 77 pixel tall vertical strip and a 77 pixel wide by 25 pixel tall horizontal strip that intersect in a 13 by 25 pixel region. This setup and intersection produces a rotated T-like shape centered in the 100 by 100 pixel grid. 
To generate the noisy, blurred figure, we used a Moffat kernel~\cite{Moffat-kernel} that arises from stellar photography (parameters $\alpha=1.5, \beta=1.2$ and length scale $5$) and added uniform $[-0.1,0.1]$ noise (roughly 38\% of max blurred value 0.261) for each pixel before scaling by $1/0.3$ and clamping to $[0,1]$ range. This stellar photography scenario was chosen to simulate a binary image reconstruction scenario.  Let $\vf$ be noisy, blurry image with values in the range $[0,1]$. Let $G$ be the grid graph associated with the grid underlying the image $\vf$. Make sure to read the \Cref{chap:mqi} before reading the details of the reconstruction algorithm. Then we construct the following augmented graph: connect $s$ to $i \in V$ with weight $\delta f_i d_i$ (where $d$ is the degree); connect $t$ to $i \in V$ where $f_i = 0$ with weight $\infty$. We show the mincut solution $S$ for the two values of $\delta$ explained in the problem. This corresponds to a minimization problem similar to \eqref{eq:mqi-orig-problem}, namely $\minimize \cut(S) - \delta \sum_{i \in S}{d_i f_i} \text{ subject to } S \subseteq \{ i \mid f_i > 0 \} $, which uses a \emph{biased} notion of volume $\nu(S) = \sum_{i \in S} d_i f_i$ (as in \Cref{sec:qcut-general}). In this case, if $\delta$ is 0.04, then we can no longer continue improving the result and we end up with the convex set. For $\delta = 0.11$, we have the rough reconstruction of the original shape. For our own purposes, we used a Julia implementation~\cite{Veldt-2019-localflow} of the flow code that is available in the \path{mqi-images} subdirectory in~\citep{git:localgraphclustering_experiments}.

\paragraph{\Cref{fig:usroads-results}}
All details are given in the main text of the survey. The code is in the Jupyter notebook \texttt{usroads-figures.ipynb} in the subdirectory~\path{usroads} available in~\citep{git:localgraphclustering_experiments}.

\paragraph{\Cref{tab:usroads}} This table provides additional details for the results of \Cref{fig:usroads-results}. The code for this experiment is in the Jupyter notebook \texttt{usroads-figures.ipynb} in the subdirectory~\path{usroads} available in~\citep{git:localgraphclustering_experiments}.

\paragraph{\Cref{fig:astro}}
This dataset has been obtained from~\citet{lawlor2016mapping}.
It is a $k=32$-nearest neighbor graph constructed on the Main Galaxy Sample (MGS) in SDSS Data Release 7. Each galaxy is captured in a 3841-band spectral profile. Each spectra is normalized based on the median signal over 520 bands selected in~\citet{lawlor2016mapping}. Since the results are sensitive to this set and it is not available elsewhere, the indices of the bands were
\begin{quote} \tiny
856, 857, 858, 859, 860, 861, 862, 863, 864, 865, 866, 867, 868, 869, 870, 871, 872, 873, 874, 875, 876, 877, 878, 879, 880, 881, 882, 883, 884, 885, 886, 887, 888, 889, 890, 891, 892, 893, 894, 895, 896, 897, 898, 899, 900, 901, 902, 903, 904, 905, 906, 907, 908, 909, 910, 911, 912, 913, 914, 915, 916, 917, 918, 919, 920, 921, 922, 923, 924, 925, 926, 927, 928, 929, 930, 931, 932, 933, 934, 935, 936, 937, 938, 939, 940, 941, 942, 943, 944, 945, 946, 947, 948, 949, 950, 951, 952, 953, 954, 955, 956, 957, 1251, 1252, 1253, 1254, 1255, 1256, 1257, 1258, 1259, 1260, 1261, 1262, 1263, 1264, 1265, 1266, 1267, 1268, 1269, 1270, 1271, 1272, 1273, 1274, 1275, 1276, 1277, 1278, 1279, 1280, 1281, 1282, 1283, 1284, 1285, 1286, 1287, 1288, 1289, 1290, 1291, 1292, 1293, 1294, 1295, 1296, 1297, 1298, 1299, 1300, 1301, 1302, 1303, 1304, 1305, 1306, 1307, 1308, 1309, 1310, 1311, 1312, 1313, 1314, 1315, 1316, 1317, 1318, 1319, 1320, 1321, 1322, 1323, 1324, 1325, 1326, 1327, 1328, 1329, 1330, 1331, 1332, 1333, 1334, 1335, 1336, 1337, 1338, 1339, 1340, 1341, 1342, 1343, 1344, 1345, 1346, 1347, 1348, 1349, 1350, 1351, 1352, 1353, 1354, 1355, 1356, 1357, 1358, 1359, 1360, 1361, 1362, 1363, 1364, 1365, 1366, 1367, 1368, 1369, 1370, 1371, 1372, 1373, 1374, 1375, 1376, 1377, 1378, 1379, 1380, 1381, 1382, 1383, 1384, 1385, 1386, 1387, 1388, 1389, 1390, 1391, 1392, 1393, 1394, 1395, 1396, 1397, 1398, 1399, 1400, 1401, 1402, 1403, 1404, 1405, 1406, 1407, 1408, 1409, 1410, 1411, 1412, 1413, 1414, 1415, 1416, 1417, 1418, 1419, 1420, 1421, 1422, 1423, 1424, 1425, 1426, 1427, 1428, 1429, 1430, 1431, 1432, 1433, 1434, 1435, 1948, 1949, 1950, 1951, 1952, 1953, 1954, 1955, 1956, 1957, 1958, 1959, 1960, 1961, 1962, 1963, 1964, 1965, 1966, 1967, 1968, 1969, 1970, 1971, 1972, 1973, 1974, 1975, 1976, 1977, 1978, 1979, 1980, 1981, 1982, 1983, 1984, 1985, 1986, 1987, 1988, 1989, 1990, 1991, 1992, 1993, 1994, 1995, 1996, 1997, 1998, 1999, 2000, 2001, 2002, 2003, 2004, 2005, 2006, 2007, 2008, 2009, 2010, 2011, 2012, 2013, 2014, 2015, 2016, 2017, 2018, 2019, 2020, 2021, 2022, 2023, 2024, 2025, 2026, 2027, 2106, 2107, 2108, 2109, 2110, 2111, 2112, 2113, 2114, 2115, 2116, 2117, 2118, 2119, 2120, 2121, 2122, 2123, 2124, 2125, 2126, 2127, 2128, 2129, 2130, 2131, 2132, 2133, 2134, 2135, 2136, 2137, 2138, 2139, 2140, 2141, 2142, 2143, 2144, 2145, 2146, 2147, 2148, 2149, 2150, 2151, 2152, 2153, 2154, 2155, 2156, 2157, 2158, 2159, 2160, 2161, 2162, 2163, 2164, 2165, 2166, 2167, 2168, 2169, 2170, 2171, 2172, 2173, 2174, 2175, 2176, 2177, 2178, 2179, 2180, 2181, 2182, 2183, 2184, 2185, 2186, 2187, 2188, 2189, 2190, 2191, 2192, 2193, 2194, 2195, 2196, 2197, 2198, 2199, 2200, 2201, 2202, 2203, 2204, 2205, 2206, 2207, 2208, 2209, 2210, 2211, 2212, 2213, 2214, 2215, 2216, 2217, 2218, 2219, 2220, 2221, 2222, 2223, 2224, 2225, 2226, 2227, 2228, 2229, 2230, 2231, 2232, 2233, 2234, 2235, 2236, 2237, 2238, 2239, 2240, 2241, 2242, 2243, 2244, 2245, 2246, 2247, 2248, 2249, 2250, 2251, 2252, 2253, 2254, 2255, 2256, 2257, 2258
\end{quote}
We create a node for each galaxy and connect vertices if either is within the $16$ closest vertices to the other based on a Euclidean after this median normalization.
The graph is then weighted proportional to this distance and the distance to the the 8th nearest neighbor based on a $k$-nearest neighbor tuning procedure in manifold learning. (The results for spectral embeddings are somewhat sensitive to this procedure.) Formally, let $\rho_i$ be the distance to the $8$th nearest neighbor (or $\infty$ if all of these distances are 0). We add a weighted undirected edge based on node $i$ to node $j$ with distance $d_{i,j}$ as $W_{i,j} = \exp(-(d_{i,j}/\rho_i))$. If $i$ and $j$ are both nearest neighbors, then we increment the weights, so the construction is symmetric. Each node also has a self-loop with weight 1. The adjacency matrix of the graph has 32,229,812 non-zeros, which is 15,856,315 edges and 517,182 self-loops.
The code for this experiment is in the Jupyter notebook \texttt{hexbingraphplots\_global.jl} in the subdirectory~\path{flow_embedding/hexbin_plots} available in~\citep{git:localgraphclustering_experiments}. The full code to process the graph is available upon request.

\paragraph{\Cref{fig:cluster-improvement}} For this experiment we used seeded PageRank to find the seed set for the flow algorithm \MQI and \LFI.
We set the teleportation parameter of the seeded PageRank algorithm to $0.01$.
The code for this experiment is in the Jupyter notebook \texttt{plot\_cluster\_improvement.ipynb} in the subdirectory~\path{cluster_improvement} available in~\citep{git:localgraphclustering_experiments}.

\paragraph{\Cref{fig:grow-shrink-image}}
In our experiments constructing the graph from the image, we follow \Cref{sec:image-to-graph} and we set $r=80$, $\sigma_p^2=\mathcal{O}(10^2)$ and $\sigma_c^2=l/10$, where $l$ is the maximum between the row and column length of the image.
The code for this experiment is in the Jupyter notebook \texttt{image\_segmentation.ipynb} in the subdirectory~\path{image_segmentation} available in~\citep{git:localgraphclustering_experiments}.

\paragraph{\Cref{fig:JH}} The input is a 2-hop BFS set starting from a random target node. We independently generate 25 such BFS sets. The transparency level of red or blue nodes is determined by the ratio of including each node in the resulting sets.
The code for this experiment is in the Jupyter notebook \texttt{social.ipynb} in the subdirectory~\path{social} available in~\citep{git:localgraphclustering_experiments}. Specific details about tuning can also be found in the code.

\paragraph{\Cref{fig:ssl-all}} For every class we randomly select a small percentage of labeled nodes, the exact percentages are given in the main text. The nodes that are selected from each class are considered a single seed set.
For each seed set and for each class we use seeded PageRank with teleportation parameter equal to $0.01$. This procedure provides one PageRank vector per class. For each unlabeled node in the graph we look at the corresponding coordinates in the PageRank vectors and we give to each unlabeled node the label that corresponds to the largest value in the PageRank vectors. For flow methods, for every labelled node that is used, we run one step of breadth-first-search to expand the single seed node to a seed set. The expanded seed set is used as input to the flow methods. We find a cluster and each node in the cluster is considered to have the same label
as the seed node. Based on this technique, it is possible that one node can be allocated in more than one classes, we consider such node as false positives.
The code for this experiment is in the Jupyter notebook \texttt{semisupervised\_learning.ipynb} in the subdirectory~\path{ssl} available in~\citep{git:localgraphclustering_experiments}. The MNIST graph was weighted for this experiment. The distance between two images is computed by a radial basis function with width to be 2. To robustify the process of rounding diffusion vector to class labels, we use a strategy from~\citet{Gleich-2015-robustifying}, which involves rounding to classes based on the node with the smallest rank in the ranked-list of each diffusion vector.

\paragraph{\Cref{tab:runtime}} The code for this experiment is in the Jupyter notebooks in the subdirectory~\path{large_scale} available in~\citep{git:localgraphclustering_experiments}.

\paragraph{\Cref{fig:usroads-embed}}
We use the eigenvector of the Laplacian matrix $\mD - \mA$ associated with the smallest non-zero eigenvalues to compute the vectors $v_1$ and $v_2$. The coordinates of the plot are generated by assigning $x$ and $y$ based on the rank of a node in $v_1$ and $v_2$ in a sorted order. This has the effect of stretching out the eigenvector layout, which often compresses many nodes at similar point. The color of the nodes is proportional to the east-west latitude. The code for this experiment is in the notebook \texttt{usroads-embed.ipynb} in the subdirectory~\path{usroads} in~\citep{git:localgraphclustering_experiments}.

\paragraph{\Cref{fig:usroads-local-embed}}
We use \Cref{algo:flow-coords} with $N=500$ sets, $k=1$, $d=20$, $c=2$, along with \LFI[0.1] as the improve algorithm. For the local spectral embedding, we use the same seeding parameters with seeded PageRank with $\rho=\mbox{1e-6}$. When we create the matrix $\mX$ for seeded PageRank, we take the base-10 logarithm of the result value (which is always between $0$ and $1$). For vertices with 0 values, we assign them $-10$, which is lower than any other value. We found that this gave a more useful embedding and helped the spectral show more structure. The node labeled ``Rest of graph'' was manually placed in both because the embedding does not suggest a natural place for this. Here, we also used the rank of the node in a sorted order, which helps to spread out nodes that are all placed in exactly the same location. The code for this experiment is in the Jupyter notebook \texttt{usroads-embed.ipynb} in the subdirectory~\path{usroads} available in~\citep{git:localgraphclustering_experiments}.

\paragraph{\Cref{fig:astro-embed}}
We use \Cref{algo:flow-coords} with $N=500$ sets, $k=1$, $d=3$, $c=2$, along with \LFI[0.1] as the improve algorithm. We used the same local spectral methodology as in \Cref{fig:usroads-local-embed}. The large red node represents the remainder of the graph and all ``unembedded nodes,'' which is manually placed to highlight edges to the rest of the graph. Here, we also used the rank of the node in a sorted order, which helps to spread out nodes that are all placed in exactly the same location. The code for this experiment is in the Python script \texttt{flow\_embedding.py} in the subdirectory~\path{flow_embedding} available in~\citep{git:localgraphclustering_experiments} and Jupyter notebooks in the subdirectory~\path{flow_embedding/hexbin_plots} available in~\citep{git:localgraphclustering_experiments}. The Python script needs to be run first to generate data and then the notebook can be used to generate the figures.

\paragraph{\Cref{fig:astro-cond-hist}} The code is in the Jupyter notebooks \texttt{social.ipynb} in the subdirectory~\path{flow_embedding/cond_hists} available in~\citep{git:localgraphclustering_experiments}. They both need the embedding results from \Cref{fig:astro-embed} to generate the figures.

\paragraph{The rationale for the the local flow embedding procedure}
We now briefly justify the motivation for the structure of the local flow embedding algorithm. The key idea is that spectral algorithms are based on linear operations: if we have any way of sampling the reference set $R$ with a normalized set indicator $T$ such that $E[T] = \frac{1}{|R|} \mathbf{1}_R$, then if $f$ is a linear function -- such as an exact seeded PageRank computation -- we have $E[f(T)] = f(\frac{1}{|R|} \mathbf{1}_R)$. This expectation corresponds to the seeded PageRank result on the entire set. To include another dimension, we would seek to find an orthogonal direction to $E[f(T)]$, such as is done with constrained eigenvector computations. It is this linear function perspective that inspired our flow-embedding algorithm: collect samples of $f(T_i)$ into a matrix and then use the SVD on the samples of $T$ to approximate $E[f(T)]$ and the orthogonal component (given by the second singular vector). While some of these arguments can be formalized and made rigorous for a linear function, that is an orthogonal discussion (pun intended). Here, we simply use the observation that this perspective enables us to use a nonlinear procedure $f$ without any issue. This gave rise to the \Cref{algo:flow-coords}, which differs only in that we grow the sets $T \to R_i$ by including all vertices within graph distance $d$.
\section{Converting Images to Graphs}
\label{sec:image-to-graph}

\begin{figure}[t]
	\includegraphics[width=\linewidth]{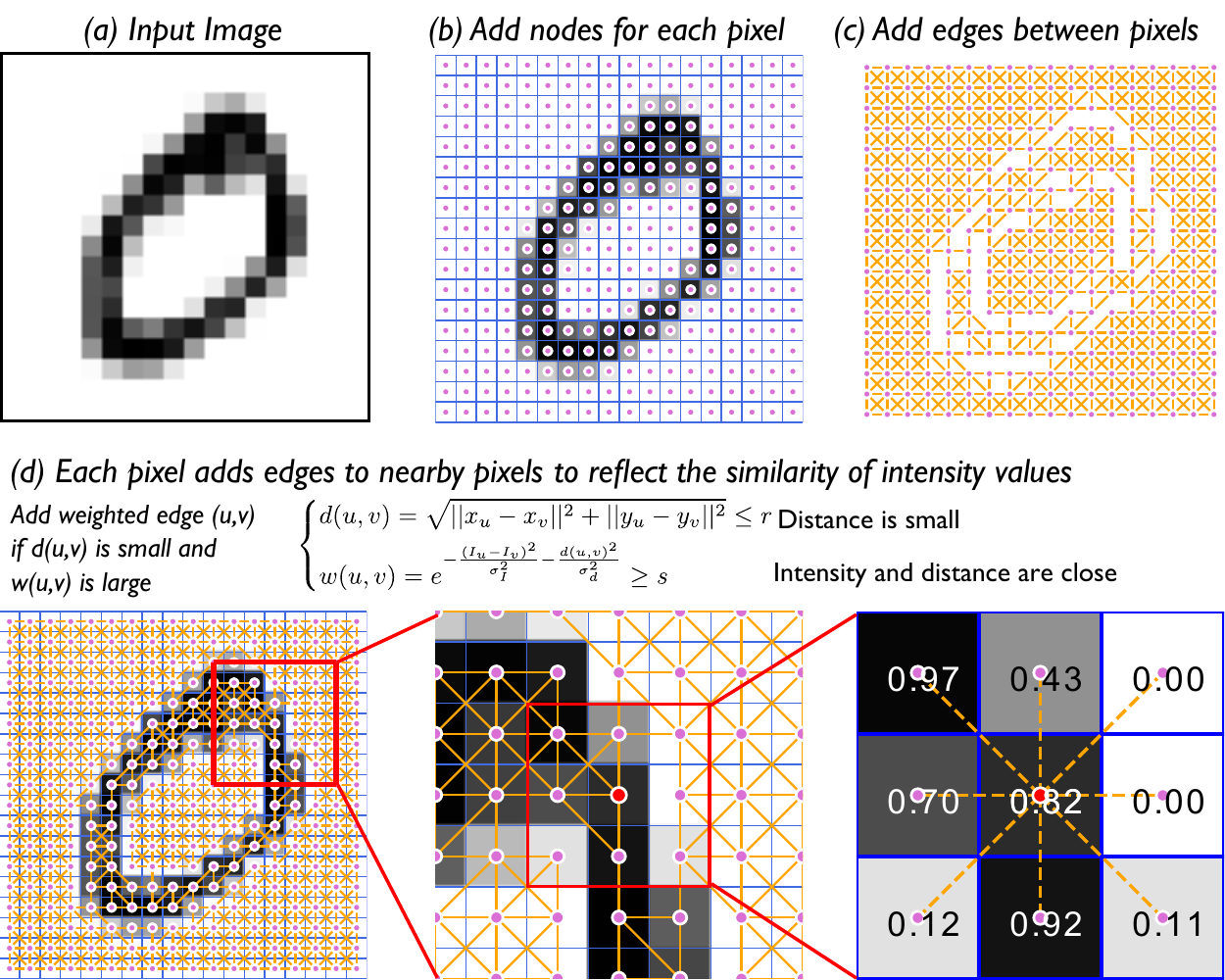}
	\caption{We turn an image into a graph by adding a node for every pixel (b). Then we connect the nodes if the associated pixels are close by (distance less than $r$) as well has have similar pixel values). We weight the edge by the degree of similarity. The resulting graph has small conductance sets when there are regions with similarly colored pixels.  }
	\label{fig:image-to-graph}
\end{figure}

For illustration purposes, we use images to generate graphs in various examples throughout the survey. The purpose of this construction is that visually distinct segments of the picture should have small conductance. Given an image we create a weighted nearest neighbor graph using a Gaussian kernel as described in~\citet{SM2000}.
We create a node for each pixel. Then we connect pixels with weighted edges. In particular, let $w_{ij}$ denote the the weight of the edge between pixels $i$ and $j$, let $p_i\in\mathbb{R}^2$ be the position of pixel $i$, $c_i\in\mathbb{R}^3$ is the color representation of pixel $i$, $\sigma_d^2$ is the variance for the position, $\sigma_I^2$ is the variance for the color. Then, we define the edge weights as
\begin{equation*}
w_{ij} :=
\begin{cases}
e^{- \frac{\|p_i - p_j\|^2_2}{\sigma_d^2} - \frac{\|c_i - c_j\|^2_2}{\sigma_I^2}} & \mbox{if} \ \|p_i - p_j\|^2 \le r  \\
0 & \mbox{otherwise}
\end{cases}
\end{equation*}
Note that there is a region $r$ that restricts the feasible edges, illustrated in \Cref{fig:image-to-graph}.
\section*{Acknowledgements}
\footnotesize
We would like to thank many individuals for discussions about these ideas over the years. We would also like to especially thank Nate Veldt for a careful reading of an initial draft, Charles Colley for reviewing a later draft, both Di Wang and Satish Rao for discussions on geometric aspects of flow algorithms, and finally Kent Quanrud for many helpful pointers. 	
	\bibliographystyle{dgleich-bib3}
	\bibliography{references} 
\end{document}